\documentclass[twoside,11pt]{article}
\usepackage{jair, theapa, rawfonts}

\jairheading{70}{2021}{1031-1116}{08/2020}{03/2021}
\ShortHeadings{Induction and Exploitation of Subgoal Automata for RL}{Furelos-Blanco, Law, Jonsson, Broda, \& Russo}
\firstpageno{1031}

\usepackage{amssymb}
\usepackage{amsmath}
\usepackage{amsthm}
\usepackage{mathtools}
\usepackage{marvosym}
\usepackage[marvosym]{tikzsymbols}
\usepackage{tikz}
\usetikzlibrary{automata,positioning}
\usetikzlibrary{matrix}
\usepackage[caption=false]{subfig}
\usepackage{enumerate}
\usepackage{multicol}
\usepackage{algorithm}
\usepackage[noend]{algpseudocode}
\usepackage{thm-restate}
\usepackage{booktabs}
\usepackage{pifont}
\usepackage{siunitx}
\usepackage{etoolbox}
\usepackage{rotating}
\usepackage{pgfplots}
\usepackage{url}
\usepackage{bm}

\algnewcommand{\algorithmicbreak}{\textbf{break}}

\newcommand{\cmark}{\ding{51}}%
\newcommand{\xmark}{\ding{55}}%

\DeclareMathOperator{\codeif}{\mathtt{:-} }
\mathchardef\mhyphen="2D 
\newcommand{\methodname}{ISA}

\newcommand{\methodnamefullhighlight}{\textbf{I}nduction of \textbf{S}ubgoal \textbf{A}utomata for Reinforcement Learning}

\theoremstyle{plain}
\newtheorem{definition}{Definition}[section]
\newtheorem{example}{Example}[section]
\newtheorem{theorem}{Theorem}[section]
\newtheorem{lemma}{Lemma}[section]
\newtheorem{proposition}{Proposition}[section]
\newtheorem{assumption}{Assumption}

\tikzset{%
	in place/.style={
		auto=false,
		fill=white,
		inner sep=2pt,
	},
}

\newcommand{\rulesep}{\unskip\ \vrule\ }

\definecolor{pblue}{HTML}{4477AA} 
\definecolor{pgreen}{HTML}{228833} 
\definecolor{pred}{HTML}{EE6677} 
\definecolor{pyellow}{HTML}{CCBB44} 
\definecolor{pcyan}{HTML}{66CCEE} 
\definecolor{pgray}{HTML}{BBBBBB}

\makeatletter
\newenvironment{customlegend}[1][]{%
	\begingroup
	\pgfplots@init@cleared@structures
	\pgfplotsset{#1}%
}{%
	\pgfplots@createlegend
	\endgroup
}%

\def\addlegendimage{\pgfplots@addlegendimage}

\begin{document}

\title{Induction and Exploitation of Subgoal Automata for Reinforcement Learning}

\author{\name Daniel Furelos-Blanco \email d.furelos-blanco18@imperial.ac.uk\\
	\name Mark Law \email mark.law09@imperial.ac.uk \\
	\addr Department of Computing\\Imperial College London\\London, SW7 2AZ, United Kingdom
	\AND
	\name Anders Jonsson \email anders.jonsson@upf.edu \\
	\addr Department of Information and Communication Technologies \\ Universitat Pompeu Fabra\\Roc Boronat 138, 08018 Barcelona, Spain
	\AND
	\name Krysia Broda \email k.broda@imperial.ac.uk\\
	\name Alessandra Russo \email a.russo@imperial.ac.uk \\
	\addr Department of Computing\\Imperial College London\\London, SW7 2AZ, United Kingdom}


\maketitle

\begin{abstract}
In this paper we present {\methodname}, an approach for learning and exploiting subgoals in episodic reinforcement learning (RL) tasks. {\methodname} interleaves reinforcement learning with the induction of a subgoal automaton, an automaton whose edges are labeled by the task's subgoals expressed as propositional logic formulas over a set of high-level events. A subgoal automaton also consists of two special states: a state indicating the successful completion of the task, and a state indicating that the task has finished without succeeding. A state-of-the-art inductive logic programming system is used to learn a subgoal automaton that covers the traces of high-level events observed by the RL agent. When the currently exploited automaton does not correctly recognize a trace, the automaton learner induces a new automaton that covers that trace. The interleaving process guarantees the induction of automata with the minimum number of states, and applies a symmetry breaking mechanism to shrink the search space whilst remaining complete. We evaluate {\methodname} in several grid-world and continuous state space problems using different RL algorithms that leverage the automaton structures. We provide an in-depth empirical analysis of the automaton learning performance in terms of the traces, the symmetry breaking and specific restrictions imposed on the final learnable automaton. For each class of RL problem, we show that the learned automata can be successfully exploited to learn policies that reach the goal, achieving an average reward comparable to the case where automata are not learned but handcrafted and given beforehand.
\end{abstract}

\section{Introduction}
Reinforcement learning (RL) is a family of algorithms for controlling an agent that acts in an environment with the purpose of maximizing some measure of cumulative reward it receives. These algorithms have played a key role in recent breakthroughs like human-level video game playing from raw sensory input \shortcite{MnihKSRVBGRFOPB15} and mastering complex board games \shortcite{Silver18}. However, despite of these impressive advancements, RL algorithms still struggle to discover and exploit the structure of a task. A possible way to represent these structures is through (temporal) abstractions.

Finite-state automata have been extensively used as a means for abstraction across different areas of Artificial Intelligence (AI), including the control of agents in robotics~\shortcite{Brooks89} and games~\shortcite{Buckland04}, as well as automated planning~\shortcite{BonetPG09,HuG11,Aguas0J18}. In the context of RL, automata have been used for multiple purposes, such as to represent abstract decision hierarchies~\shortcite{ParrR97,LeonettiIP12}, be used as memory in partially observable environments~\shortcite{MeuleauPKK99,IcarteWKVCM19}, or ease the interpretation of the policies encoded by a neural network~\shortcite{KoulFG19}. In particular, \shortciteA{IcarteKVM18} recently proposed reward machines (RMs), which are automata that play the role of reward functions while revealing the task's structure to the RL agent. This approach has drawn attention from the community and several works have recently attempted to learn RMs~\shortcite{IcarteWKVCM19,XuGAMNTW20} or similar kinds of automata~\shortcite{furelosblanco2020aaai,GaonB20}.

In this paper we propose {\methodname} (\methodnamefullhighlight), a method for learning and exploiting a minimal automaton that encodes the subgoals of an episodic goal-oriented task. Indeed, these automata are called \emph{subgoal automata} since each transition is labeled by a subgoal, which is a boolean formula over a set of high-level events that characterizes the task. A (possibly empty) set of high-level events is sensed by the agent at each step. Besides, subgoal automata have accepting and rejecting states. The former indicate the successful completion of the task, while the latter indicate that the task has been finished but without succeeding.
 
We represent subgoal automata using Answer Set Programming \shortcite<ASP,>{GelfondK14}, a logic programming language. A state-of-the-art \emph{inductive logic programming} (ILP) system for learning ASP programs, ILASP~\shortcite{ILASP_system}, is used to learn the automata. Specifically, given a set of traces of high-level events, and sufficiently large numbers of automaton states and allowed edges from one state to another, ILASP learns the transitions between states such that the traces are correctly recognized. For instance, all those traces achieving the task's goal must finish in the accepting state. To speed up automaton learning, we devise a \emph{symmetry breaking} mechanism that discards multiple equivalent automata in order to shrink the search space.

{\methodname} \emph{interleaves} automaton learning and reinforcement learning. The automaton learner is executed only when the RL agent finds a trace not correctly recognized by the current automaton. The interleaving scheme guarantees that the induced subgoal automaton is minimal (i.e., has the minimum number of states).

Importantly, subgoal automata address two types of abstraction: state abstraction and action abstraction \cite{Konidaris19,HoAGL19}. The set of automaton states is an abstraction of the original state space: they determine the level of completion of a given task. That is, they indicate which subgoals have been achieved and which remain to be achieved. Conversely, the subgoals labeling the edges can be seen as local objectives of abstract actions. The latter has been successfully addressed by hierarchical reinforcement learning \shortcite<HRL,>{BartoM03a}, which divides a single task into several subtasks that can be solved separately. We use HRL methods to exploit subgoal automata by defining one subtask per subgoal, as well as methods that exploit similar automaton structures like the aforementioned reward machines. All in all, abstractions in the form of subgoal automata can potentially:
\begin{enumerate}
	\item Make learning simpler since mastering a subtask should be easier than mastering the whole task.
	\item Allow for better exploration since the agent moves more quickly between abstract states (i.e., different levels of completion of the task).
	\item Allow for generalization between different tasks if they share common subtasks.
	\item Handle partial observability by acting as an external memory.
\end{enumerate}

We evaluate {\methodname} in several grid-world and continuous state space tasks. We show that a subgoal automaton can be simultaneously induced from interaction and exploited by different RL algorithms to learn a policy that achieves the task's goal. Importantly, the performance achieved by {\methodname} in the limit is comparable to that where a handcrafted automaton is given beforehand. Furthermore, we make a thorough analysis of how reinforcement learning affects automaton learning and vice versa.

The description of our approach previously appeared in a conference paper \shortcite{furelosblanco2020aaai}. Compared to the conference version, the present paper includes the following novel material:
\begin{itemize}
	\item A method for breaking symmetries in our automata that speeds up the automaton learning phase.
	\item An extensive experimental analysis of our interleaving method. The experiments include two new domains, one of which is characterized by a continuous state space and, thus, requires the use of function approximation techniques. Besides, we also evaluate a hierarchical RL algorithm that was not included in the conference version of the paper.
	\item A detailed description of recent related work, some of which was not yet available at the time of submission of the conference paper.
	\item A discussion on the limitations of our work and ideas to be developed in future work.
\end{itemize}

The paper is organized as follows. Section~\ref{sec:background} introduces the background of our work. The formalization of the main components of our approach (the tasks, the automata and the traces) is given in Section~\ref{sec:problem_formulation}. Section~\ref{sec:asp_representation_subgoal_automata} describes how subgoal automata and traces are represented using ASP, while Section~\ref{sec:learn_subgoal_automata_from_traces} explains how this class of automata can be learned from traces. In Section~\ref{sec:structural_properties} we introduce a set of constraints for checking and guaranteeing that a given automaton complies with specific structural properties (determinism and symmetry breaking constraints). Section~\ref{sec:interleaved_automata_learning_algorithm} details how automaton learning is interleaved with different RL algorithms that exploit the resulting automata. The effectiveness of our method across different tasks is evaluated in Section~\ref{sec:experiments}, followed by a discussion on related work in Section~\ref{sec:related_work}. Section~\ref{sec:conclusions} concludes the paper and suggests directions for future work.

\section{Background}
\label{sec:background}

In this section we briefly summarize the key background concepts of reinforcement learning and inductive learning of answer set programs, which are the main building blocks of our approach.

\subsection{Reinforcement Learning}
Reinforcement learning \shortcite<RL,>{SuttonB98} is a family of algorithms for learning to act in an unknown environment. Typically, this learning process is formulated as a \emph{Markov Decision Process (MDP)}, i.e., a tuple $\mathcal{M}=\langle S,A,p,r,\gamma\rangle$, where $S$ is a finite set of states, $A$ is a finite set of actions, $p:S\times A \to \Delta(S)$ is a transition probability function,\footnote{For any finite set $X$, $\Delta(X) = \{\mu \in \mathbb{R}^X \mid \sum_x\mu(x) = 1, \mu(x) \geq 0~(\forall x) \}$ is the probability simplex over $X$. } $r:S \times A \times S \to \mathbb{R}$ is a reward function, and $\gamma \in [0,1)$ is a discount factor. At time $t$, the agent observes state $s_t \in S$, executes action $a_t \in A$, transitions to the next state $s_{t+1} \sim p(\cdot|s_t, a_t)$ and receives reward $r(s_t, a_t, s_{t+1})$. We consider {\em episodic} MDPs that \emph{terminate} in a given set of terminal states, which can be either goal states or undesirable states (i.e.,~dead-ends). Let $S_T\subseteq S$ be the set of terminal states and $S_G\subseteq S_T$ the set of goal states. The aim is to find a \emph{policy} $\pi:S \to A$, a mapping from states to actions,\footnote{In this work we learn \emph{deterministic} policies, but it could be extended to support stochastic policies too.} that maximizes the expected sum of discounted reward (or \emph{return}), $R_t=\mathbb{E}[\sum_{k=t}^n\gamma^{k-t} r_k]$, where $n$ is the last step of the episode.

In model-free RL the transition probability function $p$ and reward function $r$ are unknown to the agent, and a policy is learned via interaction with the environment. Q-learning \shortcite{Watkins89} computes an \emph{action-value function} $Q(s,a) = \mathbb{E}[R_t|s_t=s, a_t=a]$ that estimates the return from each state-action pair when following an approximately optimal policy. In each iteration $t$ the estimates are updated as 
\begin{equation*}
Q(s_t,a_t) = Q(s_t,a_t) + \alpha \left( r(s_t,a_t,s_{t+1}) + \gamma\max_{a'}Q(s_{t+1},a') - Q(s_t,a_t)\right),
\end{equation*}
where $\alpha$ is a learning rate and $s_{t+1}$ is the state after applying $a_t$ in $s_t$. The term $r(s_t,a_t,s_{t+1}) + \gamma\max_{a'}Q(s_{t+1},a')$ is the \emph{target}, while the whole expression within parentheses is the \emph{Bellman error}. Usually, an $\epsilon$-greedy policy selects a random action with probability $\epsilon$ and the action maximizing $Q(s,a)$ otherwise. The policy is defined as the action that maximizes $Q(s,a)$ in each $s$.

In this work we focus on \emph{Partially Observable MDPs} \shortcite<POMDPs,>{KaelblingLC98}, which are MDPs where the agent cannot observe the real state of the environment. Since we later use the terms ``observable'' and ``observation'' to denote other concepts, we refer to unobserved states as \emph{latent}, and to observed states as \emph{visible}. Formally, a POMDP is a tuple $\mathcal{M}^\Sigma=\langle S,\Sigma,A,p,r,\gamma,\nu \rangle$, where $S,A,p,r$ and $\gamma$ are defined as for MDPs, $\Sigma$ is the set of visible states, and $\nu:S \to \Delta(\Sigma)$ is a mapping from latent states to probability distributions over visible states. In the \emph{episodic} setting, just like episodic MDPs, a POMDP includes a set of terminal states $S_T\subseteq S$ and a set of goal states $S_G\subseteq S_T$. We model the interaction between the agent and an episodic POMDP environment as follows. At time $t$, the latent state is $s_t\in S$, and the agent observes a tuple $\bm\sigma_t=\langle\sigma^\Sigma_t,\sigma^T_t,\sigma^G_t \rangle$, where $\sigma^\Sigma_t \in \Sigma$ is a visible state such that $\sigma^\Sigma_t \sim \nu(\cdot |s_t)$, $\sigma^T_t \in \lbrace \bot,\top \rbrace$ indicates whether $s_t$ is terminal (i.e., $\sigma_t^T=\mathbb{I}[s_t \in S_T]$),\footnote{The indicator function $\mathbb{I}(x)$ returns true ($\top$) if the condition $x$ holds, and false ($\bot$) otherwise.} and $\sigma^G_t \in \lbrace \bot, \top \rbrace$ indicates whether $s_t$ is a goal state (i.e., $\sigma_t^G=\mathbb{I}[s_t \in S_G]$).\footnote{Note that $\sigma^T$ is true whenever $\sigma^G$ is true since $S_G \subseteq S_T$ (i.e., all goal states are terminal).} If the state is non-terminal, the agent executes action $a_t\in A$, the environment transitions to the next latent state $s_{t+1} \sim p(\cdot|s_t,a_t)$, and the agent observes a new tuple $\bm\sigma_{t+1}$  and receives reward $r(s_t,a_t,s_{t+1})$. The policy $\pi:\Sigma^* \to A$ becomes a mapping from histories of visible states to actions. Note that the history is required since the Markov property might not hold for the states in $\Sigma$ and, similarly, the reward function becomes non-Markovian over $\Sigma$.

\paragraph{Options \shortcite{SuttonPS99}} address temporal abstraction in RL. Given an MDP $\mathcal{M}=\langle S,A,p,r,\gamma\rangle$, an \emph{option}  is a tuple $\omega=\langle I_\omega, \pi_\omega, \beta_\omega\rangle$ where $I_\omega \subseteq S$ is the option's initiation set, $\pi_\omega:S \to A$ is the option's policy, and $\beta_\omega:S\to [0,1]$ is the option's termination condition.\footnote{In the case of POMDPs, the option components are defined using the set of visible states $\Sigma$ instead of the set of latent states $S$.} An option is available in state $s \in S$ if $s\in I_\omega$. If the option is started, the actions are chosen according to $\pi_\omega$. The option terminates at a given state $s\in S$ with probability $\beta_\omega(s)$. Note that an action $a \in A$ can be viewed as an option that terminates in any state with probability 1.

An MDP whose action set is extended with options is a Semi-Markov Decision Process (SMDP). The learning methods for SMDPs require minimal changes with respect to MDPs. The counterpart of Q-learning for SMDPs is called SMDP Q-learning \shortcite{BradtkeD94}. The update it performs when an option $\omega$ has terminated is:
\begin{equation*}
Q(s_t,\omega_t) = Q(s_t,\omega_t) + \alpha\left(r + \gamma^k\max_{\omega'}Q(s_{t+k},\omega')- Q(s_t,\omega_t)\right),
\end{equation*}
where $k$ is the number of steps between $s_t$ and $s_{t+k}$, and $r=\sum^k_{j=1} \gamma^{j-1} r_{t+j}$ is the cumulative discounted reward over this time. Similarly to Q-learning, $\alpha$ is a learning rate and $s_{t+k}$ is the state in which option $\omega_t$ terminates.

Intra-option learning \shortcite{Kaelbling93b,SuttonPS98} is an often used method for learning option policies. During intra-option learning, an experience $\langle s_t,a_t,r_{t+1},s_{t+1}\rangle$ generated by the current option can be used not only to update its policy, but also other options' policies. Since experience accumulates faster, the convergence speed is often increased.

\subsection{Inductive Learning of Answer Set Programs}
In this section we describe answer set programming (ASP) and the ILASP system for learning ASP programs.

\subsubsection{Answer Set Programming} 
Answer Set Programming \shortcite<ASP,>{GelfondK14} is a declarative programming language for knowledge representation and reasoning. An ASP problem is expressed in a logical format and the models (called answer sets) of its representation provide the solutions to that problem. In the paragraphs below we describe the main concepts of ASP used in the paper.

An \emph{atom} is an expression of the form $\mathtt{p(t_1,\ldots,t_n)}$ where $\mathtt{p}$ is a \emph{predicate} symbol of arity $\mathtt{n}$ and $\mathtt{t_1,\ldots,t_n}$ are \emph{terms}. If $\mathtt{n=0}$, we omit the parentheses. In this paper, a term can be either a variable or a constant. By convention, variables are denoted using upper case (e.g., \texttt{X} or \texttt{Y}), while constants are written in lower case (e.g., \texttt{coffee} or \texttt{mail}). An atom is said to be \emph{ground} if none of its terms is a variable. A \emph{literal} is an atom $\mathtt{a}$ or its negation $\mathtt{not~a}$. The \texttt{not} symbol is called \emph{negation as failure}.

An ASP \emph{program} $P$ is a set of rules. In this paper, we assume that this set is formed by normal rules, choice rules and constraints. Given an atom $\mathtt{h}$ and a set of literals $\mathtt{b_1,\ldots,b_n}$, a \emph{normal rule} is of the form $\mathtt{h \codeif b_1, \ldots, b_n}$, where $\mathtt{h}$ is the \emph{head} and $\mathtt{b_1,\ldots,b_n}$ is the \emph{body} of the rule. A normal rule with an empty body is a \emph{fact}. A \emph{choice rule} is of the form $\mathtt{lb\{h_1,\ldots,h_m\}ub \codeif b_1,\ldots, b_n}$, where $\mathtt{lb}$ and $\mathtt{ub}$ are integers, $\mathtt{h_1,\ldots,h_m}$ are atoms and $\mathtt{b_1,\ldots, b_n}$ are literals. Rules of the form $\mathtt{\codeif b_1, \ldots, b_n}$ are called \emph{constraints}.

Given a set of ground atoms (or \emph{interpretation}) $I$, a ground normal rule is satisfied if the head is satisfied by $I$ when the body literals are satisfied by $I$. The head of a choice rule is satisfied by $I$ if and only if the number of satisfied atoms in the head is between $\mathtt{lb}$ and  $\mathtt{ub}$ (both included), i.e., $\mathtt{lb}\leq |I \cap \{\mathtt{h_1,\ldots,h_m} \}| \leq \mathtt{ub}$. A ground constraint is satisfied if the body is not satisfied by $I$. The \emph{reduct} $P^I$ of a program $P$ with respect to $I$ is built in 4 steps~\shortcite{LawRB15}:\footnote{This is a non-standard form of building the reduct, but it is proven to be equivalent \shortcite{LawRB15} to the standard definitions (\shortciteR<e.g.,>{CalimeriFGIKKLM20}).}
\begin{enumerate}
	\item Replace the heads of all constraints with $\bot$.
	\item For each choice rule $R$: 
	\begin{itemize}
		\item if its head is not satisfied, replace its head with $\bot$, or
		\item if its head is satisfied then remove $R$ and for each atom $\mathtt{h}$ in the head of $R$ such that $\mathtt{h} \in I$, add the rule $\mathtt{h} \codeif body(R)$ (where $body(R)$ is the set of literals forming the body of $R$).
	\end{itemize}
	\item Remove any rule $R$ such that the body of $R$ contains the negation of an atom in $I$.
	\item Remove all negation from any remaining rules.
\end{enumerate}
An interpretation $I$ is an \emph{answer set} of $P$ if and only if (1) $I$ satisfies the rules in $P^I$, and (2) no subset of $I$ satisfies the rules in $P^I$.

\begin{example}
	Given the following ASP program formed by two facts and two normal rules
	\begin{align*}
		P = \begin{Bmatrix*}[l]
			\mathtt{p(X) \codeif not~q(X), r(X).} & \mathtt{q(X) \codeif not~p(X), r(X).} & \mathtt{r(1).} & \mathtt{r(2).}
		\end{Bmatrix*},
	\end{align*}
	the following four interpretations are all the answer sets of $P$:
	\begin{align*}
		\begin{matrix*}[c]
			\left\lbrace \mathtt{r(1).~r(2).~p(1).~p(2).} \right\rbrace & \left\lbrace \mathtt{r(1).~r(2).~p(1).~q(2).} \right\rbrace \\
			\left\lbrace \mathtt{r(1).~r(2).~p(2).~q(1).} \right\rbrace & \left\lbrace \mathtt{r(1).~r(2).~q(1).~q(2).} \right\rbrace
		\end{matrix*}
	\end{align*}
	Note that the interpretation $\lbrace \mathtt{r(1).~r(2).~p(1).~p(2).~q(1).~q(2).} \rbrace$ is not an answer set.
\end{example}

\subsubsection{ILASP}
ILASP \shortcite<Inductive Learning of Answer Set Programs,>{ILASP_system} is an inductive logic programming system for learning ASP programs from  partial answer sets.

A \emph{context-dependent partial interpretation}~\shortcite<CDPI,>{LawRB16} is a pair $\langle \langle e^{inc}, e^{exc} \rangle,\allowbreak e^{ctx} \rangle$, where:
\begin{itemize}
	\item $\langle e^{inc}, e^{exc}\rangle$ is a pair of sets of atoms, called a \emph{partial interpretation}. We refer to $e^{inc}$ and $e^{exc}$ as the \emph{inclusions} and \emph{exclusions} respectively.
	\item $e^{ctx}$ is an ASP program, called a \emph{context}.
\end{itemize}
A program $P$ \emph{accepts} a CDPI $\langle \langle e^{inc}, e^{exc} \rangle, e^{ctx} \rangle$ if and only if there is an answer set $A$ of $P \cup e^{ctx}$ such that $e^{inc} \subseteq A$ and $e^{exc}~\cap A = \emptyset$.

An \emph{ILASP task}~\shortcite{LawRB16} is a tuple $T=\langle B,S_M,\langle E^+, E^-\rangle\rangle$ where 
\begin{itemize}
	\item $B$ is the ASP background knowledge, which describes a set of known concepts before learning;
	\item $S_M$ is the set of ASP rules allowed in the hypotheses; and
	\item $E^+$ and $E^-$ are sets of CDPIs called, respectively, the positive and negative examples.
\end{itemize}
A hypothesis $H \subseteq S_M$ is an \emph{inductive solution} of $T$ if and only if: 
\begin{enumerate}
	\item $\forall e \in E^+$, $B \cup H$ accepts $e$, and
	\item $\forall e \in E^-$, $B \cup H$ does not accept $e$.
\end{enumerate}

\begin{example}
	Let $T=\langle B, S_M,\langle E^+, E^- \rangle \rangle$ be an ILASP task where:
	\begin{align*}
		\begin{split}
			B&=\left\lbrace \mathtt{p \codeif not~q.} \right\rbrace,\\
			E^+&=\begin{Bmatrix*}[c]
				\left\langle\left\langle\left\lbrace\mathtt{p}\right\rbrace, \left\lbrace\mathtt{q}\right\rbrace \right\rangle, \left\lbrace\mathtt{r.} \right\rbrace\right\rangle,\\
				\left\langle\left\langle\left\lbrace\mathtt{q}\right\rbrace, \left\lbrace\mathtt{p}\right\rbrace \right\rangle, \left\lbrace\mathtt{r.} \right\rbrace\right\rangle,\\
				\left\langle\left\langle\left\lbrace\mathtt{p}\right\rbrace, \left\lbrace\mathtt{q}\right\rbrace \right\rangle, \emptyset\right\rangle
			\end{Bmatrix*},
		\end{split}
		\quad
		\begin{split}
			S_M&=\begin{Bmatrix*}[c]
				\mathtt{q.}&\mathtt{q \codeif not~p.}\\\mathtt{q \codeif p,r.}&\mathtt{q \codeif not~p,r.}
			\end{Bmatrix*},\\
			E^-&=\begin{Bmatrix*}[c]
				\left\langle\left\langle\left\lbrace\mathtt{q}\right\rbrace, \left\lbrace\mathtt{p}\right\rbrace \right\rangle, \emptyset\right\rangle
			\end{Bmatrix*}.
		\end{split}
	\end{align*}
	A candidate hypothesis is $H=\lbrace\mathtt{q \codeif not~p,r.}\rbrace$. Now we have to check whether the positive examples are accepted, and the negative is not:
	\begin{itemize}
		\item For the positive examples $\left\langle\left\langle\left\lbrace\mathtt{p}\right\rbrace, \left\lbrace\mathtt{q}\right\rbrace \right\rangle, \left\lbrace\mathtt{r.} \right\rbrace\right\rangle$ and $\left\langle\left\langle\left\lbrace\mathtt{q}\right\rbrace, \left\lbrace\mathtt{p}\right\rbrace \right\rangle, \left\lbrace\mathtt{r.} \right\rbrace\right\rangle$, the program $B \cup H \cup \lbrace\mathtt{r.} \rbrace$ has two answer sets $A_1=\lbrace\mathtt{p,r} \rbrace$ and $A_2=\lbrace\mathtt{q,r} \rbrace$. Then, $B \cup H \cup \lbrace\mathtt{r.} \rbrace$ accepts the first example because $\lbrace\mathtt{p}\rbrace\subseteq A_1$ and $\lbrace\mathtt{q}\rbrace \cap A_1 = \emptyset$, and also the second one because $\lbrace\mathtt{q}\rbrace\subseteq A_2$ and $\lbrace\mathtt{p}\rbrace \cap A_2 = \emptyset$.
		\item For the positive example $\left\langle\left\langle\left\lbrace\mathtt{p}\right\rbrace, \left\lbrace\mathtt{q}\right\rbrace \right\rangle, \emptyset\right\rangle$ and the negative example $\left\langle\left\langle\left\lbrace\mathtt{q}\right\rbrace, \left\lbrace\mathtt{p}\right\rbrace \right\rangle, \emptyset\right\rangle$, the program $B\cup H \cup \emptyset$ has a single answer set $A_1'=\lbrace\mathtt{p} \rbrace$. Then, $B\cup H \cup \emptyset$ accepts the first example since $\lbrace\mathtt{p}\rbrace\subseteq A'_1$ and $\lbrace\mathtt{q}\rbrace \cap A'_1 = \emptyset$, and does not accept the second one since $\lbrace\mathtt{q}\rbrace \subsetneq A'_1$ (also, $\lbrace\mathtt{p}\rbrace \cap A_1'\neq \emptyset$).
	\end{itemize}
	Therefore, $H$ is an inductive solution of $T$. In contrast, $H'=\lbrace \mathtt{q.} \rbrace$ is not an inductive solution.  For instance, given the positive example $\left\langle\left\langle\left\lbrace\mathtt{p}\right\rbrace, \left\lbrace\mathtt{q}\right\rbrace \right\rangle, \emptyset\right\rangle$, the program $B \cup H'\cup\emptyset$ has a single answer set, $A''_1=\lbrace\mathtt{q}\rbrace$. Then, this program does not accept the example because $\lbrace\mathtt{p} \rbrace \subsetneq A_1''$ (also, $\lbrace\mathtt{q}\rbrace \cap A_1'' \neq \emptyset$), which causes $H'$ not to be an inductive solution. 
\end{example}

\section{Problem Formulation}
\label{sec:problem_formulation}
The objectives of this work are twofold:
\begin{enumerate}
	\item Propose a method for learning a subgoal automaton from traces of a given reinforcement learning task (Sections~\ref{sec:asp_representation_subgoal_automata}-\ref{sec:structural_properties}).
	\item Propose a method that interleaves automaton learning and reinforcement learning (Section~\ref{sec:interleaved_automata_learning_algorithm}).
\end{enumerate}

In this section we formalize the main components of our automaton-driven RL approach, including the class of RL tasks, the notion of traces generated by the RL agent and the definition of subgoal automata. In later sections, we show how subgoal automata and traces are represented in ASP (see Section~\ref{sec:asp_representation_subgoal_automata}), and how learning a subgoal automaton can be formalized as an ILASP learning task that uses traces as examples (see Section~\ref{sec:learn_subgoal_automata_from_traces}).

\subsection{Tasks}
\label{sec:tasks_def}
The class of RL tasks we consider are \emph{episodic POMDPs} $\mathcal{M}^\Sigma=\langle S,S_T, S_G,\Sigma,A,p,r,\gamma,\nu\rangle$ enhanced by a set of \emph{observables} $\mathcal{O}$. An observable is a propositional event that the agent can sense while interacting with the environment. A \emph{labeling function} $L:\Sigma \to 2^\mathcal{O}$ maps a visible state into a subset of observables (or \emph{observation}) $O \subseteq \mathcal{O}$ perceived by the agent in that state. We emphasize that an observation solely depends on a visible state.

We assume that the combination of a visible state and a history of observations seen during an episode is sufficient to obtain the Markov property. Formally, the set of latent states $S$ is a subset of the cross product of the set of visible states $\Sigma$ and a history of observations $(2^\mathcal{O})^\ast$. Given this assumption, a policy over $\Sigma \times (2^\mathcal{O})^\ast$ could be learned; however, this space can become very large for long histories. Throughout the paper, we explain how the automata we propose can be seen as a compact representation of observation histories. In addition, we also assume that the history of observations seen during an episode is sufficient to determine whether a terminal state is reached and, if so, whether it is a goal state.

We use the \textsc{OfficeWorld} environment \shortcite{IcarteKVM18} as a running example to explain our method. It consists of a $9\times 12$ grid (see Figure~\ref{fig:officeworld_grid}) where an agent ($\Strichmaxerl[1.25]$) can move in the four cardinal directions; that is, the action set is $A=\{\text{up},\text{down}, \text{left}, \text{right}\}$. The agent always moves in the intended direction (i.e., actions are deterministic), and remains in the same location if it moves towards a wall. The set of visible states $\Sigma$ is the set of all locations in the grid (i.e., the agent knows the coordinate of the location it is stepping on). The set of observables is $\mathcal{O}=\lbrace\texttt{\Coffeecup}, \texttt{\Letter}, o, A, B,C,D,\ast\rbrace$, all of which correspond to visible locations. The agent picks up the coffee and the mail when it steps on locations {\Coffeecup} and {\Letter} respectively, and delivers them to the office when it steps on location $o$. The decorations $\ast$ break if the agent steps on them. There are also four locations labeled $A$, $B$, $C$ and $D$. Three tasks with different goals are defined in this environment:
\begin{itemize}
	\item \textsc{Coffee}: deliver coffee to the office.
	\item \textsc{CoffeeMail}: deliver coffee and mail to the office.
	\item \textsc{VisitABCD}: visit $A$, $B$, $C$ and $D$ in order.
\end{itemize}
The tasks terminate when the goal is achieved or a decoration is broken (this is a dead-end state). A reward of 1 is given when the goal is achieved, else the reward is 0.

\begin{figure}
	\tikzset{digit/.style = { minimum height = 5mm, minimum width=5mm, anchor=center }}
	\newcommand{\setcell}[3]{\edef\x{#2 - 0.5}\edef\y{9.5 - #1}\node[digit,name={#1-#2}] at (\x, \y) {#3};}
	\centering
	\begin{tikzpicture}[scale=0.5]
	\draw[white] (0, 0) grid (13, 10);
	
	\setcell{9}{2}{0} \setcell{9}{3}{1} \setcell{9}{4}{2} \setcell{9}{5}{3} \setcell{9}{6}{4} \setcell{9}{7}{5} \setcell{9}{8}{6} \setcell{9}{9}{7} \setcell{9}{10}{8} \setcell{9}{11}{9} \setcell{9}{12}{10} \setcell{9}{13}{11}
	
	\setcell{8}{1}{0} \setcell{7}{1}{1} \setcell{6}{1}{2} \setcell{5}{1}{3} \setcell{4}{1}{4} \setcell{3}{1}{5} \setcell{2}{1}{6} \setcell{1}{1}{7} \setcell{0}{1}{8}
	
	\draw[gray] (1, 1) grid (13, 10);
	\draw[very thick, scale=3] (1/3, 1/3) rectangle (13/3, 10/3);
	
	\draw[very thick, scale=1] (4, 1) rectangle (4, 2); \draw[very thick, scale=1] (4, 9) rectangle (4, 10);
	\draw[very thick, scale=1] (4, 3) rectangle (4, 8);
	\draw[very thick, scale=1] (7, 1) rectangle (7, 2); \draw[very thick, scale=1] (7, 3) rectangle (7, 8); \draw[very thick, scale=1] (7, 9) rectangle (7, 10);
	\draw[very thick, scale=1] (10, 1) rectangle (10, 2); \draw[very thick, scale=1] (10, 3) rectangle (10, 8); \draw[very thick, scale=1] (10, 9) rectangle (10, 10);
	\draw[very thick, scale=1] (1, 4) rectangle (2, 4); \draw[very thick, scale=1] (3, 4) rectangle (11, 4); \draw[very thick, scale=1] (12, 4) rectangle (13, 4);
	\draw[very thick, scale=1] (1, 7) rectangle (2, 7); \draw[very thick, scale=1] (3, 7) rectangle (5, 7); \draw[very thick, scale=1] (6, 7) rectangle (8, 7); \draw[very thick, scale=1] (9, 7) rectangle (11, 7); \draw[very thick, scale=1] (12, 7) rectangle (13, 7);
	\setcell{1}{3}{$D$} \setcell{1}{6}{$\ast$} \setcell{1}{9}{$\ast$} \setcell{1}{12}{$C$}
	\setcell{2}{5}{\Coffeecup} \setcell{2}{6}{\Strichmaxerl[1.25]}
	\setcell{4}{3}{$\ast$} \setcell{4}{6}{$o$} \setcell{4}{9}{\Letter} \setcell{4}{12}{$\ast$}
	\setcell{6}{10}{\Coffeecup}
	\setcell{7}{3}{$A$} \setcell{7}{6}{$\ast$} \setcell{7}{9}{$\ast$} \setcell{7}{12}{$B$}
	\end{tikzpicture}
	\caption{Example grid used in the \textsc{OfficeWorld} environment \shortcite{IcarteKVM18}.}
	\label{fig:officeworld_grid}
\end{figure}
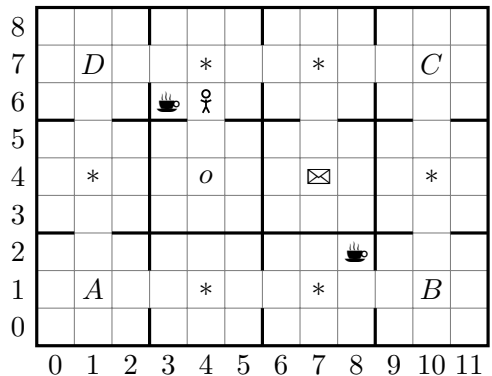

\subsection{Traces}
\label{sec:traces_def}
We define the different kinds of traces that can be generated in our class of RL tasks.

\begin{definition}[Execution trace]
	An execution trace $\lambda=\langle\bm\sigma_0,a_0,r_1,\bm\sigma_1,a_1,\ldots,a_{n-1},r_n,\allowbreak\bm\sigma_n \rangle$ is a finite sequence of tuples $\bm\sigma_i = \langle\sigma^\Sigma_i,\sigma^T_i,\sigma^G_i \rangle$, actions and rewards induced by a (potentially changing) policy during an episode. An execution trace $\lambda$ can be one of the following:
	\begin{itemize}
		\item A goal execution trace $\lambda^G$ if $\sigma^G_n = \top$ (i.e., a latent goal state has been reached).
		\item A dead-end execution trace $\lambda^D$ if $\sigma^T_n = \top \land \sigma^G_n = \bot$ (i.e., a latent dead-end state has been reached).
		\item An incomplete execution trace $\lambda^I$ if $\sigma^T_n=\bot$ (i.e., the final latent state is not terminal).
	\end{itemize}
\end{definition}

Execution traces are the traces perceived by an RL agent. However, these are not the traces that are used to learn the subgoal automata. Since subgoal automata aim to provide the RL agent with a subgoal structure that is independent from the state space, they are defined at a higher level of abstraction. Therefore, the subgoal automaton learner needs traces of higher level events as an input. These traces are called observation traces.

\begin{definition}[Observation trace]
	An observation trace $\lambda_{L,\mathcal{O}}$ is a sequence of observations $O_i\subseteq \mathcal{O}, 0 \leq i \leq n$, obtained by applying a labeling function $L$ to each visible state $\sigma^\Sigma_i\in \Sigma$ in an execution trace $\lambda=\langle \bm{\sigma}_0 ,a_0,r_1,\bm\sigma_1,a_1, \ldots,\allowbreak a_{n-1},r_n, \bm\sigma_n\rangle$. Formally, 
	\begin{equation*}
	\lambda_{L,\mathcal{O}}=\left\langle O_0,\ldots,O_n \mid O_i = L(\sigma^\Sigma_i), \sigma^\Sigma_i \in \lambda \right\rangle.
	\end{equation*}
\end{definition}

For the rest of the paper, we may simply use the term \emph{trace} when we refer to an observation trace.

A \emph{set of execution traces} is denoted by $\Lambda = \Lambda^G \cup \Lambda^D \cup \Lambda^I$, where $\Lambda^G$, $\Lambda^D$ and $\Lambda^I$ are sets of goal, dead-end and incomplete execution traces, respectively. The associated set of observation traces is analogously denoted by $\Lambda_{L,\mathcal{O}}=\Lambda_{L,\mathcal{O}}^G \cup \Lambda_{L,\mathcal{O}}^D \cup  \Lambda_{L,\mathcal{O}}^I$.

\begin{example}
	The first trace below is a goal execution trace for the \textsc{OfficeWorld}'s \textsc{Coffee} task using the grid in Figure~\ref{fig:officeworld_grid}. The second trace is the resulting observation trace. 
	\begin{align*}
	\lambda^G =&~\langle \langle(4,6),\bot,\bot\rangle, \leftarrow, 0, \langle(3,6),\bot,\bot\rangle, \leftarrow, 0 , \langle(3,6),\bot,\bot\rangle, \rightarrow, 0, \langle(4,6),\bot,\bot\rangle,\downarrow, 0,\\ &~~\langle(4,5),\bot,\bot\rangle, \downarrow, 1, \langle(4,4),\top,\top\rangle \rangle,\\
	\lambda^G_{L,\mathcal{O}} =&\left\langle \{\}, \{\text{\Coffeecup}\}, \{\text{\Coffeecup}\}, \{\}, \{\}, \{o\} \right\rangle.
	\end{align*}
	The visible states in $\lambda^G$ correspond to positions in the grid, whereas the arrows correspond to the different actions: up ($\uparrow$), down ($\downarrow$), left ($\leftarrow$) and right ($\rightarrow$).
\end{example}

\subsection{Subgoal Automata}
\label{sec:subgoal_automata_def}
Now we formally define the kind of automaton that is learned and used together with the RL component in our approach. The edges that characterize this class of automata are labeled by propositional logic formulas over a set of observables $\mathcal{O}$. These formulas can be interpreted as the subgoals of the task represented by the automaton. Therefore, we refer to these automata as \emph{subgoal automata}.\footnote{For the rest of the paper, we may simply use the term \emph{automata} when we refer to subgoal automata.}

\begin{definition}[Subgoal automaton]
	A subgoal automaton is a tuple $\mathcal{A}=\langle U, \mathcal{O}, \delta_\varphi, u_0, u_A,\allowbreak u_R \rangle$ where 
	\begin{itemize}
		\item $U$ is a finite set of automaton states,
		\item $\mathcal{O}$ is a finite set of observables (or alphabet),
		\item $\delta_\varphi: U \times 2^\mathcal{O} \to U$ is a deterministic transition function that takes an automaton state and a subset of observables (or observation) and returns an automaton state,
		\item $u_0 \in U$ is the unique initial state,
		\item $u_A \in U$ is the unique absorbing accepting state, and
		\item $u_R \in U$ is the unique absorbing rejecting state.
	\end{itemize}
\end{definition}

The accepting state ($u_A$) denotes the task's goal achievement. In contrast, the rejecting state ($u_R$) indicates that the goal can no longer be achieved. Therefore, these states are both absorbing meaning that they do not have transitions to other states. That is, $\delta_\varphi(u,O) = u$ for $u \in \{u_A, u_R\}$ and any observation $O\subseteq \mathcal{O}$.

To determine how the behavior of an agent during an episode is evaluated by a subgoal automaton, we introduce the concept of \emph{automaton traversal}.

\begin{definition}[Automaton traversal]
	Given a subgoal automaton $\mathcal{A}$ and an observation trace $\lambda_{L,\mathcal{O}}=\langle O_0,\ldots,O_n\rangle$, an automaton traversal $\mathcal{A}(\lambda_{L,\mathcal{O}})=\langle v_0, v_1, \ldots, v_{n+1} \rangle$ is a unique sequence of automaton states such that
	\begin{enumerate}
		\item $v_0 = u_0$, and
		\item $\delta_\varphi(v_i, O_i)=v_{i+1}$ for $i=0,\ldots,n$.
	\end{enumerate}
	A subgoal automaton $\mathcal{A}$ accepts an observation trace $\lambda_{L,\mathcal{O}}$ if the automaton traversal $\mathcal{A}(\lambda_{L,\mathcal{O}})=\langle v_0, v_1, \ldots, v_{n+1} \rangle$ is such that $v_{n+1}=u_A$. Analogously, $\mathcal{A}$ rejects $\lambda_{L,\mathcal{O}}$ if $v_{n+1} = u_R$.
\end{definition}

Now that we know how a subgoal automaton evaluates an observation trace, we need to determine whether this evaluation (acceptance, rejection, or neither) complies with the type of the trace (goal, dead-end, or incomplete). The following definition introduces the concept of \emph{validity} with respect to an observation trace (i.e., whether the type of the trace matches the automaton's evaluation). This concept is crucial to prove the correctness of the ASP encoding (see Section~\ref{sec:asp_representation_subgoal_automata}), as well as to learn subgoal automata (see Section~\ref{sec:learn_subgoal_automata_from_traces}).
	
\begin{definition}
	Given an observation trace $\lambda^*_{L,\mathcal{O}}$, where $\ast \in \lbrace G,D,I\rbrace$, a subgoal automaton $\mathcal{A}$ is valid with respect to $\lambda^*_{L,\mathcal{O}}$ if one of the following holds:
	\begin{itemize}
		\item $\mathcal{A}$ accepts $\lambda^*_{L,\mathcal{O}}$ and $\ast=G$ (i.e., $\lambda^*_{L,\mathcal{O}}$ is a goal trace).
		\item $\mathcal{A}$ rejects $\lambda^*_{L,\mathcal{O}}$ and $\ast=D$ (i.e., $\lambda^*_{L,\mathcal{O}}$ is a dead-end trace).
		\item $\mathcal{A}$ does not accept nor reject $\lambda^*_{L,\mathcal{O}}$ and $\ast=I$ (i.e., $\lambda^*_{L,\mathcal{O}}$ is an incomplete trace).
	\end{itemize}
	\label{def:valid_trace}
\end{definition}

The transition function $\delta_\varphi$ is constructed from a \emph{logical transition function} $\varphi$ that maps state pairs into propositional formulas over $\mathcal{O}$, each representing a subgoal of the task. Our approach represents and learns logical transition functions.

\begin{definition}[Logical transition function]
	A logical transition function $\varphi: U \times U \to \emph{DNF}_\mathcal{O}$ is a transition function that maps a state pair into a disjunctive normal form (DNF) formula over $\mathcal{O}$, where $\varphi(u,u)=\bot$ for each $u\in U$ (i.e., $\varphi$ only represents transitions to different states).
	\label{def:logical_transition_function}
\end{definition}

Expressing $\varphi(u,u')$ as a DNF formula allows representing multiple edges between the same pair of states $u$ and $u'$. That is, each of the conjunctions inside a DNF formula $\varphi(u,u')$ labels a different edge between $u$ and $u'$. The following notation is used throughout the paper:
\begin{itemize}
	\item $|\varphi(u,u')|$ denotes the number of conjunctive formulas that form the formula $\varphi(u,u')$,
	\item $\text{conj}_i$ denotes the $i$-th conjunction (left-to-right) in the DNF formula $\varphi(u,u')$, and
	\item $O \models \varphi(u,u')$ denotes that observation $O \subseteq \mathcal{O}$ satisfies the DNF formula $\varphi(u,u')$. Note that $O$ is used as a truth assignment where observables in the set (i.e., $o\in O$) are true and observables that are not in the set (i.e., $o \notin O$) are false. Formally,
	\begin{equation*}
	O \models \varphi(u,u') \equiv \exists\text{conj}_i \in \varphi(u,u') \text{ such that } O \models \text{conj}_i.
	\end{equation*}
\end{itemize}

\begin{figure}
	\centering
	\begin{tikzpicture}[shorten >=1pt,node distance=2.06cm,on grid,auto]
	\node[state,initial] (u_0)   {$u_0$};
	\node[state,accepting] (u_acc) [below =3cm of u_0]  {$u_{A}$};
	\node[state] (u_1) [left =3.5cm of u_acc]   {$u_1$};
	\node[state] (u_rej) [right =3.5cm of u_acc]  {$u_{R}$};
	
	\path[->] (u_0) edge [loop above] node {otherwise} ();
	\path[->] (u_1) edge [loop below] node {otherwise} ();
	\path[->] (u_acc) edge [loop right] node {otherwise} ();
	\path[->] (u_rej) edge [loop below] node {otherwise} ();
	
	\path[->] (u_0) edge [bend right] node[in place] {$\texttt{\Coffeecup} \land \neg o$} (u_1);
	\path[->] (u_0) edge [bend left] node[in place] {$\ast \land \neg\texttt{\Coffeecup}$} (u_rej);
	\path[->] (u_1) edge[bend right] node[in place] {$\ast \land \neg o$} (u_rej);
	\path[->] (u_1) edge[bend left] node[in place] {$o$} (u_acc);
	\path[->] (u_0) edge node[in place] {$\texttt{\Coffeecup} \land o$} (u_acc);
	\end{tikzpicture}
	
	\caption{Subgoal automaton for the \textsc{OfficeWorld}'s \textsc{Coffee} task.}
	\label{fig:officeworld_coffee_rm}
\end{figure}
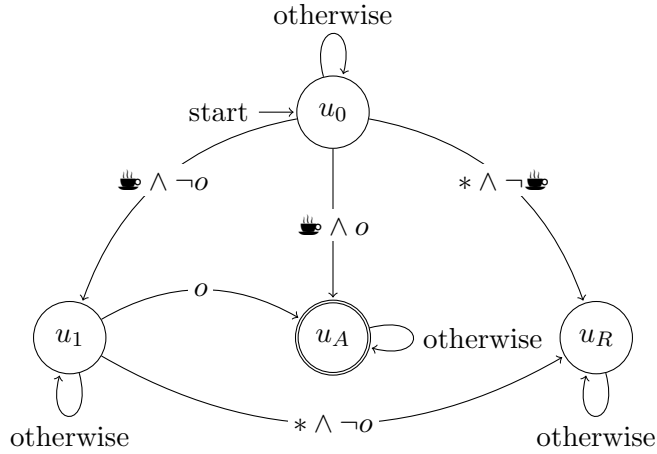

Given a logical transition function $\varphi$, the transition function $\delta_\varphi$ can be formally defined in terms of $\varphi$ as follows:
\begin{equation}
\delta_\varphi(u,O)= \begin{cases} 
u' & \text{if } O \models \varphi(u,u')\\
u  & \text{if } \nexists u' \in U \text{ such that }  O \models \varphi(u,u')
\end{cases}.
\label{eq:delta_from_phi}
\end{equation}
Note that loop transitions are implicitly defined by the absence of a satisfied formula on outgoing transitions to other states. Besides, this mapping only works if $\varphi$ is deterministic; that is, given a state $u\in U$ and an observation $O\subseteq \mathcal{O}$, at most one formula is satisfied. Formally, there \emph{are not} two states $u',u''\in U$ such that $O \models \varphi(u,u'), O \models \varphi(u,u''),$ and $u'\neq u''$. Determinism is guaranteed when all pairs of outgoing transitions from a given state to two different states are mutually exclusive; that is, an observable appears positively in one edge and negatively in another.

\begin{example}
	Figure~\ref{fig:officeworld_coffee_rm} shows a minimal subgoal automaton for the \textsc{OfficeWorld}'s \textsc{Coffee} task. The edges are labeled according to the logical transition function given below and loop transitions are only taken if no outgoing transition holds.
	\begin{align*}
	\begin{aligned}[c]
	\varphi(u_0,u_1) &= \text{\Coffeecup} \land \neg o \\
	\varphi(u_1,u_A) &= o
	\end{aligned}
	\qquad
	\begin{aligned}[c]
	\varphi(u_0,u_A) &= \text{\Coffeecup} \land o \\
	\varphi(u_1,u_R) &= \ast \land \neg o
	\end{aligned}
	\qquad
	\begin{aligned}[c]
	\varphi(u_0,u_R) &= \ast \land \neg \text{\Coffeecup}\\
	~
	\end{aligned}
	\end{align*}
	For all absent pairs of states $(u,u')$, $\varphi(u,u')=\bot$. The automaton covers two accepting cases: (1) $\text{\Coffeecup}$ and $o$ are observed in the same tile (i.e., direct path from $u_0$ to $u_A$) and (2) $\text{\Coffeecup}$ and $o$ are observed in different tiles (i.e., path from $u_0$ to $u_A$ through $u_1$). The transition function is deterministic because all pairs of outgoing transitions from a given state to two different states are mutually exclusive (e.g., the formulas $o$ and $\ast \land \neg o$ are mutually exclusive because $o$ appears positively in the former and negatively in the latter). Note that, in this case, no pairs of states have multiple edges; that is, $|\varphi(u,u')|=1$ for all pairs of states.
\end{example}

\begin{example}
	The automaton traversal for the observation trace $\lambda_{L,\mathcal{O}}=\langle \{\}, \{\text{\Coffeecup}\}, \{\}, \{\}, \{o\} \rangle$ in the automaton of Figure \ref{fig:officeworld_coffee_rm}  is $\mathcal{A}(\lambda_{L,\mathcal{O}})=\langle u_0, u_0, u_1, u_1, u_1, u_A \rangle$.
\end{example}

\section{Representation of Subgoal Automata in Answer Set Programming}
\label{sec:asp_representation_subgoal_automata}
In this section we explain how subgoal automata are represented using Answer Set Programming (ASP). First, we describe how traces and subgoal automata are represented. Then, we present the general rules that describe the behavior of a subgoal automaton. Finally, we prove the correctness of the representation.

\begin{definition}[ASP representation of an observation trace]
	Given an observation trace $\lambda_{L,\mathcal{O}}=\langle O_0,\ldots,O_n\rangle$, $M(\lambda_{L,\mathcal{O}})$ denotes the set of ASP facts that describe it:
	\begin{equation*}
	M(\lambda_{L,\mathcal{O}}) = \begin{matrix*}[l]
	\left\lbrace \mathtt{obs(}o, t). \mid 0 \leq t \leq n, o \in O_t\right\rbrace \cup \\
	\left\lbrace\mathtt{step}(t). \mid 0 \leq t \leq n \right\rbrace \cup \\
	\left\lbrace\mathtt{ last(}n\mathtt{).}\right\rbrace.
	\end{matrix*}
	\end{equation*}
	\label{def:asp_trace_representation}
\end{definition}
The $\mathtt{obs}(o,t)$ predicate indicates that observable $o \in \mathcal{O}$ is observed at step $t$, $\mathtt{step}(t)$ states that $t$ is a step of the trace, and $\mathtt{last}(n)$ indicates that the trace ends at step $n$.

\begin{example}
	The set of ASP facts for the observation trace $\lambda_{L,\mathcal{O}}=\langle\lbrace a\rbrace,\lbrace\rbrace,\lbrace b,c\rbrace\rangle$ is $M(\lambda_{L,\mathcal{O}})=\lbrace \mathtt{obs}(a,0).,~\mathtt{obs}(b,2).,~\mathtt{obs}(c,2).,~\mathtt{step}(0),~\mathtt{step}(1).,~\mathtt{step}(2).,~\mathtt{last}(2).\rbrace$.	
\end{example}

\begin{definition}[ASP representation of a subgoal automaton]
	\label{def:asp_subgoal_automata}
	Given a subgoal automaton $\mathcal{A}=\langle U,\mathcal{O},\delta_\varphi,u_0,u_A, u_R \rangle$, $M(\mathcal{A})=M_U(\mathcal{A}) \cup M_\varphi(\mathcal{A})$ denotes the set of ASP rules that describe it, where:
	\begin{equation*}
	M_U\left(\mathcal{A}\right) = \left\lbrace \mathtt{state}\left(u\right). \mid u \in U \right\rbrace
	\end{equation*}
	and
	\begin{equation*}
	M_\varphi\left(\mathcal{A}\right)=\left\lbrace\begin{array}{@{}l|c@{}}
	\mathtt{ed}(u,u',i).                                                                    &  \\
	\bar{\varphi}(u,u',i,\mathtt{T) \codeif not~obs(}o_1\mathtt{,T), step(T).} & u \in U\setminus \left\lbrace u_A,u_R \right\rbrace, \\
	\multicolumn{1}{c|}{\vdots}                                                                          & u' \in U \setminus \left\lbrace u \right\rbrace, \\
	\bar{\varphi}(u,u',i,\mathtt{T) \codeif not~obs(}o_n\mathtt{,T), step(T).} & 1 \leq i \leq \left|\varphi(u,u')\right|, \\
	\bar{\varphi}(u,u',i,\mathtt{T) \codeif obs(}o_{n+1}\mathtt{,T), step(T).} & \emph{conj}_i \in \varphi(u,u'), \\
	\multicolumn{1}{c|}{\vdots}                                                                          & \begin{split}\emph{conj}_i &= o_1 \land \cdots \land o_n\\ &\land \neg o_{n+1} \land \cdots \land \neg o_m\end{split} \\
	\bar{\varphi}(u,u',i,\mathtt{T) \codeif obs(}o_m,\mathtt{T), step(T).}     & 
	\end{array}\right\rbrace.
	\end{equation*}
\end{definition}
The rules in $M(\mathcal{A})$ are described as follows:
\begin{itemize}
	\item Facts $\mathtt{state}(u)$ indicate that $u$ is an automaton state.
	\item Facts $\mathtt{ed}(u,u',i)$ indicate that there is a transition from state $u$ to $u'$ using edge $i$. Note that $i$ is the $i$-th conjunction in the DNF formula $\varphi(u,u')$.\footnote{Remember that each conjunction in the DNF formula $\varphi(u,u')$ represents a different edge between states $u$ and $u'$.}
	\item Normal rules whose \emph{head} is of the form $\bar{\varphi}(u,u',i,\mathtt{T})$ state that the transition from state $u$ to state $u'$ with edge $i$ \emph{does not} hold at step $\mathtt{T}$. The \emph{body} of these rules consists of a single $\mathtt{obs}(o,\mathtt{T})$ literal and an atom $\mathtt{step}(\mathtt{T})$ indicating that $\mathtt{T}$ is a step. Remember that we represent variables using upper case letters, which is the case of steps $\mathtt{T}$ here.
\end{itemize}
Note that $\bar{\varphi}$ represents the negation of the logical transition function $\varphi$. This is because, as discussed in more detail later in Section~\ref{sec:learn_subgoal_automata_from_traces}, learning the negation of the logical transition function $\varphi$ makes the search space smaller and, thus, makes the learning process faster.

\begin{example}
	\label{ex:asp_rules_automaton}
	The following rules represent the automaton in Figure~\ref{fig:officeworld_coffee_rm} (p.~\pageref{fig:officeworld_coffee_rm}).
	\begin{equation*}
	\begin{Bmatrix*}[l]
	\mathtt{state}(u_0).~\mathtt{state}(u_1).&\mathtt{state}(u_A).~\mathtt{state}(u_R). \\
	\mathtt{ed}(u_0, u_1, 1).~\mathtt{ed}(u_0, u_A, 1).~\mathtt{ed}(u_0, u_R, 1).                   &\mathtt{ed}(u_1, u_A, 1).~\mathtt{ed}(u_1, u_R, 1).\\
	\bar{\varphi}(u_0, u_1, 1, \mathtt{T}) \codeif \mathtt{not~obs(\text{\Coffeecup}, T), step(T).} &\bar{\varphi}(u_0, u_1, 1, \mathtt{T) \codeif obs(}o\mathtt{, T), step(T).}\\
	\bar{\varphi}(u_0, u_A, 1, \mathtt{T) \codeif not~obs(\text{\Coffeecup}, T), step(T).}          &\bar{\varphi}(u_0, u_A, 1, \mathtt{T}) \codeif \mathtt{not~obs(}o\mathtt{, T), step(T).}\\
	\bar{\varphi}(u_0, u_R, 1, \mathtt{T}) \codeif \mathtt{not~obs(\ast, T), step(T).}              &\bar{\varphi}(u_0, u_R, 1, \mathtt{T}) \codeif \mathtt{obs(\text{\Coffeecup}, T), step(T).}\\
	\bar{\varphi}(u_1, u_A, 1, \mathtt{T) \codeif not~obs(}o\mathtt{, T), step(T).}&\\
	\bar{\varphi}(u_1, u_R, 1, \mathtt{T) \codeif not~obs(\ast, T), step(T).}                       &\bar{\varphi}(u_1, u_R, 1, \mathtt{T) \codeif obs(}o\mathtt{, T), step(T).}  
	\end{Bmatrix*}
	\end{equation*}	
\end{example}

\paragraph{General Rules.} In order to check whether an automaton accepts or rejects an observation trace, it is necessary to reason about the automaton's behavior. This is done by means of a set of rules $R$ that define how a subgoal automaton processes an observation trace. This set of rules is given by the union of different components, $R = R_\varphi \cup R_\delta \cup R_\mathtt{st}$. The subsets $R_\varphi$ and $R_\delta$ define the rules related to the automaton transition function:
\begin{itemize}
	\item The first rule in $R_\varphi$ defines the logical transition function $\varphi$ in terms of its negation $\bar{\varphi}$ and $\mathtt{ed}$ atoms. The second rule indicates that an outgoing transition from state $\mathtt{X}$ is taken at step $\mathtt{T}$. 
	\begin{align*}
	R_\varphi &= \begin{Bmatrix*}[l]
	\varphi\mathtt{(X,Y,E,T) \codeif not~}\bar{\varphi}\mathtt{(X,Y,E,T),  ed(X,Y,E), step(T).}\\
	\mathtt{out\_\varphi(X,T) \codeif \varphi(X,\_,\_,T).}
	\end{Bmatrix*}
	\end{align*}
	\item The rules in $R_\delta$ define the transition function $\delta$ in terms of $\varphi$, as defined in Equation~\ref{eq:delta_from_phi}~(p.~\pageref{eq:delta_from_phi}). The first rule states that $\mathtt{X}$ transitions to $\mathtt{Y}$ at step $\mathtt{T}$ if an outgoing transition to $\mathtt{Y}$ holds at that step. In contrast, the second rule indicates that state $\mathtt{X}$ transitions to itself at step $\mathtt{T}$ if no outgoing transition is satisfied at that step.
	\begin{align*}
	R_{\delta} &= \begin{Bmatrix*}[l]
	\mathtt{\delta(X,Y,T) \codeif \varphi(X,Y,\_,T).}\\
	\mathtt{\delta(X, X, T) \codeif not~out\_\varphi(X,T), state(X), step(T).}
	\end{Bmatrix*}
	\end{align*}
\end{itemize}

The subset $R_\mathtt{st}$ is used to define the automaton traversal of the trace (that is, the sequence of visited automaton states), and the criteria for accepting or rejecting a trace. The $\mathtt{st(T,X)}$ atoms indicate that a trace is in state $\mathtt{X}$ at step $\mathtt{T}$. The first rule defines that the agent is in $u_0$ at step $\mathtt{0}$. The second rule determines that at step $\mathtt{T\mathord{+}1}$ the agent will be in state $\mathtt{Y}$ if it is in state $\mathtt{X}$ at step $\mathtt{T}$ and a transition between them holds at that step. The third (resp.\ fourth) rule indicates that the observation trace is accepted (resp.\ rejected) if the state at the trace's last step is $u_A$ (resp.\ $u_R$).
\begin{equation*}
R_\mathtt{st} = \begin{Bmatrix*}[l]
\mathtt{st(0},u_0).\\
\mathtt{st(T\mathord{+}1,Y) \codeif st(T,X),\delta(X,Y,T).}\\
\mathtt{accept \codeif last(T), st(T\mathord{+}1 },u_A).\\
\mathtt{reject \codeif last(T), st(T\mathord{+}1},u_R).
\end{Bmatrix*}
\end{equation*}

\begin{restatable}[Correctness of the ASP encoding]{proposition}{propositionaspcorrectness}
	Given a finite observation trace $\lambda_{L,\mathcal{O}}^\ast$, where $\ast \in \lbrace G, D, I\rbrace$, and an automaton $\mathcal{A}$ that is valid with respect to $\lambda_{L,\mathcal{O}}^\ast$, the program $P=M(\mathcal{A}) \cup R \cup M(\lambda_{L,\mathcal{O}}^\ast)$ has a unique answer set $AS$ and (1) $\mathtt{accept}\in AS$ if and only if $\ast = G$, and (2) $\mathtt{reject} \in AS$ if and only if $\ast = D$.
	\label{prop:asp_correctness}
\end{restatable}
\begin{proof}
	See Appendix \ref{proof:correctness_asp}.
\end{proof}

\section{Learning Subgoal Automata from Traces}
\label{sec:learn_subgoal_automata_from_traces}
This section describes our approach for learning a subgoal automaton. Firstly, we formalize the task of learning an automaton from traces.

\begin{definition}
	An automaton learning task is a tuple $T_{\mathcal{A}}=\langle U, \mathcal{O}, u_0, u_A, u_R, \Lambda_{L,\mathcal{O}}, \kappa \rangle$, where
	\begin{itemize}
		\item $U \supseteq \lbrace u_0, u_A, u_R\rbrace$ is a set of automaton states, where $u_0$ is the initial state, $u_A$ is the accepting state and $u_R$ is the rejecting state;
		\item $\mathcal{O}$ is a set of observables;
		\item $\Lambda_{L,\mathcal{O}}=\Lambda^G_{L,\mathcal{O}} \cup \Lambda^D_{L,\mathcal{O}} \cup \Lambda^I_{L,\mathcal{O}}$ is a set of observation traces; and
		\item $\kappa$ is the maximum number of directed edges $(u,u')$ from a state $u\in U$ to another state $u'\in U\setminus\{u\}$.
	\end{itemize}
	An automaton $\mathcal{A}$ is a solution of $T_\mathcal{A}$ if and only if it is valid with respect to all the traces in $\Lambda_{L,\mathcal{O}}$; that is, if and only if it accepts all goal traces in $\Lambda_{L,\mathcal{O}}^G$, rejects all dead-end traces in $\Lambda_{L,\mathcal{O}}^D$, and does not accept nor reject any incomplete trace in $\Lambda_{L,\mathcal{O}}^I$.
\end{definition}
Note that (i) $\kappa$ can be seen as the maximum number of disjuncts that a DNF formula $\varphi(u,u')$ between two states $u$ and $u'$ can have, and (ii) $U$ will be the set of states of the learned automaton.

Given an automaton learning task $T_\mathcal{A}$, we map it into an ILASP learning task $M(T_{\mathcal{A}})=\langle B,S_M,\langle E^+, \emptyset \rangle\rangle$ and use the ILASP system~\shortcite{ILASP_system} to find a minimal inductive solution $M_\varphi(\mathcal{A}) \subseteq S_M$ that covers the examples.\footnote{Note that we do not use \emph{negative examples} ($E^-=\emptyset$).} We define the different components of $M(T_{\mathcal{A}})$ below.

\paragraph{Background Knowledge.}
The background knowledge $B=B_U \cup R$ is a set of rules that describe the general behavior of any subgoal automaton. The set of rules $B_U$ consists of $\mathtt{state}(u)$ facts for each automaton state $u \in U$, while $R$ is the set of general rules that defines how subgoal automata process observation traces (see the definition in Section~\ref{sec:asp_representation_subgoal_automata}). Describing known concepts through the background knowledge is useful to avoid learning everything from scratch. In this case, we only need to learn the edges of the automata, and not the general definitions of how subgoal automata work (i.e., the rules in $R$).

\paragraph{Hypothesis Space.}
The hypothesis space $S_M$ contains all $\mathtt{ed}$ and $\bar{\varphi}$ rules that characterize a transition from a non-terminal state $u \in U \setminus \{u_A,u_R\}$ to a different state $u' \in U \setminus \{u\}$ using edge $i\in[1,\kappa]$. Formally, it is defined as
\begin{equation*}
S_M=\left\lbrace\begin{array}{@{}l|c@{}}
\mathtt{ed}(u,u',i).                                                  & u \in U\setminus \left\lbrace u_A,u_R \right\rbrace, \\
\bar{\varphi}(u,u',i,\mathtt{T) \codeif obs}(o, \mathtt{T), step(T).} & u' \in U \setminus \left\lbrace u \right\rbrace, \\
\bar{\varphi}(u,u',i,\mathtt{T) \codeif not~obs}(o, \mathtt{T), step(T).}  & i \in \left[1, \kappa \right], o \in \mathcal{O}
\end{array}\right\rbrace.
\end{equation*}
Loop transitions are not included since they are unsatisfiable formulas (see Definition~\ref{def:logical_transition_function}). Note that it is possible to learn unlabeled transitions, which are taken unconditionally (that is, regardless of the current observation). For example, for a transition from $u$ to $u'$ using edge $i$, an inductive solution may only include $\mathtt{ed}(u,u',i)$ and not $\bar{\varphi}(u,u',i,\mathtt{T})$.

As mentioned before, the learner induces the negation $\bar{\varphi}$ of a logical transition function $\varphi$.\footnote{Remember that the set of rules $R$ introduced in Section~\ref{sec:asp_representation_subgoal_automata} defines $\varphi$ in terms of $\bar{\varphi}$.} A different hypothesis space where the learned rules characterize $\varphi$ directly could have been defined. However, this requires guessing the maximum number of literals that label a transition between two states.\footnote{This is because ILASP has the maximum length of learnable rules as a parameter. This is a problem to enforce determinism (see Section~\ref{sec:determinism}) since we do not know how many literals are going to be needed to make two formulas mutually exclusive. Besides, allowing for an arbitrarily large number of literals to overcome the problem increases the hypothesis massively.} Therefore, we represent subgoal automata using $\bar{\varphi}$ and instead of imposing a maximum size for the conjunctive formulas, we impose a limit ($\kappa$) on the number of edges from one state to another.

It is important to realize that the learned hypothesis is denoted by $M_\varphi(\mathcal{A})$ and not $M(\mathcal{A})$ (see Definition~\ref{def:asp_subgoal_automata}). The set of automaton states is given in the background knowledge $B$, and the hypothesis space $S_M$ only contains transition rules. Hence, the hypothesis is a smallest subset  of transition rules that covers all the examples. Since the set of automaton states is provided through the background knowledge, a minimal automaton (i.e., an automaton with the minimum number of states) is only guaranteed to be learned when the set of automaton states is the minimal one. The mechanism that interleaves reinforcement learning and automaton learning described in Section~\ref{sec:interleaved_automata_learning_algorithm} ensures that the learned automaton is minimal for a specific $\kappa$.

\paragraph{Example Sets.}
Given a set of traces $\Lambda_{L,\mathcal{O}}=\Lambda^G_{L,\mathcal{O}} \cup \Lambda^D_{L,\mathcal{O}} \cup \Lambda^I_{L,\mathcal{O}}$, the set of \emph{positive examples} is defined as 
\begin{equation*}
E^+=\lbrace\langle e^*, M(\lambda_{L,\mathcal{O}}) \rangle \mid \ast \in \lbrace G,D,I \rbrace,\lambda_{L,\mathcal{O}} \in \Lambda^*_{L,\mathcal{O}} \rbrace,
\end{equation*}
where 
\begin{itemize}
	\item $e^G=\langle \lbrace\mathtt{accept}\rbrace,\lbrace\mathtt{reject}\rbrace \rangle$,
	\item $e^D=\langle \lbrace\mathtt{reject}\rbrace, \lbrace\mathtt{accept}\rbrace \rangle$, and
	\item $e^I=\langle \lbrace\rbrace,\lbrace\mathtt{accept, reject}\rbrace \rangle$
\end{itemize}
are the partial interpretations for goal, dead-end and incomplete traces. The $\mathtt{accept}$ and  $\mathtt{reject}$ atoms express whether a trace is accepted or rejected by the automaton; hence, goal traces must only be accepted, dead-end traces must only be rejected, and incomplete traces cannot be accepted or rejected. Note that the context of each example is the set of ASP facts $M(\lambda_{L,\mathcal{O}})$ that represents the corresponding trace (see Definition~\ref{def:asp_trace_representation}).

\paragraph{Correctness of the Learning Task.} The following theorem captures the correctness of the automaton learning task.
\begin{theorem}
	Given an automaton learning task $T_\mathcal{A}=\langle U, \mathcal{O},u_0,u_A,u_R,\Lambda_{L,\mathcal{O}},\kappa\rangle$, an automaton $\mathcal{A}$ is a solution of $T_\mathcal{A}$ if and only if $M_\varphi(\mathcal{A})$ is an inductive solution of $M(T_\mathcal{A})=\langle B, S_M, \langle E^+,\emptyset\rangle\rangle$. 
	\label{theorem:first_theorem}
\end{theorem}
\begin{proof}
	Assume $\mathcal{A}$ is a solution of $T_\mathcal{A}$.
	
	$\iff$ $\mathcal{A}$ is valid with respect to all traces in $\Lambda_{L,\mathcal{O}}$ (i.e., $\mathcal{A}$ accepts all traces in $\Lambda^G_{L,\mathcal{O}}$, rejects all traces in $\Lambda^D_{L,\mathcal{O}}$ and does not accept nor reject any trace in $\Lambda^I_{L,\mathcal{O}}$).
	
	$\iff$ By Proposition~\ref{prop:asp_correctness}, for each trace $\lambda^*_{L,\mathcal{O}}\in\Lambda^*_{L,\mathcal{O}}$ where $\ast \in \lbrace G,D,I\rbrace$, $M(\mathcal{A}) \cup R \cup M(\lambda_{L,\mathcal{O}}^\ast)$ has a unique answer set $AS$ and (1) $\mathtt{accept}\in AS$ if and only if $\ast = G$, and (2) $\mathtt{reject} \in AS$ if and only if $\ast = D$.
	
	$\iff$ For each example $e \in E^+$, $R\cup M(\mathcal{A})$ accepts $e$.
	
	$\iff$ For each example $e \in E^+$, $B\cup M_\varphi(\mathcal{A})$ accepts $e$ (the two programs are identical).
	
	$\iff M_\varphi(\mathcal{A})$ is an inductive solution of $M(T_\mathcal{A})$.
\end{proof}

\paragraph{Optimization.} To make the formalization simpler, we have always included $u_A$ and $u_R$ in the set of automaton states $U$ of the automaton learning task $T_\mathcal{A}$. However, in practice, $u_A$ is not included in $U$ when the set of goal traces $\Lambda^G_{L,\mathcal{O}}$ is empty. Likewise, $u_R$ is not included in $U$ when the set of dead-end traces $\Lambda^D_{L,\mathcal{O}}$ is empty. Removing these states when they are not needed is helpful to make the hypothesis space smaller.

\section{Verification of Structural Properties}
\label{sec:structural_properties}
The formalization of the automaton learning task described in Section~\ref{sec:learn_subgoal_automata_from_traces} induces non-deterministic automata. That is, there can exist an observation that simultaneously satisfies two formulas labeling outgoing edges to two different states. In order to comply with the definition of the logical transition function (see Definition~\ref{def:logical_transition_function}, p.~\pageref{def:logical_transition_function}), we need to constrain it to be deterministic. The automaton states act as a proxy for determining what is the current level of completion of the task; in other words, which subgoals have been achieved. If the automaton is non-deterministic, several different levels of completion can be active simultaneously, which is impractical from the perspective of an RL agent. These agents need to know exactly in which stage they are in order to make an appropriate decision.

Determinism is not the only property we can impose. Note that an automaton can be easily transformed into an equivalent one by rearranging its state and edge identifiers. Therefore, ILASP can visit multiple symmetric automata during the search for a solution that covers the examples. By avoiding revisiting parts of the search space, the automaton learning time can be greatly reduced.

In this section we introduce rules that allow us to verify that a given automaton is deterministic and complies with specific properties that characterize a canonical representation (e.g., a criteria for naming the automaton states). Crucially, these verification rules can also be used to enforce these properties during search.

Both the determinism and symmetry properties are related to the automaton structure; that is, the edges and formulas that label them. The ASP representation $M(\mathcal{A})$ of a subgoal automaton $\mathcal{A}$ given in Definition \ref{def:asp_subgoal_automata} represents the formulas on the edges as rules. Consequently, we cannot impose constraints over them easily. To solve this problem, we map the $\bar{\varphi}$ rules in $M(\mathcal{A})$ to facts of the form $\mathtt{pos}(u,u',i,o)$ and $\mathtt{neg}(u,u',i,o)$. These facts express that observable $o\in\mathcal{O}$ appears positively (resp. negatively) in the edge $i$ from state $u$ to state $u'$. Formally, given the ASP encoding $M(\mathcal{A})$ of automaton $\mathcal{A}$, its corresponding mapping into facts $F(M(\mathcal{A}))$ is defined as follows:
\begin{equation*}
F(M(\mathcal{A}))=\left\lbrace\begin{array}{@{}l|l@{}}
\mathtt{state}(u).                                                         & \mathtt{state}(u).\\
\mathtt{ed}(u,u',i).                                                       & \mathtt{ed}(u,u',i).\\
\mathtt{pos}(u,u',i,o_1).                                                  & \bar{\varphi}(u,u',i,\mathtt{T) \codeif not~obs(}o_1\mathtt{,T), step(T).} \\
\multicolumn{1}{c|}{\vdots}                                                & \multicolumn{1}{c}{\vdots}  \\
\mathtt{pos}(u,u',i,o_n).                                                  & \bar{\varphi}(u,u',i,\mathtt{T) \codeif not~obs(}o_n\mathtt{,T), step(T).} \\
\mathtt{neg}(u,u',i,o_{n+1}).                                              & \bar{\varphi}(u,u',i,\mathtt{T) \codeif obs(}o_{n+1}\mathtt{,T), step(T).} \\
\multicolumn{1}{c|}{\vdots}                                                & \multicolumn{1}{c}{\vdots}    \\
\mathtt{neg}(u,u',i,o_{m}).                                                & \bar{\varphi}(u,u',i,\mathtt{T) \codeif obs(}o_m,\mathtt{T), step(T).} 
\end{array}\right\rbrace.
\end{equation*}
Note that the right hand side of the set corresponds to a given $M(\mathcal{A})$. The facts $\mathtt{state}$ and $\mathtt{ed}$ are replicated, while the rules $\bar{\varphi}$ are transformed into $\mathtt{pos}$ and $\mathtt{neg}$ facts.\footnote{Note that we could have learned the factual representation $F(M(\mathcal{A}))$ instead of the one based on normal rules $M(\mathcal{A})$. We opted for the latter since it requires less grounding and can be potentially used to represent more complex automata in future work (see Section~\ref{sec:conclusions}).} Given this factual representation of a subgoal automaton, we can now define constraints for enforcing determinism and a canonical structure over these facts.

\begin{example}
	The set of facts below represents the formulas on the edges of the automaton in Figure~\ref{fig:officeworld_coffee_rm} (p.~\pageref{fig:officeworld_coffee_rm}) generated from the logical transition function in Example~\ref{ex:asp_rules_automaton} (p.~\pageref{ex:asp_rules_automaton}).
	\begin{equation*}
	\begin{Bmatrix*}[l]
	\mathtt{pos}(u_0,u_1,1,\text{\Coffeecup}). & \mathtt{neg}(u_0,u_1,1,o).\\
	\mathtt{pos}(u_0,u_A,1,\text{\Coffeecup}). & \mathtt{pos}(u_0,u_A,1,o).\\
	\mathtt{pos}(u_0,u_R,1,\ast).              & \mathtt{neg}(u_0,u_R,1,\text{\Coffeecup}).\\
	\mathtt{pos}(u_1,u_A,1,o).\\
	\mathtt{pos}(u_1,u_R,1,\ast).              & \mathtt{neg}(u_1,u_R,1,o).
	\end{Bmatrix*}
	\end{equation*}
	
\end{example}

Verifying that a learned automaton complies with a set of structural properties can be done over its factual encoding. That is, those automata that violate the properties are discarded as solutions to the automaton learning task. However, this is clearly computationally expensive. ILASP allows the above mapping to be included in the learning task through meta-program injection~\shortcite{LawRB18}. This enables the learner to verify the properties during the search for an automaton, which effectively shrinks the search space and speeds up automaton learning.

In what follows, we use the factual representation above to encode the determinism constraints (Section~\ref{sec:determinism}) and a canonical structure for breaking symmetries (Section~\ref{sec:symmetry_breaking}).

\subsection{Determinism}
\label{sec:determinism}
As described in Section~\ref{sec:subgoal_automata_def}, a logical transition function is deterministic if no observation can satisfy two outgoing formulas to two different states simultaneously. As these formulas are conjunctions of literals,  two formulas are mutually exclusive (i.e., cannot be both satisfied at the same time) if and only if an observable appears positively in the first formula and negatively in the second (or vice versa). The set of rules below encodes this definition by means of the $\mathtt{mutex(X,Y,EY,Z,EZ)}$ predicate which indicates that the formula on the edge from $\mathtt{X}$ to $\mathtt{Y}$ with index $\mathtt{EY}$ is mutually exclusive with the edge $\mathtt{X}$ to $\mathtt{Z}$ with index $\mathtt{EZ}$. The first and second rules specify that two outgoing edges from state $\mathtt{X}$ to two different states ($\mathtt{Y}$ and $\mathtt{Z}$) are mutually exclusive if an observable $\mathtt{O}$ appears positively in one edge and negatively in the other.\footnote{The comparison $\mathtt{Y\mathord{<}Z}$ is done instead of $\mathtt{Y!\mathord{=}Z}$ for efficiency purposes. Note that both comparisons are equivalent in this context. The former imposes a lexicographical order to evaluate the rules and thus avoids reevaluating the expression when $\mathtt{Y}$ and $\mathtt{Z}$ are interchanged.} The third rule enforces edges from a state $\mathtt{X}$ to two different states $\mathtt{Y}$ and $\mathtt{Z}$ to be mutually exclusive.
\begin{equation*}
\begin{Bmatrix*}[l]
\mathtt{mutex(X, Y, EY, Z, EZ) \codeif pos(X, Y, EY, O), neg(X, Z, EZ, O), Y\mathord{<}Z.}\\
\mathtt{mutex(X, Y, EY, Z, EZ) \codeif neg(X, Y, EY, O), pos(X, Z, EZ, O), Y\mathord{<}Z.}\\
\mathtt{\codeif not~mutex(X, Y, EY, Z, EZ), ed(X, Y, EY), ed(X, Z, EZ), Y\mathord{<}Z.}
\end{Bmatrix*}
\end{equation*}

\subsection{Symmetry Breaking}
\label{sec:symmetry_breaking}
In this section we describe rules which aim to enforce that the induced automata follow a canonical structure in order to make the search for solution faster by breaking symmetries. 

There are several types of symmetries we are interested in breaking. Firstly, any two states except for $u_0$, $u_A$ and $u_R$ are interchangeable. For example, Figure~\ref{fig:isomorphishms} shows two automata whose states $u_1,u_2$ and $u_3$ can be used interchangeably.
\begin{figure}[t]
	\centering
	\subfloat[\textsc{VisitABCD}\label{fig:isomorphisms_patrol}]{
		\resizebox{0.42\columnwidth}{!}{
			\begin{tikzpicture}[shorten >=1pt,node distance=3cm,on grid,auto]
			\node[state,initial] (u_0)   {$u_0$};
			\node[state,fill=lightgray] (u_1) [right =of u_0] {$u_1$};
			\node[state,fill=lightgray] (u_2) [below =of u_1] {$u_2$};
			\node[state,fill=lightgray] (u_3) [below =of u_2] {$u_3$};
			\node[state,accepting] (u_A) [left =of u_3] {$u_A$};
			\node[state] (u_R) [left =of u_2] {$u_R$};
			
			\path[->] (u_0) edge node[in place] {$A \land \neg \ast$} (u_1);
			\path[->] (u_1) edge node[in place] {$B \land \neg \ast$} (u_2);
			\path[->] (u_2) edge node[in place] {$C \land \neg \ast$} (u_3);
			\path[->] (u_3) edge node[in place] {$D \land \neg \ast$} (u_A);
			
			\path[->] (u_0) edge node [in place] {$\ast$} (u_R);
			\path[->] (u_1) edge node [in place] {$\ast$} (u_R);
			\path[->] (u_2) edge node [in place] {$\ast$} (u_R);
			\path[->] (u_3) edge node [in place] {$\ast$} (u_R);
			\end{tikzpicture}
		}
	}
	\hfill
	\subfloat[\textsc{CoffeeMail}\label{fig:isomorphisms_coffee_mail}]{
		\centering
		\resizebox{0.53\columnwidth}{!}{
			\begin{tikzpicture}[shorten >=1pt,node distance=3.2cm,on grid,auto]
			\node[state,initial] (u_0)   {$u_0$};
			\node[state,accepting] (u_A) [below =of u_0] {$u_A$};
			\node[state,fill=lightgray] (u_1) [left =of u_A] {$u_1$};
			\node[state,fill=lightgray] (u_2) [right =of u_A] {$u_2$};
			\node[state,fill=lightgray] (u_3) [below =of u_A] {$u_3$};
			
			\path[->] (u_0) edge[bend right=20] node [swap] {$\text{\Coffeecup} \land \neg\text{\Letter}$} (u_1);
			\path[->] (u_0) edge[bend left=20] node {$\neg\text{\Coffeecup} \land \text{\Letter}$} (u_2);
			\path[->] (u_0) edge[bend right, below] node[in place,pos=0.6] {$\text{\Coffeecup} \land \text{\Letter} \land \neg o$} (u_3);
			\path[->] (u_0) edge node {$\text{\Coffeecup} \land \text{\Letter} \land o$} (u_A);
			
			\path[->] (u_1) edge[bend right=20] node [swap] {$\text{\Letter} \land \neg o$} (u_3);
			\path[->] (u_1) edge[above] node [swap] {$\text{\Letter} \land o$} (u_A);
			
			\path[->] (u_2) edge[bend left=20] node {$\text{\Coffeecup} \land \neg o$} (u_3);
			\path[->] (u_2) edge[above] node {$\text{\Coffeecup} \land o$} (u_A);
			
			\path[->] (u_3) edge node [swap] {$o$} (u_A);
			\end{tikzpicture}
		}
	}
	\caption{Minimal subgoal automata for two \textsc{OfficeWorld} tasks. Self-loops and transitions to the rejecting state $u_R$ in (b) are omitted for simplicity. The shaded states can be interchanged in the absence of symmetry breaking.
	}
	\label{fig:isomorphishms}
\end{figure}
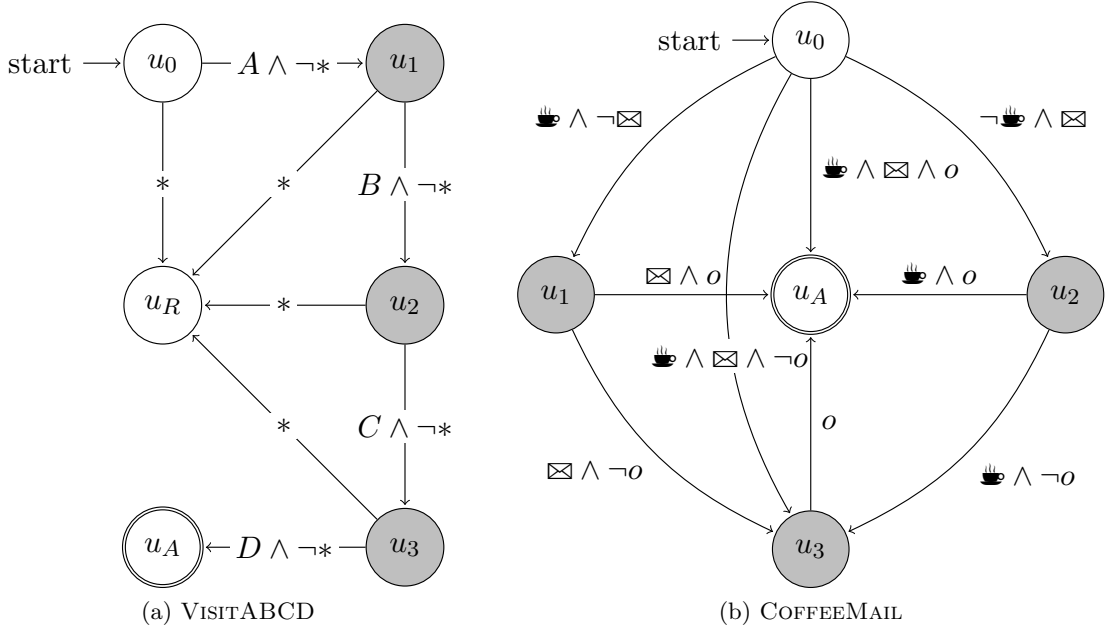
Secondly, there can also be symmetries in the automaton edges. The inductive solution to the automaton learning task $M(T_\mathcal{A})$ contains transition rules whose edge indices range between 1 and $\kappa$ (the maximum number of edges from one state to another). For instance, if $\kappa=2$ a potentially learned representation of the \textsc{Coffee} automaton (see Figure~\ref{fig:officeworld_coffee_rm}, p.~\pageref{fig:officeworld_coffee_rm}) is shown below. Note that edges can arbitrarily be labeled 1 or 2 even though there is a single edge between every pair of states.
\begin{equation*}
\begin{Bmatrix*}[l]
\mathtt{ed}(u_0, u_1, 2).~\mathtt{ed}(u_0, u_A, 1).~\mathtt{ed}(u_0, u_R, 2).                   &\mathtt{ed}(u_1, u_A, 1).~\mathtt{ed}(u_1, u_R, 2).\\
\bar{\varphi}(u_0, u_1, 2, \mathtt{T}) \codeif \mathtt{not~obs(\text{\Coffeecup}, T), step(T).} &\bar{\varphi}(u_0, u_1, 2, \mathtt{T) \codeif obs(}o\mathtt{, T), step(T).}\\
\bar{\varphi}(u_0, u_A, 1, \mathtt{T) \codeif not~obs(\text{\Coffeecup}, T), step(T).}          &\bar{\varphi}(u_0, u_A, 1, \mathtt{T}) \codeif \mathtt{not~obs(}o\mathtt{, T), step(T).}\\
\bar{\varphi}(u_0, u_R, 2, \mathtt{T}) \codeif \mathtt{not~obs(\ast, T), step(T).}              &\bar{\varphi}(u_0, u_R, 2, \mathtt{T}) \codeif \mathtt{obs(\text{\Coffeecup}, T), step(T).}\\
\bar{\varphi}(u_1, u_A, 1, \mathtt{T) \codeif not~obs(}o\mathtt{, T), step(T).}\\
\bar{\varphi}(u_1, u_R, 2, \mathtt{T) \codeif not~obs(\ast, T), step(T).}                       &\bar{\varphi}(u_1, u_R, 2, \mathtt{T) \codeif obs(}o\mathtt{, T), step(T).}  
\end{Bmatrix*}.
\end{equation*}
Finally, the indices of two edges between the same pair of states can be also interchanged.

The idea of our symmetry breaking mechanism is to impose a unique assignment of state and edge indices given a labeling of the automaton edges. In Section~\ref{sec:symmetry_breaking_method} we propose a symmetry breaking mechanism for a particular class of labeled directed graphs, whereas in Section~\ref{sec:symmetry_breaking_application} we explain how this method applies to subgoal automata, which can be represented as graphs in this class.

\subsubsection{Graph Indexing}
\label{sec:symmetry_breaking_method}
In this section we propose a symmetry breaking mechanism for a particular class of labeled directed graphs. 

Let $\mathcal{L}=\{l_1,\ldots,l_k\}$ be a set of \emph{labels}, and let $G=(V,E)$ be a labeled directed graph with a set of nodes $V = \{v_1,\ldots,v_n\}$ and a set of edges $E$. Each edge in $E$ is of the form $(u,v,L)$, where $u,v\in V$ are the two connected nodes, and $L\subseteq \mathcal{L}$ is a subset of labels. For each node $u \in V$, let $E^o(u) = \{(v,w,L) \in E \mid u = v\}$ be the set of outgoing labeled edges from $u$, and let $E^i(u) = \{(v,w,L) \in E \mid u = w\}$ be the set of incoming labeled edges.

We define a class $\mathcal{G}$ of labeled directed graphs by imposing three assumptions:
\begin{assumption}
	\label{assump:sym_starting_node}
	The node $v_1$ is a designated start node.
\end{assumption}
\begin{assumption}
	\label{assump:sym_reachability}
	Each node $u \in V \setminus \{v_1\}$ is reachable on a directed path from $v_1$.
\end{assumption}
\begin{assumption}
	\label{assump:sym_unique_label_sets}
	Outgoing label sets from each node are unique, i.e.~for each $u \in V$ and label set $L \subseteq \mathcal{L}$ there is at most one edge $(u,v,L)\in E^o(u)$.
\end{assumption}
As a consequence of Assumption~\ref{assump:sym_reachability}, it holds that $|E^i(u)| \geq 1$ for each $u \in V \setminus \{v_1\}$.

\begin{example}
	The following figure is a labeled directed graph $G=\langle V,E\rangle$ that belongs to class $\mathcal{G}$, where $V=\{v_1,\ldots,v_5\}$. The set of labels is $\mathcal{L}=\{a,b,c,d,e,f\}$.
	\begin{center}
		\begin{tikzpicture}[shorten >=1pt,node distance=2.5cm,on grid,auto]
		\node[state,initial] (u_1)   {$v_1$};
		\node[state] (u_3) [above right =of u_1] {$v_5$};
		\node[state] (u_2) [below right =of u_1] {$v_4$};
		\node[state] (u_4) [above right =of u_2] {$v_3$};	
		\node[state] (u_5) [right =of u_4] {$v_2$};	
		
		\path[->] (u_1) edge node {$\{a,f\}$} (u_3);
		\path[->] (u_1) edge node[swap] {$\{b,e\}$} (u_2);
		\path[->] (u_1) edge node {$\{a,b\}$} (u_4);
		\path[->] (u_3) edge node {$\{b\}$} (u_4);
		\path[->] (u_2) edge node[swap] {$\{a\}$} (u_4);
		\path[->] (u_4) edge[bend left] node {$\{c\}$} (u_5);
		\path[->] (u_4) edge[bend right] node[swap] {$\{d\}$} (u_5);
		\end{tikzpicture}
	\end{center}
	\label{ex:labeled_graph}
\end{example}

\paragraph{Label Set Ordering.} Given a set of labels $\mathcal{L}=\{l_1,\ldots,l_k\}$, we impose a total order on label sets as follows.

\begin{definition}
	\label{def:observation_ordering}
	A label set $L\subseteq \mathcal{L}$ is lower than a label set $L'\subseteq \mathcal{L}$, denoted $L<L'$, if there exists a label $l_{1\leq m \leq k} \in \mathcal{L}$ such that
	\begin{enumerate}
		\item $l_m \notin L$ and $l_m \in L'$, and
		\item there is not a label $l_{m'\mid m'<m}$ such that $l_{m'} \in L$ and $l_{m'} \notin L'$.
	\end{enumerate}
\end{definition}
A label set $L \subseteq \mathcal{L}$ can be mapped into a binary string $B(L)=\{0,1\}^k$ where $B_m(L)=1$ if $l_m\in L$ and 0 otherwise. Note that $B_m(L)$ denotes the $m$-th digit in  $B(L)$. Then, a label set $L$ is lower than another label set $L'$ if its binary representation $B(L)$ is lexicographically lower than $B(L')$.
\begin{example}
	Given the label set $\mathcal{L}=\{a,b,c,d,e,f\}$, the following inequalities between some of its subsets hold. Note that the second column contains the corresponding binary representations of the sets.
	\begin{equation*}
	\begin{matrix*}[c]%
	\{\} &<& \{a,d,f\}   & (000000 < 100101),\\
	\{d\} &<& \{c\} & (000100 < 001000), \\
	\{b,e\} &<& \{a, f\} & (010010 < 100001),\\
	\{a,f\} &<& \{a,b\} & (100001 < 110000).
	\end{matrix*}
	\end{equation*}
	\label{ex:observation_comparison}
\end{example}

\paragraph{Graph Indexing.} Given a graph $G\in\mathcal{G}$, we next introduce a {\em graph indexing} that assigns unique integers to each node in $V$ and, for each node $u\in V$, to each outgoing edge in $E^o(u)$. Formally, a graph indexing is a tuple $\mathcal{I}(G) = \langle f,\{\Gamma_u\}_{u\in V} \rangle$ of bijections defined as
\begin{align*}
f &: V \rightarrow \{1,\ldots,|V|\} \;\; \mathrm{s.t.} \;\; f(v_1) = 1,\\
\Gamma_u &: E^o(u) \rightarrow \{1,\ldots,|E^o(u)|\}, \;\; \forall u \in V.
\end{align*}
Hence a graph indexing always assigns 1 to the designated start node $v_1$. Since outgoing label sets are unique due to Assumption~\ref{assump:sym_unique_label_sets}, we use $\Gamma_u(L)$ as shorthand for $\Gamma_u(u,v,L)$.

Given a graph indexing $\mathcal{I}(G)$, we introduce an associated {\em parent function} $\Pi_\mathcal{I}:V \setminus \{v_1\} \to \{1,\ldots,|V|\} \times \mathbb{N}$ from nodes (excluding the start node $v_1$) to pairs of integers, defined as
\begin{equation*}
\Pi_\mathcal{I}(v)=\min_{(u,v,L)\in E^i(v)} (f(u), \Gamma_u(L)).
\end{equation*}
Here, the minimum is with respect to a lexicographical ordering of integer pairs. Hence $\Pi_\mathcal{I}(v)=(i,e)$ is the smallest integer $i$ assigned to any node on an incoming edge to $v$ and, in the case of ties, the smallest integer $e$ on such an edge. Note that the parent function is well-defined since $|E^i(v)|\geq 1$ for each $v\in V\setminus\{v_1\}$ due to the reachability assumption.

We are particularly interested in the graph indexing that corresponds to a breadth-first search (BFS) traversal of the graph $G$, which we proceed to define.
\begin{definition}\label{def:bfs}
	A graph indexing $\mathcal{I}(G)$ is a {\em BFS traversal} if the following conditions hold:
	\begin{enumerate}
		\item For each pair of nodes $u$ and $v$ in $V\setminus\{v_1\}$, $\Pi_\mathcal{I}(u) < \Pi_\mathcal{I}(v) \Leftrightarrow f(u) < f(v)$.
		\item For each node $u\in V$ and each pair of outgoing edges $(u,v,L)$ and $(u,v',L')$ in $E^o(u)$, $L<L' \Leftrightarrow \Gamma_u(L)<\Gamma_u(L')$. 
	\end{enumerate}
\end{definition}
Due to the second condition in Definition~\ref{def:bfs}, the bijection $\Gamma_u$ of each node $u\in V$ clearly orders outgoing edges by their label sets. Due to the first condition, the bijection $f$ orders node $u$ before node $v$ if the parent function of $u$ is smaller than that of $v$. In a BFS traversal from $v_1$, nodes are processed in the order they are first visited. In this context, the parent function identifies the edge used to visit a node for the first time. Together, these facts imply that $f$ assigns integers to nodes in the order they are visited by a BFS traversal from $v_1$, given that the label set ordering is used to break ties among edges. This BFS traversal can be characterized by a BFS subtree whose edges are defined by the parent function.

\begin{example}
	The figure below shows a graph indexing for the graph in Example~\ref{ex:labeled_graph}, with nodes and edges labeled by their assigned integer. This graph indexing is a BFS traversal since nodes are ordered according to their distance from the start node $v_1$, and since the edge integers used to break ties are consistent with the label set ordering shown in Example~\ref{ex:observation_comparison}. The parent function is given by $\Pi_\mathcal{I}(v_2)=(4,1)$, $\Pi_\mathcal{I}(v_3)=(1,3)$, $\Pi_\mathcal{I}(v_4)=(1,1)$, and $\Pi_\mathcal{I}(v_5)=(1,2)$, and the corresponding BFS subtree appears in bold.
	\begin{center}
		\begin{tikzpicture}[shorten >=1pt,node distance=2.5cm,on grid,auto]
		\node[state,initial] (u_1)   {$1: v_1$};
		\node[state] (u_3) [above right =of u_1] {$3: v_5$};
		\node[state] (u_2) [below right =of u_1] {$2: v_4$};
		\node[state] (u_4) [above right =of u_2] {$4: v_3$};	
		\node[state] (u_5) [right =of u_4] {$5: v_2$};
		
		\path[->, line width=2pt] (u_1) edge node {$2:\{a,f\}$} (u_3);
		\path[->, line width=2pt] (u_1) edge node[swap] {$1:\{b,e\}$} (u_2);
		\path[->, line width=2pt] (u_1) edge node {$3:\{a,b\}$} (u_4);
		\path[->] (u_3) edge node {$1:\{b\}$} (u_4);
		\path[->] (u_2) edge node[swap] {$1:\{a\}$} (u_4);
		\path[->] (u_4) edge[bend left] node {$2:\{c\}$} (u_5);
		\path[->, line width=2pt] (u_4) edge[bend right] node[swap] {$1:\{d\}$} (u_5);
		\end{tikzpicture}
	\end{center}
	\label{ex:labeled_graph_with_edge_indices}
\end{example}

The important property that we exploit about BFS traversals is that they are unique, which we prove in Lemma~\ref{lemma:unique}. To prove it, we use the result in Proposition~\ref{prop:unique} when BFS is applied to any kind of directed graph where all nodes are reachable.

\begin{proposition}
	If BFS visits the neighbors of each node in a fixed order, the resulting tree is unique.
	\label{prop:unique}
\end{proposition}
\begin{proof}
	By contradiction. Assume that two different BFS trees, $T$ and $T'$, are produced using the same visitation criteria. Then $T$ contains an edge $(u,v)$ that is not in $T'$, and $T'$ contains an edge $(u',v)$ that is not in $T$. In the case of $T$, this means that $u$ was visited before $u'$. Analogously, $u'$ was visited before $u$ to produce $T'$. Therefore, the visitation criteria is different for each of the BFS trees. This is a contradiction.
\end{proof}

\begin{lemma}\label{lemma:unique}
	Each graph $G\in\mathcal{G}$ has a unique associated BFS traversal $\mathcal{I}(G)$. 
\end{lemma}

\begin{proof}	
	Intuitively, the lemma holds because there is only one way to perform a BFS traversal from $v_1$, given that we use the label set ordering to break ties among edges.
	
	Formally, for each $u\in V$, since outgoing label sets are unique, there is a unique bijection $\Gamma_u$ that satisfies the second condition in Definition~\ref{def:bfs}. If $\Gamma_u$ orders outgoing edges in any other way, there will always be two outgoing edges $(u,v,L)$ and $(u,v',L')$ in $E^o(u)$ such that $L<L'$ and $\Gamma_u(L)>\Gamma_u(L')$, thus violating the condition.
	
	Then, if we fix the correct definition of $\Gamma_u$ for each $u\in V$, there is a unique bijection $f$ that satisfies the first condition in Definition~\ref{def:bfs}. To recover this bijection we can simply perform a BFS traversal from $v_1$, using the bijections $\Gamma_u$, $u\in V$, to break ties among edges. By Proposition~\ref{prop:unique}, this results in a unique BFS tree. If $f$ orders nodes in any other way, there will always be two nodes $u$ and $v$ in $V\setminus\{v_1\}$ such that $\Pi_\mathcal{I}(u)<\Pi_\mathcal{I}(v)$ and $f(u)>f(v)$, thus violating the condition.
\end{proof}

Appendix~\ref{app:sat_encoding_labeled_directed_graph} shows an encoding of our symmetry breaking mechanism in the form of a satisfiability (SAT) formula and formally prove several of its properties.

\subsubsection{Application to Subgoal Automata}
\label{sec:symmetry_breaking_application}
In this section we devise a symmetry breaking mechanism based on the graph indexing from the previous subsection. The key idea is to encode rules that force ILASP to generate a graph indexing which is also a BFS traversal. Since this graph indexing is unique due to Lemma~\ref{lemma:unique}, ILASP can only represent each graph in one way, precluding multiple symmetric variations. In fact, we could have used any unique graph indexing for this purpose.

We first show that a subgoal automaton $\mathcal{A}=\langle U, \mathcal{O}, \delta_\varphi, u_0, u_A, u_R \rangle$ is a special case of a label directed graph $G=\langle V,E\rangle$ in the class $\mathcal{G}$. The set of automaton states $U$ corresponds to the set of nodes $V$, while the logical transition function $\varphi$ corresponds to the set of edges $E$. Crucially, a subgoal automaton complies with all three assumptions we made about graphs in the class $\mathcal{G}$:
\begin{itemize}
	\item Assumption~\ref{assump:sym_starting_node} holds because $\mathcal{A}$ has an initial state $u_0$.
	\item Assumption~\ref{assump:sym_reachability} is enforced by the set of rules below. The first rule defines $u_0$ to be reachable, while the second rule indicates that a state is reachable if it has an incoming edge from a reachable state. Finally, the third rule enforces all states to be reachable.
	\begin{equation*}
	\begin{Bmatrix*}[l]
	\mathtt{reachable}(u_0).\\
	\mathtt{reachable(Y) \codeif reachable(X), ed(X, Y, \_).}\\
	\mathtt{\codeif not~reachable(X), state(X).}
	\end{Bmatrix*}
	\end{equation*}
	\item Assumption~\ref{assump:sym_unique_label_sets} holds because the automaton is deterministic (see Section~\ref{sec:determinism}). If the automaton is deterministic, the formulas labeling two outgoing edges from a given state to two different states are mutually exclusive and, thus, are different (which fulfills the assumption). Besides, two outgoing edges from a given state to another state cannot be equal because the automaton learner, ILASP, induces a minimal hypothesis (in our case, a minimal set of transition rules). Since a hypothesis with two equally labeled edges is not minimal, the assumption still holds.
\end{itemize}

Even though a subgoal automaton complies with the three assumptions of labeled directed graphs, there are two differences between them:
\begin{enumerate}
	\item The edges of a subgoal automaton are labeled by propositional formulas over a set of observables $\mathcal{O}$, whereas the edges of a labeled directed graph are defined over a set of labels $\mathcal{L}$.
	\item The edge indices in an indexed labeled directed graph are different from those in our representation of a subgoal automaton. In the former, the edge indices from a given node range between 1 to the number of outgoing edges from than node. In contrast, the indices range from 1 to the number of edges between each pair of states. In other words, the edge indices from a given node in the labeled directed graph are unique, whereas in a subgoal automaton they can be repeated.
\end{enumerate} 
To address these differences we just need to set a mapping from the representation used by the subgoal automata to that characterizing a labeled directed graph. The set of symmetry breaking constraints is defined on top of this mapping. Appendix~\ref{app:asp_encodings} shows two possible ways of encoding the symmetry breaking mechanism in ASP using the factual representation introduced at the beginning of this section.

\section{Interleaved Automaton Learning}
\label{sec:interleaved_automata_learning_algorithm}
In this section we describe {\methodname} (\methodnamefullhighlight), a method that combines reinforcement learning and automaton learning. Firstly, we explain two reinforcement learning algorithms that we use to exploit a given subgoal automaton (Section~\ref{sec:interleaved_rl_algorithms}). Secondly, we explain how reinforcement learning and automaton learning are interleaved to learn a minimal automaton (Section~\ref{sec:interleaved_algorithm_detailed}). The automaton learning process includes the constraints introduced in Section~\ref{sec:structural_properties} for enforcing determinism while the symmetry breaking constraints are optional (i.e., we can still learn a usable deterministic subgoal automaton without the latter).

\subsection{Reinforcement Learning Algorithms}
\label{sec:interleaved_rl_algorithms}
In this section we describe two methods to learn a policy for an episodic POMDP $\mathcal{M}^\Sigma=\langle S,S_T, S_G,\Sigma,A,p,r,\gamma,\nu \rangle$ by exploiting the automaton structure given by its subgoals. Each method is characterized by a different way of using \emph{options} \shortcite{SuttonPS99}:
\begin{enumerate}
	\item Learning an option for each outgoing edge in the automaton and a metacontroller to choose between options in each automaton state (Section~\ref{sec:interleaved_hrl}).
	\item Learning an option for each automaton state (Section~\ref{sec:interleaved_qrm}).
\end{enumerate}

In general, a subgoal automaton is used as follows. Firstly, the agent selects an option when it reaches an automaton state. Note that in (1)~there can be multiple options to choose from, whereas in (2) there is a single option. Once an option is chosen, the agent selects actions according to that option's policy until its termination. An option terminates when either (a) the episode ends, or (b) when a formula on an outgoing edge from the current automaton state is satisfied. After the agent experiences a tuple $\langle\bm\sigma_t,a_t,\bm\sigma_{t+1}\rangle$ in automaton state $u$ at step $t$, it transitions to the automaton state $u'=\delta_\varphi(u, L(\sigma^\Sigma_{t+1}))$ given by the transition function of the automaton. Remember that the labeling function $L$ maps a visible state into an observation.

The automaton is used as an external memory since each automaton state indicates which subgoals have been achieved so far. Crucially, the use of the automaton as an external memory causes tasks to have the Markov property given two assumptions made earlier (see Section~\ref{sec:tasks_def}). First, that observations only depend on visible states. Second, that the \emph{combination} of a visible state and a history of observables (i.e., a sequence of achieved subgoals) is sufficient to be Markovian. Because each option is only executed for a specific history of observations (represented by the current automaton state), we can define an option's policy over visible states only.

In the following sections we describe the features of the RL methods in detail including:
\begin{itemize}
	\item how options are modeled,
	\item how policies are learned, and
	\item which optimality guarantees they have.
\end{itemize}

\subsubsection{Learning an Option for each Outgoing Edge and a Metacontroller for each Automaton State (HRL)}
\label{sec:interleaved_hrl}
The edges of a subgoal automaton are labeled by propositional formulas over a set of observables $\mathcal{O}$. The idea is that each of these formulas represents a subgoal of the task encoded by the automaton. Therefore, an intuitive approach to exploit the automaton structure consists in learning an option that aims to reach a state whose observation satisfies a given formula, and a metacontroller that learns which option to take at each automaton state. Since decisions are taken at two different hierarchical levels, we will refer to this approach as HRL (Hierarchical Reinforcement Learning).

\paragraph{Option Modeling.} Given a subgoal automaton $\mathcal{A}=\langle U, \mathcal{O}, \delta_\varphi, u_0, u_A, u_R\rangle$, the set of options in a \emph{non-terminal} automaton state $u\in U$ (that is, a state with outgoing transitions to other states) is
\begin{align*}
\Omega_u=\left\lbrace \omega_{u,\phi} \mid \phi \in \varphi(u,u'), u\neq u' \right\rbrace,
\end{align*}
where $\omega_{u,\phi}$ is the option that attempts to satisfy the formula $\phi$ in state $u\in U$. Note that $\phi$ is a disjunct of the DNF formula $\varphi(u,u')$.\footnote{Remember that the logical transition function $\varphi: U \times U \to \text{DNF}_\mathcal{O}$ maps an automaton state pair into a DNF formula over $\mathcal{O}$.} Each option at the automaton state $u$ is a tuple $\omega_{u,\phi}=\langle I_u, \pi_\phi, \beta_u \rangle$ where: 
\begin{itemize}
	\item The initiation set is simply the set of visible states (formally, $I_u=\Sigma$) since when the automaton state $u$ is reached, it does not matter which is the current visible state $\sigma^\Sigma\in\Sigma$: an option in that automaton state will be started.
	\item The policy $\pi_\phi:\Sigma \to A$ maps a visible state to a primitive action with the goal of satisfying the propositional formula $\phi$.
	\item The termination condition $\beta_u:\Sigma \times \lbrace\bot,\top\rbrace \times \lbrace\bot,\top\rbrace \to [0,1]$ indicates that the option terminates if any formula on an outgoing edge from $u$ holds (i.e., the formula does not necessarily have to be $\phi$) or a terminal state is reached: 
	\begin{align*}
		\beta_u(\bm\sigma)=\begin{cases}
			1 & \text{if}~\exists u' \in U \text{ such that } L(\sigma^\Sigma) \models \varphi(u,u') \textnormal{ or } \sigma^T = \top \\
			0 & \text{otherwise}
		\end{cases}.
	\end{align*}
\end{itemize}
Note that the policy $\pi_\phi$ can be shared by different options in the automaton. This opportunity arises when two different automaton states have outgoing edges labeled by the same formula $\phi$. Indeed, we store a dictionary that maps a formula $\phi$ into its Q-function and use it for the different options that depend on $\phi$.

In the case of \emph{terminal} states (e.g., $u_A$ and $u_R$), the set of available options $\Omega_u$ at automaton state $u\in U$ is formed by one-step options, each corresponding to a primitive action $a\in A$. Their initiation sets are the same as for non-terminal automaton states.

\paragraph{Policy Learning.} In this approach, decisions are taken at two levels and, thus, two types of policies are learned: (1) policies over options (i.e., a metacontroller) and (2) option policies.

A policy over options $\Pi_u:\Sigma \to \Omega_u$ in state $u\in U$ maps a visible state into an option available at $u$. These policies are learned using SMDP Q-learning \shortcite{BradtkeD94} with $\epsilon$-greedy exploration. Given an experience tuple $\langle\bm\sigma_t,\omega_t,\bm\sigma_{t+k}\rangle$, where $\omega_t \in \Omega_{u}$ is an option taken in automaton state $u$ at step $t$, the update rule is the following:
\begin{align*}
Q_u(\sigma^\Sigma_t,\omega_t) = Q_u(\sigma^\Sigma_t,\omega_t) + \alpha \left(r+  \gamma^k \max_{\omega'\in \Omega_{u'}}Q_{u'}({\sigma^\Sigma_{t+k}},\omega') - Q_u(\sigma^\Sigma_t,\omega_t)\right),
\end{align*}
where $k$ is the number of steps between $\bm\sigma_t$ and $\bm\sigma_{t+k}$, $r$ is the cumulative discounted reward over this time, and $u'$ is the automaton state when the option ends. Note that the discounted term depends on  $u'$ (i.e., the policy in $u$ depends on that in $u'$) and becomes 0 when $\sigma^T_{t+k}$ is true (i.e., the latent state is terminal) since there is no applicable action thereafter.

An option policy $\pi_\phi:\Sigma \to A$ aiming to satisfy a formula $\phi$ is not learned using the rewards from the POMDP. Instead, we use a pseudoreward function $r_\phi:\Sigma \times \lbrace \bot,\top \rbrace \times \lbrace\bot,\top \rbrace \to \mathbb{R}$, which is defined as follows:
\begin{align*}
r_\phi(\bm\sigma)= \begin{cases} 
r_{success} & \text{if}~L(\sigma^\Sigma) \models \phi \\
r_{deadend} & \text{if}~\sigma^T=\top \land \sigma^G=\bot \\
r_{step} & \text{otherwise}
\end{cases},
\end{align*}
where $r_{success}>0$ is given when the next state satisfies $\phi$, $r_{deadend} \leq 0$ is given if the next state is a dead-end state, and $r_{step} \leq 0$ is given after every step otherwise.\footnote{In Section~\ref{sec:experiments} we instantiate $r_\phi$ in two different ways and show the impact they have on learning.} Note that the last case includes the scenario in which a formula different from $\phi$ labeling an edge from the current automaton state is satisfied. These policies are learned using Q-learning \shortcite{Watkins89} with $\epsilon$-greedy exploration. The update rule for a given formula $\phi$ and an experience tuple $\langle\bm\sigma_t,a_t,\bm\sigma_{t+1} \rangle$ is:
\begin{equation}
Q_\phi(\sigma^\Sigma_t,a_t)=Q_\phi(\sigma^\Sigma_t,a_t)+\alpha\left(r_\phi(\bm\sigma_{t+1}) + \gamma\max_{a'}Q_\phi(\sigma^\Sigma_{t+1},a')-Q_\phi(\sigma^\Sigma_t,a_t) \right),
\label{eq:formula_qvalue_update}
\end{equation}
where the second term of the target becomes 0 when either a terminal state is reached (i.e., $\sigma^T_{t+1}$ is true) or the next observation satisfies $\phi$ (i.e., $L(\sigma^\Sigma_{t+1}) \models \phi$).

Intra-option learning is easily applicable to update an option policy $\pi_{\phi'}$ while another policy $\pi_\phi$ is being followed. That is, given an option policy $\pi_\phi$, an experience tuple $\langle\bm\sigma_t,a_t,\bm\sigma_{t+1}\rangle$ generated by this policy is used to update the Q-value of $(\sigma^\Sigma_t,a_t)$ of another formula $\phi'$ through Equation~\ref{eq:formula_qvalue_update}.

\paragraph{Optimality.} Since the sets of available options $\Omega_u$ in non-terminal states $u\in U$ do not include any primitive actions as one-step options, then optimal policies over the set of available options are in general suboptimal policies of the core MDP~\shortcite{Dietterich00,BartoM03a}.

\subsubsection{Learning an Option for each Automaton State (QRM)}
\label{sec:interleaved_qrm}
Instead of learning one option for each outgoing edge and a metacontroller for each automaton state, we can learn a single policy over the space $\Sigma \times U$. This policy is distributed among automaton states; specifically, we learn a single option for each automaton state. However, despite of this distribution, the global policy is still coupled everywhere since we bootstrap the action-value from one automaton state to the next (details below). Therefore, the policy of each single option in the automaton chooses the action that appears globally best. In contrast, the previous method (HRL) decouples the option policies by making them independent of one another (i.e., each attempts to satisfy a specific formula).

The method we describe here is better known as Q-learning for Reward Machines \shortcite<QRM,>{IcarteKVM18}. QRM was created to exploit the structure of Reward Machines (RMs), a family of automata similar to our subgoal automata. The main difference is that each transition in a RM is not only labeled by a propositional formula over a set of observables, but also by a reward function. An in-depth comparison is made in Section~\ref{sec:reward_machines_rw}. We explain QRM in terms of options for a better comparison with the method we presented in the previous section.

\paragraph{Option Modeling.} Given a subgoal automaton $\mathcal{A}=\langle U, \mathcal{O}, \delta_\varphi, u_0, u_A, u_R\rangle$, each state $u\in U$ encapsulates an option $\omega_u=\langle I_{u}, \pi_{u}, \beta_{u} \rangle$ where:
\begin{itemize}
	\item The initiation set $I_u$ and the termination $\beta_u$ are defined as in Section~\ref{sec:interleaved_hrl} for non-terminal automaton states.
	\item The policy $\pi_{u}:\Sigma \to A$ selects the action that appears globally best at a given state (i.e., the action that leads to the fastest achievement of the task's goal). In other words, the policy does not attempt to satisfy a particular formula: it will eventually satisfy the formula that appears to be the best to reach the task's goal. Therefore, the agent may act towards reaching different formulas for different regions of the state space.
\end{itemize}

\paragraph{Policy Learning.} Given an experience tuple $\langle\bm\sigma_t, a_t, \bm\sigma_{t+1}\rangle$, the policy of an option $\omega_u$ is learned through Q-learning updates of the form:
\begin{equation}
Q_u(\sigma^\Sigma_t,a_t) = Q_u(\sigma^\Sigma_t,a_t) + \alpha \left( r(u,u')+\gamma \max_{a'}Q_{u'}(\sigma^\Sigma_{t+1},a') - Q_u(\sigma^\Sigma_t,a_t) \right),
\label{eq:qrm_update}
\end{equation}
where the discounted term depends on the next automaton state $u'$ (i.e., the policy in the automaton state $u$ is coupled with that of $u'$) and becomes 0 when the next state is terminal (i.e., $\sigma^T_{t+1}$ is true) since there is no applicable action thereafter. Crucially, the reward $r$ used in the update does not come from the POMDP. As said before, QRM is originally applied on automata whose transitions are also labeled by reward functions. The reward $r$ is obtained by evaluating those functions. Since subgoal automata are not labeled with reward functions, we assume that the reward function $r:S\times A \times S \to \mathbb{R}$ of the underlying POMDP is as follows:\footnote{This reward function is chosen because all evaluation tasks in Section~\ref{sec:experiments} are characterized by it.}
\begin{equation*}
r(s,a,s')= \begin{cases}
1 & \text{if}~ s' \in S_G\\
0 & \text{otherwise}
\end{cases}.
\end{equation*}
Then, given the current automaton state $u \in U$ and the next automaton state $u'\in U$ the reward $r$ in Equation~\ref{eq:qrm_update} is always 0 except when we transition to the accepting state:
\begin{equation*}
r(u,u') = \begin{cases}
1 & \text{if}~ u \neq u_A, u' = u_A\\
0 & \text{otherwise}
\end{cases}.
\end{equation*}

QRM performs the update in Equation~\ref{eq:qrm_update} for all the options given a single $\langle\bm\sigma_t,a_t,\bm\sigma_{t+1}\rangle$ experience. That is, given the option $\omega_u$ of automaton state $u\in U$, the next automaton state $u'\in U$ used for bootstrapping is determined by evaluating the observation $L(\sigma^\Sigma_{t+1})$ of the next visible state in $u$. Note this is a form of intra-option learning: we update the policies of all options from the experience generated by a single option's policy.

\paragraph{Optimality.} \shortciteA{IcarteKVM18} proved that in the tabular case QRM is guaranteed to converge to an optimal policy in the limit. Essentially, optimality is possible since action-values are bootstrapped from one automaton state to the next.

\paragraph{Reward Shaping.} A subgoal automaton does not only provide the subgoals of a given task, but also gives an intuition of how far the agent is from achieving the task goal. Intuitively, the closer the agent is to the accepting state, the closer it is to the task goal. Therefore, we can provide the agent with an extra positive reward signal  when it gets closer to the accepting state. The idea of giving additional rewards to the agent to guide its behavior is known as reward shaping.

\shortciteA{NgHR99} proposed a function that provides the agent with additional reward while guaranteeing that optimal policies remain unchanged:
\begin{equation*}
F(s_t,a_t,s_{t+1})=\gamma\Phi(s_{t+1})-\Phi(s_t),
\end{equation*}
where $\gamma$ is the MDP's discount factor and $\Phi : S \to \mathbb{R}$ is a real-valued function. The automaton structure can be exploited by defining $F:(U \setminus \{u_A,u_R\})\times U \to \mathbb{R}$ in terms of the automaton states instead \shortcite{CamachoIKVM19,furelosblanco2020aaai}:
\begin{equation*}
F(u,u')= \gamma\Phi(u')-\Phi(u),
\end{equation*}
where $\Phi:U\to\mathbb{R}$. Consequently, Equation~\ref{eq:qrm_update} is rewritten as:
\begin{equation*}
Q_u(\sigma^\Sigma_t,a_t) = Q_u(\sigma^\Sigma_t,a_t) + \alpha \left( r(u,u')+F(u,u')+\gamma \max_{a'}Q_{u'}(\sigma^\Sigma_{t+1},a') - Q_u(\sigma^\Sigma_t,a_t) \right).
\end{equation*}
Since we want the value of $F(u,u')$ to be positive when the agent gets closer to the accepting state $u_A$, we define $\Phi$ as
\begin{equation*}
\Phi(u)=|U| - d(u, u_A),
\end{equation*}
where $|U|$ is the number of automaton states, and $d(u,u_A)$ is a measure of distance between $u$ and $u_A$. If $u_A$ is unreachable from $u$, then $d(u,u_A)=\infty$.\footnote{In practice, we use a sufficiently big number to represent $\infty$ (e.g., $10^6$).} Note that $|U|$ acts as an upper bound of the maximum length (i.e., number of directed edges) of an acyclic path between $u$ and $u_A$. The distance between $u$ and $u_A$ can be given either by the length of the shortest path between them ($d_{\min}$) or by the length of the longest acyclic path between them ($d_{\max}$).

\begin{example}
	The following figures show the additional rewards generated by the reward shaping function using $d_{\min}$ (left) and $d_{\max}$ (right) with $\gamma=0.99$ in the \textsc{Coffee} task's automaton (see Figure~\ref{fig:officeworld_coffee_rm}, p.~\pageref{fig:officeworld_coffee_rm}). The numbers inside the states correspond to the values returned by $\Phi$, whereas the numbers on the edges are the values returned by $F$. Note that $|U|=4$ and that the only difference in the $\Phi$ values occurs in the initial state.
	\begin{center}
		\begin{tikzpicture}[shorten >=1pt,node distance=2.06cm,on grid,auto]
		\node[state,initial] (u_0)   {$3.0$};
		\node[state,accepting] (u_acc) [below =3cm of u_0]  {$4.0$};
		\node[state] (u_1) [left =3cm of u_acc] {$3.0$};
		\node[state] (u_rej) [right =3cm of u_acc]  {$-\infty$};
		\path[->] (u_0) edge [loop above] node {$-0.03$} ();
		\path[->] (u_1) edge [loop below] node {$-0.03$} ();
		\path[->] (u_0) edge [bend right] node[in place] {$-0.03$} (u_1);
		\path[->] (u_0) edge [bend left] node[in place] {$-\infty$} (u_rej);
		\path[->] (u_1) edge[bend right] node[in place] {$-\infty$} (u_rej);
		\path[->] (u_1) edge[bend left] node[in place] {$+0.96$} (u_acc);
		\path[->] (u_0) edge node[in place] {$+0.96$} (u_acc);
		\end{tikzpicture}
		\hfill
		\begin{tikzpicture}[shorten >=1pt,node distance=2.06cm,on grid,auto]
		\node[state,initial] (u_0)   {$2.0$};
		\node[state,accepting] (u_acc) [below =3cm of u_0]  {$4.0$};
		\node[state] (u_1) [left =3cm of u_acc]   {$3.0$};
		\node[state] (u_rej) [right =3cm of u_acc]  {$-\infty$};
		\path[->] (u_0) edge [loop above] node {$-0.02$} ();
		\path[->] (u_1) edge [loop below] node {$-0.03$} ();
		\path[->] (u_0) edge [bend right] node[in place] {$+0.97$} (u_1);
		\path[->] (u_0) edge [bend left] node[in place] {$-\infty$} (u_rej);
		\path[->] (u_1) edge[bend right] node[in place] {$-\infty$} (u_rej);
		\path[->] (u_1) edge[bend left] node[in place] {$+0.96$} (u_acc);
		\path[->] (u_0) edge node[in place] {$+1.96$} (u_acc);
		\end{tikzpicture}
	\end{center}
	\label{ex:reward_shaping_qrm}
\end{example}

\subsection{Interleaved Automaton Learning Algorithm}
\label{sec:interleaved_algorithm_detailed}
In this section we describe how the {\methodname} algorithm interleaves reinforcement learning and automaton learning. Given an episodic POMDP $\mathcal{M}^\Sigma=\langle S,S_T, S_G,\Sigma,A,p,r,\gamma,\nu\rangle$, a set of observables $\mathcal{O}$, a labeling function $L:\Sigma \to 2^\mathcal{O}$, and a maximum number of edges $\kappa$ between two states, {\methodname} aims to learn and iteratively refine a subgoal automaton $\mathcal{A}=\langle U, \mathcal{O}, \delta_\varphi, u_0, u_A, u_R \rangle$ from the experience of a reinforcement learning agent. The subgoal automaton $\mathcal{A}$ is the automaton with the smallest number of states that has at most $\kappa$ edges from one state to another and is valid with respect to all the traces observed by the agent.

\begin{algorithm}[ht]
	\caption{{\methodname} Algorithm}
	\label{alg:interleaving_pseudocode}
	\begin{algorithmic}[1]
		\Require An initial state ($u_0$), an accepting state ($u_A$), a rejecting state ($u_R$), a set of observables $\mathcal{O}$, a labeling function $L$, and max. number of edges between two states ($\kappa$).
		\State $U \leftarrow \lbrace u_0,u_A,u_R \rbrace$
		\State $\mathcal{A} \leftarrow \langle U, \mathcal{O}, \delta_\varphi, u_0,u_A,u_R \rangle$
		\State $\Lambda_{L,\mathcal{O}} \leftarrow \lbrace\rbrace$ \Comment Set of counterexamples
		\State \textsc{InitQFunctions}($\mathcal{A}$)
		\For{$l=0$ \textbf{to} num\_episodes}
		\State $\sigma^\Sigma,\sigma^T,\sigma^G\leftarrow$ \textsc{Env.InitialState}()
		\State $u \leftarrow \delta_\varphi(u_0, L(\sigma^\Sigma))$
		\State $\lambda_{L,\mathcal{O}} \leftarrow \langle L(\sigma^\Sigma)\rangle$ \Comment Initialize trace
		\If {\textsc{IsCounterexample}($\sigma^T,\sigma^G, u$)}
		\State \textsc{OnCounterexampleFound}($\lambda_{L,\mathcal{O}}$)
		\State $u \leftarrow \delta_\varphi(u_0, L(\sigma^\Sigma))$
		\EndIf
		\State $t \leftarrow 0$
		\While{$t < \text{max\_episode\_length} \land \sigma^T = \bot$} \Comment Run episode
		\State $a \leftarrow$ \textsc{SelectAction}($\sigma^\Sigma,u$)
		\State $\sigma'^\Sigma,\sigma'^T,\sigma'^G,r \leftarrow$ \textsc{Env.Step}($a$)
		\State $u' \leftarrow \delta_\varphi(u,L(\sigma'^\Sigma))$
		\State \textsc{UpdateTrace}($\lambda_{L,\mathcal{O}},L(\sigma'^\Sigma)$)
		\If {\textsc{IsCounterexample}($\sigma'^T, \sigma'^G, u'$)}
		\State \textsc{OnCounterexampleFound}($\lambda_{L,\mathcal{O}}$)
		\State \algorithmicbreak
		\Else
		\State \textsc{UpdateQFunctions}($\sigma^\Sigma$, $a$, $\sigma'^\Sigma$, $\sigma'^T$, $\sigma'^G$, $L(\sigma'^\Sigma)$, $r$)
		\EndIf
		\State $\sigma^\Sigma \leftarrow \sigma'^\Sigma; \sigma^T \leftarrow \sigma'^T; u \leftarrow u'$
		\State $t \leftarrow t+1$
		\EndWhile
		\EndFor
		\Function{IsCounterexample}{$\sigma^T,\sigma^G,u$}
		\State \Return $(\sigma^G=\top \land u \neq u_A) \vee  (\sigma^T=\top \land \sigma^G=\bot \land u \neq u_R) \vee (\sigma^T=\bot \land u \in \{u_A, u_R\})$
		\EndFunction
		\Function{OnCounterexampleFound}{$\lambda_{L,\mathcal{O}}$}
		\State $\Lambda_{L,\mathcal{O}} \leftarrow \Lambda_{L,\mathcal{O}} \cup \lbrace \lambda_{L,\mathcal{O}} \rbrace$
		\State is\_unsat $\leftarrow$ \textbf{true}
		\While {is\_unsat}
		\State $\mathcal{A}, \text{is\_unsat} \leftarrow$ \textsc{LearnAutomaton}($U,u_0,u_A,u_R,\Lambda_{L,\mathcal{O}}, \kappa$)
		\If {is\_unsat}
		\State $U \leftarrow U \cup \{u_{|U|-2}\}$
		\EndIf
		\EndWhile
		\State \textsc{ResetQFunctions}($\mathcal{A}$)
		\EndFunction
	\end{algorithmic}
\end{algorithm}

Algorithm~\ref{alg:interleaving_pseudocode} contains the pseudocode describing how RL and automaton learning are interleaved. The pseudocode applies to both of the methods described in Section~\ref{sec:interleaved_rl_algorithms} and consists of two functions related to automaton learning:
\begin{itemize}
	\item The \textsc{IsCounterexample} function (lines~25-26) checks whether the current automaton state $u$ correctly recognizes the current tuple $\bm\sigma=\langle\sigma^\Sigma,\sigma^T,\sigma^G \rangle$. It returns true in the following cases:
	\begin{itemize}
		\item a goal state is reached and $u$ is not the accepting state ($\sigma^G=\top \land u \neq u_A$), or
		\item a dead-end state is reached and $u$ is not the rejecting state ($\sigma^T=\top \land \sigma^G = \bot \land u \neq u_R$), or
		\item the state is not a terminal state and $u$ is either the accepting or rejecting state ($\sigma^T=\bot \land u \in \{u_A, u_R\}$).
	\end{itemize}
	\item The \textsc{OnCounterexampleFound} function (lines~27-34) determines what to do when a trace $\lambda_{L,\mathcal{O}}$ is not correctly recognized by the current automaton:
	\begin{enumerate}[(a)]
		\item Add $\lambda_{L,\mathcal{O}}$ to the corresponding set of traces (line~28):
		\begin{itemize}
			\item to the set of goal traces $\Lambda_{L,\mathcal{O}}^G$ if a goal state is reached ($\sigma^G=\top$), or
			\item to the set of dead-end traces $\Lambda_{L,\mathcal{O}}^D$ if a dead-end state is reached ($\sigma^T=\top \land \sigma^G=\bot$), or
			\item to the set of incomplete traces $\Lambda_{L,\mathcal{O}}^I$ if a terminal state is not reached ($\sigma^T=\bot$).
		\end{itemize}
		\item Run the automaton learner (lines~29-33). If the automaton learning task is unsatisfiable, it means that the hypothesis space does not include the automaton we are looking for. Therefore, we add a new state to $U$.\footnote{Non-special states (i.e., not $u_0,u_A$ or $u_R$) are labeled from 1 upwards ($u_1, u_2,\ldots$).} We adopt this iterative deepening strategy to find the subgoal automaton with the lowest number of states with at most $\kappa$ edges from one state to another.
		\item When a new automaton is learned, the Q-functions are reset (line~34). We later explain what resetting the Q-functions means for the two RL algorithms we use.
	\end{enumerate}
\end{itemize}
We now describe the main function of the algorithm:
\begin{enumerate}
	\item Initially, the set of states is formed by the initial state $u_0$, the accepting state $u_A$ and the rejecting state $u_R$ (line~1). The automaton is initialized such that it does not accept nor reject anything; that is, there are no edges between the states in $U$ (line~2). The set of counterexample traces and the Q-functions are also initialized (lines~3-4).
	\item When an episode starts, the current automaton state $u$ is $u_0$. One transition is then applied depending on the agent's initial observation $L(\sigma^\Sigma)$ (lines~6-7). For instance, if the agent initially observes $\lbrace\text{\Coffeecup}\rbrace$ in $\textsc{OfficeWorld}$'s $\textsc{Coffee}$ task, then $u$ must be the state where the agent has already observed {\Coffeecup} (see Figure~\ref{fig:officeworld_coffee_rm}, p.~\pageref{fig:officeworld_coffee_rm}). The episode trace $\lambda_{L,\mathcal{O}}$ is initialized with the initial observation $L(\sigma^\Sigma)$ (line~8). If a counterexample is detected at the beginning of the episode (line~9), a new automaton is learned (line~10), the automaton state is reset (line~11) and the episode continues.
	\item At each episode's step, we select an action $a$ in state $\sigma^\Sigma$ (line~14) and apply it in the environment (line~15). Based on the new observation $L(\sigma'^\Sigma)$, we get the next state $u'$ (line~16) and update the observation trace $\lambda_{L,\mathcal{O}}$ (line~17). If a counterexample trace $\lambda_{L,\mathcal{O}}$ is found (line~18), a new automaton is learned (line~19) and the episode ends (line~20). Else, the Q-functions are updated (line~22) and the episode continues. 
\end{enumerate}

Theorem~\ref{theorem:finite_steps} shows that if the target automaton is in the hypothesis space, there will only be a finite number of learning steps in the algorithm before it converges to such automaton (or an equivalent one).
\begin{theorem}
	Given a target finite automaton $\mathcal{A_*}$, there is no infinite sequence $\rho$ of automaton-counterexample pairs $\langle\mathcal{A}_i, e_i\rangle$ such that $\forall i$: (1) $\mathcal{A}_i$ covers all examples $e_1, \ldots, e_{i-1}$, (2) $\mathcal{A}_i$ does not cover $e_i$, and (3) $\mathcal{A}_i$ is in the finite hypothesis space $S_M$.
	\label{theorem:finite_steps}
\end{theorem}
\begin{proof}
	By contradiction. Assume that $\rho$ is infinite. Given that $S_M$ is finite, the number of possible automata is finite.  Hence, some automaton $\mathcal{A}$ must appear in $\rho$ at least twice, say as $\mathcal{A}_i=\mathcal{A}_j, i<j$. By definition, $\mathcal{A}_i$ does not cover $e_i$ and $\mathcal{A}_j$ covers $e_i$. This is a contradiction.
\end{proof}

In the following paragraphs we describe important aspects regarding the implementation of the algorithm. Firstly, we describe how the Q-functions are managed for the two different algorithms we consider (HRL and QRM). Secondly, we introduce two optimizations to make automaton learning more efficient.

\paragraph{Management of Q-functions.} A critical aspect of the algorithm is how Q-functions are initialized, updated and reset when a new automaton is learned. In the tabular case, all the Q-values for the different state-action and/or state-option pairs are initialized to 0 for both HRL and QRM. The Q-functions are updated using the rules described in Sections~\ref{sec:interleaved_hrl} and \ref{sec:interleaved_qrm} respectively. Note that in the case of HRL, Algorithm~\ref{alg:interleaving_pseudocode} omits the call to the function responsible for updating the metacontroller Q-functions, which occurs when the selected option terminates.\footnote{We have omitted it for the generality of the pseudocode. Note that according to the termination condition defined in Section~\ref{sec:interleaved_rl_algorithms}, an option ends when the episode terminates or the next automaton state is different. Therefore, such a check can be done at the end of each step (between lines~22 and 23).}

Ideally, when a new automaton is learned, we would like to reuse the knowledge from the previous automata into the new one:
\begin{itemize}
	\item In the case of HRL, as described in Section~\ref{sec:interleaved_hrl}, the policies of all options that have been used throughout learning are stored in a dictionary. Therefore, these policies can be reused whenever their corresponding formulas appear in the automaton. There are two choices to make in this case:
	\begin{enumerate}
		\item Which Q-functions to update. There are two alternatives: (i) update all the stored Q-functions or (ii) update only the Q-functions of the formulas appearing in the current automaton. While (i) is more costly, the fact that all Q-functions are updated makes them more reliable if they are to be used in the future. Experimentally, we use (i) for the tabular case and (ii) in the function approximation case because of the running time.
		\item Reuse Q-functions by copying those defined for similar formulas. Given that options correspond to propositional formulas, we can compare how similar two formulas are and, thus, initialize a new Q-function from an existing one. We use the number of matching positive literals between two formulas; in case of a draw, the formula whose Q-function has been updated the most is chosen. Experimentally, the number of matching positive literals works better than the number of all matching literals (positive and negative). Taking into account negative literals can lead to negative transfer because they usually emerge from the determinism constraints to make two formulas mutually exclusive; therefore, they do not usually represent the ``important'' part of the subgoal formula in the tasks we use later for evaluation. For this reason, we use the number of matching positive literals in the experiments.
	\end{enumerate}
	Unlike option policies, it is more difficult to determine when to transfer the metacontroller (i.e., a policy over options) from one automaton to another. In this case, we create a new Q-function and do not reuse any previous knowledge.
	\item In the case of QRM, the policy at each automaton state selects the action that appears best towards satisfying the task's final goal. Therefore, it can attempt to satisfy different formulas for different regions of the state space. This makes the transfer of policies between automata non-trivial. 
	In this case we just reinitialize all the Q-functions, which causes the agent to forget everything it learned. Nevertheless, the reward shaping mechanism introduced before can be helpful to alleviate this problem.
\end{itemize}

\paragraph{Optimizations.} In practice we use two additional optimizations to make the automaton learning phase more efficient:
\begin{enumerate}
	\item The agent does not learn an automaton for the first time until a goal trace is found (i.e., the goal is achieved). Experimentally, we have observed that in tasks where dead-ends are frequent, the learner constantly finds counterexamples with the aim to refine the paths to the rejecting state. Starting to learn automata only when there is a goal trace in the counterexample set has proved to be a better strategy.
	\item As explained in Section~\ref{sec:learn_subgoal_automata_from_traces}, the rejecting state $u_R$ is not included in the set of states $U$ if the set of dead-end traces is empty. This avoids an unnecessary increase in the number of rules in the hypothesis space, specially in tasks without dead-end states.
\end{enumerate}

\section{Experiments}
\label{sec:experiments}
In this section, we evaluate the effectiveness of {\methodname} in different domains. Our analysis focuses on evaluating how the behavior of the RL agent and the task being learned affect automaton learning and vice versa. Firstly, we describe the main characteristics of our evaluation methodology. Secondly, we make a thorough analysis of the performance of our approach using the \textsc{OfficeWorld}~\shortcite{IcarteKVM18}, \textsc{CraftWorld}~\shortcite{AndreasKL17} and \textsc{WaterWorld}~\shortcite{IcarteKVM18} domains.

We use ILASP2 to learn the automata with a 2 hour timeout for each automaton learning task. All experiments ran on 3.40GHz Intel\textsuperscript{\textregistered} Core\texttrademark~i7-6700 processors. The code is available at \url{https://github.com/ertsiger/induction-subgoal-automata-rl}.

\subsection{Experimental Setting}
\label{sec:experimental_setting}
In this section, we describe how we evaluate our approach and introduce some restrictions that we can apply on the automaton learning task. Then, we explain the nomenclature we use for the different RL algorithms we evaluate, and how we report the results.

\paragraph{Sets of POMDPs.} We consider the problem of learning an automaton given a set of POMDPs $\mathcal{D}= \lbrace \mathcal{M}^\Sigma_1, \ldots,\mathcal{M}^\Sigma_{|\mathcal{D}|} \rbrace$. All POMDPs in $\mathcal{D}$ correspond to the same task (e.g., \textsc{OfficeWorld}'s \textsc{Coffee} task) and are enhanced by the same set of observables $\mathcal{O}$. However, they do not need to share the same state and action spaces. Algorithm~\ref{alg:interleaving_pseudocode} is applied on a set of POMDPs $\mathcal{D}$ by iteratively running one full episode for each POMDP, and repeating until a maximum total number of episodes is reached. We impose a maximum episode length $N$ to guarantee all episodes terminate in a reasonable amount of time, especially when the terminal states are hard to reach. There are two reasons for which these sets are useful:
\begin{enumerate}
	\item The automaton will generalize to several POMDPs. For example, in \textsc{OfficeWorld} the coffee ({\Coffeecup}) and the office ($o$) can sometimes be in the same location. Therefore, the learned automaton should reflect these two situations: when {\Coffeecup} and $o$ are together and when they are not (see Figure~\ref{fig:officeworld_coffee_rm}, p.~\pageref{fig:officeworld_coffee_rm}). Furthermore, using different POMDPs can help to avoid overgeneralization, which is related to the fact that a minimal automaton is learned from positive examples only (see discussion in Section~\ref{sec:conclusions}). Figure~\ref{fig:overgeneralization_single_mdp} shows an \textsc{OfficeWorld} grid that if used alone to learn an automaton for \textsc{Coffee} would produce the automaton on the right. Since the agent (\Strichmaxerl[1.25]) cannot see any trace that reaches $o$ without having seen {\Coffeecup}, it does not learn that observing {\Coffeecup} is important to reach the goal.
	
	\item The set of POMDPs may contain some POMDPs in which it is easier to reach the goal than in others. Therefore, automaton learning can be initialized earlier. The RL agent can then immediately exploit the automaton in the harder POMDPs to effectively reduce the amount of exploration needed to reach the goal in them.
\end{enumerate}

\begin{figure}
	\centering
	\subfloat{
		\tikzset{digit/.style = { minimum height = 5mm, minimum width=5mm, anchor=center }}
		\newcommand{\setcell}[3]{\edef\x{#2 - 0.5}\edef\y{9.5 - #1}\node[digit,name={#1-#2}] at (\x, \y) {#3};}
		\centering
		\begin{tikzpicture}[scale=0.5]
		\draw[gray] (1, 1) grid (13, 10);
		\draw[very thick, scale=3] (1/3, 1/3) rectangle (13/3, 10/3);
		
		\draw[very thick, scale=1] (4, 1) rectangle (4, 2); \draw[very thick, scale=1] (4, 9) rectangle (4, 10);
		\draw[very thick, scale=1] (4, 3) rectangle (4, 8);
		\draw[very thick, scale=1] (7, 1) rectangle (7, 2); \draw[very thick, scale=1] (7, 3) rectangle (7, 8); \draw[very thick, scale=1] (7, 9) rectangle (7, 10);
		\draw[very thick, scale=1] (10, 1) rectangle (10, 2); \draw[very thick, scale=1] (10, 3) rectangle (10, 8); \draw[very thick, scale=1] (10, 9) rectangle (10, 10);
		\draw[very thick, scale=1] (1, 4) rectangle (2, 4); \draw[very thick, scale=1] (3, 4) rectangle (11, 4); \draw[very thick, scale=1] (12, 4) rectangle (13, 4);
		\draw[very thick, scale=1] (1, 7) rectangle (2, 7); \draw[very thick, scale=1] (3, 7) rectangle (5, 7); \draw[very thick, scale=1] (6, 7) rectangle (8, 7); \draw[very thick, scale=1] (9, 7) rectangle (11, 7); \draw[very thick, scale=1] (12, 7) rectangle (13, 7);
		\setcell{1}{3}{$D$} \setcell{1}{6}{$\ast$} \setcell{1}{9}{$\ast$} \setcell{1}{12}{$C$}
		\setcell{3}{6}{\Coffeecup} \setcell{2}{6}{$o$}
		\setcell{4}{3}{$\ast$} \setcell{5}{6}{\Strichmaxerl[1.25]} \setcell{4}{9}{\Letter} \setcell{4}{12}{$\ast$}
		\setcell{6}{10}{\Coffeecup}
		\setcell{7}{3}{$A$} \setcell{7}{6}{$\ast$} \setcell{7}{9}{$\ast$} \setcell{7}{12}{$B$}
		\end{tikzpicture}}\quad
	\subfloat{
		\scalebox{1.2}{
			\begin{tikzpicture}[shorten >=1pt,node distance=2.06cm,on grid,auto]
			\node[state,initial] (u_0)   {$u_0$};
			\node[state,accepting] (u_acc) [below =of u_0]  {$u_{A}$};
			\path[->] (u_0) edge node {$o$} (u_acc);
			\end{tikzpicture}
		}
	}
	\caption{Example of an \textsc{OfficeWorld} grid whose traces cause automata overgeneralization for the \textsc{Coffee} task.}
	\label{fig:overgeneralization_single_mdp}
\end{figure}
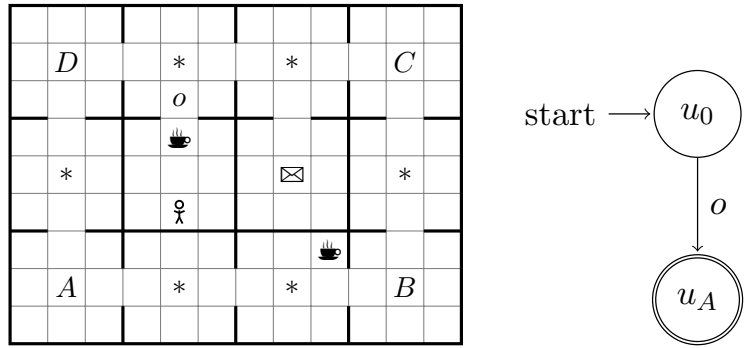

The sets of POMDPs used in these experiments are not handcrafted, but randomly generated (e.g., placing observables randomly in a grid).\footnote{The constraints we impose in the POMDP generation are explained later for each of the used domains.} Therefore, (1) there is no guarantee that certain observations will be seen (e.g., two observables in the same tile of a grid), and (2) the difficulty of the POMDPs is not fully controlled. Given the two previous consequences of using randomly generated sets of POMDPs, we will evaluate how different sets affect automaton learning and RL. Note that it is easy to handcraft tasks such that the target automaton is learned. In the case of a grid-world, like \textsc{OfficeWorld}, we need at most one grid for each path to the accepting state (in the absence of dead-ends). However, we believe that using a set of randomly generated POMDPs reduces our bias on the automaton learning.

\paragraph{Restrictions.} The following are restrictions that can be imposed on the automaton structure, the traces and the observables, and that are later used in the evaluation:\footnote{The details of these restrictions (e.g., implementation and examples) are given in Appendix~\ref{app:automata_learning_restrictions}.}
\begin{itemize}
	\item Avoid learning purely negative formulas. We assume that the tasks' subgoals cannot be characterized \emph{only} by the non-occurrence of certain observables. That is, the formula labeling an edge cannot be formed only by negated observables. The minimal automata for the tasks we consider in this paper comply with this assumption, which helps to slightly simplify the automaton learning phase.
	\item Acyclicity. There are tasks whose corresponding minimal automata do not contain cycles (i.e., a previously visited automaton state cannot be revisited). For instance, the three \textsc{OfficeWorld} tasks we have considered so far belong to this class of automata. Thus, the search space can be made smaller by ruling out solutions containing cycles.
	\item Trace compression. The traces used as counterexamples in {\methodname} depend on the agent behavior. While the agent has not managed to reach the goal, its behavior is random. Consequently, the counterexample that the agent provides to the automaton learner can be long and include many observables that are irrelevant to the task at hand. These two factors, as we will see later, have a negative impact on the time required to learn an automaton.
	\item Restricted observable set. To simplify the traces, we can use only those observables relevant to the task at hand. For example, if the task is \textsc{CoffeeMail}, then the set of observables becomes $\mathcal{O}=\lbrace\texttt{\Coffeecup}, o,\ast \rbrace$ instead of $\mathcal{O}=\lbrace\texttt{\Coffeecup}, \texttt{\Letter}, o, A, B,C,D,\ast\rbrace$. This greatly simplifies the automaton learning tasks since the hypothesis space becomes smaller and ILASP does not have to discern which are the relevant observables.
\end{itemize}

\paragraph{Reinforcement Learning Algorithms.} We use the following nomenclature for the different RL algorithms applied in the experiments:
\begin{itemize}
	\item HRL: HRL where $r_{success}=1.0$, $r_{deadend}=0.0$ and $r_{step}=0.0$.
	\item $\text{HRL}_\text{G}$: HRL where $r_{success}=1.0$, $r_{deadend}=-N$ and $r_{step}=-0.01$.
	\item QRM: QRM without reward shaping.
	\item $\text{QRM}_{\min}$: QRM with reward shaping based on the length of the shortest path to the accepting state ($d_{\min}$).
	\item $\text{QRM}_{\max}$: QRM with reward shaping based on the length of the longest acyclic path to the accepting state ($d_{\max}$).
\end{itemize}
In the case of HRL algorithms, the option policies are updated as follows:
\begin{itemize}
	\item In the tabular case, we update all the Q-functions of the formulas in the dictionary after every step (i.e., all discovered formulas during learning).
	\item In the function approximation case, we update only the Q-functions of the formulas appearing in the current automaton.
\end{itemize}

\paragraph{Reporting Results.}
We report results using tables and figures. In the \emph{tables} we report the following \emph{average} automaton learning statistics across runs where the automaton learner has not timed out and at least one automaton has been learned:
\begin{itemize}
	\item Total time (in seconds) used to run the automaton learner.
	\item Number of examples needed to learn the final automaton.
	\item Length of the examples used to learn the final automaton.
\end{itemize}
For the first two cases, the numbers in brackets correspond to the standard error, while in the last case we use the standard deviation since it is an average of the example lengths across all runs. We mark with an asterisk (*) cases where either no automaton has been learned\footnote{Remember that automata are started to be learned once a goal trace has been observed.} or the automaton learner has timed out between 1 and 10 runs. A dash (-) is used if the number of such cases is higher than 10. 

The \emph{figures} show the average reward across the POMDPs in set $\mathcal{D}$ and the number of runs (20). Each point of the learning curve represents the sum of rewards obtained by the greedy policy at a given episode. By default, the greedy policy is evaluated after every training episode for one episode. The dotted vertical lines correspond to episodes where an automaton was learned. When the automaton learner times out, the reward is set to 0 for the entire interaction.

\subsection{Experiments in {\mdseries\textsc{OfficeWorld}}}
\label{sec:officeworld_experiments}
The \textsc{OfficeWorld} domain \shortcite{IcarteKVM18}, introduced in Section~\ref{sec:problem_formulation}, is characterized by the $9\times12$ grid shown in Figure~\ref{fig:officeworld_grid} (p.~\pageref{fig:officeworld_grid}). The set of observables is $\mathcal{O}=\lbrace\texttt{\Coffeecup}, \texttt{\Letter}, o, A, B,C,D,\ast\rbrace$. The grid contains one observable of each type except for the coffee location {\Coffeecup} (2) and the decoration $\ast$ (6). The agent and the observables are randomly placed in the grid using the following criteria:
\begin{itemize}
	\item The agent cannot be initially placed with decorations $\ast$ or observables $A-D$.
	\item The decorations $\ast$ do not share a location with any other observable.
	\item The decorations $\ast$ and observables $A-D$ cannot be placed next to each other (including diagonals) nor in locations that connect two rooms (e.g., $(1,2)$ and $(1,3)$).
	\item Observables $A-D$ and the office $o$ cannot be in the same location.
\end{itemize}
Note that {\Coffeecup}, {\Letter} and $o$ are allowed to be in the same location, and that {\Coffeecup} and {\Letter} can share a location with any observable $A-D$.

The tasks in this domain are the ones we introduced in Section~\ref{sec:problem_formulation}: \textsc{Coffee}, \textsc{CoffeeMail} and \textsc{VisitABCD}. These tasks constitute a good test-bed since their automata are different and they are incrementally more challenging:
\begin{itemize}
	\item The \textsc{Coffee} task has 2 subgoals and is represented by a 4 state minimal automaton.
	\item The \textsc{CoffeeMail} task has 3 subgoals and is represented by a 6 state minimal automaton.
	\item The \textsc{VisitABCD} task has 4 subgoals and is represented by a 6 state minimal automaton.
\end{itemize}
In the tasks we consider in this paper, the number of subgoals is defined as the number of directed edges of the longest acyclic path from the initial state to the accepting state in the minimal automaton. In general, the longer the sequence of subgoals is, the harder it becomes to achieve the goal.

Table~\ref{tab:officeworld_experiments_params} shows the parameters used throughout these experiments. The parameters $\alpha$, $\epsilon$ and $\gamma$ are the same for both the metacontrollers and the options in the case of HRL. In the following paragraphs we compare the learning curves produced by the two types of RL algorithms we have previously presented (HRL and QRM). We show how learning the automata in an interleaved manner affects RL and how introducing guidance (e.g., reward shaping in QRM) influences when counterexamples are found. In Appendix~\ref{app:additional_experimental_results} we analyze how parameter tuning affects automaton learning and reinforcement learning.

\begin{table}[]
	\centering
	\begin{tabular}{ll}
		\toprule
		Learning rate ($\alpha$)                  & 0.1    \\
		Exploration rate ($\epsilon$)             & 0.1    \\
		Discount factor ($\gamma$)                & 0.99   \\
		Number of episodes                        & 10,000 \\
		Avoid learning purely negative formulas & \cmark \\
		Number of tasks ($|\mathcal{D}|$)         & 50     \\
		Maximum episode length ($N$)              & 250    \\
		Trace compression                         & \cmark    \\
		Enforce acyclicity                        & \cmark    \\
		Number of disjuncts ($\kappa$)         & 1      \\
		Use restricted observable set             & \xmark \\
		\bottomrule
	\end{tabular}
	\caption{Parameters used in the \textsc{OfficeWorld} experiments.}
	\label{tab:officeworld_experiments_params}
\end{table}

Figure~\ref{fig:rl_algorithm_hancrafted} shows how the learning curves with interleaved automaton learning (ISA-HRL, ISA-QRM) compare to those obtained with handcrafted automata (HRL, QRM). We observe the following:
\begin{itemize}
	\item The algorithms using auxiliary guidance ($\text{HRL}_\text{G}, \text{QRM}_{\min}$ and $\text{QRM}_{\max}$) converge faster than their respective basic versions (HRL and QRM). The use of  auxiliary reward signals helps to explore the state space more effectively, which results in observing counterexample traces early. Consequently, automaton learning is less frequent in the last episodes, which is convenient to avoid resetting the Q-functions late and leave enough episodes to converge.
	\item $\text{QRM}_{\max}$ converges faster than $\text{QRM}_{\min}$ except in \textsc{VisitABCD} where they perform exactly the same because there is a single path to the accepting state. As previously shown in Example~\ref{ex:reward_shaping_qrm}, $\text{QRM}_{\max}$ provides a positive reward signal for any path that allows the agent to approach the accepting state. In contrast, $\text{QRM}_{\min}$ only provides a positive signal for the shortest path(s). If the shortest path is not available in a certain grid (e.g., {\Coffeecup} and $o$ occurring together), $\text{QRM}_{\min}$ gives a negative reward for choosing the only available path to the accepting state. Consequently, convergence is not as fast as in $\text{QRM}_{\max}$.
	\item HRL converges faster than QRM across the different tasks. In the absence of reward shaping, QRM needs to satisfy the formula on an edge to the accepting state to start propagating positive reward through the different states. On the other hand, HRL can independently update the Q-functions of each of the formulas in the automaton without having to achieve the task's goal.
	\item The curves for the settings involving automaton learning perform closely to the ones with handcrafted automata. Naturally, sometimes the convergence is slower due to the fact that a proper automaton cannot be exploited from episode 0. This is noticeable in \textsc{VisitABCD}, where a stable automaton requires several relearning steps to be found.
\end{itemize}

\begin{sidewaysfigure}
	\centering
	\subfloat{
		\resizebox{0.32\columnwidth}{!}{
			\includegraphics{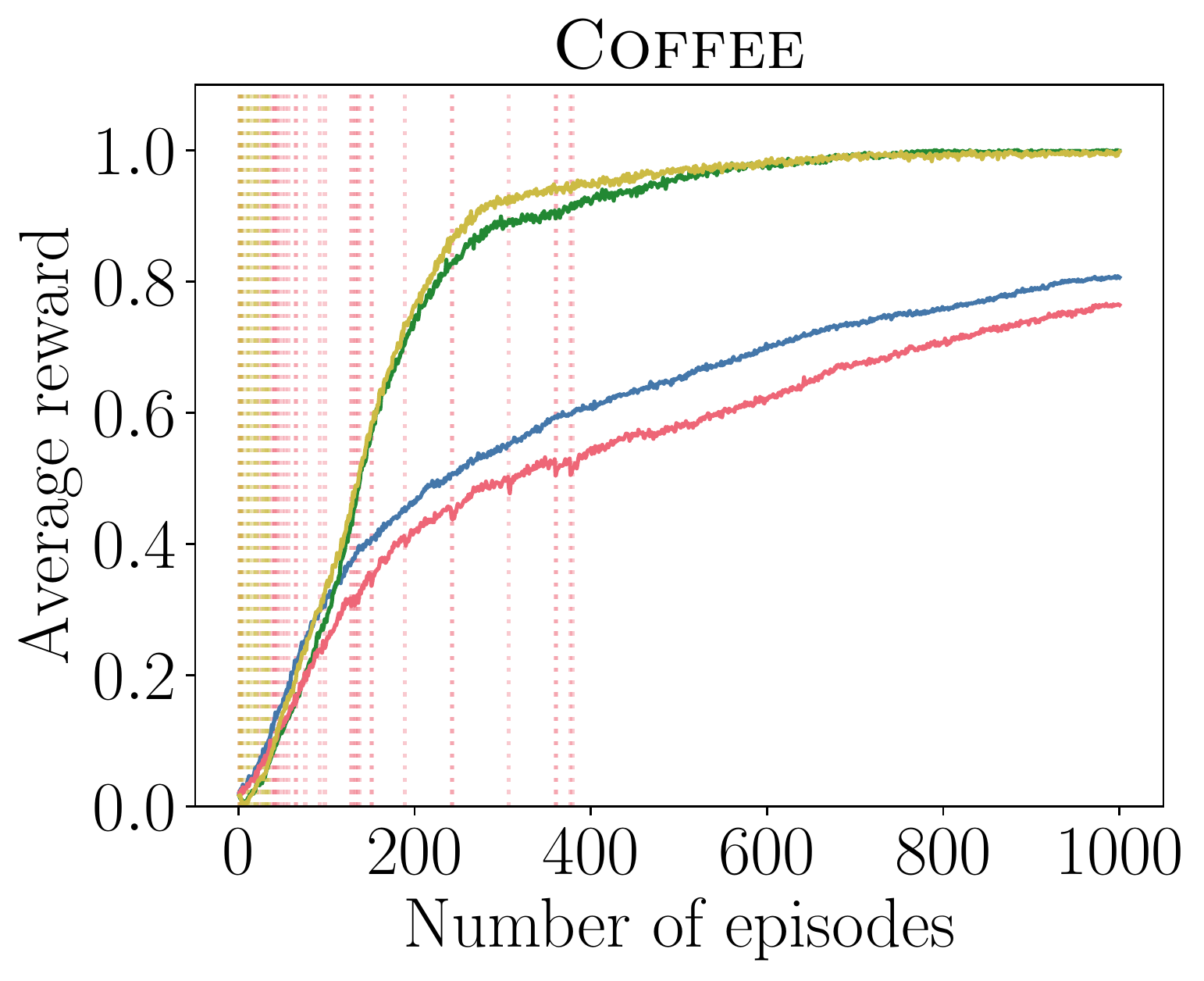}
		}
	}
	\subfloat{
		\resizebox{0.32\columnwidth}{!}{
			\includegraphics{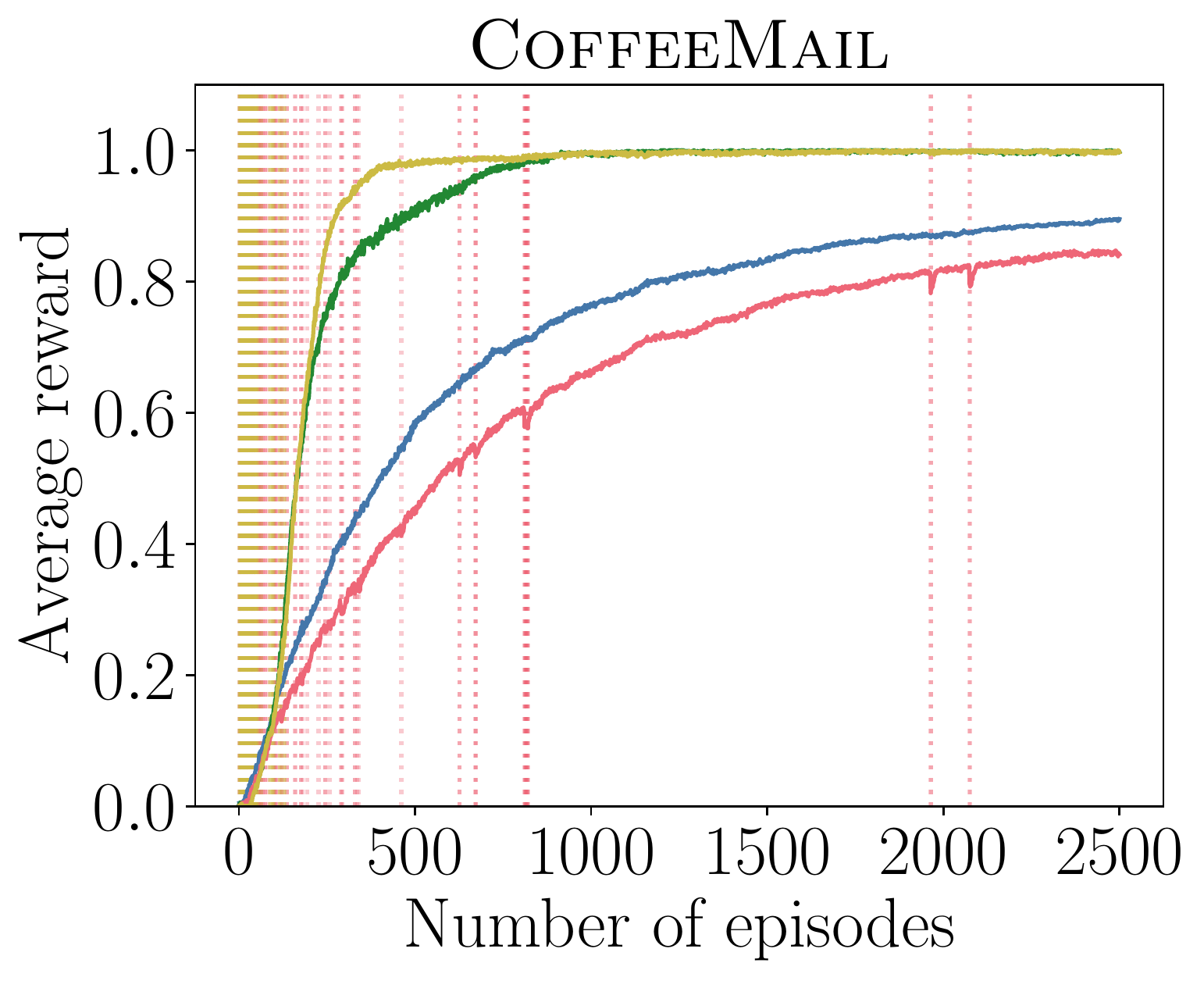}
		}
	}
	\subfloat{
		\resizebox{0.32\columnwidth}{!}{
			\includegraphics{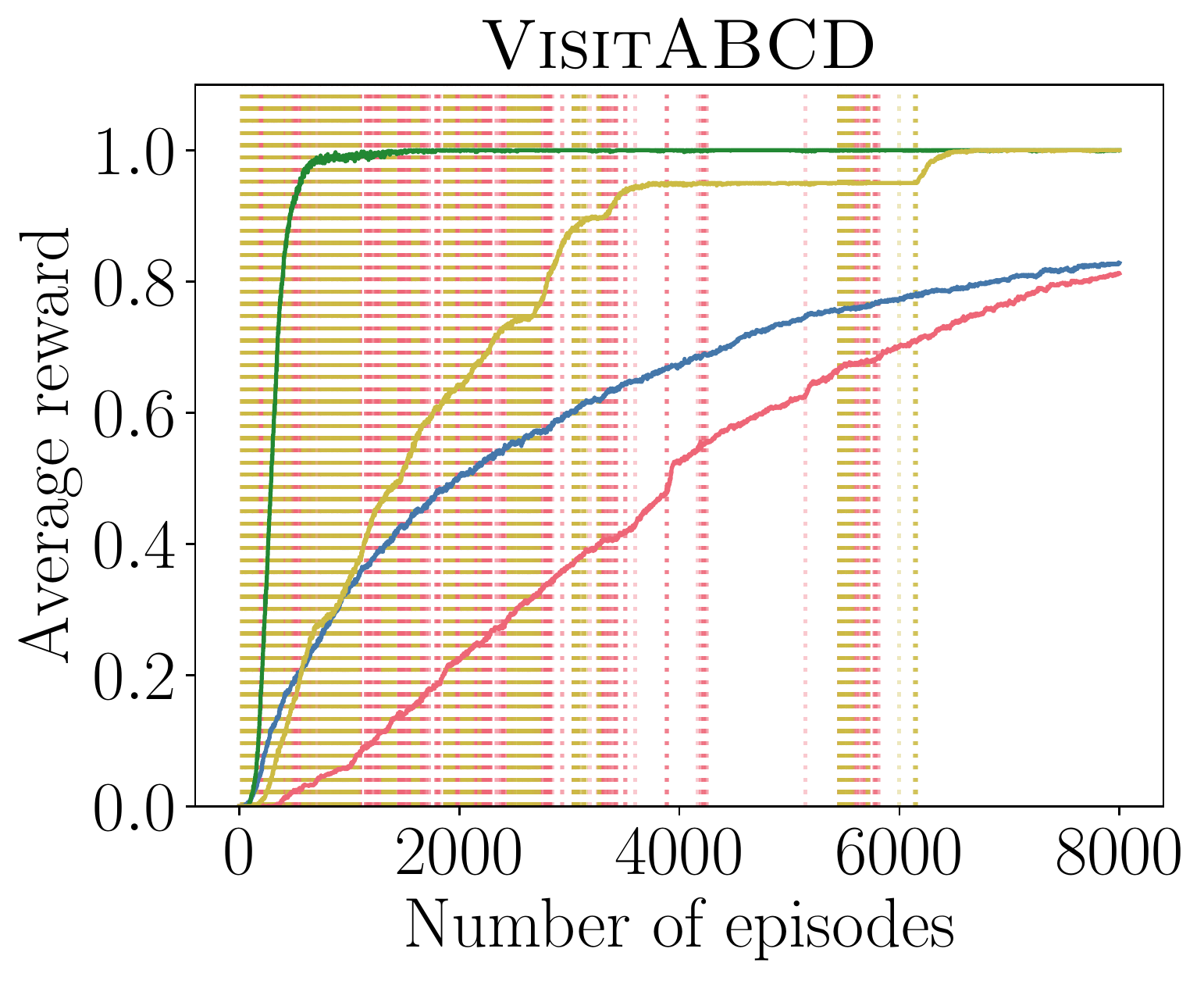}
		}
	}
	\begin{center}
		\begin{tikzpicture}
			\begin{customlegend}[legend columns=-1,legend style={column sep=1ex},legend cell align={left},legend entries={HRL,$\text{HRL}_\text{G}$,ISA-HRL,$\text{ISA-HRL}_\text{G}$}]
				\addlegendimage{pblue}
				\addlegendimage{pgreen}
				\addlegendimage{pred}
				\addlegendimage{pyellow}
			\end{customlegend}
		\end{tikzpicture}
	\end{center}
	\subfloat{
		\resizebox{0.32\columnwidth}{!}{
			\includegraphics{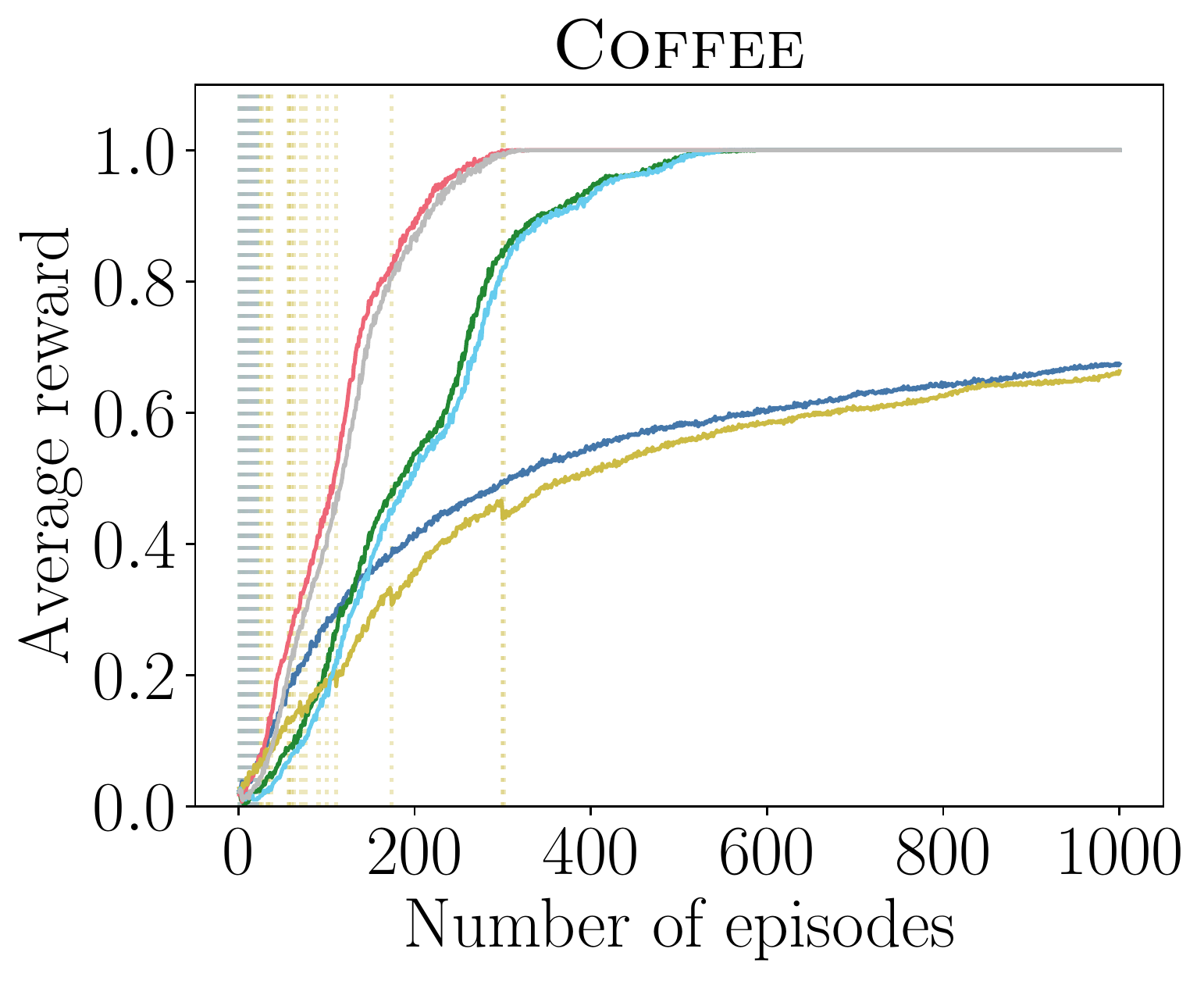}
		}
	}
	\subfloat{
		\resizebox{0.32\columnwidth}{!}{
			\includegraphics{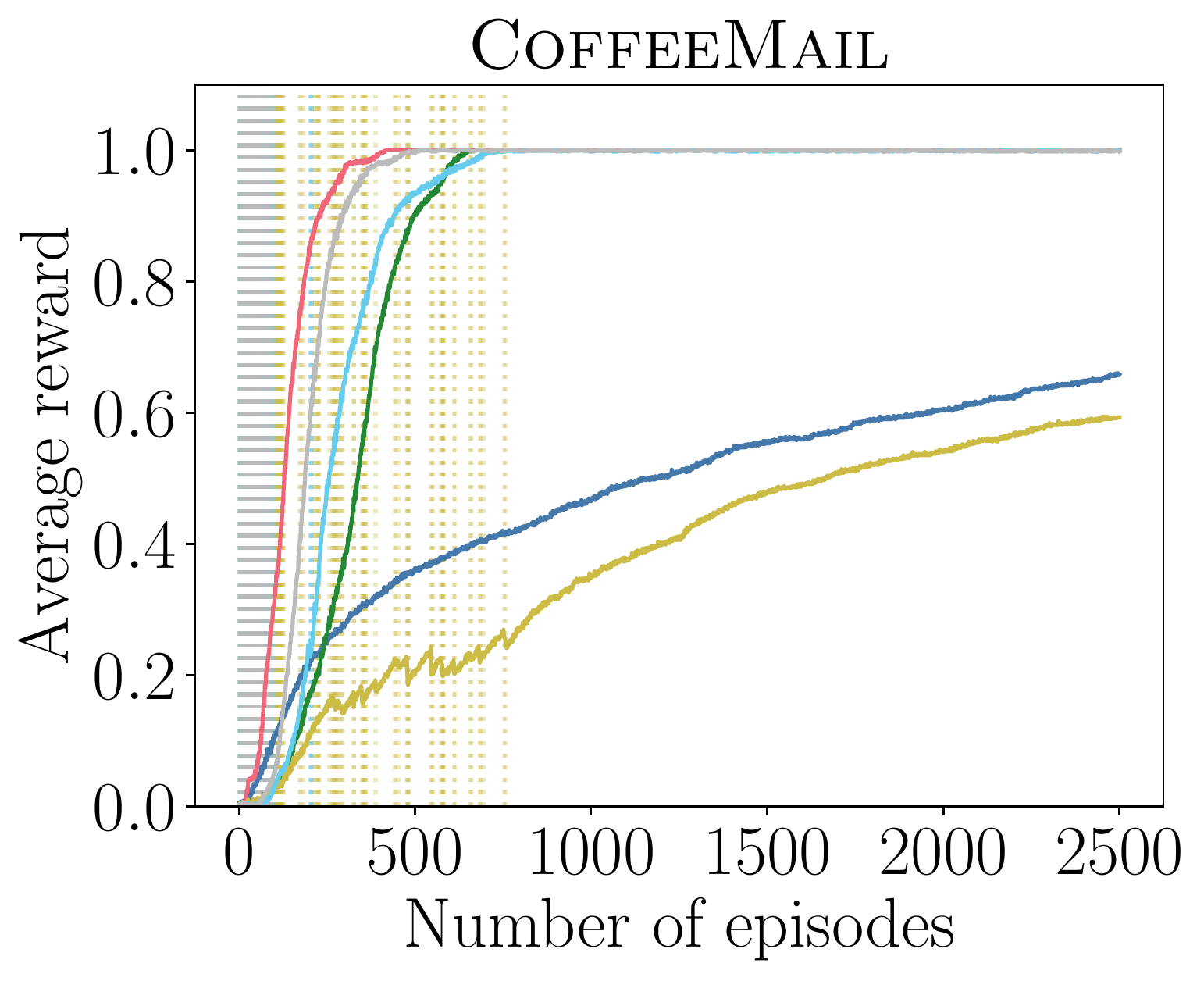}
		}
	}
	\subfloat{
		\resizebox{0.32\columnwidth}{!}{
			\includegraphics{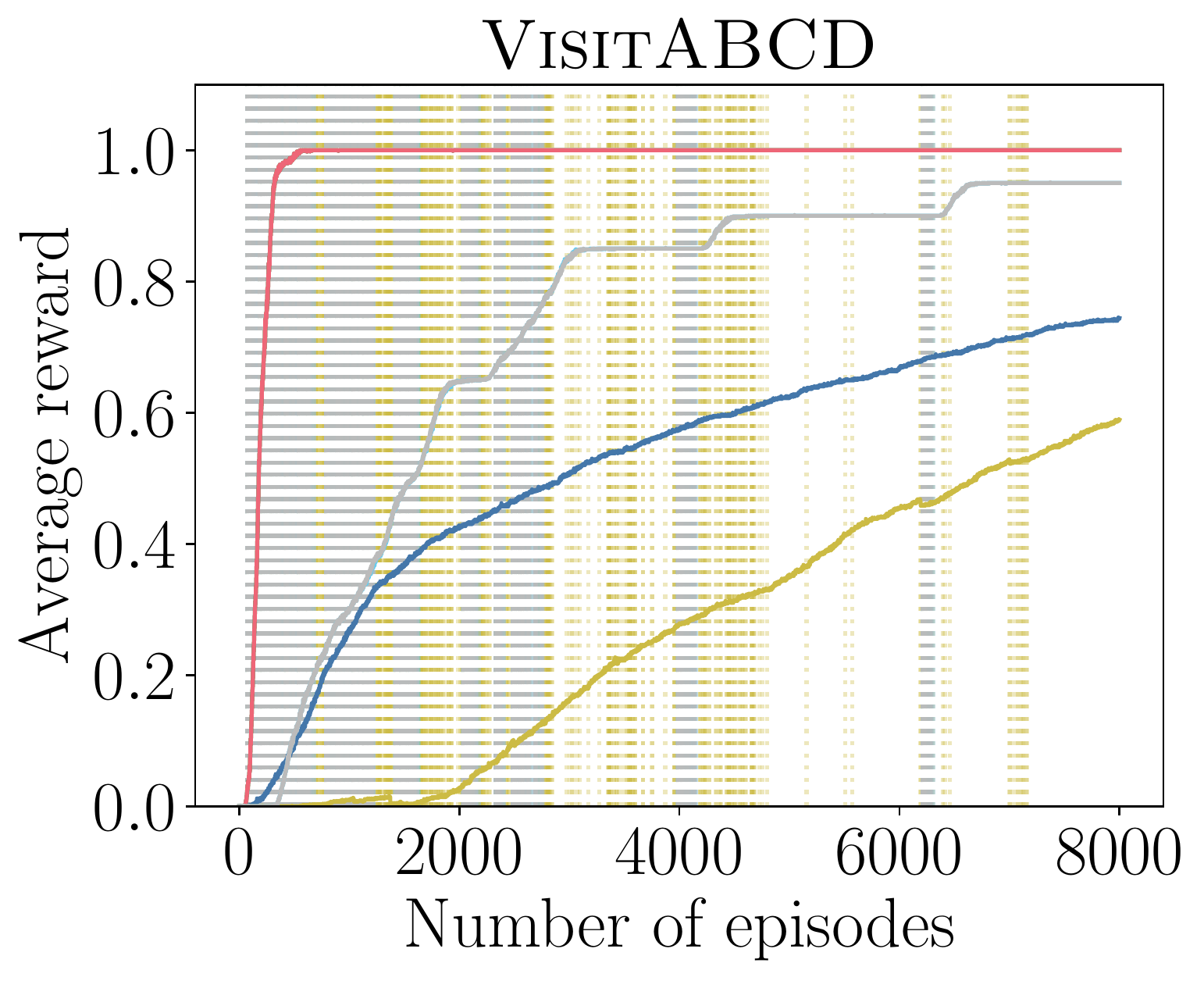}
		}
	}
	\begin{center}
		\begin{tikzpicture}
			\begin{customlegend}[legend columns=-1,legend style={column sep=1ex},legend cell align={left},legend entries={QRM,$\text{QRM}_{\min}$,$\text{QRM}_{\max}$, ISA-QRM,$\text{ISA-QRM}_{\min}$,$\text{ISA-QRM}_{\max}$}]
				\addlegendimage{pblue}
				\addlegendimage{pgreen}
				\addlegendimage{pred}
				\addlegendimage{pyellow}
				\addlegendimage{pcyan}
				\addlegendimage{pgray}
			\end{customlegend}
		\end{tikzpicture}
	\end{center}
	\caption{Learning curves for different RL algorithms in the \textsc{OfficeWorld} tasks when interleaved automaton learning is off (HRL, QRM) and on (ISA-HRL, ISA-QRM).}
	\label{fig:rl_algorithm_hancrafted}
\end{sidewaysfigure}

Figure~\ref{fig:forgetting_effect_coffee_hrl_qrm} shows the impact of automaton learning on the HRL and QRM learning curves of the \textsc{Coffee} task in a single run.\footnote{Note that abrupt changes in the learning curves do not occur in other plots because they are averaged across 20 runs. Hence, changes are smoother and not fully visible.} Remember that an HRL agent only forgets what it learned at the metacontroller level and keeps the Q-functions of the formulas unchanged. Besides, it reuses the Q-functions between similar formulas. In contrast, QRM forgets everything it has learned. The plot illustrates this behavior: while HRL quickly recovers, QRM requires a few more episodes to match HRL again.

\begin{figure}
	\centering
	\includegraphics[scale=0.6]{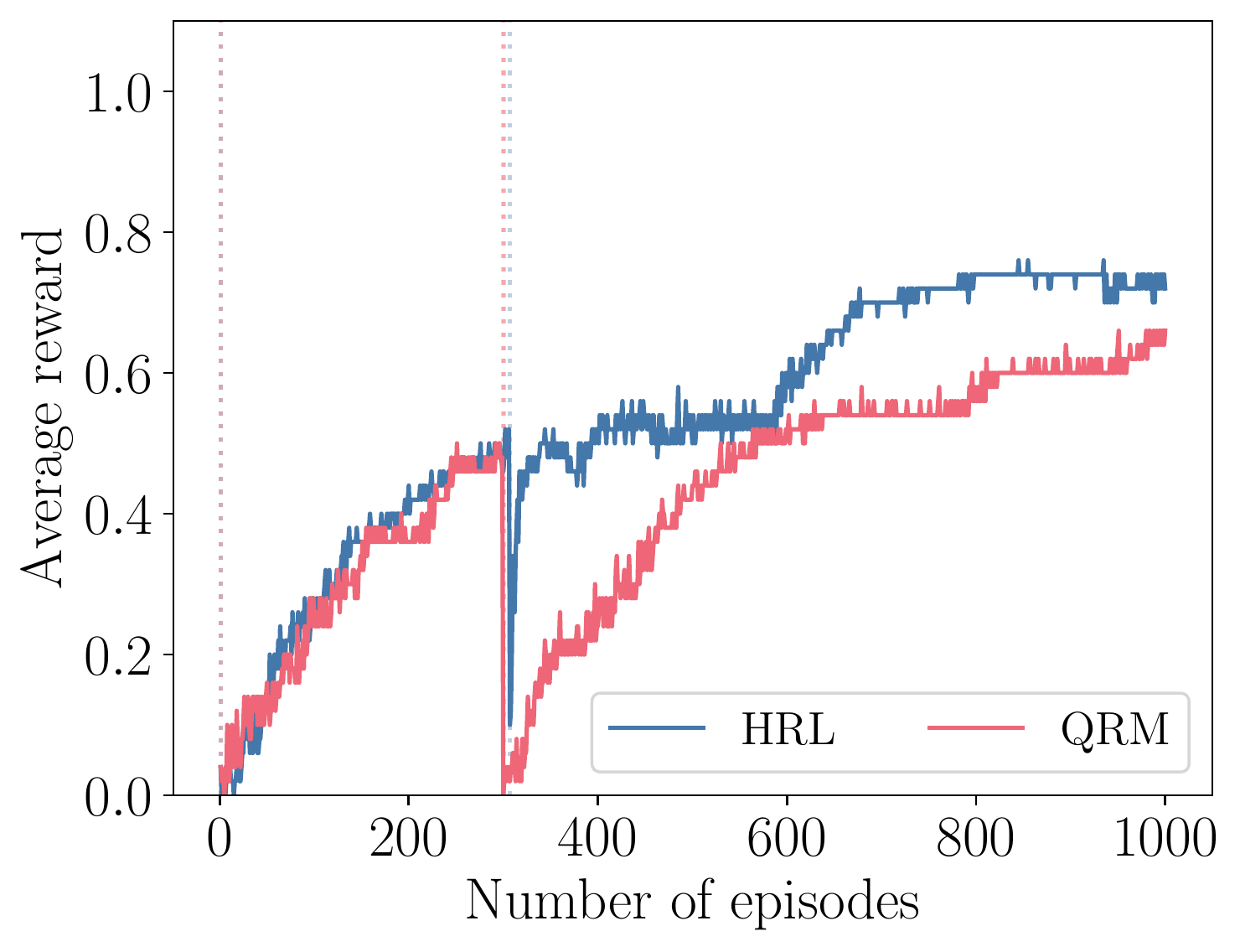}
	\caption{Example of the impact that interleaved automaton learning has on the learning curves of the \textsc{Coffee} task. An automaton is learned around episode 300. While HRL quickly recovers (it only has to relearn the policies over options), QRM needs some more episodes because it forgets everything.}
	\label{fig:forgetting_effect_coffee_hrl_qrm}
\end{figure}

Table~\ref{tab:automata_learning_statistics_officeworld} shows the automaton learning statistics for the presented \textsc{OfficeWorld} tasks using $\textnormal{HRL}_\textnormal{G}$. We observe that:
\begin{itemize}
	\item The running time increases with the number of subgoals. An automaton for the \textsc{VisitABCD} task takes (on average) more time than one for the \textsc{CoffeeMail} task even though they are both characterized by automata with the same number of states.
	\item  The number of examples increases with the number of subgoals of the task. The number of goal examples is approximately the same as the number of paths to the accepting state. For instance, in \textsc{VisitABCD} there is only one such path, so the number of goal examples is approximately 1.\footnote{The number of goal trace examples can sometimes be higher than 1 for \textsc{VisitABCD} if the first used goal trace example is complex (e.g., longer than needed or with many unnecessary symbols), thus making the subgoals unclear. In such cases a simpler goal trace might be found as a counterexample.} On the other hand, the observables that characterize the \textsc{CoffeeMail} automaton can appear jointly or not; consequently, there are more paths to the accepting state and, thus, the required number of goal examples increases. Furthermore, while there is a relationship between the number of goal examples and the number of paths to the accepting state, we do not observe such relationship between the number of dead-end examples and the number of paths to the rejecting state. The number of dead-end and incomplete examples is higher than that of goal examples; thus, we hypothesize that these two kinds of examples are mainly used to refine the automaton given the set of goal examples.
	\item The example length increases with the number of subgoals. Intuitively, the more subgoals, the longer the agent will have to interact with the environment to achieve the goal. Therefore, the observed counterexamples tend to be longer for the tasks with more subgoals.
\end{itemize}

\begin{table}
	\centering
	\resizebox{\textwidth}{!}{
		\begin{tabular}{lrrrrrrrr}
			\toprule
			                    & \multicolumn{1}{c}{Time (s.)}     & & \multicolumn{4}{c}{\# Examples}                                                                       & &  \multicolumn{1}{c}{Example Length}  \\
			\cmidrule{4-7}
			                    &                                   & & \multicolumn{1}{c}{All} & \multicolumn{1}{c}{$G$} & \multicolumn{1}{c}{$D$} & \multicolumn{1}{c}{$I$} & &                                     \\
			\midrule
			\textsc{Coffee}     & 0.4 (0.0)                         & & 8.7 (0.4)               & 2.4 (0.1)               & 3.0 (0.1)               & 3.2 (0.3)               & & 2.8 (2.1)                                    \\
			\textsc{CoffeeMail} & 18.9 (3.3)                        & & 29.0 (1.5)              & 3.9 (0.3)               & 9.3 (0.6)               & 15.8 (1.0)              & & 4.0 (2.6)                            \\
			\textsc{VisitABCD}  & 163.2 (44.3)                      & & 54.9 (3.8)              & 1.6 (0.1)               & 15.2 (0.9)              & 38.1 (3.1)              & & 5.5 (3.1)                           \\
			\bottomrule
		\end{tabular}
	}
	\caption{Automaton learning statistics for the \textsc{OfficeWorld} tasks using $\text{HRL}_\text{G}$.}
	\label{tab:automata_learning_statistics_officeworld}
\end{table}

\subsection{Experiments in {\mdseries\textsc{CraftWorld}}}
The \textsc{CraftWorld} domain~\shortcite{AndreasKL17} consists of a $39\times 39$ grid without walls. The grid contains raw materials (wood, grass, iron) and tools/workstations (toolshed, workbench, factory, bridge, axe), which constitute the set of observables $\mathcal{O}$. There are 5 labeled locations for each material, and 2 labeled locations for each tool/workstation. Like in \textsc{OfficeWorld}: (i) the agent moves in the four cardinal directions and remains in the same location if it tries to cross the grid's limits, and (ii) at each timestep, the agent knows at which cell of the grid it is (i.e., the history of achieved subgoals is not stored in the state), and sees the observables at that cell. The grids are randomly generated such that all items must be in different locations (i.e., the observations consist of one observable at most).
 
The tasks in this domain consist in observing a specific sequence of materials and tools/workstations. We use the set of tasks in \shortcite{IcarteKVM18} for evaluation:
\begin{enumerate}
	\item \textsc{MakePlank}: wood, toolshed.
	\item \textsc{MakeStick}: wood, workbench.
	\item \textsc{MakeCloth}: grass, factory.
	\item \textsc{MakeRope}: grass, toolshed.
	\item \textsc{MakeShears}: iron, wood, workbench (the iron and the wood can be observed in any order).
	\item \textsc{MakeBridge}: iron, wood, factory (the iron and the wood can be observed in any order).
	\item \textsc{GetGold}: iron, wood, factory, bridge (the iron and the wood can be observed in any order).
	\item \textsc{MakeBed}: wood, toolshed, grass, workbench (the grass can be observed anytime before the workbench.)
	\item \textsc{MakeAxe}: wood, workbench, iron, toolshed (the iron can be observed anytime before the toolshed).
	\item \textsc{GetGem}: wood, workbench, iron, toolshed, axe (the iron can be observed anytime before the toolshed).
\end{enumerate}
Tasks 1-4 have 2 subgoals and are represented by 3 state minimal automata. Tasks 5-6 have 3 subgoals and are represented by 5 state minimal automata. Task 7 has 4 subgoals and is represented by a 6 state minimal automaton. Tasks 8-9 have 4 subgoals and are represented by 7 state minimal automata. Task 10 has 5 subgoals and is represented by an 8 state minimal automaton. The agent gets a reward of 1 upon the goal's achievement and 0 otherwise. Unlike \textsc{OfficeWorld}, this domain has no dead-end states, so the set of dead-end examples is always empty. 

Table~\ref{tab:craftworld_experiments_params} lists the parameters used in these experiments. The only difference with respect to the default parameters used in \textsc{OfficeWorld} is the number of tasks $|\mathcal{D}|$, which is 100 in this case. Experimentally, we observed that using 100 instead of 50 was a better choice for tasks 8-10 which, as we will explain later, occasionally time out. 

\begin{table}[]
	\centering
	\begin{tabular}{ll}
		\toprule
		Learning rate ($\alpha$)                  & 0.1    \\
		Exploration rate ($\epsilon$)             & 0.1    \\
		Discount factor ($\gamma$)                & 0.99   \\
		Number of episodes                        & 10,000 \\
		Avoid learning purely negative formulas & \cmark \\
		Number of tasks ($|\mathcal{D}|$)         & 100     \\
		Maximum episode length ($N$)              & 250    \\
		Trace compression                         & \cmark    \\
		Enforce acyclicity                        & \cmark    \\
		Number of disjuncts ($\kappa$)         & 1      \\
		Use restricted observable set             & \xmark \\
		\bottomrule
	\end{tabular}
	\caption{Parameters used in the \textsc{CraftWorld} experiments.}
	\label{tab:craftworld_experiments_params}
\end{table}

Table~\ref{tab:automata_learning_statistics_craftworld} shows the automaton learning statistics for the presented \textsc{CraftWorld} tasks using $\text{HRL}_\text{G}$. Note that we have divided the tasks into several groups according to the number of subgoals they have and the number of states that their corresponding minimal automata have. Figure~\ref{fig:rl_algorithm_hancrafted_crafworld} shows the learning curves for one representative of each group of tasks\footnote{The learning curves are similar between members of each group, so we report just one of them.} with and without interleaved learning of automata. We observe the following:
\begin{itemize}
	\item Like in the \textsc{OfficeWorld} tasks, the more subgoals and automaton states, the higher the values for the collected metrics (running time, number of examples and example length). Besides, the number of goal examples still corresponds to the number of paths from the initial state to the accepting state. The figure also shows that, as before, learning becomes more frequent as the tasks become harder.
	\item The automaton learning statistics are very close between groups of tasks, especially for the ones having simpler automata. As the tasks become harder, the differences between tasks in the same group become bigger (e.g., \textsc{MakeBed} and \textsc{MakeAxe}). Naturally, it is extremely unlikely that an agent observes two equivalent sets of examples for two different tasks, especially when examples become longer (as we have seen before, the more subgoals a task has, the longer the examples become). Therefore, it is normal that these differences arise for harder tasks.
	\item The running time increases dramatically from \textsc{GetGold} to \textsc{MakeBed} and \textsc{MakeAxe}. Actually, the automaton learner has timed out a few times for the latter tasks: 5 for \textsc{MakeBed} and 4 for \textsc{MakeAxe}. Furthermore, in the case of \textsc{GetGem}, the harder task, it has timed out 9 times. The number of timeouts varies between algorithms, which is probably caused by exploration. For example, standard HRL has timed out 8 times for \textsc{MakeAxe} and only once for \textsc{MakeBed}. 
	\item The difference between the curves where interleaved automaton learning is on (ISA-HRL, ISA-QRM) and off (HRL, QRM) is small for most of the tasks, like in \textsc{OfficeWorld}. This shows that the induced automata are useful to learn a policy that reaches the goal. Note that for the hardest tasks (\textsc{MakeAxe}, \textsc{MakeBed} and \textsc{GetGem}) the gap between HRL and QRM with respect to ISA-HRL and ISA-QRM is usually bigger than for other tasks. This is due to the presence of timeouts in the approaches that induce automata, as explained before.
	\item When an automaton is handcrafted, $\text{QRM}_{\min}$ performs like $\text{QRM}_{\max}$ because the minimum and maximum distances to the accepting state are the same in these automata. This similarity also occurs when an automaton is learned; however, they are not identical since the intermediately learned automata may cause some variances.
	\item The approaches not using guidance (HRL, QRM) start converging faster than the ones that use it ($\text{HRL}_\text{G}, \text{QRM}_{\min}, \text{QRM}_{\max}$). However, the latter eventually learn to reach the goal earlier in the interaction. The initially slower convergence for the approaches using guidance can be due to the fact that guidance encourages more exploration. A \textsc{CraftWorld} grid is bigger than an \textsc{OfficeWorld} grid, which causes the agent to explore the environment for longer and delays the start of convergence. However, all the knowledge acquired while exploring is later quickly exploited and the learning curves  surpass those of the approaches without guidance.
\end{itemize}

In Section~\ref{sec:conclusions} we present several proposals for future work to reduce the running time that ILASP needs to find an automaton, thus making the approach more scalable to tasks whose minimal automata have many states, such as \textsc{GetGem}.

\begin{table}
	\centering
	\resizebox{\textwidth}{!}{
		\begin{tabular}{lrrrrrrr}
			\toprule
			                    & \multicolumn{1}{c}{Time (s.)}     & & \multicolumn{3}{c}{\# Examples}                                             & &  \multicolumn{1}{c}{Example Length}  \\
			\cmidrule{4-6}
			                    &                                   & & \multicolumn{1}{c}{All} & \multicolumn{1}{c}{$G$} & \multicolumn{1}{c}{$I$} & &                                       \\
			\midrule
			\textsc{MakePlank}  & 0.2 (0.0)                         & & 4.4 (0.3)               & 1.4 (0.1)               & 3.0 (0.3)               & & 2.2 (1.2)                             \\
			\textsc{MakeStick}  & 0.2 (0.0)                         & & 3.6 (0.2)               & 1.2 (0.1)               & 2.4 (0.2)               & & 2.2 (1.4)                             \\
			\textsc{MakeCloth}  & 0.3 (0.0)                         & & 4.9 (0.4)               & 1.2 (0.1)               & 3.6 (0.4)               & & 2.4 (1.5)                             \\
			\textsc{MakeRope}   & 0.2 (0.0)                         & & 4.2 (0.3)               & 1.2 (0.1)               & 2.9 (0.3)               & & 2.4 (1.4)                             \\
			\midrule
			\textsc{MakeShears} & 2.0 (0.3)                         & & 16.2 (0.8)              & 3.3 (0.2)               & 12.8 (0.8)              & & 3.4 (1.7)                             \\
			\textsc{MakeBridge} & 1.7 (0.3)                         & & 15.5 (1.3)              & 3.0 (0.2)               & 12.5 (1.2)              & & 3.0 (1.5)                             \\
			\midrule
			\textsc{GetGold}    & 60.7 (25.7)                       & & 30.6 (3.2)              & 2.2 (0.2)               & 28.5 (3.2)              & & 4.0 (1.9)                             \\
			\midrule
			\textsc{MakeBed}    & 2140.4 (1071.7)*                  & & 37.1 (3.1)*             & 3.8 (0.3)*              & 33.3 (3.0)*             & & 4.0 (1.7)*                           \\
			\textsc{MakeAxe}    & 2990.3 (717.7)*                   & & 46.6 (3.5)*             & 3.3 (0.2)*              & 43.3 (3.5)*             & & 4.3 (1.9)*                            \\
			\midrule
			\textsc{GetGem}     & 6179.4 (2784.8)*                  & & 116.8 (14.7)*           & 1.2 (0.1)*              & 115.6 (14.7)*            & & 5.2 (2.0)*                                     \\
			\bottomrule
		\end{tabular}
	}
	\caption{Automaton learning statistics for the \textsc{CraftWorld} tasks using $\text{HRL}_\text{G}$.}
	\label{tab:automata_learning_statistics_craftworld}
\end{table}

\begin{sidewaysfigure}
	\centering
	\begin{minipage}[c]{0.24\linewidth}
		\resizebox{\linewidth}{!}{
			\includegraphics{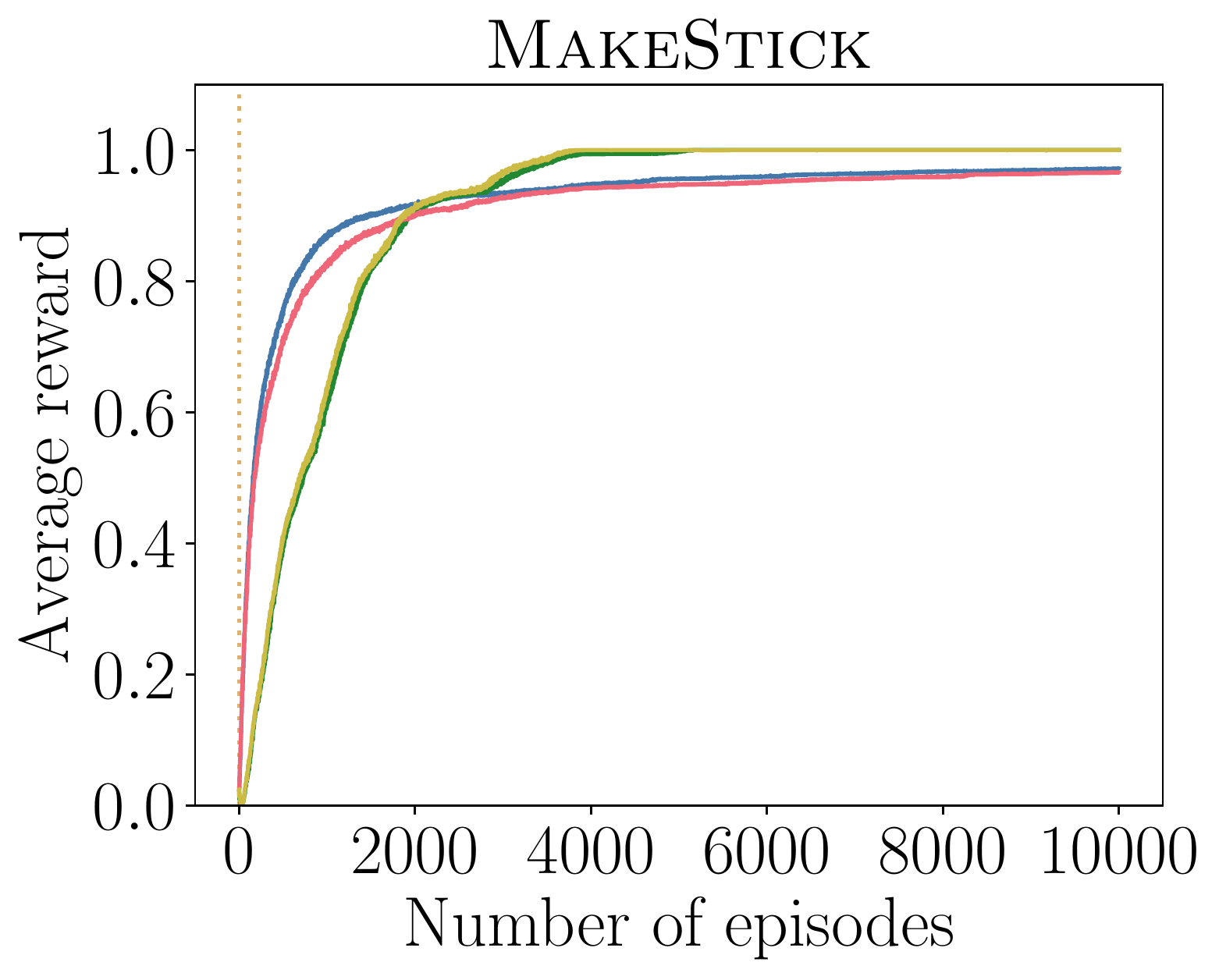}
		}
	\end{minipage}
	\begin{minipage}[c]{0.24\linewidth}
		\resizebox{\linewidth}{!}{
			\includegraphics{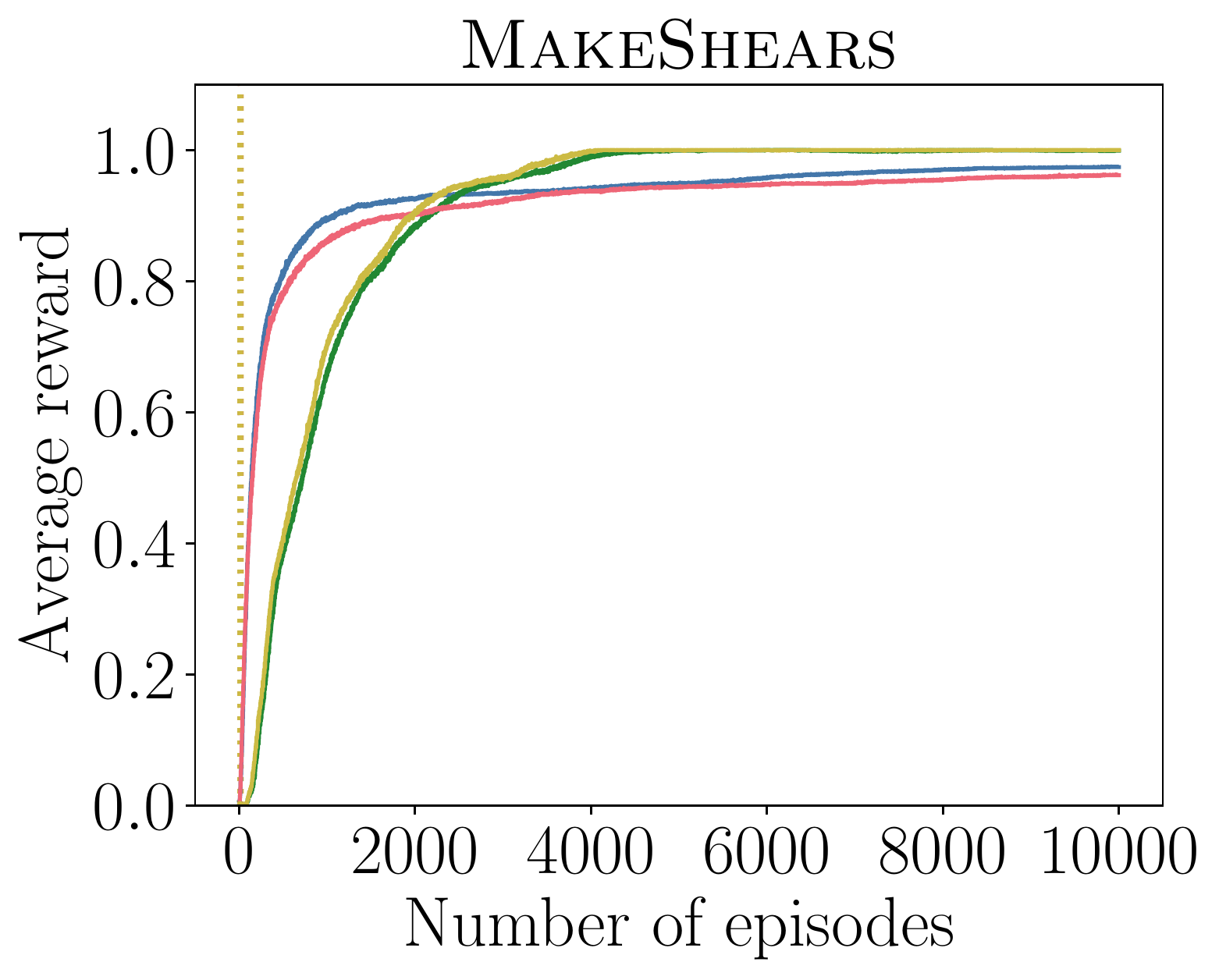}
		}
	\end{minipage}
	\rulesep
		\begin{minipage}[c]{0.24\linewidth}
		\resizebox{\linewidth}{!}{
			\includegraphics{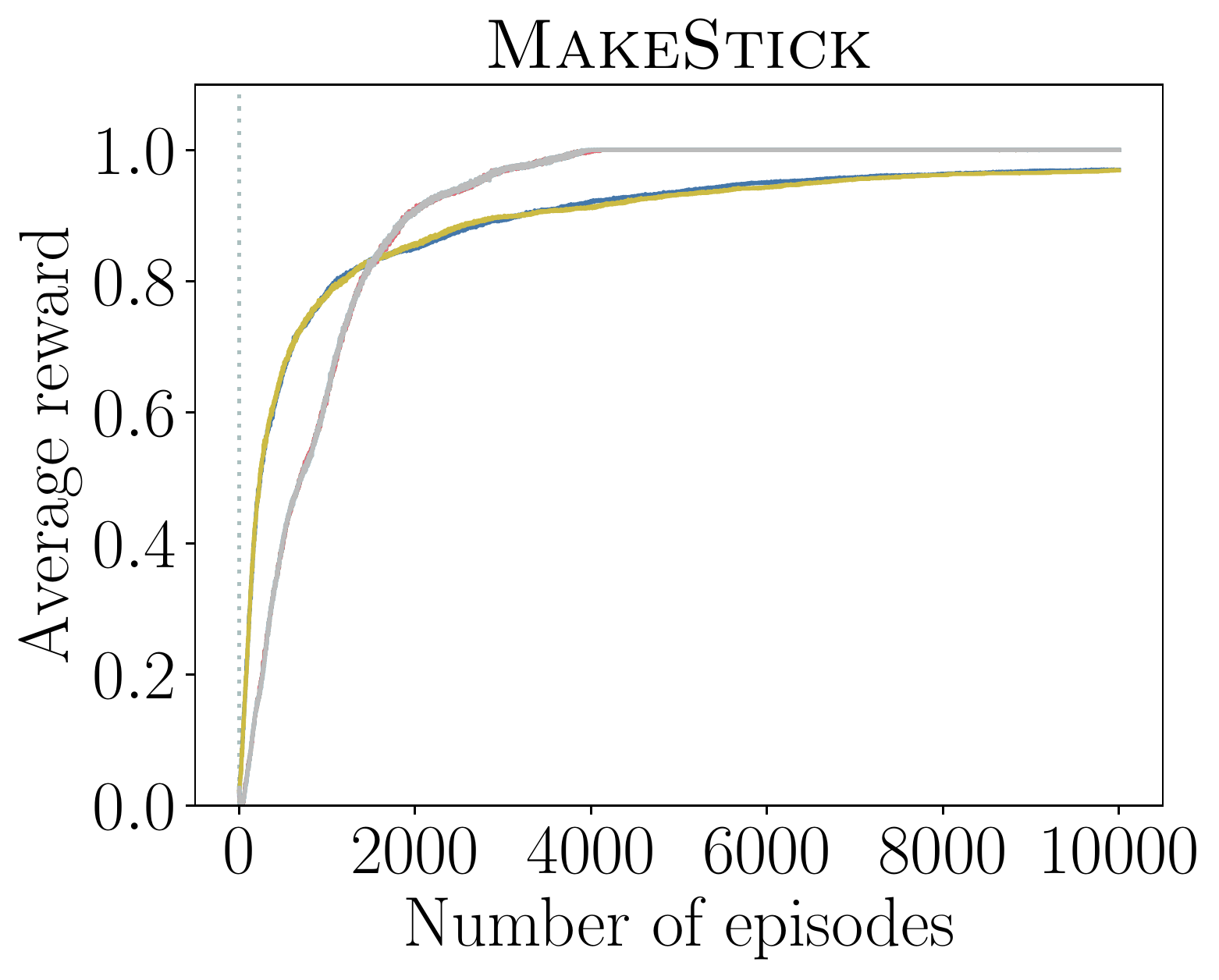}
		}
	\end{minipage}
	\begin{minipage}[c]{0.24\linewidth}
		\resizebox{\linewidth}{!}{
			\includegraphics{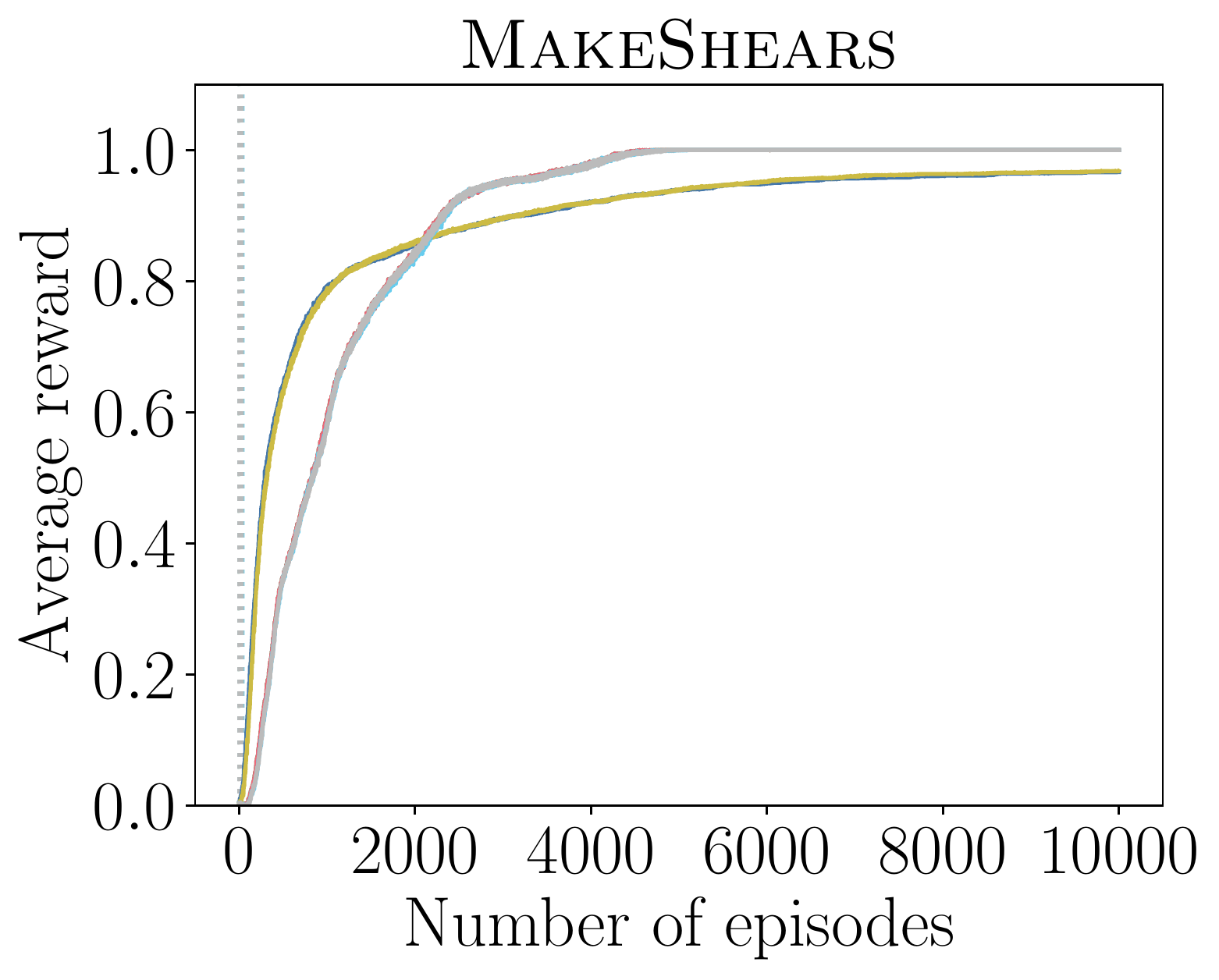}
		}
	\end{minipage}

	\begin{minipage}[c]{0.24\linewidth}
		\resizebox{\linewidth}{!}{
			\includegraphics{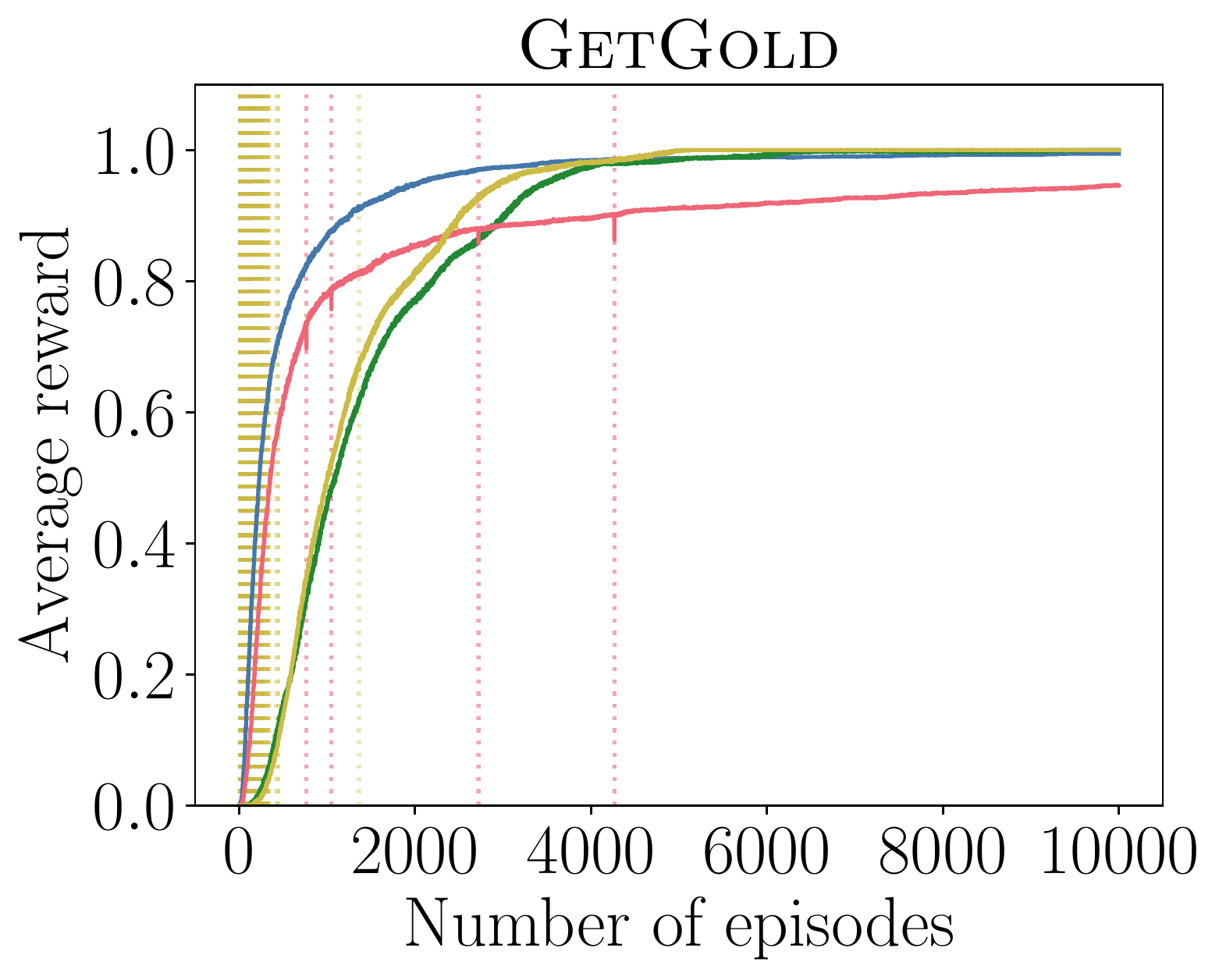}
		}
	\end{minipage}
	\begin{minipage}[c]{0.24\linewidth}
		\resizebox{\linewidth}{!}{
			\includegraphics{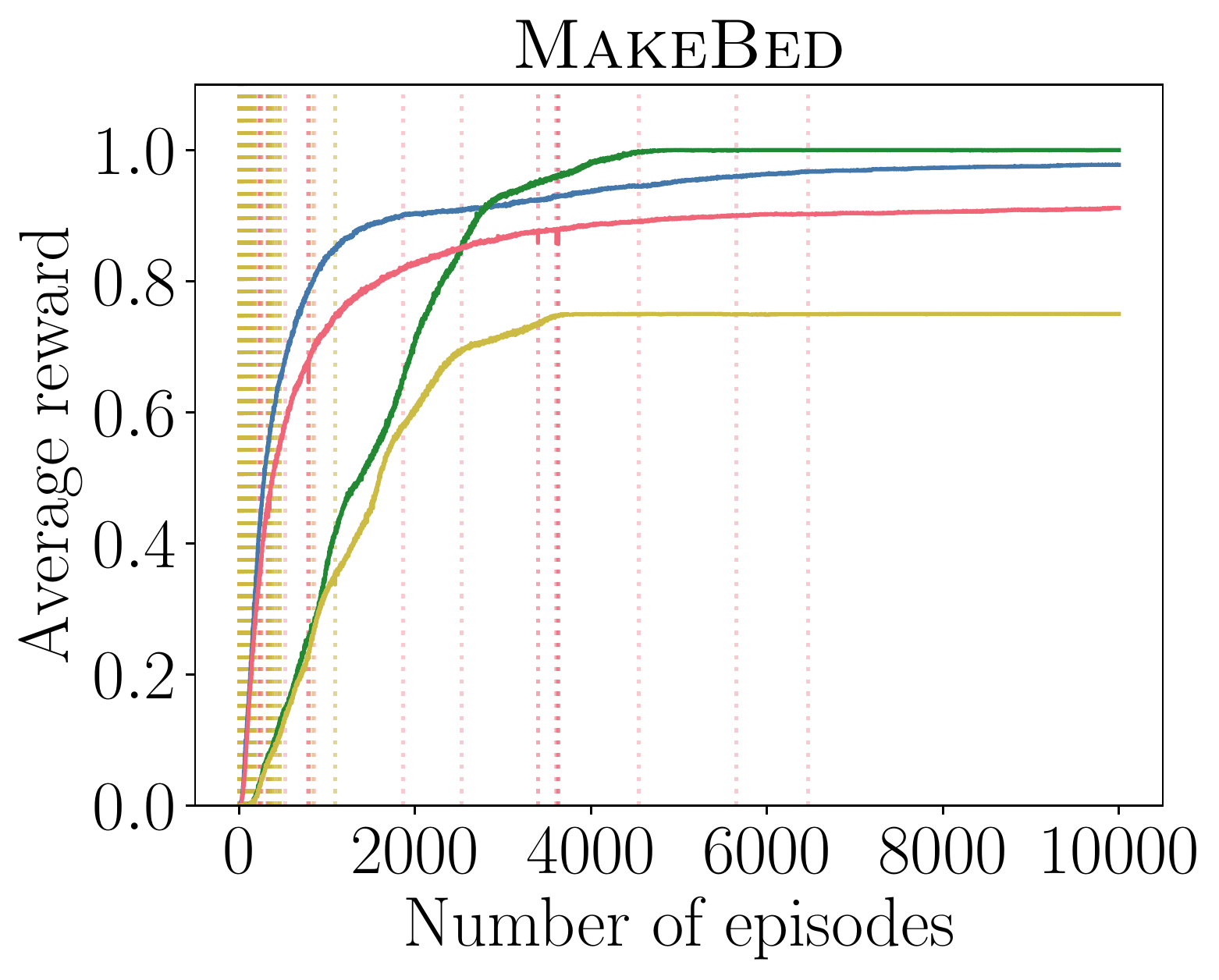}
		}
	\end{minipage}
	\rulesep
	\begin{minipage}[c]{0.24\linewidth}
		\resizebox{\linewidth}{!}{
			\includegraphics{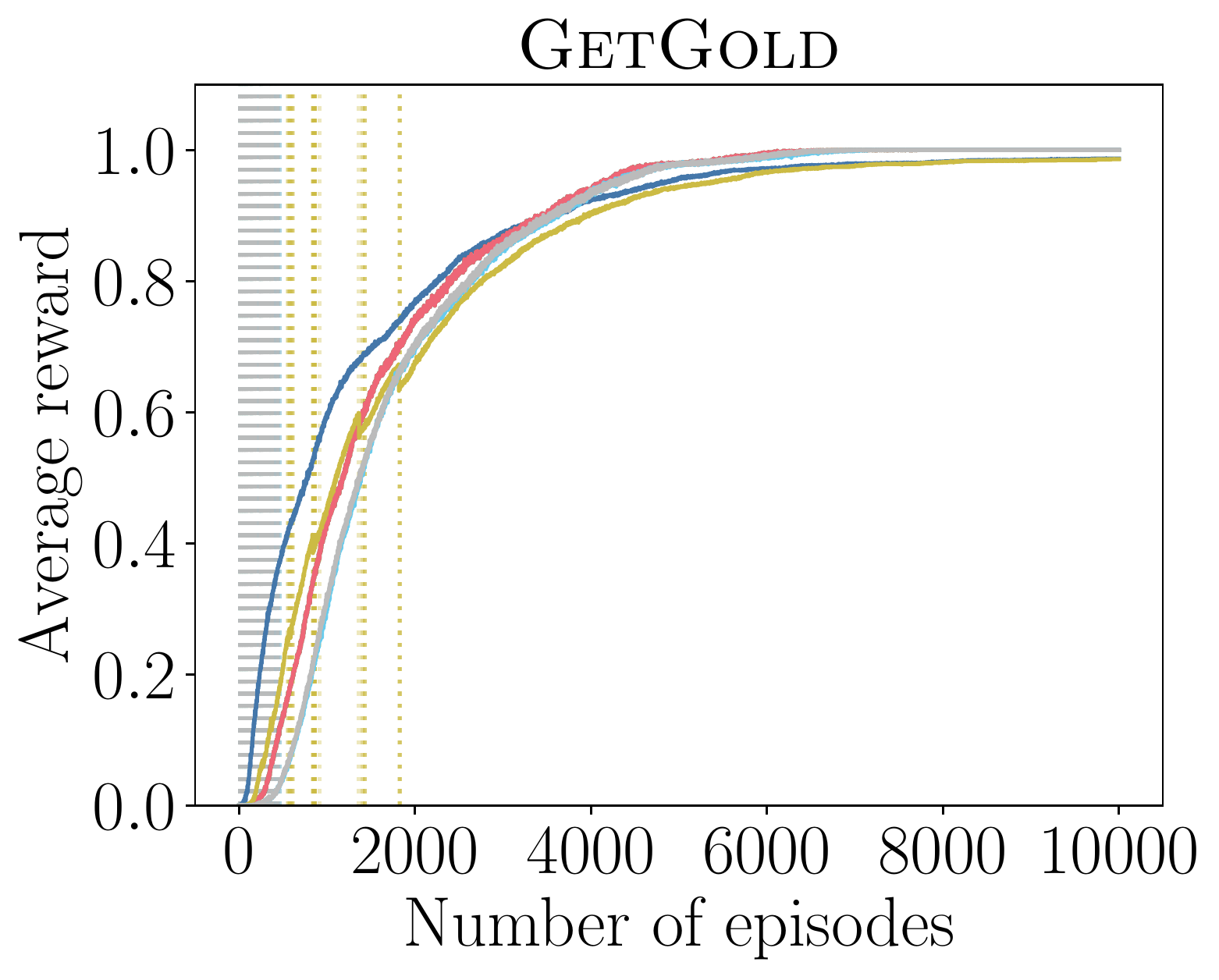}
		}
	\end{minipage}
	\begin{minipage}[c]{0.24\linewidth}
		\resizebox{\linewidth}{!}{
			\includegraphics{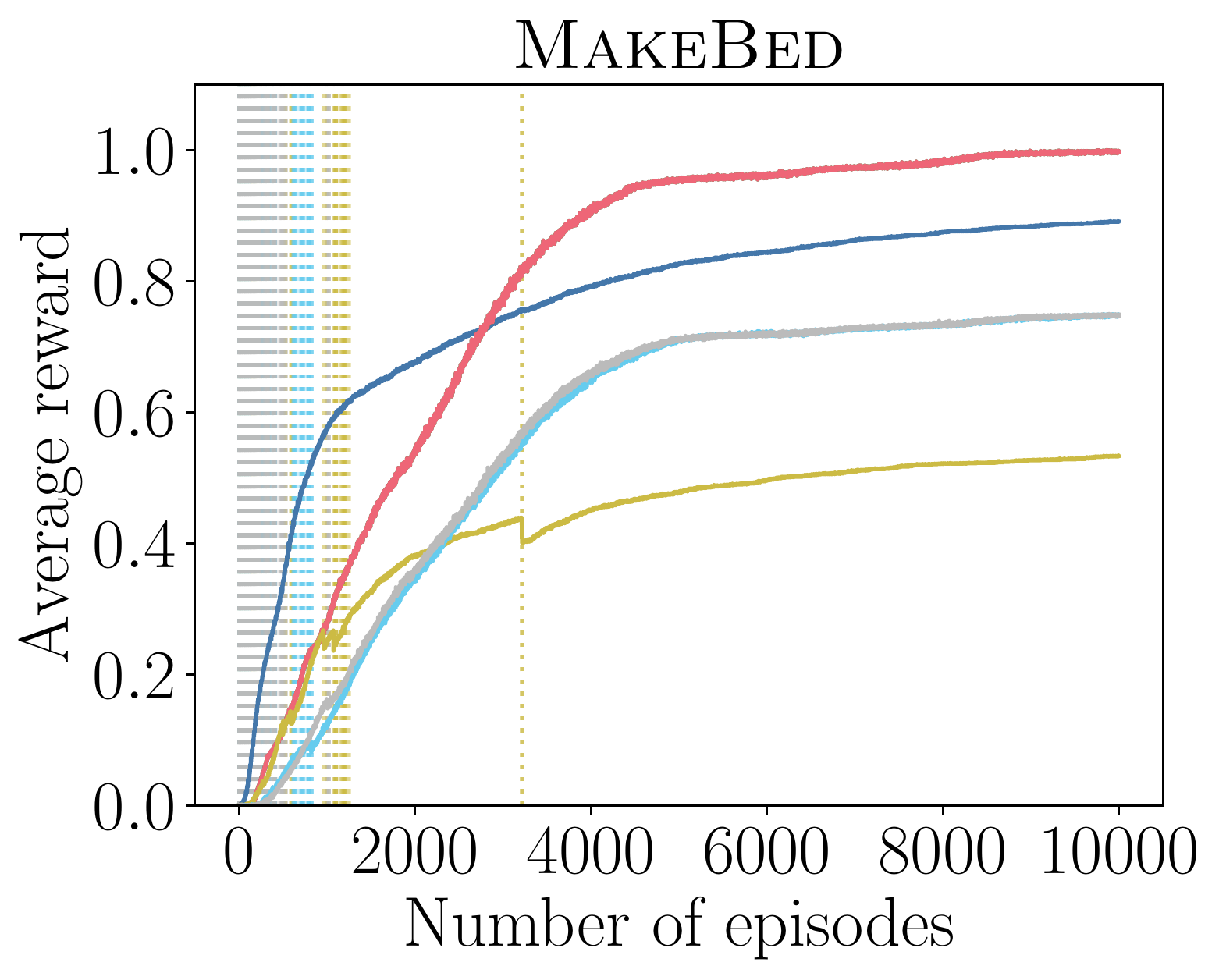}
		}
	\end{minipage}

	\begin{minipage}[c]{0.24\columnwidth}
		\resizebox{\linewidth}{!}{
			\includegraphics{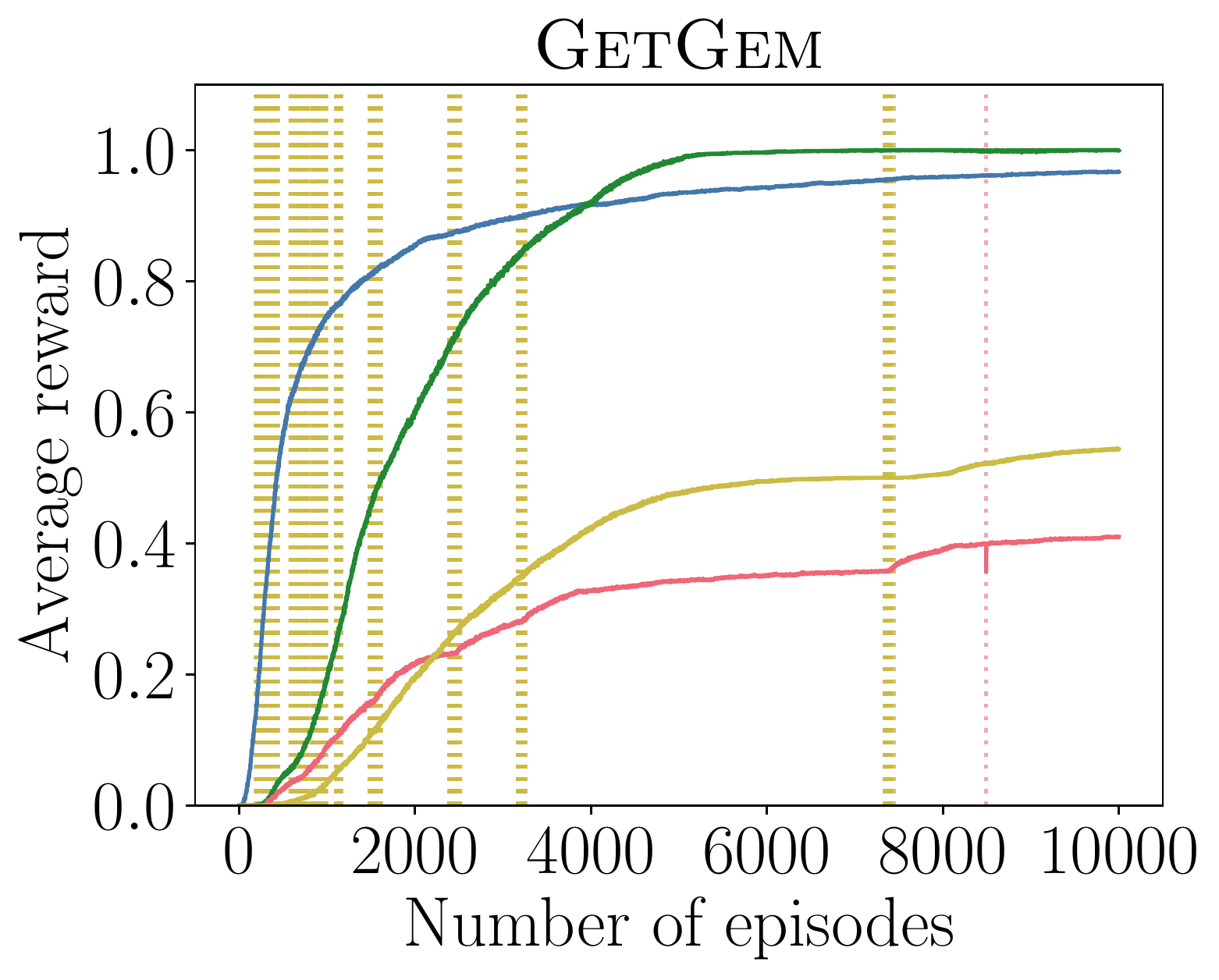}
		}
	\end{minipage}
	\begin{minipage}[c]{0.24\columnwidth}
		\centering
		\begin{tikzpicture}
		\begin{customlegend}[legend columns=1,legend style={column sep=1ex},legend cell align={left},legend entries={HRL,$\text{HRL}_\text{G}$,ISA-HRL,$\text{ISA-HRL}_\text{G}$}]
		\addlegendimage{pblue}
		\addlegendimage{pgreen}
		\addlegendimage{pred}
		\addlegendimage{pyellow}
		\end{customlegend}
		\end{tikzpicture}
	\end{minipage}
	\rulesep
	\begin{minipage}[c]{0.24\columnwidth}
		\resizebox{\linewidth}{!}{
			\includegraphics{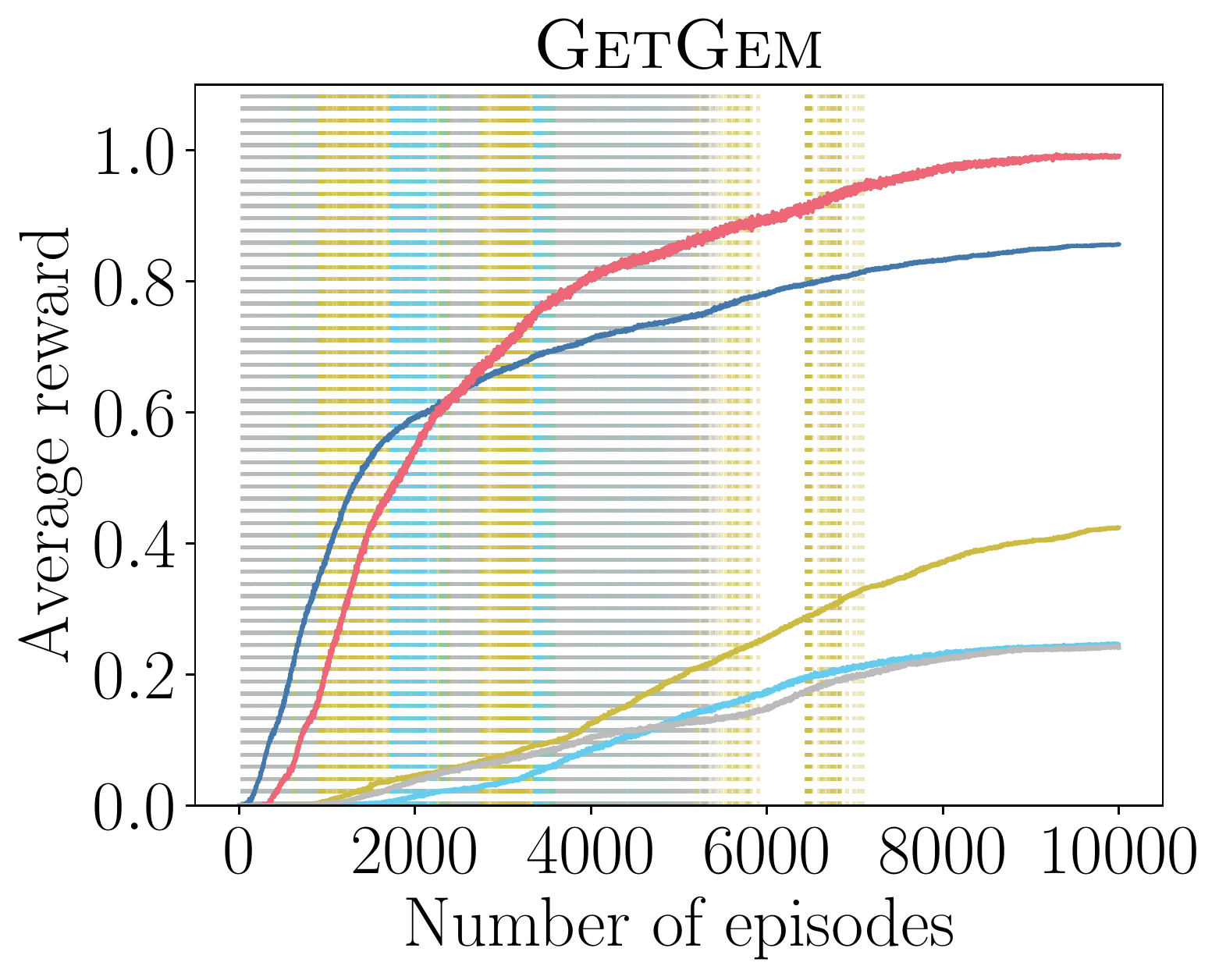}
		}
	\end{minipage}
	\begin{minipage}[c]{0.24\columnwidth}
		\centering
		\begin{tikzpicture}
		\begin{customlegend}[legend columns=1,legend style={column sep=1ex},legend cell align={left},legend entries={QRM,$\text{QRM}_{\min}$,$\text{QRM}_{\max}$, ISA-QRM,$\text{ISA-QRM}_{\min}$,$\text{ISA-QRM}_{\max}$}]
		\addlegendimage{pblue}
		\addlegendimage{pgreen}
		\addlegendimage{pred}
		\addlegendimage{pyellow}
		\addlegendimage{pcyan}
		\addlegendimage{pgray}
		\end{customlegend}
		\end{tikzpicture}
	\end{minipage}
	\caption{Learning curves for different RL algorithms in the \textsc{CraftWorld} tasks when interleaved automaton learning is off (HRL, QRM) and on (ISA-HRL, ISA-QRM). The ISA curves for \textsc{MakeBed} and \textsc{GetGem} are below the handcrafted automata curves because of the timeout runs. When a run finishes successfully, both types of curves are similar.}
	\label{fig:rl_algorithm_hancrafted_crafworld}
\end{sidewaysfigure}

\subsection{Experiments in {\mdseries\textsc{WaterWorld}}}
The \textsc{WaterWorld} domain \shortcite{IcarteKVM18}, which is illustrated in Figure~\ref{fig:waterworld_grid}, consists of a 2D box containing 12 balls of 6 different colors (2 balls per color). Each ball moves at a constant speed in a given direction. The balls bounce only when they collide with a wall. The agent is a white ball that can change its velocity in any of the four cardinal directions. The set of observables $\mathcal{O}=\{r,g,b,y,c,m\}$ is formed by the balls' colors. The agent observes a color when it overlaps with the ball painted with it. For example, in Figure~\ref{fig:waterworld_grid} the agent would observe $\{g\}$ (green). Note that several balls can overlap at the same time; thus, the agent can simultaneously observe several colors. The tasks we consider consist in observing a sequence of colors in a specific order:
\begin{itemize}
	\item \textsc{RGB}: red ($r$) then green ($g$) then blue ($b$). It consists of 3 subgoals and is represented by a 4 state minimal automaton.
	\item \textsc{RG-B}: red ($r$) then green ($g$) and (independently) blue ($b$). Note that there are two  sequences and they can be interleaved. For instance, $\langle\{r\},\{g\},\{b\}\rangle$, $\langle\{b\},\{r\},\{g\}\rangle$ and $\langle\{r\},\{b\},\{g\}\rangle$ are three possible goal traces. Note that RGB is a subcase of this task. It consists of 3 subgoals and is represented by a 6 state minimal automaton.
	\item \textsc{RGBC}: touch red ($r$) then green ($g$) then blue ($b$) then cyan ($c$). It consists of 4 subgoals and is represented by a 5 state minimal automaton.
\end{itemize}
The agent gets a reward of 1 upon the goal's achievement and 0 otherwise. The tasks we consider have no dead-end states, like in \textsc{CraftWorld}. The balls start with a random position and direction at the beginning of each episode.

\begin{figure}
	\centering
	\includegraphics[scale=0.8]{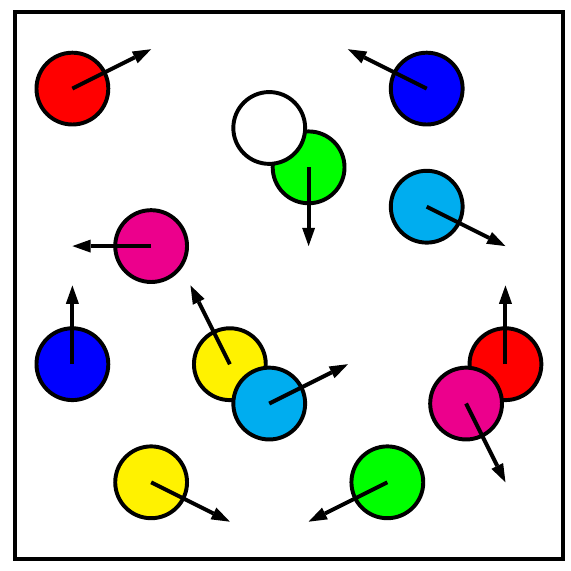}
	\caption{The \textsc{WaterWorld} domain \shortcite{IcarteKVM18}.}
	\label{fig:waterworld_grid}
\end{figure}

Unlike \textsc{OfficeWorld} and \textsc{Craftworld}, the state space is continuous, so we cannot use tabular Q-functions. Instead, like \shortciteA{IcarteKVM18}, we use a Double DQN \shortcite<{DDQN},>{HasseltGS16} to approximate the Q-functions in both HRL and QRM. The neural networks consist of 4 hidden layers of 64 neurons, each followed by a ReLU. A network's input is a vector containing the absolute position and velocity of the agent, and the relative positions and velocities of the other balls. Just like in a standard DQN \shortcite{MnihKSRVBGRFOPB15}, the output contains the estimated Q-value for each action (or option in the case of a metacontroller in HRL). We train the neural networks using the Adam optimizer \shortcite{KingmaB15} with $\alpha=\num{1e-5}$. The target networks are updated every 100 steps. A total of 50,000 episodes are run for each setting with parameters $\epsilon=0.1$ and $\gamma= 0.9$. The Q-functions are updated using batches of 32 experiences uniformly sampled from an experience replay buffer of size 50,000. The learning starts after collecting 1,000 samples.

Table~\ref{tab:waterworld_experiments_params} lists the mentioned neural network hyperparameters along with the automaton learning and RL parameters. The main difference with respect to the previous experiments is the use of a POMDP set consisting of a single POMDP. Since the balls in \textsc{WaterWorld} are constantly moving, it is easier for the agent to observe all possible combinations of observables. Therefore, we do not need to use a larger set to learn a general automaton.

\begin{table}[]
	\centering
	\begin{tabular}{ll}
		\toprule
		Learning rate ($\alpha$)                  & \num{1e-5}    \\
		Exploration rate ($\epsilon$)             & 0.1    \\
		Discount factor ($\gamma$)                & 0.9   \\
		Number of episodes                        & 50,000 \\
		Maximum episode length ($N$)              & 150    \\
		Replay memory size                        & 50,000 \\
		Replay start size                         & 1,000 \\
		Batch size                                & 32     \\
		Number of tasks ($|\mathcal{D}|$)         & 1     \\
		Target network update frequency           & 100 \\
		Trace compression                         & \cmark    \\
		Enforce acyclicity                        & \cmark    \\
		Number of disjuncts ($\kappa$)         & 1      \\
		Avoid learning purely negative formulas & \cmark \\
		Use restricted observable set             & \xmark \\
		\bottomrule
	\end{tabular}
	\caption{Parameters used in the \textsc{WaterWorld} experiments.}
	\label{tab:waterworld_experiments_params}
\end{table}

Figure~\ref{fig:rl_algorithm_hancrafted_waterworld} shows the learning curves for the \textsc{WaterWorld} tasks introduced before. We show how those that use automaton learning (ISA-HRL, ISA-QRM) compare to those obtained with handcrafted automata (HRL, QRM). The greedy policy was evaluated every 500 episodes: the evaluation consisted in running the greedy policy for 10 different episodes and averaging the reward obtained across them. The learning curves are smoothed using a sliding window of size 1,000. Table~\ref{tab:automata_learning_statistics_waterworld} shows the automaton learning statistics for these tasks using $\text{HRL}_\text{G}$.\footnote{The automaton learning statistics for the other RL algorithms are similar, so we do not report them.} We observe the following:
\begin{table}
	\centering
	\resizebox{0.9\textwidth}{!}{
		\begin{tabular}{lrrrrrrr}
			\toprule
			& \multicolumn{1}{c}{Time (s.)}     & & \multicolumn{3}{c}{\# Examples}                                             & &  \multicolumn{1}{c}{Example Length}  \\
			\cmidrule{4-6}
			&                                   & & \multicolumn{1}{c}{All} & \multicolumn{1}{c}{$G$} & \multicolumn{1}{c}{$I$} & &                                     \\
			\midrule
			\textsc{RGB}     & 3.7 (0.3)                         & & 26.9 (1.1)              & 7.0 (0.4)               & 20.0 (0.9)              & & 4.3 (2.3)                           \\
			\textsc{RG-B}    & 129.5 (23.9)                      & & 48.4 (1.2)              & 15.4 (0.4)              & 33.0 (1.0)              & & 4.4 (2.2)                            \\
			\textsc{RGBC}    & 111.8 (23.1)                      & & 61.4 (2.7)              & 11.7 (0.5)              & 49.6 (2.4)              & & 5.5 (2.7)                           \\
			\bottomrule
		\end{tabular}
	}
	\caption{Automaton learning statistics for the \textsc{WaterWorld} tasks using $\text{HRL}_\text{G}$.}
	\label{tab:automata_learning_statistics_waterworld}
\end{table}
\begin{itemize}
	\item The tasks with more subgoals run the automaton learner more often. For instance, automaton learning is concentrated at the beginning of the interaction in RGB; in contrast, it is called at many different times in RGBC.
	\item Even though the minimal automaton of RG-B has more states than that of RGB, the reward is sparser in the latter since there are not as many ways to achieve the goal as in the former. This is why RGB's learning curve does not converge faster than RG-B's. Note that the time and the number of examples needed to learn the RG-B automaton are higher than those for RGB because its set of automaton states is bigger. Finally, the fact that RG-B is the task requiring more goal traces  shows that it is the task where achieving the goal is easier although it makes automaton learning harder.
	\item Even though RGBC has more subgoals than RG-B, its running time is lower. The most likely reason is that the minimal automaton for RG-B has more states and, therefore, the hypothesis space is bigger. Naturally, the bigger the hypothesis space is, the harder it becomes to find a solution. However, as we saw in the grid-world experiments, the tasks with more subgoals require more examples. In particular, RGBC needs more incomplete examples than RG-B possibly because there are more sequences of candidate subgoals to be discarded.
	\item The approaches based on HRL perform better than the ones based on QRM, specially in RGBC. We hypothesize there are two possible causes for this behavior:
	\begin{enumerate}[(i)]
		\item The Q-functions resetting. Remember that while in HRL only the metacontrollers are reset, all the Q-functions in QRM are reset. Given that new automata are learned throughout the entire interaction, the agents using QRM rarely have the chance to converge to a stable policy.
		\item While the HRL agent commits to satisfying a given formula (determined by the metacontroller) at a given step, the QRM agent selects the globally best action at each step. Unlike the grid-worlds, in this domain the observables are constantly changing their position, so it can be more difficult to learn a function that generalizes to many scenarios (i.e., different initializations of the task).\footnote{We highlight that our evaluation QRM in the \textsc{WaterWorld} domain differs a bit from the one by \shortciteA{IcarteKVM18}. While we randomly initialize the environment at the start of every episode, they use a fixed map. In the latter case, QRM quickly converges in RGBC (we have been able to reproduce the results), but we consider that learning Q-functions that generalize to diverse scenarios is more interesting.}
	\end{enumerate}
	To determine which of these two causes is more plausible, we examine the performance of QRM with a handcrafted automaton. In general, the performance of approaches using automaton learning are very similar to the ones using a handcrafted automaton. This shows that the forgetting effect is not as present in \textsc{WaterWorld} as in the grid-world domains. The use of experience replay is likely to be responsible for this because the agents can update the Q-functions without having to relive successful experiences that happened in the past. Therefore, we conclude that cause (ii) is better supported by our experiments.
	\item There is barely any difference between the learning curves of the algorithms that do not use auxiliary guidance (HRL, QRM), and the ones that do use it ($\text{HRL}_\text{G}$, $\text{QRM}_{\min}$, $\text{QRM}_{\max}$). \shortciteA{CamachoIKVM19} also showed similar behavior in the case of QRM using handcrafted reward machines with a different reward shaping mechanism. We hypothesize that since the Q-functions must generalize to different settings (remember that all episodes start with a random configuration), an agent that does not use guidance might explore similarly to an agent that does use it. Therefore, using guidance does not help much in these tasks.
\end{itemize}
\begin{sidewaysfigure}
	\centering
	\subfloat{
		\resizebox{0.32\columnwidth}{!}{
			\includegraphics{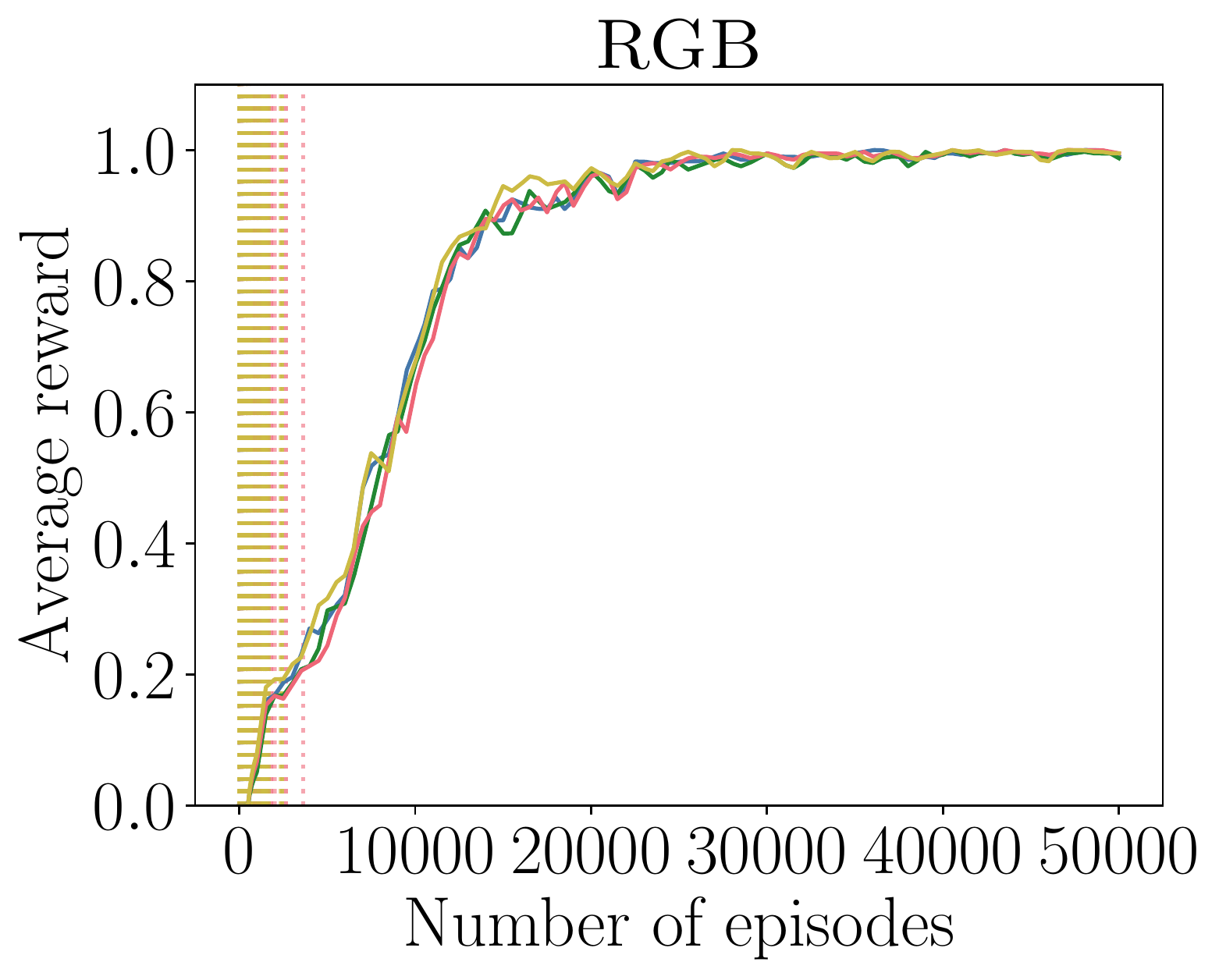}
		}
	}
	\subfloat{
		\resizebox{0.32\columnwidth}{!}{
			\includegraphics{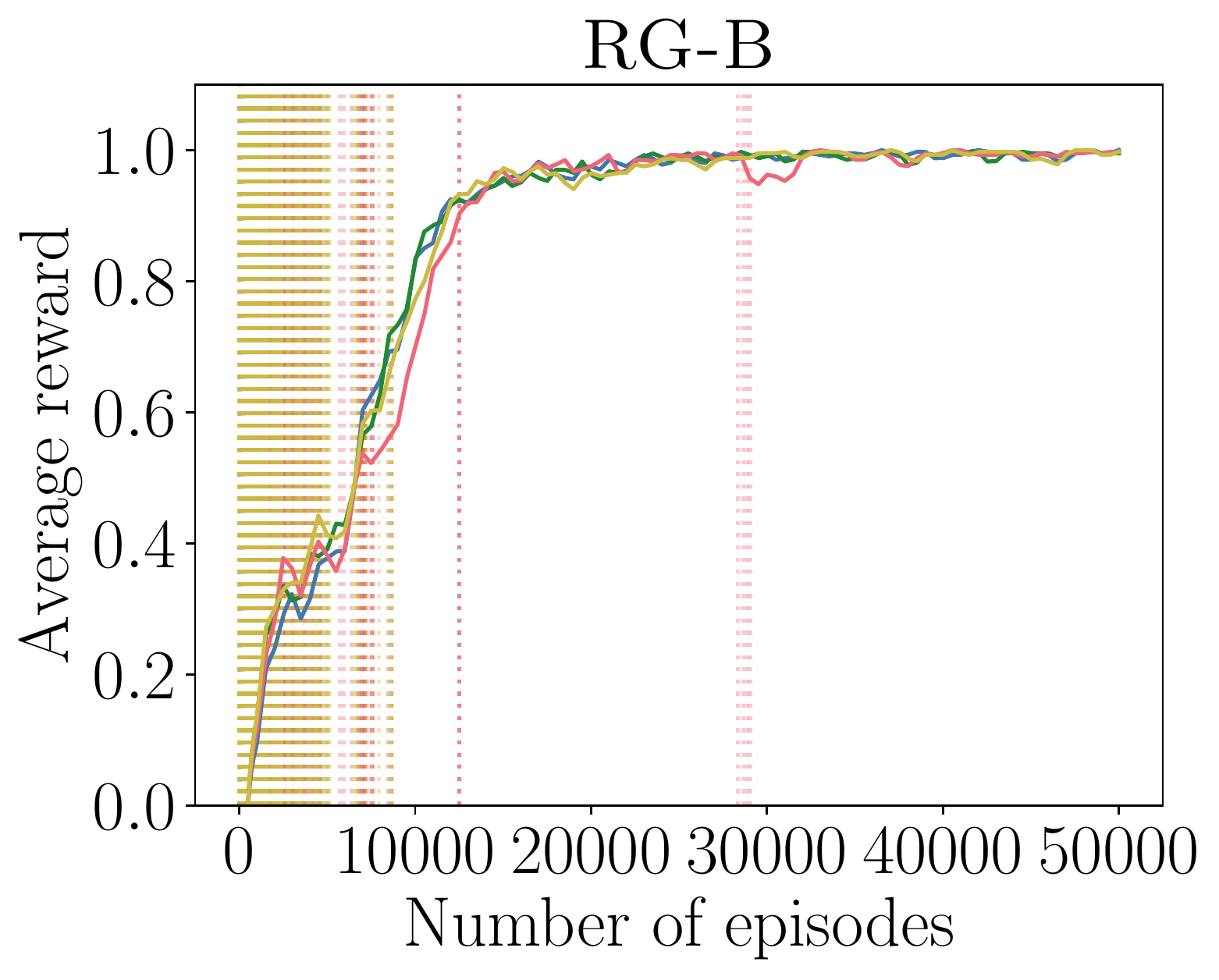}
		}
	}
	\subfloat{
		\resizebox{0.32\columnwidth}{!}{
			\includegraphics{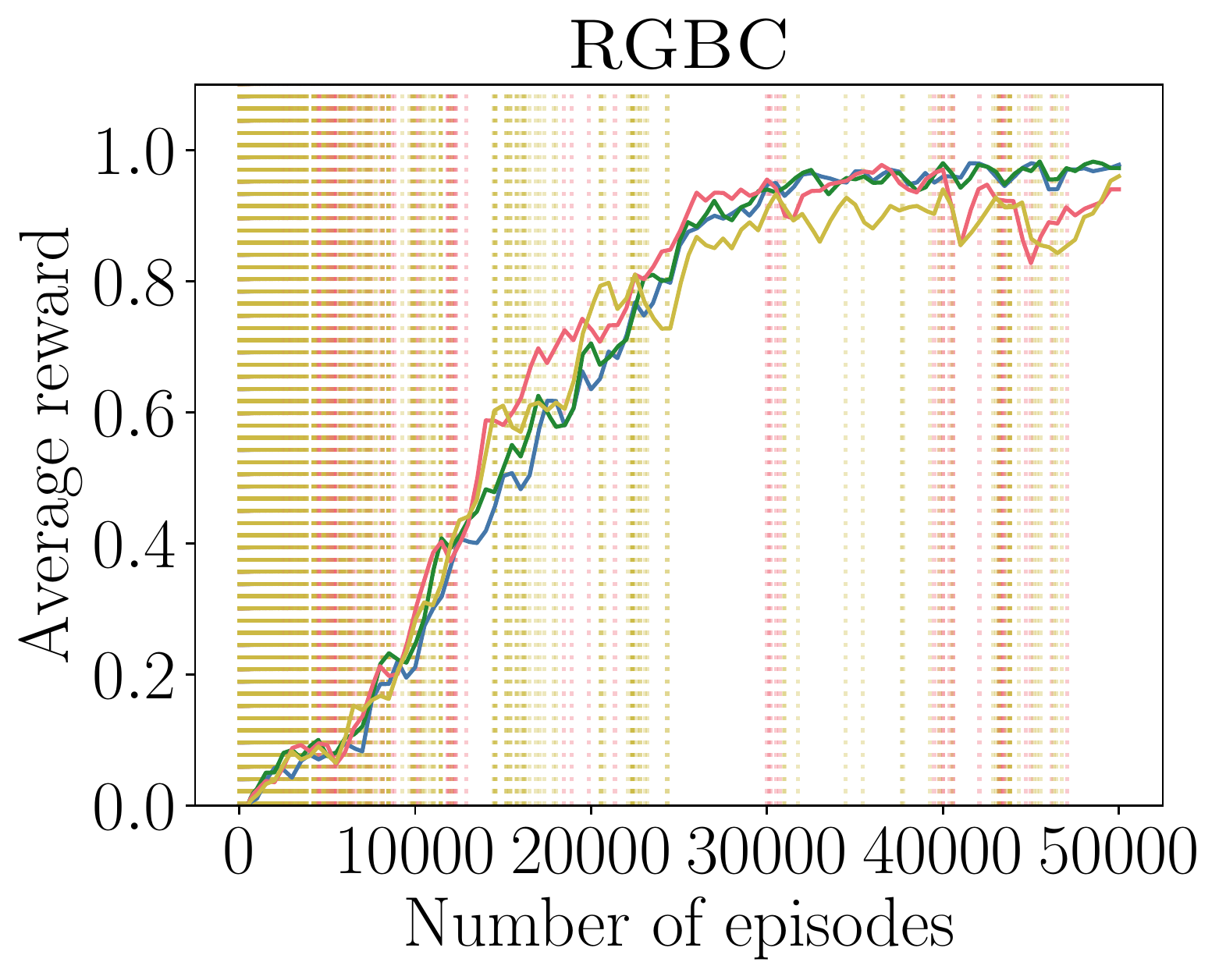}
		}
	}
	\begin{center}
		\begin{tikzpicture}
			\begin{customlegend}[legend columns=-1,legend style={column sep=1ex},legend cell align={left},legend entries={HRL,$\text{HRL}_\text{G}$,ISA-HRL,$\text{ISA-HRL}_\text{G}$}]
				\addlegendimage{pblue}
				\addlegendimage{pgreen}
				\addlegendimage{pred}
				\addlegendimage{pyellow}
			\end{customlegend}
		\end{tikzpicture}
	\end{center}
	\subfloat{
		\resizebox{0.32\columnwidth}{!}{
			\includegraphics{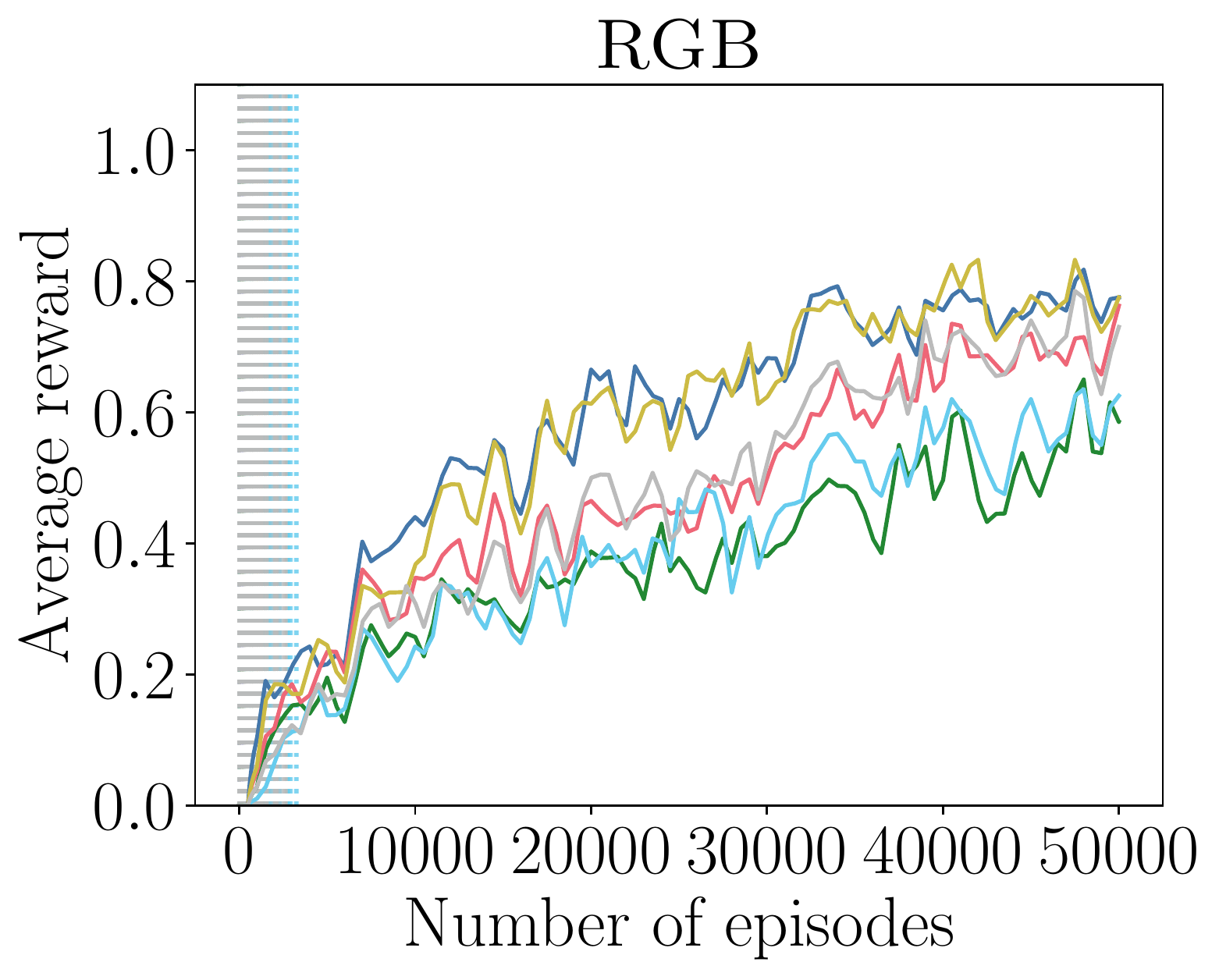}
		}
	}
	\subfloat{
		\resizebox{0.32\columnwidth}{!}{
			\includegraphics{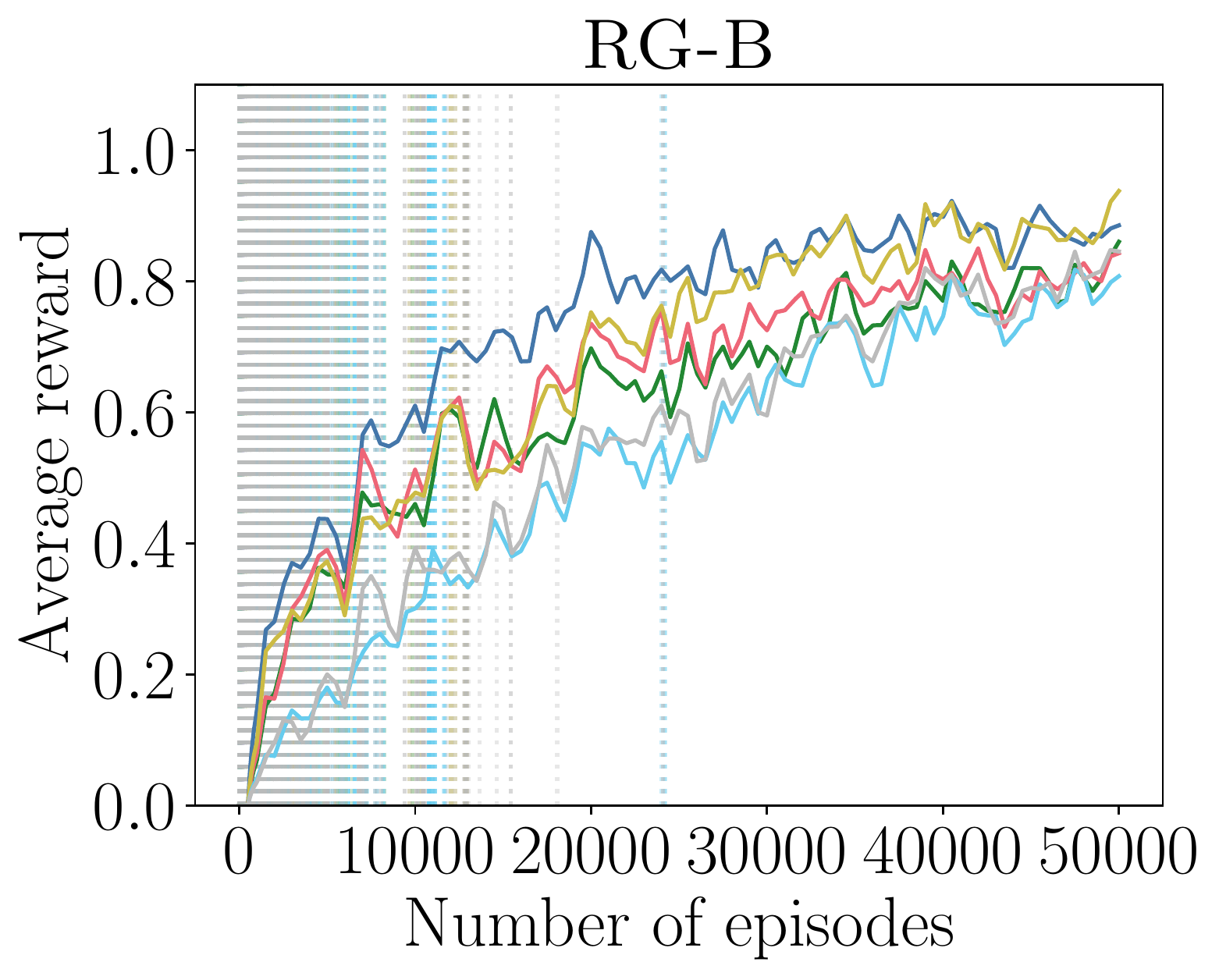}
		}
	}
	\subfloat{
		\resizebox{0.32\columnwidth}{!}{
			\includegraphics{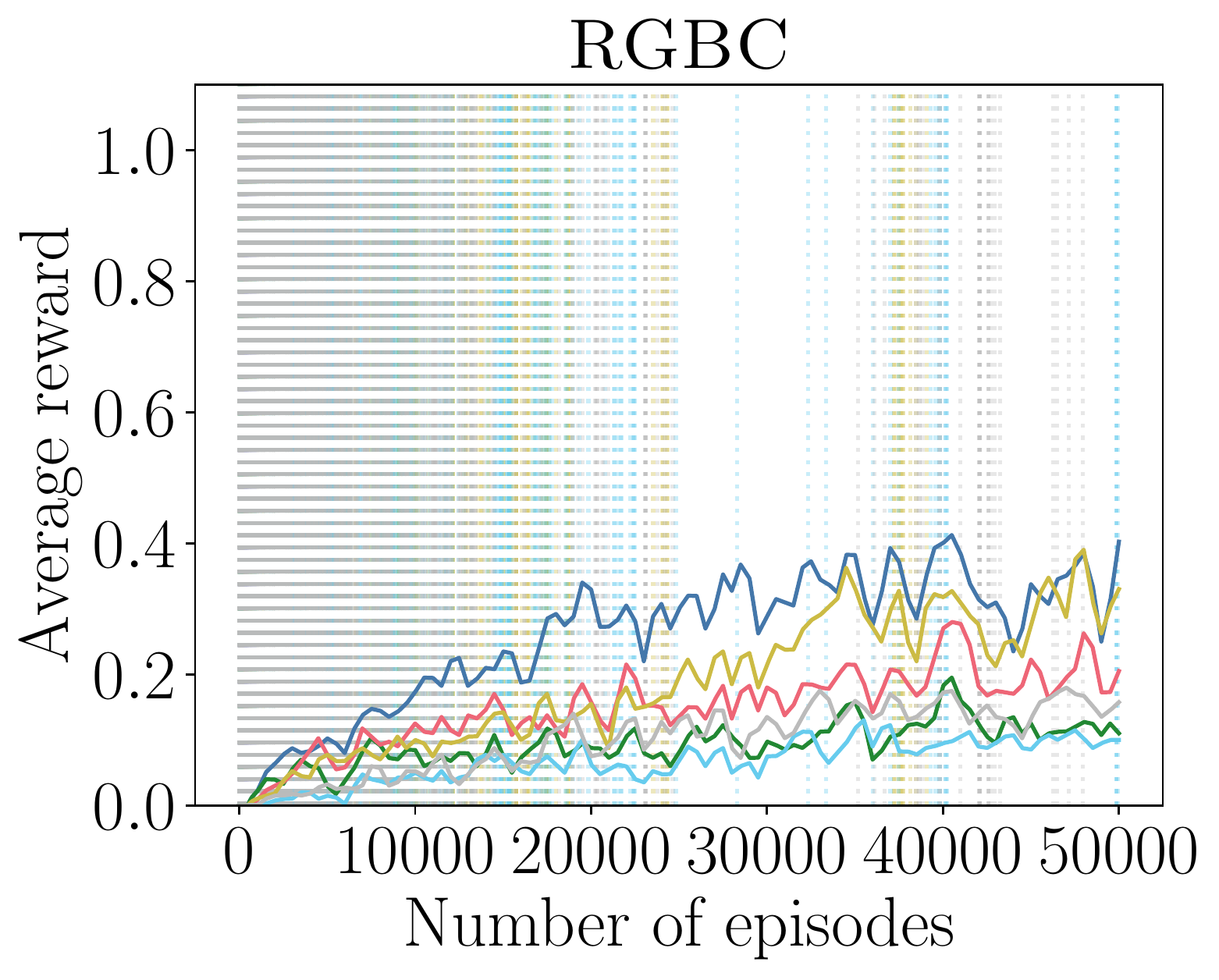}
		}
	}
	\begin{center}
		\begin{tikzpicture}
			\begin{customlegend}[legend columns=-1,legend style={column sep=1ex},legend cell align={left},legend entries={QRM,$\text{QRM}_{\min}$,$\text{QRM}_{\max}$, ISA-QRM,$\text{ISA-QRM}_{\min}$,$\text{ISA-QRM}_{\max}$}]
				\addlegendimage{pblue}
				\addlegendimage{pgreen}
				\addlegendimage{pred}
				\addlegendimage{pyellow}
				\addlegendimage{pcyan}
				\addlegendimage{pgray}
			\end{customlegend}
		\end{tikzpicture}
	\end{center}
	\caption{Learning curves for different RL algorithms in the \textsc{WaterWorld} tasks when interleaved automaton learning is off (HRL, QRM) and on (ISA-HRL, ISA-QRM).}
	\label{fig:rl_algorithm_hancrafted_waterworld}
\end{sidewaysfigure}
\subsection{Summary of the Results}
Throughout this section we have observed several commonalities between the results across domains. We provide a short summary of the main common findings below:
\begin{itemize}
	\item The higher the number of task subgoals and the number of required states are, the higher the values for the collected metrics (running time, number of examples and example length).
	\item The number of goal examples used to learn an automaton is approximately the same as the number of paths from the initial state to the accepting state. Incomplete and dead-end examples are used to refine those paths and increase as the number of subgoals and automaton states increases.
	\item Using auxiliary reward signals (or guidance) has been shown useful to speed up convergence in grid-world tasks. Importantly, as mainly shown in \textsc{OfficeWorld}, it helps to learn an automaton early in the interaction and thus reduces relearning later. This is extremely helpful in QRM since it resets all the Q-functions when a new automaton is learned. On the other hand, in \textsc{WaterWorld} there is not a big difference between approaches using guidance and those which do not use it.
	\item HRL usually converges faster than QRM, specially in the absence of guidance. In such case, QRM needs to satisfy the goal at least once to start propagating reward through the automaton states. In contrast, HRL can update the Q-functions of the formulas independently and, importantly, these can be kept  throughout the entire learning process.
	\item The difference between the approaches that learn an automaton and those that use a handcrafted one is not big except for the cases where an automaton cannot be normally learned under the imposed timeout. This shows that the learned automata properly represent the subgoal structure of the tasks.
\end{itemize}

\section{Related Work}
\label{sec:related_work}
The following sections describe the works that share more commonalities with ours. First, we qualitatively compare subgoal automata and {\methodname} with other forms of automata and automaton learning methods that have been recently used in RL. Second, we briefly discuss some work on discovering hierarchies in HRL. Finally, we describe some of the work on breaking symmetries in graphs and automata in particular.

\subsection{Automata in Reinforcement Learning}
\label{sec:reward_machines_rw}
In this section we describe others forms of automata that have been used in reinforcement learning and how their structure has been exploited. We compare these automata and the approaches for learning them to subgoal automata and {\methodname} respectively.

\paragraph{Reward Machines.} Subgoal automata are similar to another formalism that uses automata in RL called \emph{reward machines} \shortcite<RMs,>{IcarteKVM18}. Both automata consist of edges labeled by propositional formulas over a set of observables $\mathcal{O}$. The main differences with respect to our formalism are:
\begin{enumerate}
	\item RMs do not have explicit accepting and rejecting states. Therefore, they can be used to represent continuing tasks.
	\item RMs specify a reward-transition function $\delta_r: U \times U \to [S \times A \times S \to \mathbb{R}]$ that maps an automaton state pair into a reward function.
\end{enumerate}
\shortciteA{IcarteKVM18} also define \emph{simple RMs} where the reward-transition function $\delta_r:U \times U \to \mathbb{R}$ maps an automaton state pair into a reward instead of a reward function. The authors propose the QRM (Q-learning for Reward Machines) algorithm to exploit the structure of RMs, which was one of the RL algorithms we have applied on our automata.

Different methods for simultaneously learning and exploiting simple reward machines from observation traces have been recently proposed. \shortciteA{IcarteWKVCM19} introduce LRM, which formulates the RM learning problem as a discrete optimization problem and use a local search algorithm called \emph{tabu search} to solve it. They use QRM to learn the policies. Unlike our approach, LRM starts by performing random actions for a fixed number of steps to collect some traces. These traces are used for two things:
\begin{enumerate}
	\item Learn an initial automaton.
	\item Know which is the set of observables $\mathcal{O}$ of the task.\footnote{The set can be later extended if more observables are found in the following steps.} This is done by looping through the traces. Interestingly, if two events happen at the same time in these traces, then they will constitute a single observable. For example, if the traces  show that $\text{\Coffeecup}$ and $o$ happen together in \textsc{OfficeWorld}'s \textsc{Coffee} task, then the set observables $\mathcal{O}$ would include a single observable $\text{\Coffeecup}o$ instead of two distinct observables {\Coffeecup} and $o$.
\end{enumerate}
Note that these two aspects are different from our approach (ISA) because:
\begin{enumerate}
	\item ISA does not learn automata from a set of randomly collected traces. Instead, it learns a new automaton when a counterexample is found.
	\item ISA is given the set of observables $\mathcal{O}$ in advance. However, we could also determine the set of observables just like LRM does.
\end{enumerate}
Even though LRM learns an initial automaton from random traces, it uses counterexamples to refine the automata as well. LRM is also different from {\methodname} in the following aspects:
\begin{itemize}
	\item The traces used by LRM consist of observations formed by a single observable. This is due to the process for getting the set of observables $\mathcal{O}$ described above. Therefore, conjunctions are not learned but assumed to be given in $\mathcal{O}$. This can be problematic if the automaton is to be used in a task where the set of observables is different. Keeping observables separated and explicitly learning the conjunctions as we do is less constrained, but allows for better generalization. Note that the observations we use can contain an arbitrary number of observables. 
	\item LRM does not need to enforce mutual exclusivity between the edges to two different automaton states. The reason is that each observable in $\mathcal{O}$ always happens alone, as explained before. In contrast, {\methodname} needs to enforce mutual exclusivity (see Section~\ref{sec:structural_properties}).
	\item LRM's optimization scheme is used to find an automaton that is good at predicting what will be the next different observable given a maximum number of states. Thus, they do not aim to find a minimal automaton. In contrast, {\methodname} aims to find a minimal automaton that covers the example traces through the use of an iterative deepening strategy on the number of states. This has two main consequences:
	\begin{itemize}
		\item The fact that {\methodname} looks for a minimal automaton using only positive examples (i.e., traces obtained from the agent-environment interaction) makes it prone to overgeneralize in certain tasks (see Section~\ref{sec:conclusions} for a detailed discussion). LRM does not suffer from this problem given that they do not look for a minimal automaton, but an automaton that is good at predicting what will be the next different observation.
		\item Since LRM aims to be good at predicting which will be the next \emph{different} observation, it uses traces where no two consecutive observations are equal. This can be seen as a form of trace compression similar to the one we use in the experiments (see details in Appendix~\ref{app:automata_learning_restrictions}). Therefore, their method cannot be applied to tasks like counting how many times an observable has been seen in a row, while ours can handle these tasks as long as the traces are not compressed.
	\end{itemize}
	\item LRM does not classify examples into different categories. The reward machine it aims to learn can also represent continuing tasks; therefore, they do not use an explicit notion of goal or dead-end as we do.
	\item LRM does not apply a symmetry breaking mechanism and may consider different equivalent solutions during the search for an automaton. {\methodname} uses symmetry breaking constraints to shrink the search space and speed up automaton learning. 
\end{itemize}

\shortciteA{XuGAMNTW20} propose another algorithm to learn simple RMs called JIRP (Joint Inference of Reward Machines and Policies). The authors express the automaton learning problem as a SAT problem and use QRM to learn a policy for a given task. JIRP is similar to {\methodname} in the following aspects:
\begin{enumerate}
	\item It aims to learn a minimal automaton based on an iterative deepening strategy on the number of states.
	\item The automaton learner is triggered when a counterexample is found.
	\item It learns only from positive examples (i.e., attainable traces by the RL agent).
	\item It learns RMs for episodic tasks where the reward is 1 only when the goal is achieved and 0 otherwise.\footnote{Note that we assumed 0/1 reward tasks only when we used QRM and not HRL (see Section~\ref{sec:interleaved_qrm}).} Therefore, like in our case, these RMs consist of absorbing accepting and rejecting states. The rejecting state, however, seems to be induced implicitly and not indicated explicitly as in our case.
\end{enumerate}
Unlike {\methodname}, the traces used by JIRP do not only consist of observations but they also include rewards. Crucially, these sequences of rewards are used to determine counterexamples: if a trace in the environment yields a different sequence of rewards in the automaton, then that trace is a counterexample. Furthermore, there are four other differences:
\begin{itemize}
	\item JIRP does not call the automaton learner after every single counterexample. Instead, it accumulates them into a batch and calls the automaton learner periodically.
	\item JIRP learns reward machines whose edges are labeled by sets of observables, which is similar to what LRM does.\footnote{\citeauthor{XuGAMNTW20}'s paper shows RMs whose edges are labeled by propositional formulas. However, we have verified through personal communication with the authors that the transitions are indeed labeled by sets of observables. The propositional formulas were used to make the representation simpler in the paper.} In contrast, the edges of a subgoal automaton are labeled by propositional formulas over a set of observables.
	\item JIRP learns conditions for the loop transitions. In contrast, our approach takes loop transitions only when outgoing transitions to other states cannot be taken.
	\item JIRP reuses the Q-functions learned by QRM when a new automaton is learned. The Q-function at a given state $u$ is reused in a new state $u'$ if they are equivalent. Two states are equivalent if they yield the exact same sequence of rewards for all traces in the set of counterexamples. We decided not to include this mechanism in the QRM experiments for several reasons. As explained in Section~\ref{sec:interleaved_qrm}, the policy at a given automaton state selects the action that appears to be best to achieve the task's final goal. In other words, that policy might aim to satisfy any of the formulas on the outgoing edges, so it is unclear what should be transferred to the new automaton. Importantly, the transfer proposed by \citeauthor{XuGAMNTW20} is based on the sequence of rewards a trace yields from a given automaton state, and not on which subgoals the policies are trying to achieve. In our view, the latter aspect is what should be taken into account.
\end{itemize}

\shortciteA{GaonB20} learn a deterministic finite automaton (DFA), so the transitions are labeled by symbols instead of propositional formulas. However, it still shares some commonalities with reward machines and subgoal automata. The tasks they consider are also episodic and terminate when the goal is reached. The main difference with respect to subgoal automata is that the set of observables $\mathcal{O}$ contains an observable for each action. Therefore, their automata are learned from action traces. The authors use two well-known algorithms from the grammatical inference literature \shortcite{DeLaHiguera10} to learn a minimal DFA: L* and EDSM (Evidence Driven State Merging):
\begin{itemize}
	\item L*~\shortcite{Angluin87} is an active learning algorithm that can learn a DFA from just a polynomial number of queries. Typically, two types of queries are considered:
	\begin{itemize}
		\item Membership queries: the learner requests to label a trace (i.e.,~state whether the trace belongs to the language or not).
		\item Equivalence queries: the learner asks whether its automaton captures the target language. If it does not, a counterexample is returned.
	\end{itemize}
	\citeauthor{GaonB20} use the RL agent as the oracle. A membership query is answered by trying to reproduce the sequence of actions in it, whereas an equivalence query is answered by checking if the trace has been observed in the past. Note that the equivalence queries can be unfeasible traces and, thus, overgeneralization is controlled. However, their method might be prone to making wrong guesses, specially in large state and action spaces where it is unlikely that a given trace has been seen before.
	
	\item EDSM~\shortcite{LangPP98} is a state-merging approach to automaton learning, which consists of two phases:
	\begin{enumerate}
		\item Build an initial DFA called Prefix Tree Acceptor (PTA), which is a tree-structured DFA built from the prefixes of a finite set of traces such that it accepts the positive traces in the set and rejects the negative ones. \label{ref:pta_reference}
		\item Iteratively choose pairs of equivalent states to merge and produce a new automaton. If the automaton does not cover all the examples, it backtracks and chooses another pair of states. When no additional merging is possible, it stops.
	\end{enumerate}
	The automaton learned by EDSM depends on the quality of the example set and, under specific conditions, the algorithm is proved to converge to the minimal DFA. The complexity of the algorithm is polynomial in the number of examples.
	
	To apply EDSM, \citeauthor{GaonB20} keep a record of the traces that reach the goal and those that do not reach the goal. Note that this is similar to what we do, although our dead-end and incomplete traces would be both inside their set of traces that do not reach the goal because they do not have an explicit rejecting state.
	
	State-merging approaches follow a different path to minimality than our approach and \citeauthor{XuGAMNTW20}'s. The former start from a big set of states and aim to reduce it, whereas the latter begin with a small set of states and increase it when becomes insufficient to cover the examples. Nevertheless, even though these approaches are slightly different, they can both generate overgeneralized automata if only positive examples are used, as we show in Section~\ref{sec:conclusions}.
\end{itemize}
The authors combine the automaton learning approaches with Q-learning and a model-based RL algorithm called \textsc{R-max} \shortcite{BrafmanT02} and test them in tabular tasks. Note that while the minimality of our automaton is dependent on the maximum number of edges between two states ($\kappa$), the approaches by \shortciteA{XuGAMNTW20} and \shortciteA{GaonB20} do not mention this dependence (they both show examples of automata involving disjunctions).

Automaton structures have also been exploited in reward machines to give bonus reward signals. \shortciteA{CamachoIKVM19} convert reward functions expressed in various formal languages (e.g., linear temporal logic) into RMs, and propose a reward shaping method that runs value iteration on the RM states. Similarly, \shortciteA{CamachoCSM17} use automata as representations of non-Markovian rewards and exploit their structure to guide the search of an MDP planner using reward shaping. In our case, we have proposed two reward shaping mechanisms based on the maximum and minimum distances to the accepting state.

\paragraph{Hierarchical Abstract Machines (HAMs).} Throughout the paper we have shown the connection between subgoal automata and the options framework~\shortcite{SuttonPS99}, which is one of the classical approaches for HRL along with HAMs~\shortcite{ParrR97} and MAXQ~\shortcite{Dietterich00}. Even though our method is closer to options, it is also similar to HAMs in that both use an automaton. However, HAMs are non-deterministic automata whose transitions can invoke lower level machines and are not labeled by observables (the high-level policy consists in deciding which transition to fire).

\paragraph{Relational Macros.} \shortciteA{TorreySWM07} learn a kind of automata, similar to subgoal automata, called \emph{relational macros}, which are finite state machines where both states and transitions are characterized by first-order logic formulas. These formulas are built on the first-order logic predicates that describe the environment states. The formulas on the automaton states indicate which action to take, while the formulas on the transitions say when the transition is taken. The learning of the relational macros is done in two phases. The \emph{structure learning} phase finds a sequence of actions that distinguishes traces reaching the goal from those which do not and composes them into an automaton. In the \emph{ruleset learning} phase, their system learns the conditions for choosing actions and for taking transitions. The authors use Aleph~\shortcite{Srinivasan01}, an inductive-logic programming system, to learn the rules in both phases from positive and negative examples. The learned automaton is used for transfer learning: the target task follows the strategy encoded by the automaton for some steps to estimate the Q-values of the actions in the strategy, and then stops using the automaton and acts according to the Q-values. This approach differs from ours in that:
\begin{enumerate}
	\item The traces it uses to learn the automata are formed by actions and not high-level events (observables). However, similarly to us, the traces are divided into groups depending on whether they reach the goal or not.
	\item The transitions are labeled by first-order logic formulas instead of propositional formulas.
	\item It learns logic rules that describe what action to take in each automaton state, while we learn policies to choose the actions.
	\item A relational macro requires that the target task has the same action space as the source task since the rules are defined on these actions. In contrast, a subgoal automaton can be reused in another task if the set of observables and the goal are the same, even when the state and action spaces are different.
\end{enumerate}
Note that the first-order logic predicates are similar to our observables. Even though observables are propositional, both provide a high-level abstraction of the state space.

\paragraph{Policy Graphs.} \shortciteA{MeuleauPKK99} propose to represent policies with finite memory using a class of finite-state automata called policy graphs. Like subgoal automata, policy graphs are applied to POMDPs. The states of a policy graph are labeled with actions, while each edge is labeled by a single visible state.\footnote{Note that we use the terminology for POMDPs introduced in Section~\ref{sec:background}.} This differs from our case where the edges are labeled by propositional formulas over a set of high-level events (i.e., the observables). Another difference is that the transition function between states in the automaton is probabilistic. This function is represented as a parametric function and is learned through stochastic gradient descent using a set of traces obtained by the RL agent. Similarly to us, policy graphs are applied to tasks with a subset of states characterizing the tasks' goals.

\paragraph{Moore Machines.} \shortciteA{KoulFG19} transform the policy encoded by a Recurrent Neural Network (RNN) into a Moore machine, which is a quantized version of the RNN. That is, the Moore machine is defined in terms of quantized state and observation representations of the RNN. Unlike our method and the previously presented works, the authors use the resulting machine for interpretability and do not exploit its structure.

\subsection{Hierarchical Reinforcement Learning (HRL)}
In Section~\ref{sec:interleaved_automata_learning_algorithm} we described two RL algorithms for exploiting the structure of a subgoal automaton using the options framework~\shortcite{SuttonPS99}, which is one of the classical approaches for HRL along with HAMs~\shortcite{ParrR97} and MAXQ~\shortcite{Dietterich00}.

One of the core problems in the options framework is finding a set of options that helps to maximize the return instead of handcrafting such set. This problem is known as \emph{option discovery}. The method we propose in this paper, {\methodname}, can certainly be seen as an option discovery method. The family of option discovery methods where ISA fits best are \emph{bottleneck} methods, which find ``bridges'' between regions of the latent state space. In particular, each state of our automata can represent a different region of the state space, and the bottleneck is represented by the formula connecting two automaton states. The option discovery method most similar to ours, except for the reward machine related ones (see Section~\ref{sec:reward_machines_rw}), is due to \shortciteA{McGovernB01}. Their approach uses diverse density to find landmark states in state traces, and it is similar to ours because:
\begin{enumerate}
	\item It learns from traces.
	\item It classifies traces into two different categories depending on whether they achieve the goal or not.
	\item It interleaves option discovery and learning of policies for the discovered options.
\end{enumerate}
The main difference is that while our bottlenecks are propositional logic formulas, theirs are crucial states to achieve the task's goal. Therefore, they do not use/require a set of propositional events (i.e., observables) to be provided in advance.

Just like some option discovery methods \shortcite<e.g.,>{McGovernB01,StolleP02}, our approach requires the task to be solved at least once. Other methods \shortcite<e.g.,>{MenacheMS02,SimsekB04,SimsekWB05,MachadoBB17} discover options without solving the task and, thus, are also suited to continuing tasks.

Alternative formalisms to automata for expressing formal languages, like grammars, have been used to discover options. \shortciteA{LangeF19} induce a straight-line grammar, a non-branching and loop-free context-free grammar, which can only generate a single string. The authors use greedy algorithms to find a straight-line grammar from the shortest sequence of actions that leads to the goal. The production rules are then flattened, leading to one macro-action (a sequence of actions) per production rule. These macro-actions constitute the set of options.

Similarly to options, there has been work on learning the structures used in other HRL frameworks. \shortciteA{LeonettiIP12} synthesize a HAM from the set of shortest solutions to a non-deterministic planning problem, and use it to refine the choices at non-deterministic points through RL. \shortciteA{MehtaRTD08} propose a method for discovering MAXQ hierarchies from a trace that reaches the task's goal.

\subsection{Symmetry Breaking}
The symmetry breaking mechanism that we have proposed in this paper has been shown to help decrease the time needed to find a subgoal automaton that covers a set of examples. In this section we briefly mention some of the most related works to ours; that is, those addressing the problem of breaking symmetries in graphs (specifically, automata) or that use ASP to encode problems.

SAT-based approaches to learning deterministic finite automata (DFA) have used symmetry breaking constraints to shrink the search space. \shortciteA{HeuleV10} reduce the DFA learning problem to a graph coloring problem, which is translated into SAT. The graph to be colored is derived from the Prefix Tree Acceptor (PTA, see p.~\pageref{ref:pta_reference}) of the examples by connecting two states if they cannot be merged. The vertices in a $k$-clique must be colored differently and there are $k!$ different ways of coloring them. The authors propose to break these symmetries by imposing a way of assigning colors after finding a large clique using an approximation algorithm (given that the problem is NP-complete). On the other hand, similarly to us, orderings based on well-known search algorithms like breadth-first search \shortcite<BFS,>{UlyantsevZS15,ZakirzyanovMIUM19} and depth-first search \shortcite<DFS,>{UlyantsevZS16} have also been proposed. State-merging approaches to learning DFA have also used BFS to break symmetries \shortcite{LambeauDD08}. These BFS-based methods are different from ours in that they do not need to define a comparison criteria for sets of symbols since they are applied to DFA. Remember that the edges in a DFA are labeled by a single symbol, whereas the edges in a subgoal automaton might be labeled by formulas with more than one symbol. Besides, unlike previous works, we prove that the assignment of state and edge indices given by our mechanism is unique.

Given the successes of the use of symmetry breaking in SAT solving, the ASP community has also produced some work on symmetry breaking. \shortciteA{DrescherTW11} propose \textsc{sbass}, a system that detects symmetries in a ground ASP program through a reduction to a graph automorphism problem. Then, it adds constraints to the initial program to break the detected symmetries.

\shortciteA{CodishMPS19} propose a method that breaks symmetries in undirected graphs by imposing a lexicographical order in the rows of the adjacency matrix.

In our previous work \shortcite{furelosblanco2020aaai}, we introduced a method for breaking symmetries in \emph{acyclic} subgoal automata, which consists in:
\begin{enumerate}
	\item assigning an integer index to each automaton state such that $u_0$ has the lowest index and $u_A$ and $u_R$ have the highest indices; and
	\item imposing that a trace must visit automaton states in increasing order of indices.
\end{enumerate}
However, this method cannot break symmetries when there is not a trace that traverses all states in the automaton (e.g., if there are two different paths to the accepting state). We illustrate this drawback with the example below.
\begin{example}
	Figure~\ref{fig:isomorphishms} (p.~\pageref{fig:isomorphishms}) showed two automata whose states $u_1,u_2$ and $u_3$ can be used interchangeably if no symmetry breaking is used. If we assign indices $0,\ldots,3$ to states $u_0,\ldots,u_3$ and apply the symmetry breaking rule in \shortcite{furelosblanco2020aaai}, the learned automaton for \textsc{OfficeWorld}'s \textsc{VisitABCD} task will always be the one showed in Figure~\ref{fig:isomorphisms_patrol}. On the other hand, the automaton states $u_1$ and $u_2$ in Figure~\ref{fig:isomorphisms_coffee_mail} can still be switched since there is no trace that traverses both of them.
\end{example}
The method presented in this paper does not depend on the sequence of states visited by the traces. Therefore, it can break both symmetries given in the example above.

\section{Conclusions and Future Work}
\label{sec:conclusions}
In this paper we have proposed {\methodname}, a method that interleaves the learning and exploitation of an automaton whose edges encode the subgoals of an episodic goal-oriented task. The subgoals are expressed as propositional logic formulas over a set of high-level events that the agent observes when interacting with the environment. These automata are represented using a logic programming language and learned with a state-of-the-art inductive logic programming system from traces of high-level events observed by the agent. Importantly, we have devised a symmetry breaking mechanism that speeds up the automaton learning phase by avoiding to revisit equivalent solutions in the hypothesis space. Besides, the interleaving mechanism we propose ensures that the learned automata are minimal (i.e., have the fewest number of states). We have experimentally tested {\methodname} with different types of tasks, showing that it is capable of learning automata that can be exploited by existing reinforcement learning techniques. We have also shown how automaton learning affects reinforcement learning and vice versa, and that {\methodname} achieves a performance comparable to the setting where the automaton is handcrafted and given beforehand.

We now discuss possible improvements to our algorithm and directions for future work.

\paragraph{Learning from Positive and Negative Examples.} {\methodname} aims to learn a minimal subgoal automaton from positive examples only (i.e., traces that the agent can observe), which can lead to overgeneralization \shortcite{Angluin80}. More specifically, in our case, the learned automaton will be too general in tasks where there is a temporal dependency between the observables (e.g., an observable $o$ can only be observed if another observable $o'$ has been observed before). Previously, we also showed that overgeneral automata might be learned when key traces to learn all subgoals cannot be observed (see Figure~\ref{fig:overgeneralization_single_mdp}, p.~\pageref{fig:overgeneralization_single_mdp}).\footnote{Remember that we addressed this issue by learning an automaton that generalizes to a set of POMDPs.}

\begin{example}
	Imagine that the set of observables $\mathcal{O}$ for the \textsc{OfficeWorld} environment includes an observable $g$ that states whether the task's goal has been achieved. In the case of the \textsc{Coffee} task, $g$ is observed after the agent has been in the coffee location $\text{\Coffeecup}$ and then in the office location $o$; thus, a possible goal trace is $\langle\{\text{\Coffeecup}\},\{\},\{o,g\}\rangle$. The automaton learner could then output the automaton below.
	\begin{center}
		\begin{tikzpicture}[shorten >=1pt,node distance=2.06cm,on grid,auto]
		\node[state,initial] (u_0)   {$u_0$};
		\node[state,accepting] (u_acc) [right =of u_0]  {$u_{A}$};
		\path[->] (u_0) edge node {$g$} (u_acc);
		\end{tikzpicture}
	\end{center}
	This automaton is minimal and there is no other positive example that contradicts it because $g$ is only observed when the goal is reached.
	\label{example:overgeneralization}
\end{example}

To avoid overgeneralization, we need traces that are impossible to observe from the agent-environment interaction. These traces would constitute the set of \emph{negative examples}, which are supported by ILASP, the system we have used to learn the automata. For instance, the trace $\langle\{g\}\rangle$ would be a negative example for the task described in Example~\ref{example:overgeneralization}. Since it is impossible to observe such traces, the agent must be able to learn or hypothesize what is unfeasible in the environment.

To discover negative examples, our method would need to determine whether a given trace is feasible or not. We hypothesize that approaches which learn the POMDP's model (that is, the transition probability function) would be useful to solve the problem along with more sophisticated exploration strategies than $\epsilon$-greedy. 

\paragraph{Observable Discovery.} {\methodname} assumes that an appropriate set of observables is given in advance. The same occurs with other methods that learn automaton structures, like reward machines \shortcite{IcarteWKVCM19,XuGAMNTW20}. Discovering the set of observables from interaction is an important task towards automating the entire automaton learning process. Possible future work could extract observables from object keypoints in images \shortcite{KulkarniGIBRZM19} or represent the abstract states from an abstracted state space using observables.

\paragraph{Improve Scalability.} Learning a minimal automaton from examples is a well-known hard problem~\shortcite{Gold78}. The main factors that affect the scalability of our approach are the number of states, whether cycles are allowed, the observable set size,  the length of the counterexample traces, and the maximum number of edges between two states. Even though we have introduced some methods for improving scalability (symmetry breaking and trace compression), there are several interesting directions for future work:
\begin{itemize}
	\item Other trace compression methods. It has been shown that shorter traces make learning faster. More sophisticated compression techniques than the one we have presented (see Appendix~\ref{app:automata_learning_restrictions}) could be considered. However, compression can make the learning of certain automata unfeasible. Thus, the assumptions on the environments and target automata must be clearly stated.
	\item Hierarchies of automata. While learning complex automata can take some time with the current approach, learning small automata does not take long. Therefore, we can think of hierarchies of small automata that invoke each other. Furthermore, once a small automaton has been learned in one task, it might be reusable in another task.
	\item Revision of traces. Some of the collected counterexamples during learning can be long and/or contain many observables which are irrelevant to the task at hand (e.g., observing {\Letter} in \textsc{OfficeWorld}'s \textsc{Coffee} task). Thus, if at some point in the learning we realize that a counterexample is no longer needed because it is subsumed by a simpler one (i.e., shorter and without irrelevant observables), then we can replace it.
	\item Other automaton learning methods. The L* algorithm~\shortcite{Angluin87} is a query learning algorithm that can learn a DFA from just a polynomial number of membership and equivalence queries made to an oracle. One of the main criticisms of this paradigm is that it does not adapt to practical situations where there is not an oracle. However, there is recent work where Recurrent Neural Networks (RNNs) are used to play the role of the oracle \shortcite{WeissGY18}. \shortciteA{MichalenkoSVBCP19} have shown that there is a close relationship between the internal representations used by RNNs and finite state automata. Therefore, we can consider using this family of methods in the future.
\end{itemize}

\paragraph{More Expressive Automata.} A natural extension of this work is to learn automata whose edges are labeled by first-order logic formulas instead of propositional logic formulas. This would allow to enable features not currently supported by subgoal automata, such as counting (without using an arbitrary number of edges). The ILASP system can learn first-order logic rules.

\acks{The authors would like to thank the anonymous reviewers for their helpful comments and suggestions. Anders Jonsson is partially supported by the Spanish grants PCIN-2017-082 and PID2019-108141GB-I00.}

\appendix
\section{Proof of Proposition \ref{prop:asp_correctness}}
\label{proof:correctness_asp}
To prove Proposition \ref{prop:asp_correctness}, we use the following result due to \shortciteA{GelfondL88}:
\begin{theorem}
	If an ASP program $P$ is stratified, then it has a unique answer set.
	\label{theorem:stratification_unique_as}
\end{theorem}
Now, we give the definition of a stratified ASP program and proceed to prove Proposition~\ref{prop:asp_correctness}.
\begin{definition}
	An ASP program $P$ is stratified when there is a partition 
	\begin{equation*}
	P = P_0 \cup P_1 \cup \cdots \cup P_n
	\tag{$P_i$ and $P_j$ disjoint for all $i\neq j$}
	\end{equation*}
	such that, for every predicate $p$
	\begin{itemize}
		\item the definition of $p$ (all clauses with $p$ in the head) is contained in one of the partitions $P_i$
	\end{itemize}
	and, for each $1 \leq i \leq n$:
	\begin{itemize}
		\item if a predicate occurs positively in a clause of $P_i$ then its definition is contained within $\bigcup_{j \leq i} P_j$.
		\item if a predicate occurs negatively in a clause of $P_i$ then its definition is contained within $\bigcup_{j < i} P_j$.
	\end{itemize}
	\label{def:stratified_program}
\end{definition}

\propositionaspcorrectness*
\begin{proof}
	First, we prove that the program $P=M(\mathcal{A}) \cup R \cup M(\lambda_{L,\mathcal{O}}^\ast)$ has a unique answer set. By Theorem \ref{theorem:stratification_unique_as}, if $P$ is stratified then it has a unique answer set. Therefore, we show there is a possible way of partitioning $P$ following the constraints in Definition \ref{def:stratified_program}. A possible partition is $P=P_0 \cup P_1 \cup P_2 \cup P_3$, where:
	\begin{align*}
	\begin{matrix*}[c]
	P_0 = M(\lambda^\ast_{L,\mathcal{O}}), & P_1 = M(\mathcal{A}), & P_2 = R_\varphi, & P_3 = R_\delta \cup R_\mathtt{st}.
	\end{matrix*}
	\end{align*}
	Remember that $R=R_\varphi \cup R_\delta \cup R_\mathtt{st}$. The unique answer set is $AS=AS_0 \cup AS_1 \cup AS_2 \cup AS_3$, where $AS_i$ corresponds to partition $P_i$:
	\begingroup
	\allowdisplaybreaks 
	\begin{align*}
	AS_0 &= \left\lbrace\mathtt{obs}(o,t).  \mid o \in \lambda^\ast_{L,\mathcal{O}}[t], 0 \leq t \leq n \right\rbrace \cup \left\lbrace\mathtt{step}(t). \mid 0 \leq t \leq n \right\rbrace \cup \left\lbrace \mathtt{last}(n). \right\rbrace,\\\\
	AS_1 &= \begin{matrix*}[l]
	\left\lbrace\mathtt{state}(u). \mid u \in U \right\rbrace\cup  \\
	\left\lbrace\mathtt{ed}(u,u',i). \mid u, u' \in U, 1 \leq i \leq \left|\varphi(u,u')\right| \right\rbrace\cup  \\
	\left\lbrace \bar{\varphi}(u,u',i,t). \mid u,u' \in U, \text{conj}_i \in \varphi(u,u'), 0 \leq t \leq n, \lambda^\ast_{L,\mathcal{O}}[t] \not\models \text{conj}_i \right\rbrace
	\end{matrix*},\\\\
	AS_2 &= \begin{matrix*}[l]
	\left\lbrace \varphi(u,u',i,t). \mid u,u' \in U, \text{conj}_i \in \varphi(u,u'), 0 \leq t \leq n, \lambda^\ast_{L,\mathcal{O}}[t] \models \text{conj}_i \right\rbrace\cup \\
	\left\lbrace \mathtt{out}\mhyphen\varphi(u,t). \mid u \in U, 0\leq t\leq n, \exists u' \in U \text{ s.t. } \lambda^*_{L,\mathcal{O}}[t] \models \varphi(u,u') \right\rbrace
	\end{matrix*},\\\\
	AS_3&=\begin{matrix*}[l]
	\left\lbrace \delta(u,u',t). \mid u,u' \in U, 0 \leq t \leq n, \lambda^\ast_{L,\mathcal{O}}[t] \models \varphi(u,u') \right\rbrace\cup\\
	\left\lbrace \delta(u,u,t). \mid u \in U, 0 \leq t \leq n, \nexists u' \in U \text{ s.t. } \lambda^\ast_{L,\mathcal{O}}[t] \models \varphi(u,u') \right\rbrace\cup\\
	\left\lbrace \mathtt{st}(0,u_0). \right\rbrace\cup\\
	\left\lbrace \mathtt{st}(t,u). \mid 1 \leq t \leq n+1, u = \mathcal{A}(\lambda^\ast_{L,\mathcal{O}})[t]  \right\rbrace\cup\\
	\left\lbrace \mathtt{accept.} \mid \mathcal{A}(\lambda^\ast_{L,\mathcal{O}})[n+1] = u_A  \right\rbrace\cup\\
	\left\lbrace \mathtt{reject.} \mid \mathcal{A}(\lambda^\ast_{L,\mathcal{O}})[n+1] = u_R \right\rbrace
	\end{matrix*}.
	\end{align*}
	\endgroup
	
	We now prove that $\mathtt{accept} \in AS$ if and only if $\ast = G$ (i.e., the trace achieves the goal). If $\ast=G$ then, since the automaton is valid with respect to $\lambda^\ast_{L,\mathcal{O}}$ (see Definition~\ref{def:valid_trace}), the automaton traversal $\mathcal{A}(\lambda^\ast_{L,\mathcal{O}})$ finishes in the accepting state $u_A$; that is, $\mathcal{A}(\lambda^\ast_{L,\mathcal{O}})[n+1] = u_A$. This holds if and only if $\mathtt{accept} \in AS$.
	
	The proof showing that $\mathtt{reject} \in AS$ if and only if $\ast =D$ (i.e., the trace reaches a dead-end) is similar to the previous one. If $\ast=D$ then, since the automaton is valid with respect to $\lambda^\ast_{L,\mathcal{O}}$, the automaton traversal $\mathcal{A}(\lambda^\ast_{L,\mathcal{O}})$ finishes in the rejecting state $u_R$; that is, $\mathcal{A}(\lambda^\ast_{L,\mathcal{O}})[n+1] = u_R$. This holds if and only if $\mathtt{reject} \in AS$.
\end{proof}

\section{Symmetry Breaking Encodings}
\label{app:symmetry_breaking}
In this section we describe the details of our symmetry breaking mechanism. For ease of presentation, we first encode it in the form of a satisfiability (SAT) formula and formally prove several of its properties. We later explain how to convert this SAT formula into ASP rules. Finally, we propose a more efficient ASP encoding of the symmetry breaking constraints than the direct translation from SAT.

\subsection{SAT Encoding}
\label{app:sat_encoding_labeled_directed_graph}
Our idea is to define a SAT formula that encodes a BFS traversal $\mathcal{I}(G)=\langle f, \{\Gamma_u\}_{u\in V} \rangle$ of a given graph $G=(V,E)$ in the class $\mathcal{G}$, defined on a set of labels $\mathcal{L}=\{l_1,\ldots,l_k\}$. Since a graph indexing $\mathcal{I}(G)$ assigns unique integers to nodes and edges, we use $i$ to refer to a node $u$ such that $f(u)=i$, $(i,e)$ or $(i,e,L)$ to refer to an edge $(u,v,L)$ such that $f(u)=i$ and $\Gamma_u(L)=e$, and $m$ to refer to a label $l_m \in \mathcal{L}$. We sometimes extend this notation in the natural way, e.g.~by writing $\Gamma_i(L)$, $E^o(i)$ and $\Pi_\mathcal{I}(i)$.

\paragraph{Variables.} We first define a set $\mathcal{X}$ of propositional SAT variables for all combinations of symbols (sometimes with restrictions as indicated):\\

\noindent
\begin{tabular}{l@{\hspace{5pt}}l@{\hspace{5pt}}l@{\hspace{5pt}}l}
	1. & $ed(i,j,e)$, & & [edge $(i,e)$ ends in node $j$]\\
	2. & $label(i,e,m)$, & & [the label set on edge $(i,e)$ includes $l_m\in\mathcal{L}$]\\
	3. & $pa(i,j)$, & $i<j$, & [node $i$ is the parent of $j$ in the BFS subtree]\\
	4. & $sm(i,j,e)$, & $i<j$, & [$e$ is the smallest integer on a BFS edge from node $i$ to $j$]\\
	5. & $lt(i,e-1,e,m)$, & $e>1$, & [there is a label $l_{m'\mid m'\leq m}$ on $(i,e)$ and not on $(i,e-1)$]
\end{tabular}\\

\noindent
Intuitively, variables $ed(i,j,e)$ and $label(i,e,m)$ are used to encode a graph $G$ together with an associated graph indexing $\mathcal{I}(G)$, variables $pa(i,j)$ and $sm(i,j,e)$ are used to encode the parent function $\Pi_\mathcal{I}$, and variables $lt(i,e-1,e,m)$ are used to encode the label set ordering.

\paragraph{Clauses.} We next define a set $\mathcal{C}$ of clauses on $\mathcal{X}$ for all combinations of symbols (sometimes with restrictions as indicated). The first set of clauses (1-8) enforces the first condition in Definition~\ref{def:bfs}: for any two nodes $i>1$ and $j>1$, $\Pi_\mathcal{I}(i)<\Pi_\mathcal{I}(j) \Leftrightarrow i<j$.\\

\noindent
\begin{tabular}{l@{\hspace{5pt}}l@{\hspace{5pt}}l@{\hspace{5pt}}l}
	1. & $\bigvee_{i\mid i<j} pa(i,j)$, & $j>1$, & [node $j>1$ has incoming BFS edge]\\
	2. & $pa(i,j) \Rightarrow \neg pa(i',j)$, & $i<i'<j$, & [incoming BFS edge is unique]\\
	3. & $pa(i,j) \Rightarrow \bigvee_e sm(i,j,e)$, & $i<j$, & [BFS edge implies smallest integer]\\
	4. & $pa(i,j) \Rightarrow \neg ed(i',j',e)$, & $i'<i<j\leq j'$, & [respect BFS order]\\
	5. & $sm(i,j,e) \Rightarrow pa(i,j)$, & $i<j$, & [smallest integer implies BFS edge]\\
	6. & $sm(i,j,e) \Rightarrow \neg sm(i,j,e')$, & $i<j$, $e<e'$, & [smallest integer is unique]\\
	7. & $sm(i,j,e) \Rightarrow ed(i,j,e)$, & $i<j$, & [smallest integer implies edge]\\
	8. & $sm(i,j,e) \Rightarrow \neg ed(i,j',e')$, & $i<j\leq j'$, $e'<e$, & [correctly break ties]
\end{tabular}\\

\noindent
Intuitively, Clauses 1 and 2 state that each node $j>1$ has a unique parent node $i$ in the BFS subtree. Clauses 3, 5 and 6 state that each node $j>1$ has a unique incoming edge $(i,e)$ with smallest integer $e$ from its parent $i$ in the BFS subtree. Clause 7 ensures that the incoming edge $(i,e)$ to $j$ in the BFS subtree corresponds to an actual edge in the graph $G$. 

Clauses 4 and 8 constitute the core of symmetry breaking by enforcing the condition that $\Pi_\mathcal{I}(i)<\Pi_\mathcal{I}(j)$ should imply $i<j$. By definition of $\Pi_\mathcal{I}$, the incoming edge $(i,e)$ to $j$ in the BFS subtree should be the lexicographically smallest such integer pair. Hence the graph $G$ cannot contain any incoming edge $(i',e')$ to $j$ from a node $i'<i$. In addition, no node $j'>j$ can have such an incoming edge either, since otherwise its parent function would be smaller than that of $j$, thus violating the desired condition. These two facts are jointly encoded in Clause 4 by enforcing the restriction $j'\geq j$.

Likewise, if $(i,e)$ is the incoming edge to $j$ in the BFS subtree, graph $G$ cannot contain an incoming edge $(i,e')$ from the same node $i$ with $e'<e$. Again, no node $j'>j$ can have such an incoming edge either, since otherwise its parent function would be smaller than that of $j$. These two facts are jointly encoded in Clause 8 by enforcing the restriction $j'\geq j$.

The second set of clauses (9-14) assigns edge integers to the outgoing edges from each node, enforcing the second condition in Definition~\ref{def:bfs}: for each node $i$ and pair of outgoing edges $(i,e,L)$ and $(i,e',L')$, $L<L' \Leftrightarrow e<e'$. Due to the transitivity of the relation $<$, it is sufficient to check that the condition holds for all pairs of consecutive edge integers $(e-1,e)$. Clauses 9 and 10 enforce that edge integers are unique between 1 and $|E^o(i)|$.\\

\noindent
\begin{tabular}{l@{\hspace{5pt}}l@{\hspace{5pt}}l@{\hspace{5pt}}l}
	9. & $ed(i,j,e) \Rightarrow \bigvee_{j'} ed(i,j',e-1)$, & $e>1$, & [edge integers start at 1 and are contiguous]\\
	10. & $ed(i,j,e) \Rightarrow \neg ed(i,j',e)$, & $j < j'$, & [edge integers cannot be duplicated]
\end{tabular}\\

\noindent
Clauses 11-14 are used to enforce that two consecutive edges $(i,e-1,L)$ and $(i,e,L')$ satisfy $L<L'$. Formally, variable $lt(i,e-1,e,m)$ is only true if there exists $m'\leq m$ such that $l_{m'} \notin L$ and $l_{m'} \in L'$. This is implemented using the following two clauses:\\

\noindent
\begin{tabular}{l@{\hspace{5pt}}l@{\hspace{5pt}}l@{\hspace{5pt}}l}
	11. & $lt(i,e-1,e,m) \Rightarrow \neg label(i,e-1,m) \vee lt(i,e-1,e,m-1)$, & $e>1$,\\
	12. & $lt(i,e-1,e,m) \Rightarrow label(i,e,m) \vee lt(i,e-1,e,m-1)$, & $e>1$.
\end{tabular}\\

\noindent Hence if $lt(i,e-1,e,m)$ holds, either $l_m \notin L$ and $l_m \in L'$, or $lt(i,e-1,e,m-1)$ holds for $m-1$. The disjuncts mentioning $m-1$ are only evaluated when $m>1$. The next clause ensures that for each edge $(i,e)$ with $e>1$, $lt(i,e-1,e,m)$ is true for at least one label $l_m \in \mathcal{L}$:\\

\noindent
\begin{tabular}{l@{\hspace{5pt}}l@{\hspace{5pt}}l@{\hspace{5pt}}l}
	13. & $ed(i,j,e) \Rightarrow \bigvee_m lt(i,e-1,e,m), \;\; e>1$.
\end{tabular}\\

\noindent Finally, the following clause encodes the second part of Definition~\ref{def:observation_ordering}, ensuring that the label set on edge $(i,e)$ is \emph{not} lower than that on $(i,e-1)$:\\

\noindent
\begin{tabular}{l@{\hspace{5pt}}l@{\hspace{5pt}}l@{\hspace{5pt}}l}
	14. & $lt(i,e-1,e,m) \vee \neg label(i,e-1,m) \vee label(i,e,m), \;\; e>1$.
\end{tabular}\\


\paragraph{Properties.} We proceed to prove several properties about the SAT encoding. Concretely, we show that there is a one-to-one correspondence between the BFS traversal of a graph and a solution to the SAT encoding.

\begin{definition}
	Given a graph $G=(V,E)\in\mathcal{G}$ defined on a set of labels $\mathcal{L}=\{l_1,\ldots,l_k\}$ and an associated graph indexing $\mathcal{I}(G)=\langle f,\{\Gamma_u\}_{u\in V}\rangle$, let $X(G,\mathcal{I})$ be an assignment to the SAT variables in $\mathcal{X}$, assigning false to all variables in $\mathcal{X}$ except as follows:
	\begin{itemize}
		\item For each edge $(u,v,L)\in E$, $ed(f(u),f(v),\Gamma_u(L))$ is true.
		\item For each edge $(u,v,L)\in E$ and each label $l_m\in L$, $label(f(u),\Gamma_u(L),m)$ is true.
		\item For each node $v\in V\setminus\{v_1\}$ with $\Pi_\mathcal{I}(v)=(i,e)$, $pa(i,f(v))$ and $sm(i,f(v),e)$ are true.
		\item For each node $u\in V$, each pair of outgoing edges $(u,v,L)$ and $(u,v',L')$ in $E^o(u)$ such that $\Gamma_u(L)=\Gamma_u(L')-1$, and each label $l_m\in\mathcal{L}$, $lt(f(u),\Gamma_u(L),\Gamma_u(L'),m)$ is true if there exists $m'\leq m$ such that $l_{m'}\notin L$ and $l_{m'}\in L'$.
	\end{itemize}
	\label{def:graph_sat_assignment}
\end{definition}

\begin{example}
	Given the graph $G$, graph indexing $\mathcal{I}(G)$ and set of labels $\mathcal{L}=\{a,b,c,d,e,f\}$ from Example~\ref{ex:labeled_graph_with_edge_indices} (p.~\pageref{ex:labeled_graph_with_edge_indices}), the assignment $X(G,\mathcal{I})$ assigns true to the following SAT variables in $\mathcal{X}$:
	\begin{equation*}
	\begin{Bmatrix*}[l]
	ed(1,2,1) & label(1,1,2) & label(1,1,5) & ed(1,3,2) & label(1,2,1) & label(1,2,6)\\
	ed(1,4,3) & label(1,3,1) & label(1,3,2) & ed(2,4,1) & label(2,1,1) \\
	ed(3,4,1) & label(3,1,2) & ed(4,5,1) & label(4,1,4) & ed(4, 5, 2) & label(4,2,3)\\
	pa(1,2) & pa(1,3) & pa(1,4) & pa(4,5) \\
	sm(1,2,1) & sm(1,3,2) & sm(1,4,3) & sm(4,5,1)\\
	lt(1,1,2,1) & lt(1,1,2,2) & lt(1,1,2,3) & lt(1,1,2,4) & lt(1,1,2,5) & lt(1,1,2,6) \\
	lt(1,2,3,2) & lt(1,2,3,3) & lt(1,2,3,4) & lt(1,2,3,5) & lt(1,2,3,6)\\
	lt(4,1,2,3) & lt(4,1,2,4) & lt(4,1,2,5) & lt(4,1,2,6)
	\end{Bmatrix*}
	\end{equation*}
\end{example}

\begin{theorem}\label{thm:bfssat}
	Given a graph $G$ and a graph indexing $\mathcal{I}(G)$, the assignment $X(G,\mathcal{I})$ to the SAT variables in $\mathcal{X}$ satisfies all SAT clauses in $\mathcal{C}$ if and only if $\mathcal{I}(G)$ is a BFS traversal.
\end{theorem}

\begin{proof}
	$\Leftarrow$: Assume that $\mathcal{I}(G)$ is a BFS traversal (i.e., it satisfies the conditions of Definition~\ref{def:bfs}, p.~\pageref{def:bfs}). We show that each clause in $\mathcal{C}$ is satisfied:
	\begin{enumerate}
		\item $\bigvee_{i\mid i<j}pa(i,j)$ holds for $v$ such that $f(v)=j>1$ since $pa(i,j)$ is true for $\Pi_\mathcal{I}(v)=(i,e)$.
		\item $pa(i,j) \Rightarrow \neg pa(i',j)$ holds since $pa(i,j)$ is true for a single $i$.
		\item $pa(i,j) \Rightarrow \bigvee_e sm(i,j,e)$ holds for $v$ such that $f(v)=j$ since $pa(i,j)$ and $sm(i,j,e)$ are true for $\Pi_\mathcal{I}(v)=(i,e)$.
		\item $pa(i,j) \Rightarrow \neg ed(i',j',e)$ holds for $v$ such that $f(v)=j$ and $i'<i<j=j'$ since $\Pi_\mathcal{I}(v)=(i,e)$ is the lexicographically smallest integer pair on incoming edges to $v$, implying that $G$ cannot contain an edge to $v$ from a node $u$ with $f(u)=i'<i$. Moreover, since $\mathcal{I}(G)$ is a BFS traversal, we cannot have $\Pi_\mathcal{I}(w)<\Pi_\mathcal{I}(v)$ for a node $w$ with $f(w)=j'>j$, else the second condition in Definition~\ref{def:bfs} is violated. Hence $G$ cannot contain an edge to $w$ from a node $u$ with $f(u)=i'<i$, so the clause also holds for $i'<i<j<j'$.
		\item $sm(i,j,e) \Rightarrow pa(i,j)$ holds for $v$ such that $f(v)=j$ since $pa(i,j)$ and $sm(i,j,e)$ are true for $\Pi_\mathcal{I}(v)=(i,e)$.
		\item $sm(i,j,e) \Rightarrow \neg sm(i,j,e')$ holds since $sm(i,j,e)$ is true for a single $i$ and $e$.
		\item $sm(i,j,e) \Rightarrow ed(i,j,e)$ holds for $v$ such that $f(v)=j$ since $\Pi_\mathcal{I}(v)=(i,e)$ implies that there exists an edge $(u,v,L)\in E$ with $f(u)=i$ and $\Gamma_u(L)=e$.
		\item $sm(i,j,e) \Rightarrow \neg ed(i,j',e')$ holds for $v$ such that $f(v)=j$ and $j=j'$ since $\Pi_\mathcal{I}(v)=(i,e)$ is the lexicographically smallest integer pair on incoming edges to $v$, implying that $G$ cannot contain an edge $(u,v,L)$ with $f(u)=i$ and $\Gamma_u(L)<e$. For $j<j'$, the parent function of the node $w$ with $f(w)=j'$ cannot be smaller than that of $v$, else the second condition in Definition~\ref{def:bfs} is violated. Hence $G$ cannot contain an edge $(u,w,L)$ with $f(u)=i$ and $\Gamma_u(L)<e$. Thus the clause also holds for the case $j<j'$.
		\item $ed(i,j,e) \Rightarrow \bigvee_{j'} ed(i,j',e-1)$ holds for $u$ with $f(u)=i$ since $\Gamma_u$ is a bijection onto $\{1,\ldots,|E^o(u)|\}$.
		\item $ed(i,j,e) \Rightarrow \neg ed(i,j',e)$ holds for $u$ with $f(u)=i$ since $\Gamma_u$ is a bijection.
		\item $lt(i,e-1,e,m) \Rightarrow \neg label(i,e-1,m) \vee lt(i,e-1,e,m-1)$ holds for $u$, $f(u)=i$, and outgoing edges $(u,v,L)$, $(u,v',L')$ with $\Gamma_u(L)=e-1=\Gamma_u(L')-1$ since $lt(i,e-1,e,m)$ implies that either $l_m\notin L$ or there exists $m'<m$ such that $l_{m'}\notin L$ and $l_{m'}\in L'$.
		\item $lt(i,e-1,e,m) \Rightarrow label(i,e,m) \vee lt(i,e-1,e,m-1)$ holds for the same setting since $lt(i,e-1,e,m)$ implies that either $l_m\in L'$ or there exists $m'<m$ such that $l_{m'}\notin L$ and $l_{m'}\in L'$.
		\item $ed(i,j,e) \Rightarrow \bigvee_{j'} ed(i,j',e-1)$ holds since $\mathcal{I}(G)$ is a BFS traversal, implying that the bijection $\Gamma_u$ for $u$ with $f(u)=i$ satisfies $L<L'$ whenever $\Gamma_u(L)<\Gamma_u(L')$ (and in particular when $\Gamma_u(L)=e-1=\Gamma_u(L')-1$). Since $L<L'$, there has to exist at least one $m,1\leq m\leq k$ such that $l_m\notin L$ and $l_m\in L'$ due to Definition~\ref{def:observation_ordering}.
		\item $lt(i,e-1,e,m) \vee \neg label(i,e-1,m) \vee label(i,e,m)$ also holds since $\Gamma_u(L)=e-1=\Gamma_u(L')-1$ implies $L<L'$. Hence for the given $m$, Definition~\ref{def:observation_ordering} is satisfied either 1) by a label $l_{m'\leq m}$, implying that $lt(i,e-1,e,m)$ is true; or 2) by a label $l_{m'>m}$, implying that $l_m\in L$ and $l_m\notin L'$ cannot both be true.
	\end{enumerate}
	
	\noindent
	$\Rightarrow$: Assume that $X(G,\mathcal{I})$ satisfies all SAT clauses. We show that $\mathcal{I}(G)$ is a BFS traversal. First note from above that $X(G,\mathcal{I})$ satisfies all clauses except 4, 8, 13 and 14 even if $\mathcal{I}(G)$ is not a BFS traversal. Hence we can focus exclusively on these four clauses.
	
	We first analyze Clauses 13 and 14. For Clause 13 to be true, any edge $(i,e)$ induced from graph $G$ with $e>1$ has to satisfy $lt(i,e-1,e,m)$ for at least one label $l_m\in\mathcal{L}$. Let $u$ be the node with $f(u)=i$ and let $(u,v,L)$ and $(u,v',L')$ be the two outgoing edges in $E^o(u)$ such that $\Gamma_u(L)=e-1=\Gamma_u(L')-1$. By definition of $X(G,\mathcal{I})$ there exists $m,1\leq m\leq k$ such that $l_m\notin L$ and $l_m\in L'$. For the smallest such $m$ there cannot exist $m'<m$ such that $l_{m'}\in L$ and $l_{m'}\notin L'$, else Clause 14 would be violated for $m'$. Hence Definition~\ref{def:observation_ordering} holds for label $l_m$, implying $L<L'$.
	
	We next analyze Clauses 4 and 8. Let $u$ be the node such that $f(u)=j$ and $\Pi_\mathcal{I}(u)=(i,e)$, and let $v$ be any node such that $f(v)=j'>j$. Since Clause 4 holds, graph $G$ cannot contain an edge $(w,v,L)$ such that $f(w)=i'<i$. Since Clause 8 holds, graph $G$ cannot contain an edge $(w,v,L)$ such that $f(w)=i$ and $\Gamma_w(L)<e$. Since $\Pi_\mathcal{I}(v)$ cannot equal $(i,e)$, it has to be larger than $(i,e)$, implying $\Pi_\mathcal{I}(u)<\Pi_\mathcal{I}(v)$.
	
	We have shown that the two conditions in Definition~\ref{def:bfs} hold: for each pair of nodes $u$ and $v$ in $V\setminus\{v_1\}$, $\Pi_\mathcal{I}(u)<\Pi_\mathcal{I}(v)$ implies $f(u) < f(v)$, and for each node $u\in V$ and pair of outgoing edges $(u,v,L)$ and $(u,v',L')$ in $E^o(u)$, $L<L'$ implies $\Gamma_u(L)<\Gamma_u(L')$. Hence by definition $\mathcal{I}(G)$ is a BFS traversal.
\end{proof}

\begin{definition}
	Let $X$ be an assignment to the SAT variables $\mathcal{X}$ that satisfies the SAT clauses in $\mathcal{C}$. Given an edge $(i,e)$, let $L(i,e)=\{l_m\mid label(i,e,m)\}$ be the label set induced by $X$. We define a mapping $\mathcal{G}(X)=(G,\mathcal{I}(G))$ from $X$ to a graph $G=(V,E)$ and associated graph indexing $\mathcal{I}(G)=\langle f, \{\Gamma_u\}_{u\in V} \rangle$ as follows:
	\begin{itemize}
		\item The set of nodes is $V=\{v_1,\ldots,v_n\}$, where $n$ is the largest node index in the assignment, and $f(v_i)=i$ for each $v_i\in V$.
		\item The set of edges is $E=\{(v_i,v_j,L(i,e)) \mid ed(i,j,e)\}$ and $\Gamma_{v_i}(L(i,e))=e$.
		\item The parent function of each $v_j\in V$ equals $\Pi_\mathcal{I}(v_j)=(i,e)$ if $sm(i,j,e)$ is true.
	\end{itemize}
\end{definition}

\begin{theorem}\label{thm:satgraph}
	Given an assignment $X$ to the SAT variables $\mathcal{X}$ that satisfies all clauses in $\mathcal{C}$, the mapping $\mathcal{G}(X)=(G,\mathcal{I}(G))$ induces a graph $G$ in the class $\mathcal{G}$ and a well-defined graph indexing $\mathcal{I}(G)$.
\end{theorem}

\begin{proof}	
	We first show that the induced graph indexing $\mathcal{I}(G)$ is well-defined. Clearly $f$ is a bijection onto $\{1,\ldots,|V|\}$ by definition. We next show that $\Gamma_{v_i}$ is a bijection onto $\{1,\ldots,|E^o(v_i)|\}$ for each node $v_i\in V$. Clause 10 ensures that $ed(i,j,e)$ and $ed(i,j',e)$ cannot be true simultaneously for $j\neq j'$. Due to Clause 9, if edge $(i,e)$ is defined for $e>1$, then so is $(i,e-1)$. Applying this argument recursively implies that $(i,e)$ is uniquely defined for $e\in\{1,\ldots,|E^o(v_i)|\}$, where $|E^o(v_i)|$ is the largest integer of an outgoing edge from $i$.
	
	We next show that the induced label set $L(i,e)$ on each outgoing edge $(i,e)$ from $i$ is unique. Clause 13 implies that for each edge $(i,e)$ with $e>1$, $lt(i,e-1,e,m)$ is true for at least one label $l_m\in\mathcal{L}$. Clauses 11 and 12 ensure that $lt(i,e-1,e,m)$ is true only if there exists $m'\leq m$ such that $\neg label(i,e-1,m')$ and $label(i,e,m')$ are true. For the smallest such $m'$ there cannot exist $m''<m'$ such that $label(i,e-1,m'')$ and $\neg label(i,e,m'')$ are true, else Clause 14 would be violated for $m''$. Hence label $l_m$ satisfies the condition in Definition~\ref{def:observation_ordering} with respect to the induced label sets $L(i,e-1)$ and $L(i,e)$, implying $L(i,e-1)<L(i,e)$.
	
	In conclusion, we have shown that $(i,e)$ is uniquely defined for $e\in\{1,\ldots,|E^o(v_i)|\}$, and that $L(i,e-1)<L(i,e)$ holds for each pair of consecutive integers in $\{1,\ldots,|E^o(v_i)|\}$. Since $\Gamma_{v_i}$ is defined as $\Gamma_{v_i}(L(i,e))=e$ for each $e\in\{1,\ldots,|E^o(v_i)|\}$, this implies that $\Gamma_{v_i}$ is a well-defined bijection from $E^o(v_i)$ to $\{1,\ldots,|E^o(v_i)|\}$.
	
	We also need to show that the induced parent function $\Pi_\mathcal{I}$ is well-defined, i.e.~that for each $j>1$, $sm(i,j,e)$ is true for a single $i$ and $e$, and that $\Pi_\mathcal{I}(v_j)=(i,e)$ is consistent with the definition of $\Pi_\mathcal{I}$. Clauses 1 and 2 imply that $pa(i,j)$ is true for a single $i$. Clauses 3, 5 and 6 imply that $sm(i,j,e)$ can only be true for the same $i$ and $j$ as $pa(i,j)$, and that $sm(i,j,e)$ is true for a single $e$. For $\Pi_\mathcal{I}(v_j)=(i,e)$ to hold, $G$ has to contain the edge $(v_i,v_j,L(i,e))$, which is guaranteed by Clause 7. Moreover, $G$ cannot contain any edge $(v_{i'},v_j,L(i',e'))$ such that $(i',e')<(i,e)$, which is guaranteed by Clauses 4 and 8.
	
	We finally show that the induced graph $G$ belongs to the class $\mathcal{G}$, i.e.~that the three assumptions on page~\pageref{assump:sym_unique_label_sets} are satisfied. We satisfy Assumption~\ref{assump:sym_starting_node} by designating $v_1$ as the start node. We have already shown above that the induced label set $L(i,e)$ on each outgoing edge $(i,e)$ from $i$ is unique, satisfying Assumption~\ref{assump:sym_unique_label_sets}. It remains to show that Assumption~\ref{assump:sym_reachability} holds, i.e.~that each node $v_j$, $j>1$, is reachable from $v_1$. Since $sm(i,j,e)$ is true for a single $i$ and $e$ such that $i<j$, $G$ has to contain the edge $(v_i,v_j,L(i,e))$ due to Clause 7. Aggregating these incoming edges for all nodes different from $v_1$ results in a BFS subtree rooted in $v_1$, and each node $v_j$, $j>1$, is reachable from $v_1$ in this subtree.
\end{proof}

By combining Theorems~\ref{thm:bfssat} and \ref{thm:satgraph}, it follows that the mapping $\mathcal{G}(X)=(G,\mathcal{I}(G))$ of a satisfying assignment $X$ to the SAT clauses in $\mathcal{C}$ is such that $\mathcal{I}(G)$ is a BFS traversal. Since by Lemma~\ref{lemma:unique} (p.~\pageref{lemma:unique}) each graph $G\in\mathcal{G}$ has a unique BFS traversal, it follows that the SAT encoding cannot generate two permutations of node integers that represent the same graph $G$. Hence the SAT encoding breaks the symmetries in graphs such as those in Figure~\ref{fig:isomorphishms} (p.~\pageref{fig:isomorphishms}).

We remark that the SAT encoding is not forced to include edges in the induced graph $G$ other than those needed to correctly represent the parent function $\Pi_\mathcal{I}$. However, by combining the SAT encoding with another encoding for generating automata, we ensure that the encoding cannot produce two symmetric automata.

\subsection{ASP Encodings}
\label{app:asp_encodings}
In this section we present two different encodings of our symmetry breaking method in ASP for its application to subgoal automata. The first method (Appendix~\ref{app:sat_based_asp_encoding}) is a direct translation from the SAT clauses introduced in the previous section. However, there are certain aspects that can be made more efficient in ASP, so we propose a second method for encoding the symmetry breaking constraints (Appendix~\ref{app:alternative_asp_encoding}).

In order to apply the proposed symmetry breaking mechanism to a subgoal automaton, we have to take the following aspects into account:
\begin{enumerate}
	\item The edges of a subgoal automaton are labeled by propositional formulas over a set of observables $\mathcal{O}$, whereas the edges of a labeled directed graph are defined over a set of labels $\mathcal{L}$.
	\item The graph indexing presented for labeled directed graphs assigns a different edge index for each of the outgoing edges from a node. In contrast, in the ASP representation $M(\mathcal{A})$ of a subgoal automaton $\mathcal{A}$ the edge indices are unique only from one state to another (i.e., can be repeated between other pairs of states).
\end{enumerate}

We now partially describe how we address (1), which is common to both methods. We map the set of observables $\mathcal{O}$ used by the subgoal automaton into a set of labels $\mathcal{L}$, which consists of integer values for easy comparison. Each of these integer values must encode either an observable or its negation. Given a set of observables $\mathcal{O}=\{o_1,\ldots, o_{|\mathcal{O}|}\}$, the set of labels $\mathcal{L}$ is defined as:
\begin{equation*}
\mathcal{L} = \left\lbrace i, i+|\mathcal{O}| \mid o_i \in \mathcal{O}  \right\rbrace,
\end{equation*}
where $i$ corresponds to $o_i$ and $i + |\mathcal{O}|$ corresponds to its negation, $\neg o_i$. Therefore, $\mathcal{L}$ consists of integers from 1 to $2|\mathcal{O}|$ where labels $1\ldots|\mathcal{O}|$ correspond to each of the observables, and labels $|\mathcal{O}|+1\ldots 2|\mathcal{O}|$ correspond to the observable negations. This mapping is encoded in ASP using the following predicates:
\begin{itemize}
	\item $\mathtt{obs\_id}(o,i)$ indicates that observable $o\in \mathcal{O}$ is assigned id $i$.
	\item $\mathtt{num\_obs}(i)$ indicates that the set of observables has size $i$.
	\item $\mathtt{valid\_label}(l)$ indicates that $l$ is a label.
\end{itemize}
We simply ground the above predicates according to their descriptions:
\begin{equation*}
\left\lbrace \mathtt{obs\_id}(o_i, i). \mid o_i \in \mathcal{O} \right\rbrace \cup \left\lbrace \mathtt{num\_obs}(|\mathcal{O}|). \right\rbrace \cup \left\lbrace\mathtt{valid\_label}(l). \mid 1 \leq l \leq 2|\mathcal{O}|)\right\rbrace.
\end{equation*}
The mapping is not complete: we need to map the formulas on the edges into sets of labels. Specifically, we map the factual representation of the automaton introduced in Section~\ref{sec:structural_properties} into $\mathtt{label}$ facts similar to the variables of the same name used in the SAT encoding. However, these facts are slightly different between encodings, so we describe them in detail in their respective sections. The second aspect commented above is also addressed differently for each encoding.

Another common feature between the encodings is that states in the automaton are assigned an integer index for easy comparison. This is needed only for the rules enforcing the BFS traversal on the automaton, which corresponds to Clauses~1-8 in the SAT encoding. The predicate $\mathtt{state\_id}(u,i)$ denotes that the automaton state $u$ has index $i$ and its ground instances are:
\begin{equation*}
\left\lbrace \mathtt{state\_id}(u_i,i). \mid u_i \in U \setminus \left\lbrace u_A, u_R \right\rbrace \right\rbrace.
\end{equation*}
Note that, without loss of generality, the first state index is 0 (because of the initial state $u_0$)\footnote{Remember that in Section~\ref{sec:symmetry_breaking} the bijection $f$ assigns indices starting from 1.} and that not all states are assigned an index. Given that the accepting state $u_A$ and the rejecting state $u_R$ are fixed, they cannot be interchanged with any other state. Furthermore, they cannot be the parent of any other state since they are absorbing states. Thus, they are excluded from the BFS ordering. In order to easily compare the indices of two states, we introduce the set of rules below. The first rule defines the predicate $\mathtt{state\_lt}(u,u')$ to express that the index of state $u$ is lower than that of $u'$. Similarly, the second rule defines the predicate $\mathtt{state\_leq}(u, u')$ to express that the index of state $u$ is lower or equal than that of $u'$.
\begin{equation*}
\begin{Bmatrix*}[l]
\mathtt{state\_lt(X, Y) \codeif state\_id(X, XID), state\_id(Y, YID), XID\mathord{<}YID.}\\
\mathtt{state\_leq(X, Y) \codeif state\_id(X, XID), state\_id(Y, YID), XID\mathord{<=}YID.}
\end{Bmatrix*}
\end{equation*}

\subsubsection{SAT-Based Encoding}
\label{app:sat_based_asp_encoding}
The encoding we present in this section is a direct translation from the SAT encoding introduced in Appendix~\ref{app:sat_encoding_labeled_directed_graph}. The presentation of this encoding is divided into three parts. First, we introduce a mapping from the edge indices used by the subgoal automata into the edge indices used in the symmetry breaking method. Second, we describe how the previously introduced mapping from observables into labels is applied to map formulas over observables into label sets. Third, we transform the SAT clauses into ASP rules.

\paragraph{Edge Index Mapping.} 
The set of rules below encodes the edge index mapping. Firstly, we define a predicate $\mathtt{edge\_id}(i)$, where $i$ is an edge index, and ground it for values between 1 and $(|U|-1)\kappa$:
\begin{equation*}
\left\lbrace\mathtt{edge\_id}(i) \mid 1 \leq i \leq (|U|-1)\kappa \right\rbrace.
\end{equation*}
Note that $(|U|-1)\kappa$ is the maximum number of outgoing edges from a state: each state can have edges to $|U|-1$ different states and $\kappa$ edges are allowed from one state to another.

We use facts of the form $\mathtt{mapping}(u,v,e,e')$ to indicate that edge $e$ between $u$ and $v$ is mapped into $e'$. The mapping is enforced using the set of rules below. The first rule describes that an edge index $\mathtt{E}$ from $\mathtt{X}$ to $\mathtt{Y}$ is mapped into exactly one edge index $\mathtt{EE}$ in the range given by the $\mathtt{edge\_id}$ facts. The second rule enforces that two outgoing edges from a state $\mathtt{X}$ to two different states $\mathtt{Y}$ and $\mathtt{Z}$ must be mapped into different edge indices. The third rule enforces that two edge indices $\mathtt{E}$ and $\mathtt{EP}$ between the same pair of states $\mathtt{X}$ and $\mathtt{Y}$ must be mapped into different edge indices. The fourth rule indicates that if there are two edge indices $\mathtt{E}$ and $\mathtt{EP}$ between states $\mathtt{X}$ and $\mathtt{Y}$ such that $\mathtt{E\mathord{<}EP}$, then the indices to which they are mapped ($\mathtt{EE}$ and $\mathtt{EEP}$) must preserve the same ordering ($\mathtt{EE}\mathord{<}\mathtt{EEP}$).
\begin{equation*}
\begin{Bmatrix*}[l]
\mathtt{1\{mapping(X, Y, E, EE) : edge\_id(EE)\}1 \codeif ed(X, Y, E).}\\
\mathtt{\codeif mapping(X, Y, \_, EE), mapping(X, Z, \_, EE), Y\mathord{<}Z.}\\
\mathtt{\codeif mapping(X, Y, E, EE), mapping(X, Y, EP, EE), E\mathord{<}EP.}\\
\mathtt{\codeif ed(X, Y, E), ed(X, Y, EP), E\mathord{<}EP, mapping(X, Y, E, EE), mapping(X, Y, EP, EEP), EE\mathord{>}EEP.}
\end{Bmatrix*}
\end{equation*}

Given the mapping, it is straightforward to redefine the $\mathtt{ed}$ atoms, as well as the $\mathtt{pos}$ and $\mathtt{neg}$ facts used in the factual representation of the automata introduced in Section~\ref{sec:structural_properties}:
\begin{align*}
\begin{Bmatrix*}[l]
\mathtt{map\_ed(X, Y, EP) \codeif ed(X, Y, E), mapping(X, Y, E, EP).}\\
\mathtt{map\_pos(X, Y, EP, O) \codeif pos(X, Y, E, O), mapping(X, Y, E, EP).}\\
\mathtt{map\_neg(X, Y, EP, O) \codeif neg(X, Y, E, O), mapping(X, Y, E, EP).}
\end{Bmatrix*}
\end{align*}
The predicates $\mathtt{map\_ed}$, $\mathtt{map\_pos}$ and $\mathtt{map\_neg}$ are used in the ASP encoding of the symmetry breaking constraints explained later.

\paragraph{Mapping of Formulas into Label Sets.} The set of rules below uses the previously described observable-to-label and edge index mappings to transform formulas over a set of observables into label sets. The first rule sets $\mathtt{OID}$ as a label of edge $\mathtt{E}$ from $\mathtt{X}$ if the corresponding observable $\mathtt{O}$ appears positively in that edge. The second rule sets $\mathtt{OID\mathord{+}N}$ as a label of edge $\mathtt{E}$ from $\mathtt{X}$ if the corresponding observable $\mathtt{O}$ appears negatively in that edge and $\mathtt{N}$ is the number of observables.
\begin{equation*}
\begin{Bmatrix*}[l]
\mathtt{label(X, E, OID) \codeif map\_pos(X, Y, E, O), obs\_id(O, OID).}\\
\mathtt{label(X, E, OID\mathord{+}N) \codeif map\_neg(X, Y, E, O), obs\_id(O, OID), num\_obs(N).}
\end{Bmatrix*}
\end{equation*}

\paragraph{Symmetry Breaking Rules.} Similarly to the SAT encoding, we divide the resulting ASP rules into three sets. First, we introduce the set of rules enforcing the indexing given by the BFS traversal on the automaton. These rules are defined in terms of an auxiliary predicate $\mathtt{ed\_sb}(u,u',i)$ equivalent to $\mathtt{map\_ed}(u,u',i)$ but only defined for those states $u'$ that have a state index. Remember that the accepting ($u_A$) and rejecting ($u_R$) states are excluded from the traversal and, thus, they are the only states without an index. The following rule defines $\mathtt{ed\_sb}$ in terms of $\mathtt{map\_ed}$:
\begin{equation*}
\mathtt{ed\_sb(X, Y, E) \codeif map\_ed(X, Y, E), state\_id(Y, \_).}
\end{equation*}
The set of rules below that corresponds to Clauses 1-8 in the SAT encoding.
\begin{align*}
\begin{Bmatrix*}[l]
\mathtt{1\{pa(X, Y): state(X), state\_lt(X, Y)\} \codeif state(Y), state\_id(Y, YID), YID\mathord{>}0.}\\
\mathtt{\codeif pa(X, Y), pa(XP, Y), state\_lt(X, XP), state\_lt(XP, Y).}\\
\mathtt{1\{sm(X, Y, E) : edge\_id(E)\} \codeif pa(X, Y).}\\
\mathtt{\codeif pa(X, Y), ed\_sb(XP, YP, \_), state\_lt(XP, X), state\_leq(Y, YP).}\\
\mathtt{\codeif sm(X, Y, \_), not~pa(X, Y).}\\
\mathtt{\codeif sm(X, Y, E), sm(X, Y, EP), E\mathord{<}EP.}\\
\mathtt{\codeif sm(X, Y, E), not~ed\_sb(X, Y, E).}\\
\mathtt{\codeif sm(X, Y, E), ed\_sb(X, YP, EP), state\_lt(X, Y), state\_leq(Y, YP), EP\mathord{<}E.}
\end{Bmatrix*}
\end{align*}

The next set of rules encodes Clauses~9 and 10, which enforce that edge indices are unique between 1 and the number of outgoing edges from $\mathtt{X}$:
\begin{equation*}
\begin{Bmatrix*}[l]
\mathtt{\codeif map\_ed(X, Y, E), not~map\_ed(X, \_, E\mathord{-}1), E\mathord{>}1.}\\
\mathtt{\codeif map\_ed(X, Y, E), map\_ed(X, Z, E), Y\mathord{<}Z.}
\end{Bmatrix*}
\end{equation*}

The following rules encode SAT Clauses~11-14 in ASP. Note that Clauses~11 and 12 have been divided into two rules respectively to cover the different values of $\mathtt{L}$.
\begin{equation*}
\begin{Bmatrix*}[l]
\mathtt{\codeif lt(X, E\mathord{-}1, E, L), label(X, E\mathord{-}1, L), not~lt(X, E\mathord{-}1, E, L\mathord{-}1), E\mathord{>}1, L\mathord{>}1.}\\
\mathtt{\codeif lt(X, E\mathord{-}1, E, L), label(X, E\mathord{-}1, L), E\mathord{>}1, L\mathord{=}1.}\\
\mathtt{\codeif lt(X, E\mathord{-}1, E, L), not~label(X, E, L), not~lt(X, E\mathord{-}1, E, L\mathord{-}1), E\mathord{>}1, L\mathord{>}1.}\\
\mathtt{\codeif lt(X, E\mathord{-}1, E, L), not~label(X, E, L), E\mathord{>}1, L\mathord{=}1.}\\
\mathtt{1\{lt(X, E\mathord{-}1, E, L) : valid\_label(L)\} \codeif map\_ed(X, Y, E), E\mathord{>}1.}\\
\mathtt{\codeif not~lt(X, E\mathord{-}1, E, L), label(X, E\mathord{-}1, L), not~label(X, E, L), map\_ed(X, \_, E), E\mathord{>}1.}
\end{Bmatrix*}
\end{equation*}

\subsubsection{Alternative Encoding}
\label{app:alternative_asp_encoding}
The encoding we describe in this section is an alternative to the one presented before. Despite of the differences, it also enforces the graph indexing given by the BFS traversal described in Section~\ref{sec:structural_properties}.

Similarly to the SAT-based encoding, we need to address the fact that the mechanism uses an indexing for the edges different from the one we use to represent a subgoal automaton. In the SAT-based approach, we defined a mapping from the edge indices used by a subgoal automaton to those required by the symmetry breaking. In this approach, we do not use such mapping and directly operate on the edge indices used by the automata. The edge indexing used in the symmetry breaking is such that each outgoing edge from a given state has a different index. In other words, each edge is uniquely identified by an integer number. Here we preserve the same uniqueness principle by expressing the edges as $(u, (v,e))$ meaning that there is an edge from $u$ to $v$ with edge index $e$. Note that the tuple $(v,e)$ uniquely identifies each outgoing edge from $u$.

The rest of this section is divided into two parts. First, like in the SAT-based approach, we map the propositional formulas on the edges into sets of labels. Then, we describe the set of ASP rules we use for breaking symmetries.

\paragraph{Mapping of Formulas into Label Sets.}
We use the predicate $\mathtt{label}(u,(v,e),l)$ to express that label $l$ appears in the edge from state $u$ to state $v$ with index $e$. Similarly to the SAT-based encoding, the set of rules below transforms the formulas over a set of observables into label sets. The first rule sets $\mathtt{OID}$ as a label of edge $\mathtt{(Y,E)}$ from $\mathtt{X}$ if the corresponding observable $\mathtt{O}$ appears positively in that edge. Likewise, the second rule sets $\mathtt{OID\mathord{+}N}$ as a label of edge $\mathtt{(Y,E)}$ from $\mathtt{X}$ if the corresponding observable $\mathtt{O}$ appears negatively in that edge and $\mathtt{N}$ is the number of observables.
\begin{equation*}
\begin{matrix*}[l]
\begin{Bmatrix*}[l]
\mathtt{label(X, (Y, E), OID) \codeif pos(X, Y, E, O), obs\_id(O, OID).}\\
\mathtt{label(X, (Y, E), OID\mathord{+}N) \codeif neg(X, Y, E, O), obs\_id(O, OID), num\_obs(N).}
\end{Bmatrix*}
\end{matrix*}
\end{equation*}
Note that while the $\mathtt{label}$ predicate in the SAT encoding only used the edge index for referring to an outgoing edge, here we use a state-edge pair as explained at the beginning of the section.

\paragraph{Symmetry Breaking Rules.} We start by describing the rules which enforce outgoing edges from a given state to be ordered by their respective label sets. The predicate $\mathtt{ed\_lt(X, (Y, E), (YP, EP))}$ indicates that the edge from $\mathtt{X}$ to $\mathtt{Y}$ with edge index $\mathtt{E}$ is lower than the edge from $\mathtt{X}$ to $\mathtt{YP}$ with edge index $\mathtt{EP}$. The set of rules below describes how this ordering is determined and what constraints we impose on it. The first rule determines that given two outgoing edges from $\mathtt{X}$, $\mathtt{(Y,E)}$ and $\mathtt{(YP,EP)}$, either $\mathtt{(Y,E)}$ is lower than $\mathtt{(YP,EP)}$ or vice versa. Now, the order between outgoing edges from a state $\mathtt{X}$ must respect two constraints:
\begin{itemize}
	\item The second rule enforces transitivity. That is, if $\mathtt{Edge1}$ is lower than $\mathtt{Edge2}$, and $\mathtt{Edge2}$ is lower than $\mathtt{Edge3}$, then $\mathtt{Edge1}$ must be lower than $\mathtt{Edge3}$.
	\item The third rule enforces that two edges to the same state $\mathtt{Y}$ must be ordered according to their edge index. That is, given edges $\mathtt{(Y, E)}$ and $\mathtt{(Y,EP)}$ from $\mathtt{X}$ such that $\mathtt{E\mathord{<}EP}$, edge $\mathtt{(Y,E)}$ must be lower than $\mathtt{(Y,EP)}$.
\end{itemize}
\begin{equation*}
\begin{Bmatrix*}[l]
{
	\arraycolsep=1.4pt
	\begin{matrix*}[l]
	\mathtt{1 \{ ed\_lt(X, (Y, E), (YP, EP)) ; ed\_lt(X, (YP, EP), (Y, E)) \} 1} \codeif& \mathtt{ed(X, Y, E), ed(X, YP, EP),}\\ 
	&\mathtt{(Y, E) \mathord{<} (YP, EP).}
	\end{matrix*}
}\\
{
	\arraycolsep=1.4pt
	\begin{matrix*}[l]
	\codeif &\mathtt{ed\_lt(X, Edge1, Edge2), ed\_lt(X, Edge2, Edge3), not~ed\_lt(X, Edge1, Edge3),}\\ 
	&\mathtt{Edge1 !\mathord{=} Edge3.}
	\end{matrix*}
}\\
\mathtt{\codeif ed\_lt(X, (Y, E), (Y, EP)), ed(X, Y, E), ed(X, Y, EP), E\mathord{>}EP.}
\end{Bmatrix*}
\end{equation*}
Note that the previous set of rules guesses an ordering of outgoing edges from a given state. However, this ordering must comply with that of the label sets given in Definition~\ref{def:observation_ordering}. We use the predicate $\mathtt{label\_lt(X, Edge1, Edge2, L)}$ to indicate that there is label $\mathtt{L'}\leq \mathtt{L}$ that appears in $\mathtt{Edge2}$ and does not appear in a lower edge $\mathtt{Edge1}$, where both $\mathtt{Edge1}$ and $\mathtt{Edge2}$ are outgoing edges from $\mathtt{X}$. Note that this predicate encodes the first condition in Definition~\ref{def:observation_ordering} up to a specific label. The set of rules below prunes solutions where outgoing edges do not follow the established label ordering criteria. The first rule indicates that $\mathtt{label\_lt(X, Edge1, Edge2, L)}$ is true if $\mathtt{Edge1}$ is lower than $\mathtt{Edge2}$, and the label $\mathtt{L}$ does not appear in $\mathtt{Edge1}$ and appears in $\mathtt{Edge2}$. The second rule states that $\mathtt{label\_lt}$ is true for a valid label $\mathtt{L\mathord{+}1}$ if it is true for $\mathtt{L}$. The third rule states that if $\mathtt{Edge1}$ is lower than $\mathtt{Edge2}$, then the label set on $\mathtt{Edge1}$ must be lower than that on $\mathtt{Edge2}$. Note that the three last literals on the constraint enforce both conditions in Definition~\ref{def:observation_ordering}. 
\begin{equation*}
\begin{Bmatrix*}[l]
{
	\arraycolsep=1.4pt
	\begin{matrix*}[l]
	\mathtt{label\_lt(X, Edge1, Edge2, L)}&\codeif&\mathtt{ed\_lt(X, Edge1, Edge2), not~label(X, Edge1, L),}\\
	&&\mathtt{label(X, Edge2, L).}
	\end{matrix*}
}\\
\mathtt{label\_lt(X, Edge1, Edge2, L\mathord{+}1) \codeif label\_lt(X, Edge1, Edge2, L), valid\_label(L\mathord{+}1).}\\
{
	\arraycolsep=1.4pt
	\begin{matrix*}[l]
	\codeif&\mathtt{ed\_lt(X, Edge1, Edge2), label(X, Edge1, L), not~label(X, Edge2, L),}\\
	&\mathtt{not~label\_lt(X, Edge1, Edge2, L).}
	\end{matrix*}
}
\end{Bmatrix*}
\end{equation*}

The set of rules below imposes that lower edge indices cannot be left unused. First, we define a fact for each possible edge index between 1 and $\kappa$. Remember that $\kappa$ is the maximum number of edges from one state to another. Second, the constraint indicates that if there is an edge from $\mathtt{X}$ to $\mathtt{Y}$ with edge index $\mathtt{E}$ where $\mathtt{E}\mathord{>}1$, then there must be an edge between the same states but with edge index $\mathtt{E\mathord{-}1}$ as well.
\begin{equation*}
\begin{Bmatrix*}[l]
\mathtt{edge\_id(1..\kappa).}\\
\codeif \mathtt{ed(X, Y, E), not~ed(X, Y, E\mathord{-}1), edge\_id(E), E\mathord{>}1.}
\end{Bmatrix*}
\end{equation*}

Finally, we introduce a set of rules for enforcing a BFS traversal of the subgoal automaton. Firstly, like in the SAT-based approach, we use a predicate $\mathtt{ed\_sb(X,Y,E)}$ which is ground for all the edges except for those directed to a state without an index (the accepting and rejecting states):
\begin{equation*}
\mathtt{ed\_sb(X, Y, E) \codeif ed(X, Y, E), state\_id(Y, \_).}
\end{equation*}

We now introduce a predicate $\mathtt{pa(X,Y)}$ denoting that state $\mathtt{X}$ is the parent of $\mathtt{Y}$ in the BFS subtree. Note that it is equivalent to the variable $pa(i,j)$ from the SAT encoding. The set of rules below uses this predicate to enforce the BFS ordering. The first rule defines that state $\mathtt{X}$ is the parent of $\mathtt{Y}$ if there is an edge from $\mathtt{X}$ to $\mathtt{Y}$, $\mathtt{X}$ has a lower index than $\mathtt{Y}$, and there is no state $\mathtt{Z}$ whose index is lower than $\mathtt{X}$'s and has an edge to $\mathtt{Y}$.\footnote{Note that what follows $\mathtt{\#false:}$ must not hold in order to make the body of the rule true.} The second rule indicates that all states with an index except for the initial state must have a parent. The third rule imposes the BFS ordering similarly to Clause~4 in the SAT encoding. Note that the first two rules encode Clauses~1 and 2 of the SAT encoding.
\begin{equation*}
\begin{Bmatrix*}[l]
{	
	\arraycolsep=1.4pt
	\begin{matrix*}[l]
	\mathtt{pa(X, Y)} &\codeif& \mathtt{ed\_sb(X, Y, \_), state\_lt(X, Y),}\\ 
	&&\mathtt{\#false : ed\_sb(Z, Y, \_), state\_lt(Z, X).}
	\end{matrix*}
}\\
\mathtt{\codeif state\_id(Y, YID), YID\mathord{>}0, not~pa(\_, Y).}\\
\mathtt{\codeif pa(X, Y), ed\_sb(XP, YP, \_), state\_lt(XP, X), state\_leq(Y, YP).}
\end{Bmatrix*}
\end{equation*}

Now we need to enforce that the BFS children from a given state are correctly ordered; that is, those children pointed by lower edges should be identified by lower state indices. We use the $\mathtt{state\_ord(X)}$ predicate to indicate that state $\mathtt{X}$ is properly ordered with respect to its siblings (i.e., other states with the same parent state). The set of rules below enforces this ordering. The first rule defines that a state $\mathtt{Y}$ is correctly ordered with respect to its siblings if the edge from their parent \texttt{X} to \texttt{Y}, $\mathtt{(Y,E)}$, is lower than the edge to another state $\mathtt{(YP,EP)}$ if $\mathtt{Y\mathord{<}YP}$. That is, the edges must be ordered according to the order of the state indices. The second rule enforces that all states with a parent must be correctly ordered with respect to their siblings.
\begin{equation*}
\begin{Bmatrix*}[l]
{
	\arraycolsep=1.4pt
	\begin{matrix*}[l]
	\mathtt{state\_ord(Y)} &\codeif& \mathtt{ed\_sb(X, Y, E), pa(X, Y),} \\
	&&\mathtt{\#false : ed\_sb(X, YP, EP), state\_lt(Y, YP), ed\_lt(X, (YP, EP), (Y, E)).}
	\end{matrix*}
}\\
\mathtt{\codeif pa(\_, Y), not~state\_ord(Y).}
\end{Bmatrix*}
\end{equation*}

\section{Additional Experimental Details}
In this section we describe the experimental details omitted in Section~\ref{sec:experiments}. Firstly, we explain how some of the restrictions used in the experiments are implemented. Secondly, we make a thorough analysis of how enabling and disabling these restrictions (as well as other parameters) affects automaton learning and reinforcement learning.

\subsection{Automaton Learning Restrictions}
\label{app:automata_learning_restrictions}
In this section we provide details about the restrictions on automaton learning outlined in Section~\ref{sec:experimental_setting}. All the rules introduced below are added to the automaton learning task defined in Sections~\ref{sec:learn_subgoal_automata_from_traces} and \ref{sec:structural_properties}.

\paragraph{Avoid Learning Purely Negative Formulas.} This restriction consists in avoiding that any edge is labeled by a formula formed only by negated observables. The following constraint discards such solutions by enforcing an observable to occur positively whenever an observable appears negatively in a given edge:\footnote{Remember that the \texttt{pos} and \texttt{neg} facts were introduced in Section~\ref{sec:structural_properties}.}
\begin{equation*}
	\mathtt{\codeif neg(X,Y,E,\_), not~pos(X,Y,E,\_).}
\end{equation*}

\paragraph{Acyclicity.} The following set of rules enforces the automaton to be acyclic; that is, two automaton states cannot be reached from each other. The $\mathtt{path(X,Y)}$ predicate indicates there is a directed path (i.e., a sequence of directed edges) from $\mathtt{X}$ to $\mathtt{Y}$. The first rule states that there is a path from $\mathtt{X}$ to $\mathtt{Y}$ if there is an edge from $\mathtt{X}$ to $\mathtt{Y}$. The second rule indicates that there is a path from $\mathtt{X}$ to $\mathtt{Y}$ if there is an edge from $\mathtt{X}$ to an intermediate state $\mathtt{Z}$ from which there is a path to $\mathtt{Y}$. Finally, the third rule rules out those answer sets where $\mathtt{X}$ and $\mathtt{Y}$ can be reached from each other through directed edges.
\begin{equation*}
	\begin{Bmatrix*}[l]
		\mathtt{path(X, Y) \codeif ed(X, Y, \_).}\\
		\mathtt{path(X, Y) \codeif ed(X, Z, \_), path(Z, Y).}\\
		\mathtt{\codeif path(X, Y), path(Y, X).}
	\end{Bmatrix*}
\end{equation*}

\paragraph{Trace Compression.} The method we propose for shortening an observation trace is based on the following assumptions:
\begin{enumerate}
	\item Empty observations are irrelevant.
	\item Seeing the same observation twice or more in a row is equivalent to seeing it once.
\end{enumerate}
Given these two assumptions, we define a subtype of observation trace called \emph{compressed observation trace}.
\begin{definition}[Compressed observation trace]
	A compressed observation trace $\hat{\lambda}_{L,\mathcal{O}}=\langle \hat{O}_0, \ldots, \hat{O}_m \rangle$ is the result of removing empty observations and, thereafter, removing contiguous equal observations from an observation trace $\lambda_{L, \mathcal{O}}=\langle O_0, \ldots, O_n \rangle$.
	\label{def:compressed_obs_trace}
\end{definition}
\begin{example}
	The first trace below is a goal observation trace for the \textsc{OfficeWorld}'s \textsc{Coffee} task. The second trace is the resulting compressed observation trace. 
	\begin{align*}
		\lambda^G_{L,\mathcal{O}} &= \left\langle \{\}, \{\text{\Coffeecup}\}, \{\text{\Coffeecup}\}, \{\}, \{\}, \{o\} \right\rangle,\\
		\hat{\lambda}^G_{L,\mathcal{O}} &= \left\langle \{\text{\Coffeecup}\}, \{o\} \right\rangle.
	\end{align*}
\end{example}

As we will see experimentally (see Appendix~\ref{app:additional_experimental_results}), compressed observation traces are helpful to speed up automaton learning. However, their applicability is limited to tasks where the assumptions above hold. Thus, these traces should not be used to learn the automata for tasks where every single observation is important, such as ``observe {\Coffeecup} twice in a row''.

The automaton learning component of our method does not distinguish between observation traces and compressed observation traces. Therefore, we do not have to make any change in the encoding of the learning tasks. However, since the information about the number of performed steps is lost and there are no empty observations, unlabeled transitions are meaningless when these traces are used. Thus, we rule out automata with unlabeled edges using the following constraint when trace compression is enabled:
\begin{equation*}
	\mathtt{\codeif ed(X,Y,E), not~pos(X,Y,E,\_), not~ neg(X,Y,E,\_).}
\end{equation*}
This constraint rules out any inductive solution where an edge from $\mathtt{X}$ to $\mathtt{Y}$ with index $\mathtt{E}$ is not labeled by a positive or a negative literal.

In the case of the RL component, the use of compressed observation traces affects the way in which an automaton is traversed by the agent. Since empty observations and contiguous duplicated observations are ignored, the transition function $\delta_\varphi$ is only queried when the current observation (1) is not empty and (2) it is different from the last observation. If these two conditions do not hold,  the agent remains in the same automaton state. 

\subsection{Parameter Tuning Experiments}
\label{app:additional_experimental_results}
In this section we make a thorough analysis of the impact that the reinforcement learning and automaton learning parameters have on the overall performance of ISA using the \textsc{OfficeWorld} domain. Table~\ref{tab:officeworld_extra_experiments_params} shows the parameters used throughout these experiments. The parameter set at the top of the table remains unchanged for all experiments. In contrast, the parameters in the bottom part of the table can change between experiments.

\begin{table}[]
	\centering
	\begin{tabular}{ll}
		\toprule
		Learning rate ($\alpha$)                  & 0.1    \\
		Exploration rate ($\epsilon$)             & 0.1    \\
		Discount factor ($\gamma$)                & 0.99   \\
		Number of episodes                        & 10,000 \\
		Avoid learning purely negative formulas & \cmark \\
		RL algorithm                              & $\text{HRL}_\text{G}$    \\
		\midrule
		Number of tasks ($|\mathcal{D}|$)         & 50     \\
		Maximum episode length ($N$)              & 250    \\
		Trace compression                         & \cmark    \\
		Enforce acyclicity                        & \cmark    \\
		Number of disjuncts ($\kappa$)         & 1      \\
		Use restricted observable set             & \xmark \\
		\bottomrule
	\end{tabular}
	\caption{Parameters used in the \textsc{OfficeWorld} experiments. The top part of the table contains those parameters that remain unchanged, while the bottom part contains those that change across experiments.}
	\label{tab:officeworld_extra_experiments_params}
\end{table}

\subsubsection{POMDP Sets, Number of Tasks \& Number of Steps per Episode}
In these experiments we check the variability in the performance for two different sets of randomly generated POMDPs, different numbers of POMDPs in a POMDP set, and for different numbers of steps per episode. Let $\mathcal{D}_1=\lbrace\mathcal{M}^\Sigma_{1,1}, \ldots, \mathcal{M}^\Sigma_{1,100} \rbrace$ and $\mathcal{D}_2=\lbrace\mathcal{M}^\Sigma_{2,1}, \ldots, \mathcal{M}^\Sigma_{2,100} \rbrace$ be two sets of 100 POMDPs. Let $\mathcal{D}^{j}_i=\lbrace\mathcal{M}^\Sigma_{i,1}, \ldots, \mathcal{M}^\Sigma_{i,j} \rbrace \subseteq \mathcal{D}_i$ denote the subset of the first $j$ POMDPs from the $i$-th POMDP set, and $N$ be the maximum episode length (i.e., the maximum number of steps within an episode). These experiments are performed with $j\in \lbrace 10, 50, 100 \rbrace$ and $N \in \lbrace 100, 250, 500 \rbrace$. 

Given the set of POMDPs $\mathcal{D}_1$, Figure~\ref{fig:officeworld_num_tasks_steps} shows the learning curves for the different combinations of POMDP subsets and steps. Although all experiments were run for 10,000 episodes, the curves are just shown for some of them to clearly display the variances between different values of $N$. We make the following observations:
\begin{itemize}
	\item The lowest maximum episode length ($N=100$) works fine when a goal state is easy to achieve (i.e., the number of subgoals is low, like in \textsc{Coffee}). As the number of subgoals in the task increases, the maximum episode length needs to be increased to reach a task's final goal. If $N$ is not high enough, then there is only a low chance that the agent will observe the required counterexamples to refine an automaton. Remember that if the goal is not achieved at least once, the agent will not be able to exploit the knowledge given by an automaton since it will not learn any. For instance, in the case of \textsc{VisitABCD}, we observe that the learning curves barely converge when $N=100$ because no automaton is learned. Even when the number of POMDPs is high (100), a goal trace to start automaton learning is only found in 9 out of 20 runs.	
	\item Small sets of POMDPs are sufficient to learn an automaton and policies that achieve the goal in the \textsc{Coffee} task. However, in the tasks involving more subgoals, a small set of POMDPs (10) does not speed up convergence as much as bigger sets of POMDPs (50 or 100). For example, in \textsc{VisitABCD}, automata are rarely learned when the number of POMDPs is 10 for any maximum episode length. Increasing the number of POMDPs increases the chance that easier grids are included in the set and, thus, also increases the chance of observing counterexamples to start learning automata.
	\item Small values of $N$ and small number of POMDPs often cause automaton learning to occur through the entire interaction. In contrast, higher values of $N$ and bigger POMDP sets usually concentrate learning early in the interaction. Again, when $N$ is small, the agent has a lower chance of observing a counterexample trace to refine the automaton since the interaction with the environment is shorter. This is specially detrimental in the case of smaller POMDP sets where the goal is particularly difficult to achieve (e.g., if the observables are sparsely distributed in the grid).
\end{itemize}

\begin{figure}
	\centering
	\subfloat{
		\resizebox{0.32\columnwidth}{!}{
			\includegraphics{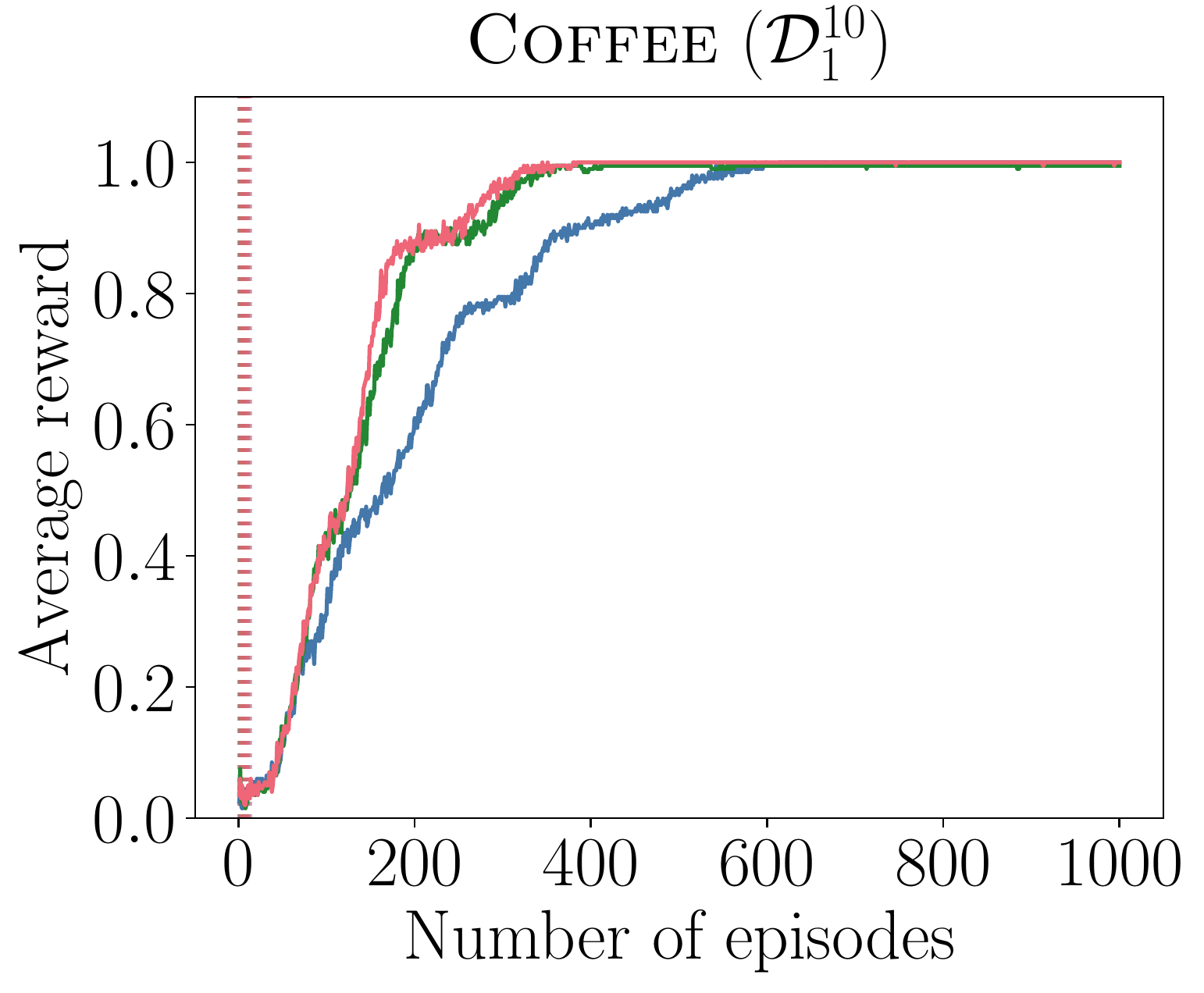}
		}
	}
	\subfloat{
		\resizebox{0.32\columnwidth}{!}{
			\includegraphics{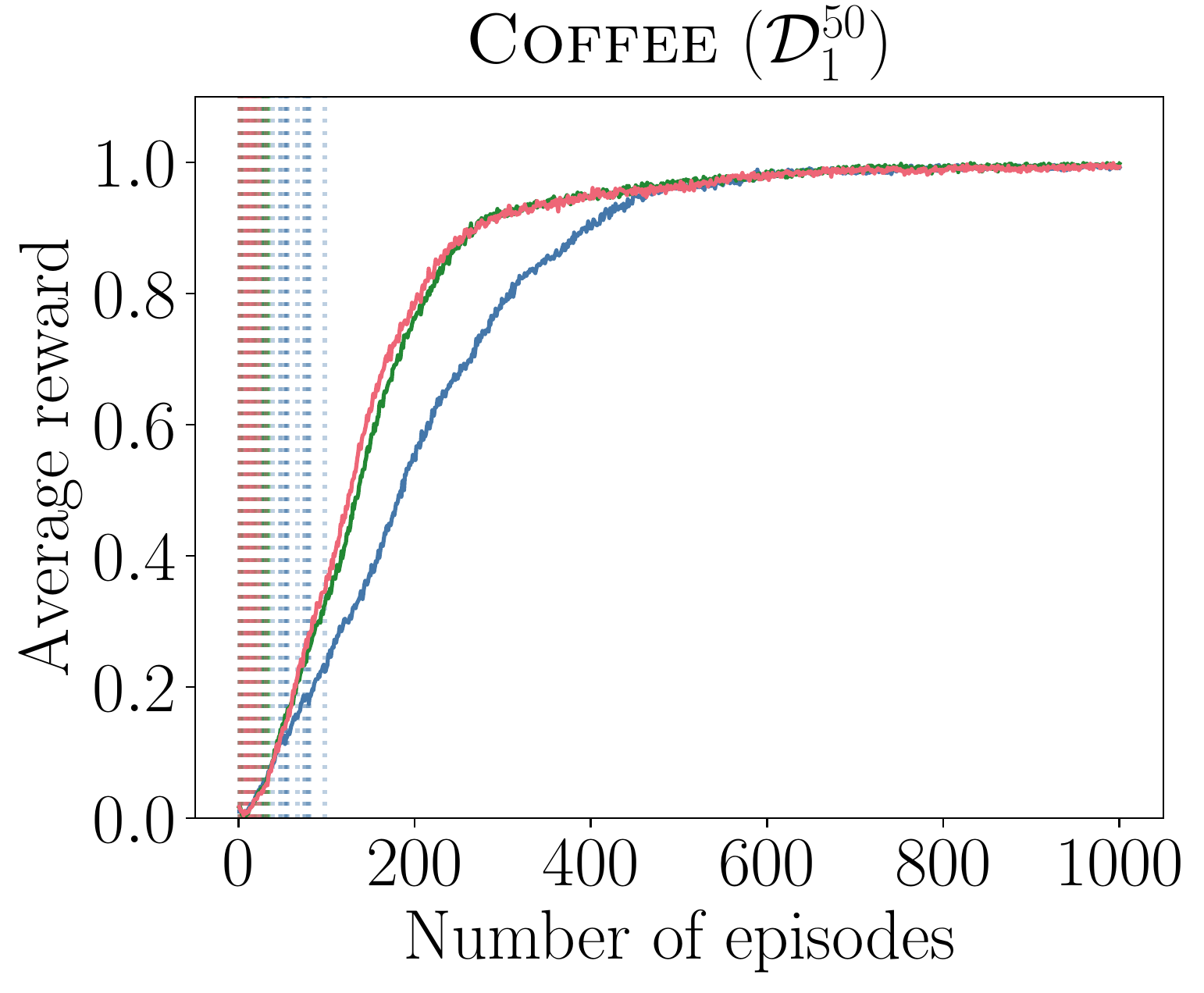}
		}
	}
	\subfloat{
		\resizebox{0.32\columnwidth}{!}{
			\includegraphics{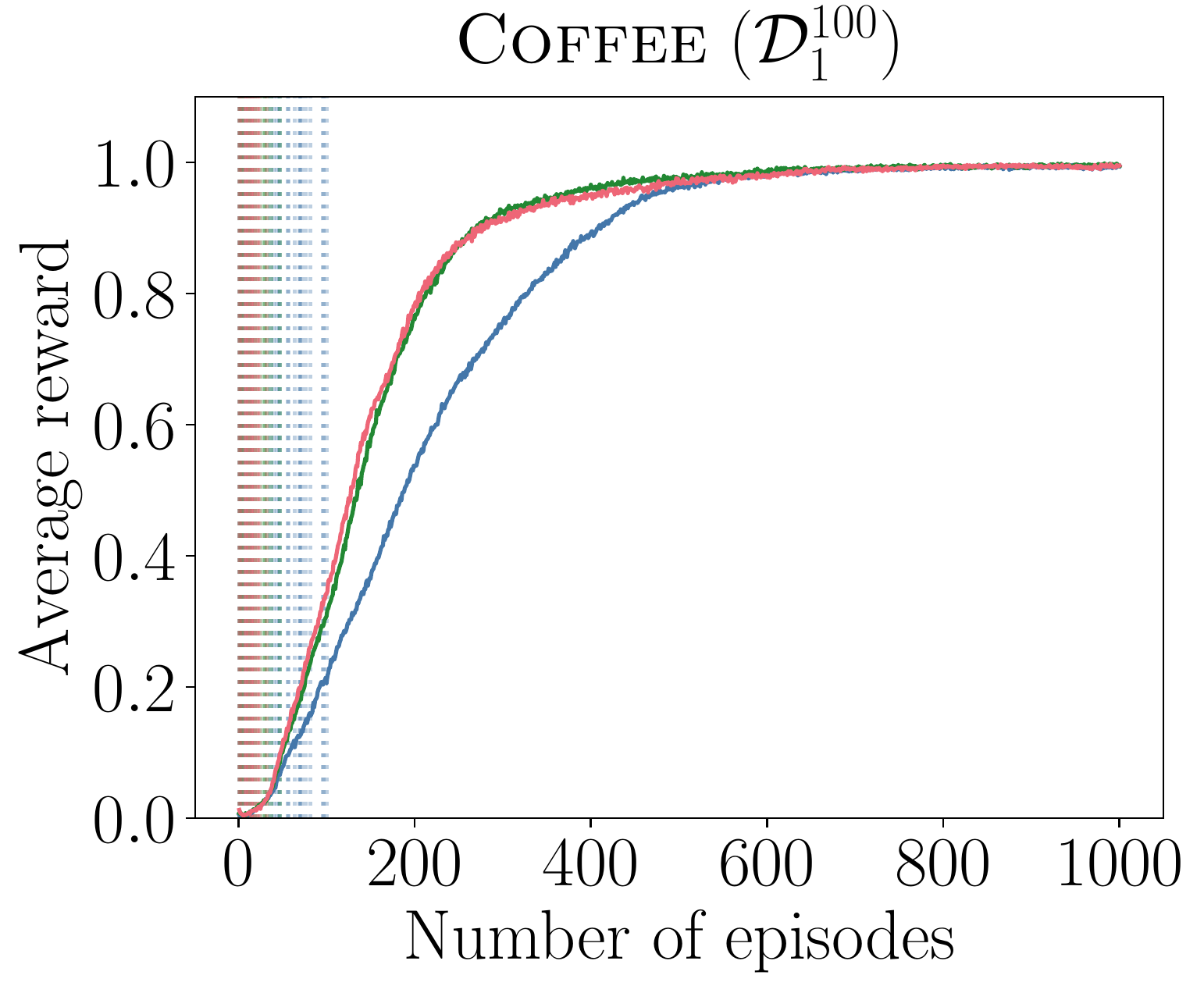}
		}
	}
	\newline
	\subfloat{
		\resizebox{0.32\columnwidth}{!}{
			\includegraphics{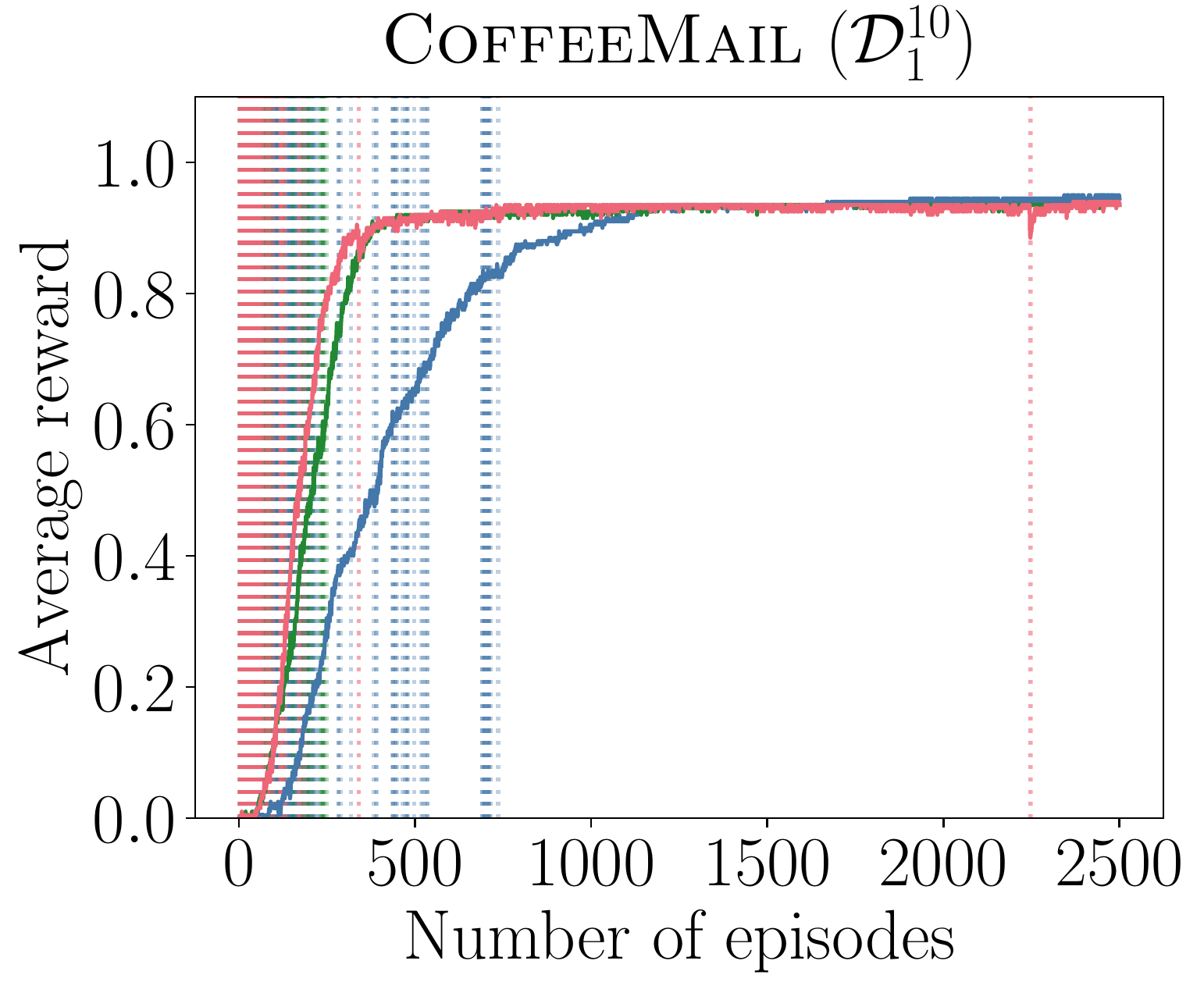}
		}
	}
	\subfloat{
		\resizebox{0.32\columnwidth}{!}{
			\includegraphics{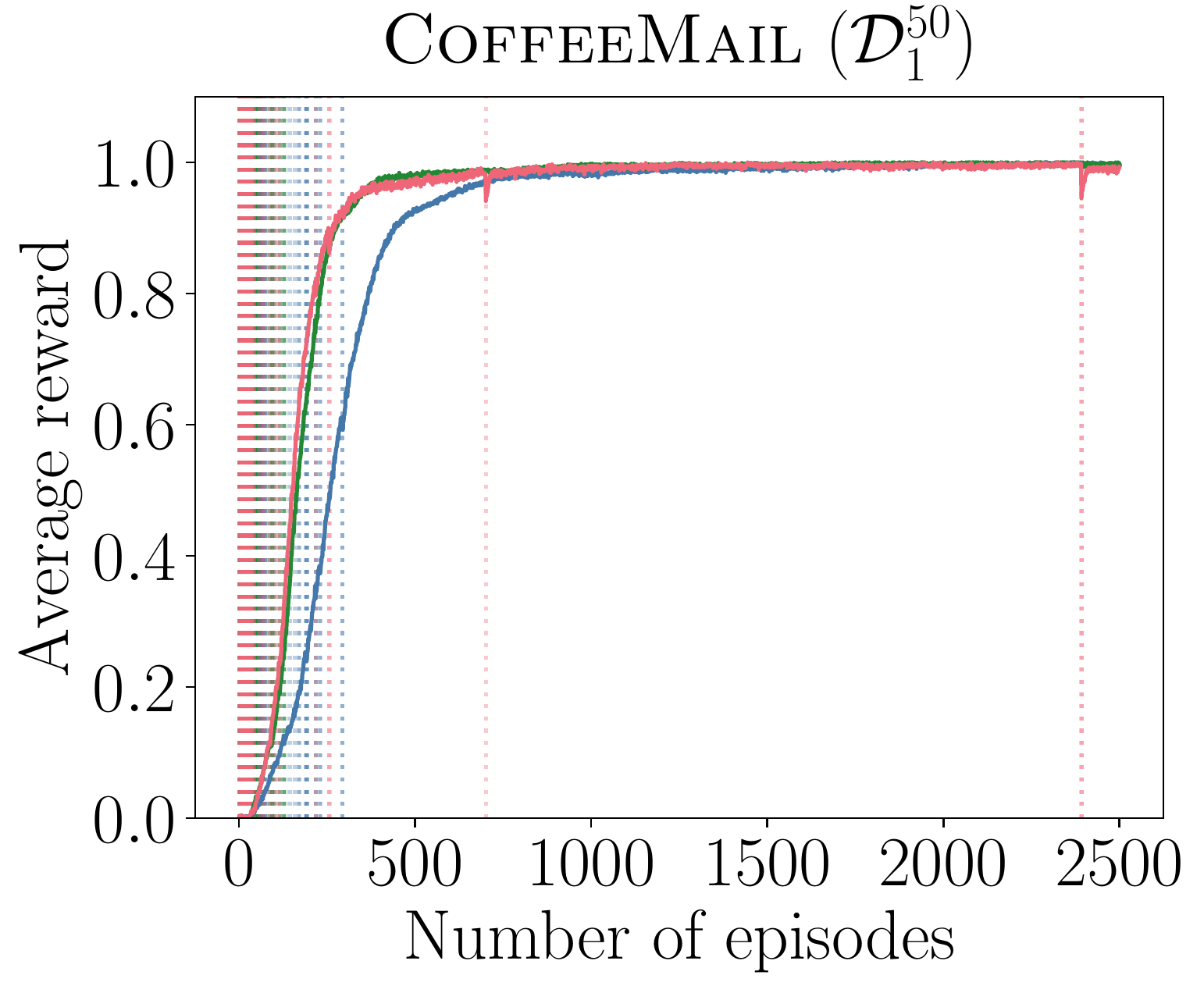}
		}
	}
	\subfloat{
		\resizebox{0.32\columnwidth}{!}{
			\includegraphics{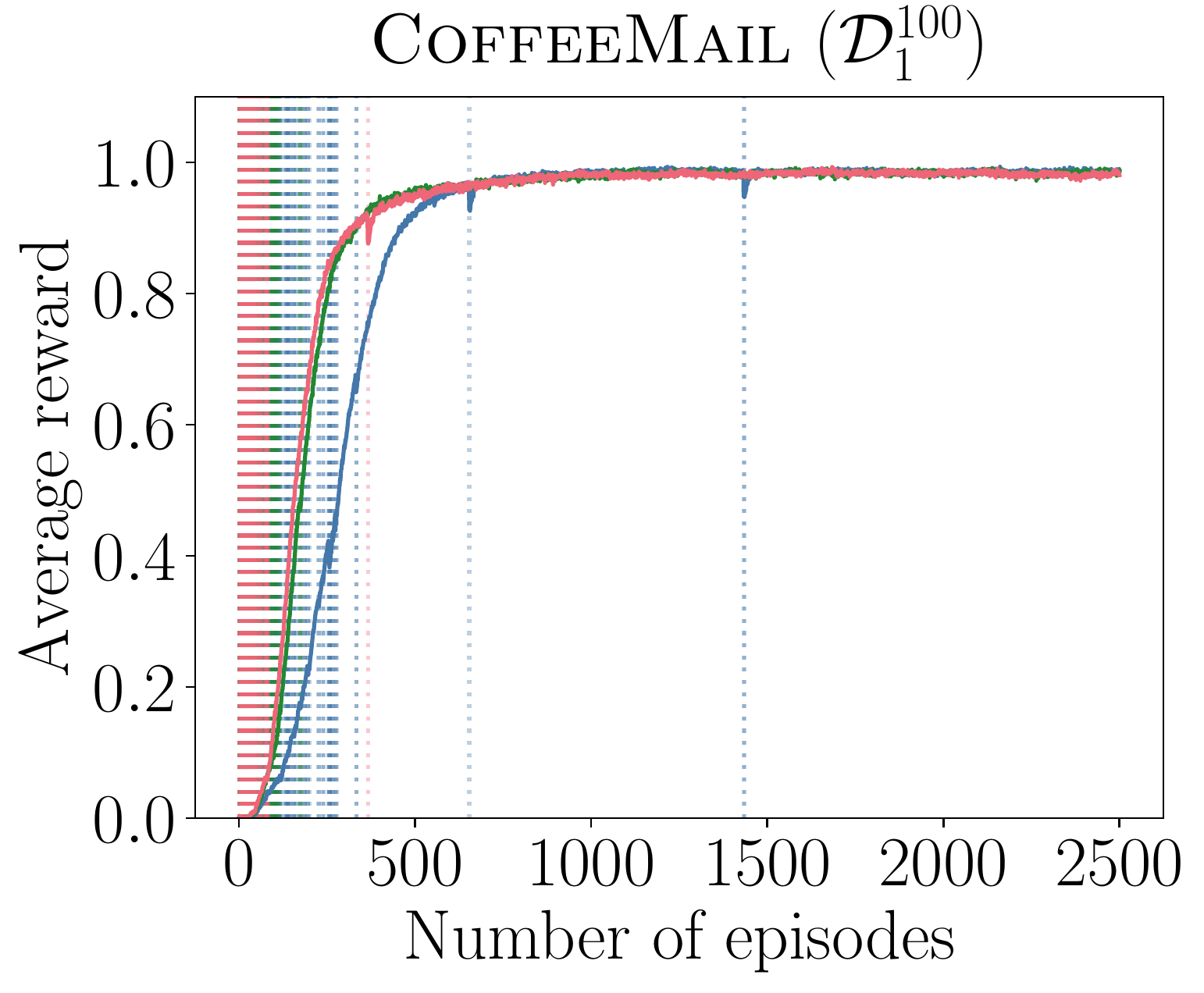}
		}
	}
	\newline
	\subfloat{
		\resizebox{0.32\columnwidth}{!}{
			\includegraphics{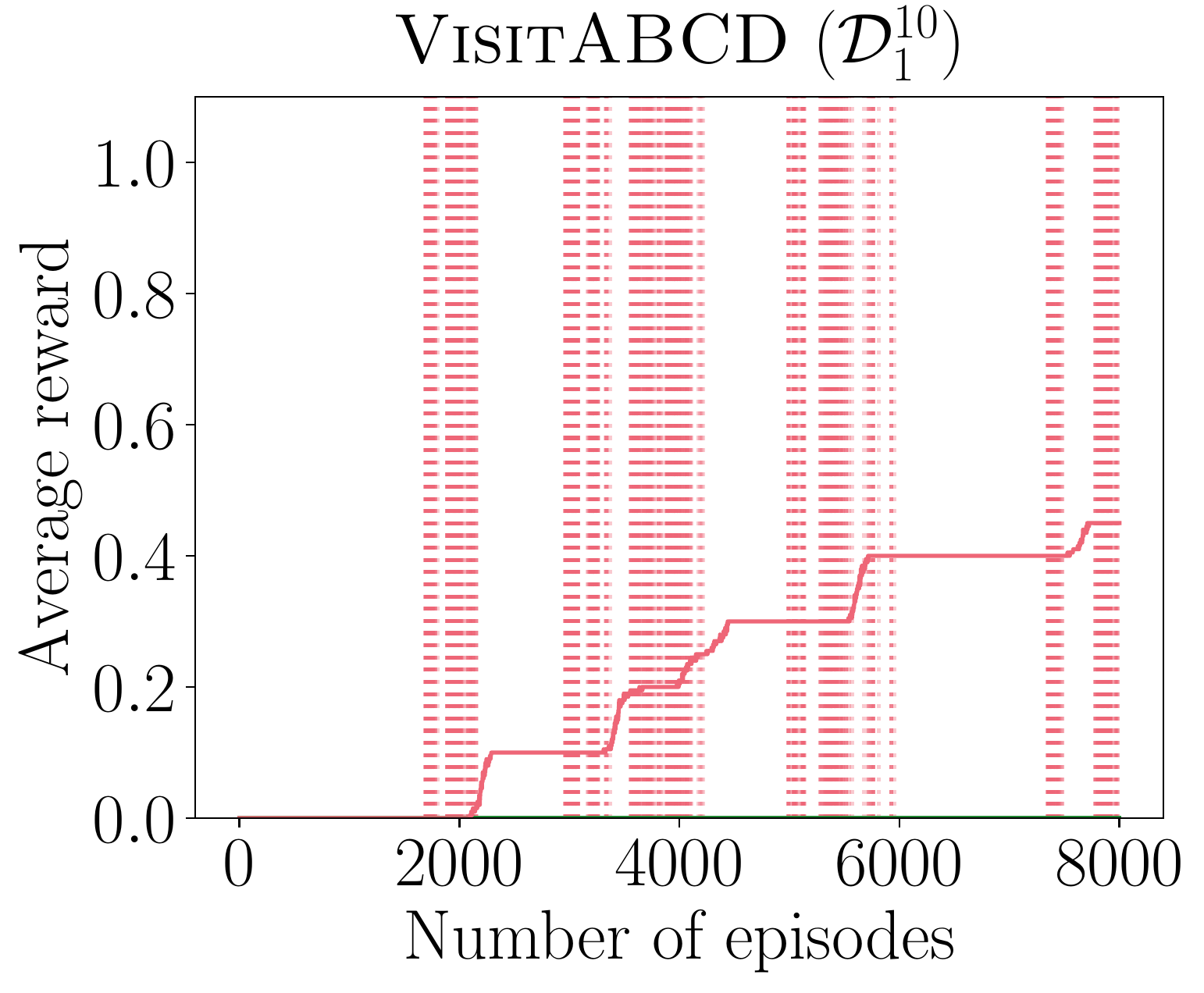}
		}
	}
	\subfloat{
		\resizebox{0.32\columnwidth}{!}{
			\includegraphics{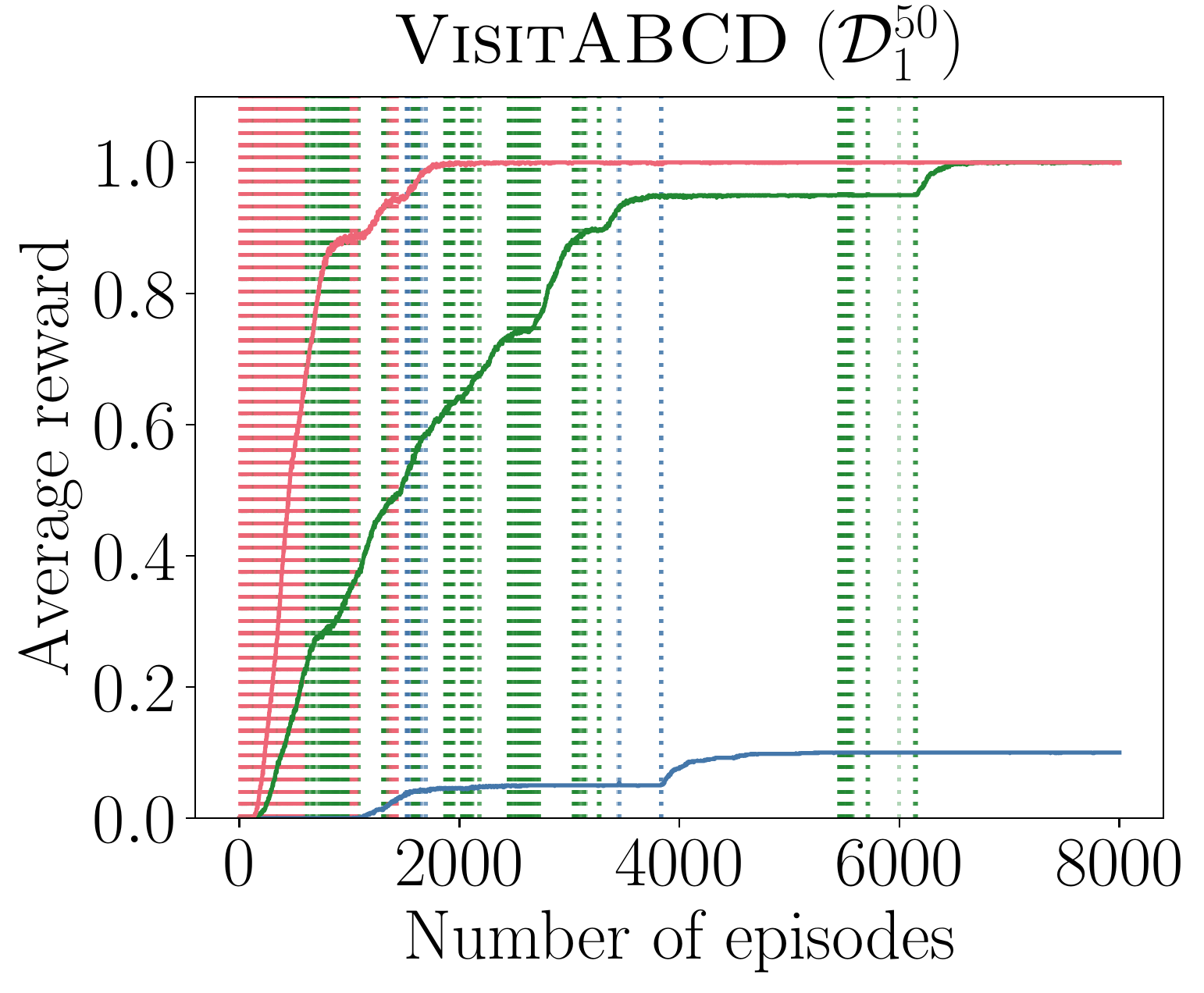}
		}
	}
	\subfloat{
		\resizebox{0.32\columnwidth}{!}{
			\includegraphics{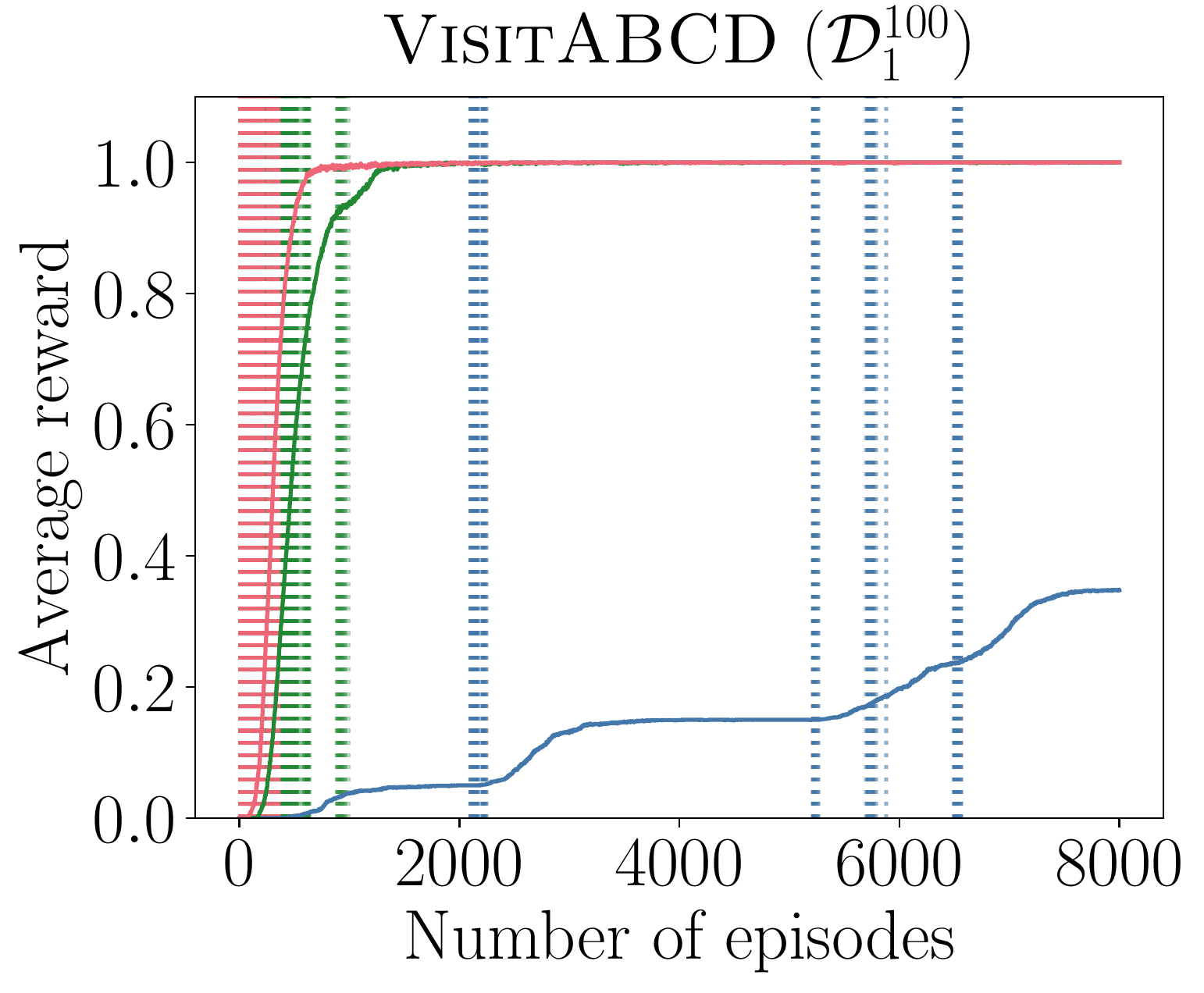}
		}
	}
	
	\begin{tikzpicture}
		\begin{customlegend}[legend columns=-1,legend style={column sep=1ex},legend cell align={left},legend entries={$N=100$,$N=250$,$N=500$}]
			\addlegendimage{pblue}
			\addlegendimage{pgreen}
			\addlegendimage{pred}
		\end{customlegend}
	\end{tikzpicture}
	\caption{Learning curves for different combinations of POMDP sets ($\mathcal{D}^{10}_1$, $\mathcal{D}^{50}_1$, $\mathcal{D}^{100}_1$) and maximum episode lengths (100, 250, 500).}
	\label{fig:officeworld_num_tasks_steps}
\end{figure}

A high maximum episode length seems to be the best choice to make sure that a policy that achieves the goal is learned. Intuitively, however, such choice can produce longer traces and can make automaton learning more complex since (1) the chance of observing irrelevant observables to the task at hand increases (i.e., the system has to learn they are not important) and (2) it is more difficult to figure out the order in which subgoals must be observed. Table~\ref{tab:automata_learning_running_times} shows the average total automaton learning time, while Table~\ref{tab:automata_learning_average_num_examples} shows the average number of examples needed to learn the last automaton in each run. The following is observed from these tables:
\begin{itemize}
	\item The running times generally increase with the maximum episode length. Moreover, Table~\ref{tab:average_example_length_specific_setting} contains the average example length for the $\mathcal{D}^{50}_1$ POMDP set, which shows that the longer the episode is, the longer the counterexamples become.
	\item The running times increase when the set of POMDPs becomes bigger, specifically when changing from $\mathcal{D}^{10}_1$ to $\mathcal{D}^{50}_1$. This is due to the presence of observations that occur within the latter and not in the former. For instance, observables {\Coffeecup} and $o$ only occur together in $\mathcal{D}^{50}_1$ and $\mathcal{D}^{100}_1$ but not in $\mathcal{D}^{10}_1$, so additional time is spent to learn the transition $\text{\Coffeecup} \land o$ for the \textsc{Coffee} and \textsc{CoffeeMail} tasks.
	\item The number of examples remains similar across different sets of POMDPs. There is only a noticeable difference between $\mathcal{D}^{10}_1$ to $\mathcal{D}^{50}_1$ because the latter includes observations that do not happen in the former, as explained before. Note that such changes do not occur in \textsc{VisitABCD} because there is a single path to the accepting state.
	\item As shown in Section~\ref{sec:officeworld_experiments}, the running time and the number of examples increases with the number of subgoals that characterizes the task. Besides, Table~\ref{tab:average_example_length_specific_setting} shows that the example length is longer for the tasks with more subgoals, which can definitely have an effect on the time needed to learn the automaton.
\end{itemize}

\begin{table}[]
	\centering
	\resizebox{\textwidth}{!}{%
		\begin{tabular}{lrrrrrrrrrrr}
			\toprule
			& \multicolumn{3}{c}{$\mathcal{D}^{10}_1$}                                                & & \multicolumn{3}{c}{$\mathcal{D}^{50}_1$}                                                & & \multicolumn{3}{c}{$\mathcal{D}^{100}_1$}                   \\
			\cmidrule{2-4} \cmidrule{6-8} \cmidrule{10-12}
			& \multicolumn{1}{c}{$N=100$} & \multicolumn{1}{c}{$N=250$} & \multicolumn{1}{c}{$N=500$} & & \multicolumn{1}{c}{$N=100$} & \multicolumn{1}{c}{$N=250$} & \multicolumn{1}{c}{$N=500$} & & \multicolumn{1}{c}{$N=100$} & \multicolumn{1}{c}{$N=250$} & \multicolumn{1}{c}{$N=500$}           \\
			\midrule
			\textsc{Coffee}     & 0.3 (0.0)                   & 0.3 (0.0)                   & 0.4 (0.0)                   & & 0.4 (0.0)                   & 0.4 (0.0)                   & 0.6 (0.0)                   & & 0.4 (0.0)                   & 0.5 (0.0)                   & 0.5 (0.0)                            \\
			\textsc{CoffeeMail} & 5.9 (1.5)                   & 4.2 (1.1)                   & 9.4 (3.3)                   & & 7.8 (1.5)                   & 18.9 (3.3)                  & 49.1 (12.0)                 & & 9.2 (1.2)                   & 29.1 (9.4)                  & 64.3 (15.3)                         \\
			\textsc{VisitABCD}  &  -                          & -                           & 2966.4 (1323.9)*             & & -                           & 163.2 (44.3)                & 311.9 (63.4)                & & -                           & 230.7 (99.9)                & 230.8 (48.2)                        \\
			\bottomrule       
		\end{tabular}%
	}
	\caption{Total automaton learning time in seconds for different combinations of POMDP sets ($\mathcal{D}^{10}_1$, $\mathcal{D}^{50}_1$, $\mathcal{D}^{100}_1$) and maximum episode lengths (100, 250, 500).}
	\label{tab:automata_learning_running_times}
\end{table}

\begin{table}[]
	\centering
	\resizebox{\textwidth}{!}{%
		\begin{tabular}{lrrrrrrrrrrr}
			\toprule
			& \multicolumn{3}{c}{$\mathcal{D}^{10}_1$}                                                & & \multicolumn{3}{c}{$\mathcal{D}^{50}_1$}                                                & & \multicolumn{3}{c}{$\mathcal{D}^{100}_1$}                   \\
			\cmidrule{2-4} \cmidrule{6-8} \cmidrule{10-12}
			& \multicolumn{1}{c}{$N=100$} & \multicolumn{1}{c}{$N=250$} & \multicolumn{1}{c}{$N=500$} & & \multicolumn{1}{c}{$N=100$} & \multicolumn{1}{c}{$N=250$} & \multicolumn{1}{c}{$N=500$} & & \multicolumn{1}{c}{$N=100$} & \multicolumn{1}{c}{$N=250$} & \multicolumn{1}{c}{$N=500$}           \\
			\midrule
			\textsc{Coffee}     & 5.8 (0.3)                   & 7.0 (0.3)                   & 7.4 (0.3)                   & & 8.4 (0.3)                   & 8.7 (0.4)                   & 10.7 (0.5)                  & & 8.6 (0.4)                   & 9.6 (0.5)                   & 9.4 (0.4)                            \\
			\textsc{CoffeeMail} & 24.2 (1.5)                  & 20.6 (1.5)                  & 24.4 (1.7)                  & & 24.8 (1.6)                  & 29.0 (1.5)                  & 33.0 (1.5)                  & & 29.6 (1.2)                  & 35.7 (1.2)                  & 35.9 (1.6)                         \\
			\textsc{VisitABCD}  &  -                          & -                           & 86.0 (12.0)*                & & -                           & 54.9 (3.8)                  & 50.6 (3.5)                  & & -                           & 55.2 (6.0)                  & 49.9 (2.5)                        \\
			\bottomrule       
		\end{tabular}%
	}
	\caption{Number of examples needed to learn the last automaton for different combinations of POMDP sets ($\mathcal{D}^{10}_1$, $\mathcal{D}^{50}_1$, $\mathcal{D}^{100}_1$) and maximum episode lengths (100, 250, 500).}
	\label{tab:automata_learning_average_num_examples}
\end{table}

We now briefly examine the results for the set $\mathcal{D}_2$. Figure~\ref{fig:officeworld_num_tasks_steps_dataset2} displays the average learning curves for different combinations of POMDP subsets and values of $N$. Table~\ref{tab:automata_learning_running_times_d2} shows the average automaton learning time, whereas Table~\ref{tab:automata_learning_average_num_examples_d2} contains the average number of examples needed to learn an automaton for different subsets of tasks and maximum episode lengths ($N$). Table~\ref{tab:average_example_length_specific_setting_d2} shows the average example length of each kind of trace for different values of $N$. Table~\ref{tab:average_number_of_examples_specific_setting_d2} contains the average number of examples needed to learn the last automaton in a specific setting. In qualitative terms, the main changes we observe with respect to $\mathcal{D}_1$ occur in the \textsc{VisitABCD} task. While $N=100$ was usually insufficient to learn an automaton in $\mathcal{D}^{50}_1$ across runs, it is enough in $\mathcal{D}^{50}_2$. By comparing the $\textsc{VisitABCD}$ curves, it is clear that $\mathcal{D}^{50}_2$ consists of POMDPs requiring less steps than those in $\mathcal{D}^{50}_1$. Furthermore, in the case of \textsc{CoffeeMail}, $\mathcal{D}^{100}_1$ has more types of joint events than $\mathcal{D}^{100}_2$, thus more examples are needed in that case and the final automaton is slightly different. This shows that the set of POMDPs has an impact on automaton learning.

\begin{table}[p]
	\centering
	\resizebox{\textwidth}{!}{
		\begin{tabular}{lrrrrrrrrrrr}
			\toprule
			& \multicolumn{3}{c}{$N=100$} & & \multicolumn{3}{c}{$N=250$} & & \multicolumn{3}{c}{$N=500$} \\
			\cmidrule{2-4} \cmidrule{6-8} \cmidrule{10-12}
			& \multicolumn{1}{c}{$G$} & \multicolumn{1}{c}{$D$} & \multicolumn{1}{c}{$I$} & &  \multicolumn{1}{c}{$G$} & \multicolumn{1}{c}{$D$} & \multicolumn{1}{c}{$I$} & & \multicolumn{1}{c}{$G$} & \multicolumn{1}{c}{$D$} & \multicolumn{1}{c}{$I$} \\
			\midrule
			\textsc{Coffee}     & 2.8 (1.0)               & 2.1 (1.1)               & 1.5 (0.9)               & & 3.6 (1.8)                & 3.0 (2.5)               & 1.9 (1.4)               & & 5.7 (4.6)               & 3.6 (2.8)               & 2.2 (1.4)    \\
			\textsc{CoffeeMail} & 4.4 (1.4)               & 3.3 (2.0)               & 3.1 (1.8)               & & 5.5 (2.7)                & 4.1 (2.8)               & 3.6 (2.2)               & & 7.2 (4.3)               & 4.9 (3.8)               & 3.5 (2.3)    \\
			\textsc{VisitABCD}  &  -                      &  -                      & -                       & &  9.2 (3.0)               & 5.1 (3.0)               & 5.4 (3.0)               & & 12.4 (4.2)              & 6.2 (4.0)               & 6.3 (4.2)    \\
			\bottomrule
		\end{tabular}
	}
	\caption{Example length of the goal, dead-end and incomplete examples used to learn the last automaton in the $\mathcal{D}^{50}_1$ setting.}
	\label{tab:average_example_length_specific_setting}
\end{table}

\begin{table}[p]
	\centering
	\resizebox{\textwidth}{!}{%
		\begin{tabular}{lrrrrrrrrrrr}
			\toprule
			& \multicolumn{3}{c}{$\mathcal{D}^{10}_2$}                                                & & \multicolumn{3}{c}{$\mathcal{D}^{50}_2$}                                                & & \multicolumn{3}{c}{$\mathcal{D}^{100}_2$}                   \\
			\cmidrule{2-4} \cmidrule{6-8} \cmidrule{10-12}
			& \multicolumn{1}{c}{$N=100$} & \multicolumn{1}{c}{$N=250$} & \multicolumn{1}{c}{$N=500$} & & \multicolumn{1}{c}{$N=100$} & \multicolumn{1}{c}{$N=250$} & \multicolumn{1}{c}{$N=500$} & & \multicolumn{1}{c}{$N=100$} & \multicolumn{1}{c}{$N=250$} & \multicolumn{1}{c}{$N=500$} \\
			\midrule
			\textsc{Coffee}     & 0.4 (0.0)                   & 0.4 (0.0)                   & 0.4 (0.0)                   & & 0.4 (0.0)                   & 0.4 (0.0)                   & 0.4 (0.0)                   & & 0.4 (0.0)                   & 0.4 (0.0)                   & 0.4 (0.0)                            \\
			\textsc{CoffeeMail} & 2.8 (0.4)                   & 3.9 (1.1)                   & 8.1 (2.2)                   & & 7.8 (3.1)                   & 10.3 (1.6)                  & 24.5 (7.5)                  & & 6.8 (1.4)                   & 30.7 (14.6)                 & 22.2 (5.3)                            \\
			\textsc{VisitABCD}  & -                           & -                           & 802.1 (297.1)               & & 130.1 (46.4)                & 290.2 (123.1)               &  481.6 (115.1)              & & 53.2 (15.6)                 & 269.2 (108.0)               & 355.8 (68.2)  \\
			\bottomrule       
		\end{tabular}%
	}
	\caption{Total automaton learning time in seconds for different combinations of POMDP sets ($\mathcal{D}^{10}_2$, $\mathcal{D}^{50}_2$, $\mathcal{D}^{100}_2$) and maximum episode lengths (100, 250, 500).}
	\label{tab:automata_learning_running_times_d2}
\end{table}

\begin{table}[p]
	\centering
	\resizebox{\textwidth}{!}{%
		\begin{tabular}{lrrrrrrrrrrr}
			\toprule
			& \multicolumn{3}{c}{$\mathcal{D}^{10}_2$}                                                & & \multicolumn{3}{c}{$\mathcal{D}^{50}_2$}                                                & & \multicolumn{3}{c}{$\mathcal{D}^{100}_2$}                   \\
			\cmidrule{2-4} \cmidrule{6-8} \cmidrule{10-12}
			& \multicolumn{1}{c}{$N=100$} & \multicolumn{1}{c}{$N=250$} & \multicolumn{1}{c}{$N=500$} & & \multicolumn{1}{c}{$N=100$} & \multicolumn{1}{c}{$N=250$} & \multicolumn{1}{c}{$N=500$} & & \multicolumn{1}{c}{$N=100$} & \multicolumn{1}{c}{$N=250$} & \multicolumn{1}{c}{$N=500$} \\
			\midrule
			\textsc{Coffee}     & 7.3 (0.3)                   & 7.6 (0.4)                   & 7.5 (0.5)                   & & 8.2 (0.3)                   & 8.3 (0.3)                   & 8.6 (0.5)                   & & 8.3 (0.3)                   & 8.0 (0.3)                   & 8.2 (0.3)     \\
			\textsc{CoffeeMail} & 17.8 (1.1)                  & 17.8 (0.8)                  & 20.6 (1.4)                  & & 24.4 (1.5)                  & 26.2 (1.2)                  & 27.2 (1.4)                  & & 24.9 (1.2)                  & 29.1 (1.4)                  & 27.4 (1.3)    \\
			\textsc{VisitABCD}  & -                           & -                           & 74.8 (8.2)                  & & 57.7 (5.6)                  & 53.8 (6.2)                  & 58.6 (3.9)                  & & 42.5 (3.3)                  & 51.9 (3.3)                  & 53.3 (3.3)                            \\
			\bottomrule       
		\end{tabular}
	}
	\caption{Number of examples needed to learn the last automaton for different combinations of POMDP sets ($\mathcal{D}^{10}_2$, $\mathcal{D}^{50}_2$, $\mathcal{D}^{100}_2$) and maximum episode lengths (100, 250, 500).}
	\label{tab:automata_learning_average_num_examples_d2}
\end{table}

\begin{table}[p]
	\centering
	\resizebox{\textwidth}{!}{
		\begin{tabular}{lrrrrrrrrrrr}
			\toprule
			& \multicolumn{3}{c}{$N=100$}                                                 & & \multicolumn{3}{c}{$N=250$}                                                  & & \multicolumn{3}{c}{$N=500$} \\
			\cmidrule{2-4} \cmidrule{6-8} \cmidrule{10-12}
			& \multicolumn{1}{c}{$G$} & \multicolumn{1}{c}{$D$} & \multicolumn{1}{c}{$I$} & &  \multicolumn{1}{c}{$G$} & \multicolumn{1}{c}{$D$} & \multicolumn{1}{c}{$I$} & & \multicolumn{1}{c}{$G$} & \multicolumn{1}{c}{$D$} & \multicolumn{1}{c}{$I$} \\
			\midrule
			\textsc{Coffee}     & 2.8 (1.0)               & 2.1 (1.1)               & 1.5 (0.9)               & & 3.6 (1.8)                & 3.0 (2.5)               & 1.9 (1.4)               & & 5.7 (4.6)               & 3.6 (2.8)               & 2.2 (1.4)             \\
			\textsc{CoffeeMail} & 3.5 (1.4)               & 3.2 (1.7)               & 2.7 (1.5)               & & 4.3 (2.8)                & 4.0 (2.6)               & 3.4 (2.6)               & & 5.1 (4.2)               & 4.1 (3.1)               & 3.5 (2.7)             \\
			\textsc{VisitABCD}  & 7.1 (1.9)               & 4.6 (2.3)               & 4.6 (1.9)               & & 8.9 (3.7)                & 5.5 (3.6)               & 5.1 (2.6)               & & 11.2 (4.6)              & 7.1 (4.8)               & 5.8 (3.3)             \\
			\bottomrule
		\end{tabular}
	}
	\caption{Example length of the goal, dead-end and incomplete examples used to learn the last automaton in the $\mathcal{D}^{50}_2$ setting.}
	\label{tab:average_example_length_specific_setting_d2}
\end{table}

\begin{table}[p]
	\centering
	\begin{tabular}{lrrrr}
		\toprule
		& \multicolumn{1}{c}{All} & \multicolumn{1}{c}{$G$} & \multicolumn{1}{c}{$D$} & \multicolumn{1}{c}{$I$} \\
		\midrule
		\textsc{Coffee}     & 8.7 (0.4)               & 2.4 (0.1)               & 3.0 (0.1)               & 3.2 (0.3)  \\
		\textsc{CoffeeMail} & 26.2 (1.2)              & 5.3 (0.3)               & 8.2 (0.5)               & 12.6 (0.9)    \\
		\textsc{VisitABCD}  & 53.8 (6.2)              & 1.6 (0.1)               & 16.2 (1.4)              & 36.0 (4.9)     \\
		\bottomrule
	\end{tabular}
	\caption{Number of examples (total, goal, dead-end and incomplete) needed to learn the last automaton in the $\mathcal{D}^{50}_2$ and $N=250$ setting.}
	\label{tab:average_number_of_examples_specific_setting_d2}
\end{table}

\begin{figure}[h]
	\centering
	\subfloat{
		\resizebox{0.32\columnwidth}{!}{
			\includegraphics{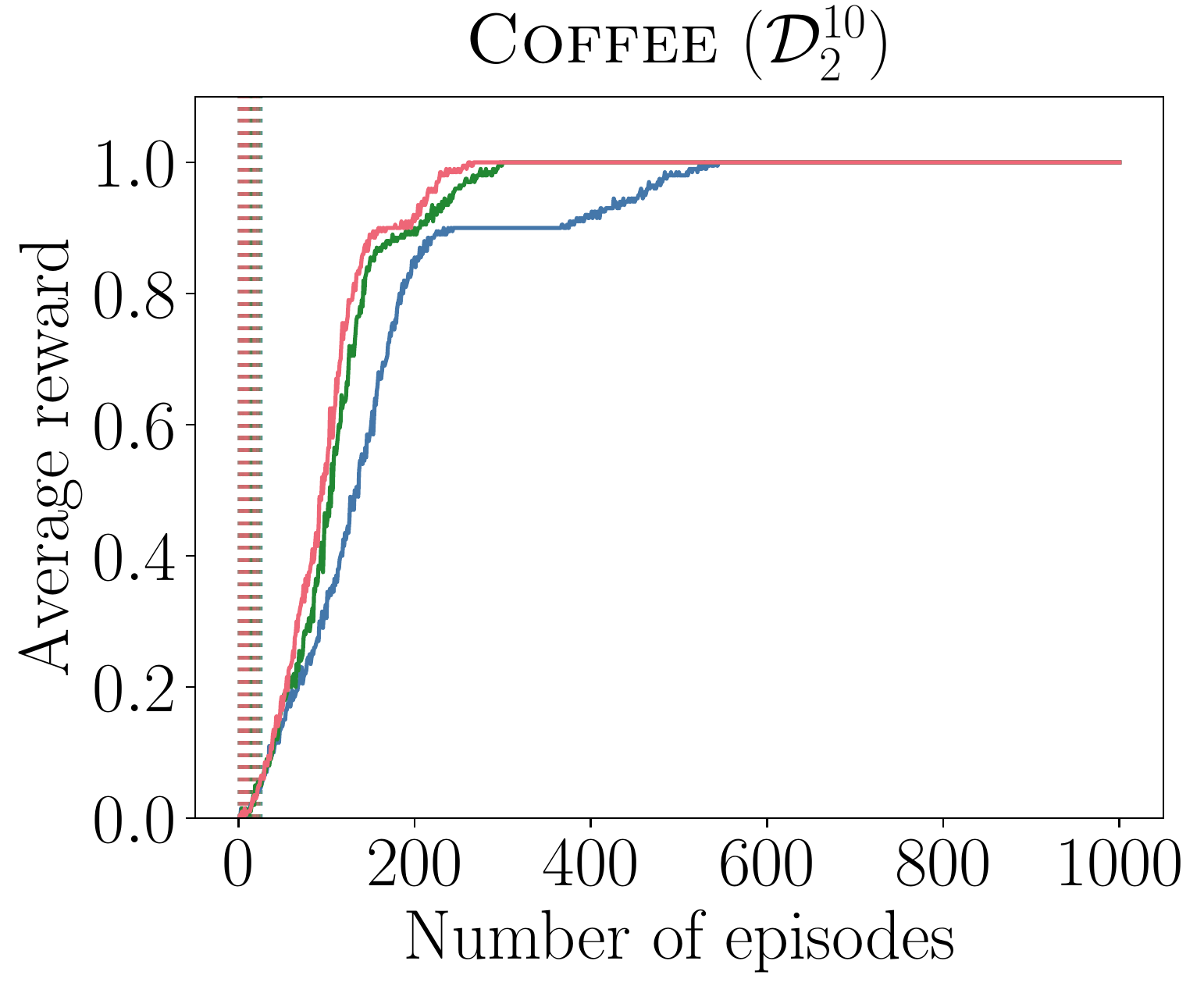}
		}
	}
	\subfloat{
		\resizebox{0.32\columnwidth}{!}{
			\includegraphics{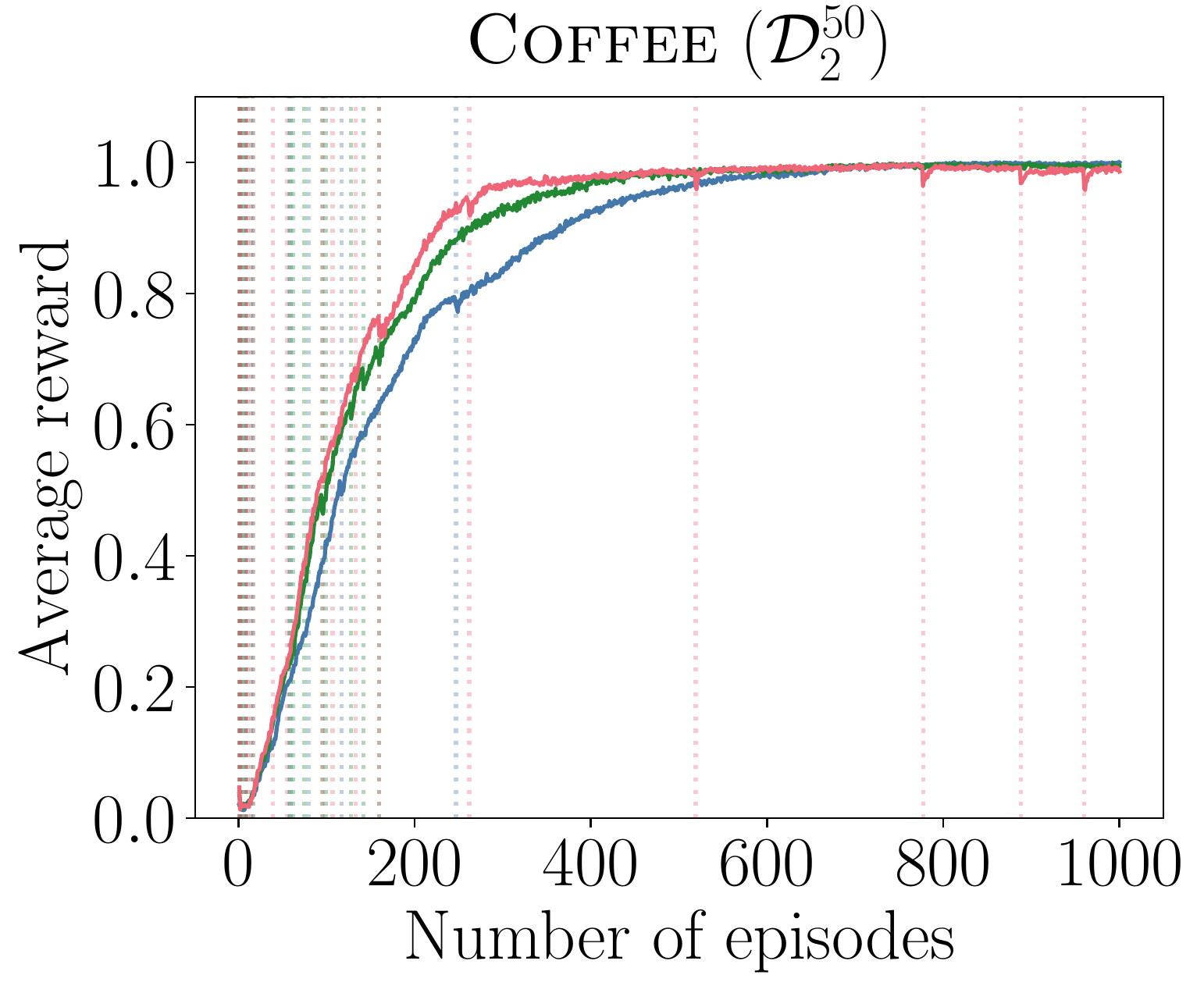}
		}
	}
	\subfloat{
		\resizebox{0.32\columnwidth}{!}{
			\includegraphics{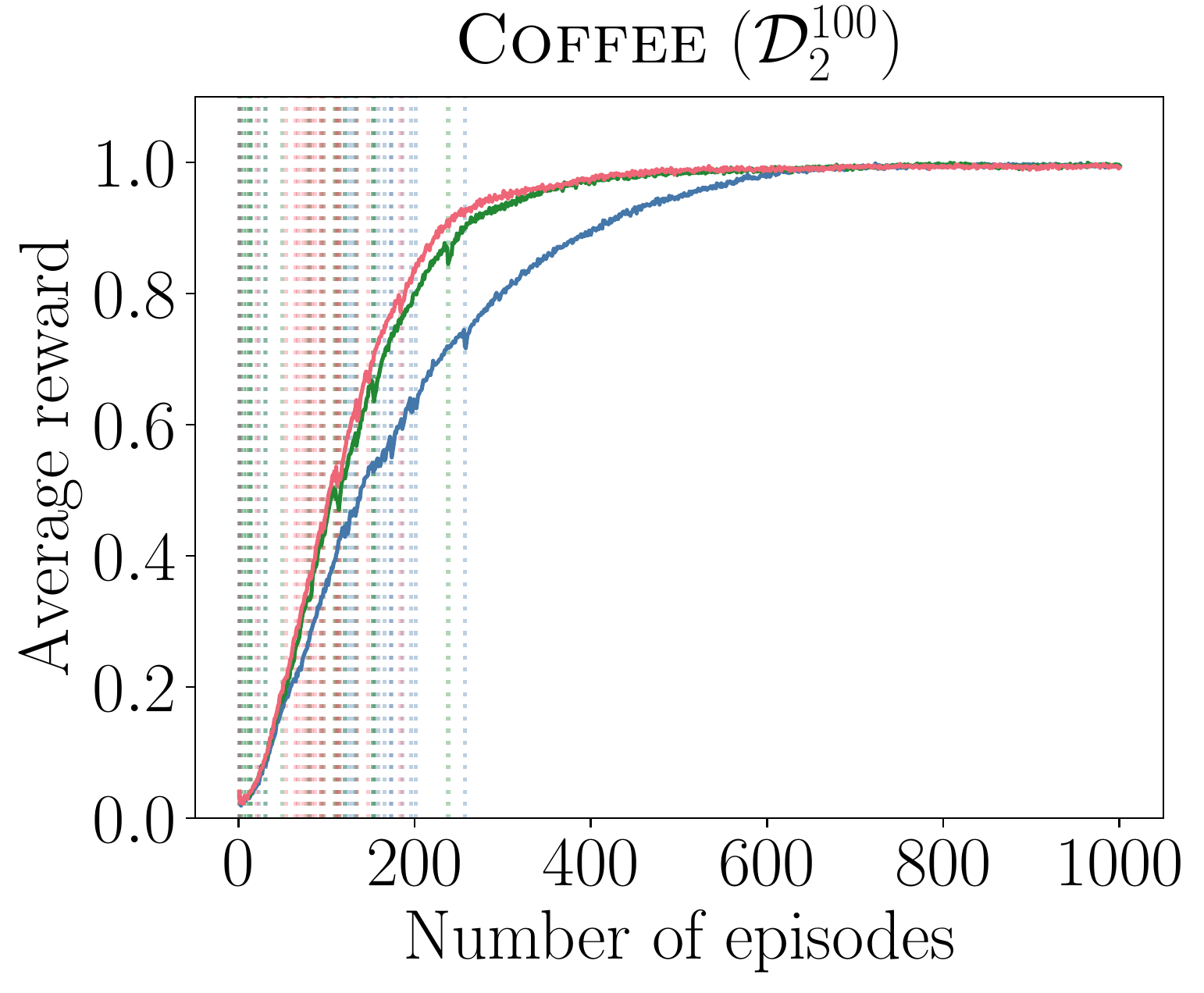}
		}
	}
	\newline
	\subfloat{
		\resizebox{0.32\columnwidth}{!}{
			\includegraphics{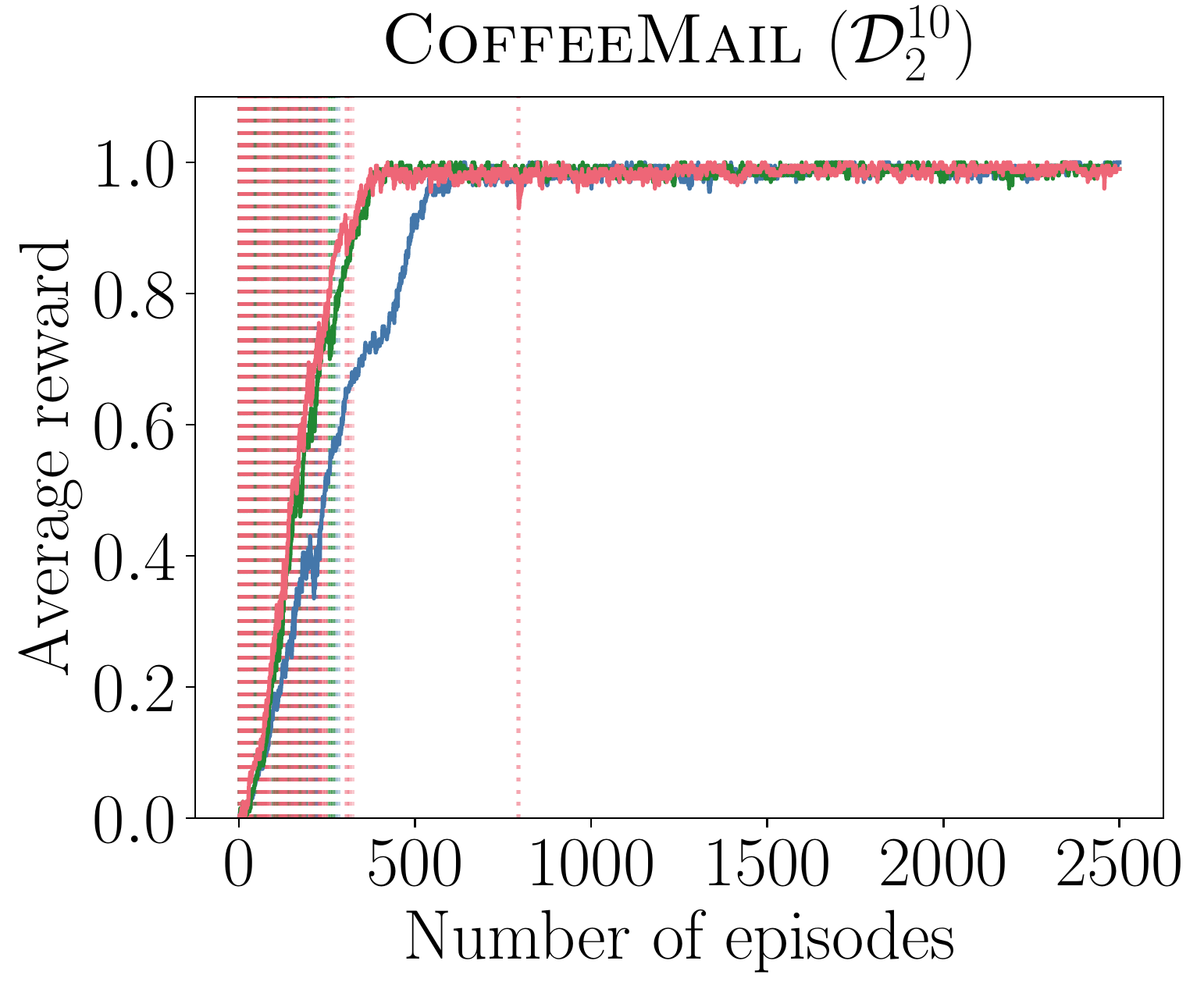}
		}
	}
	\subfloat{
		\resizebox{0.32\columnwidth}{!}{
			\includegraphics{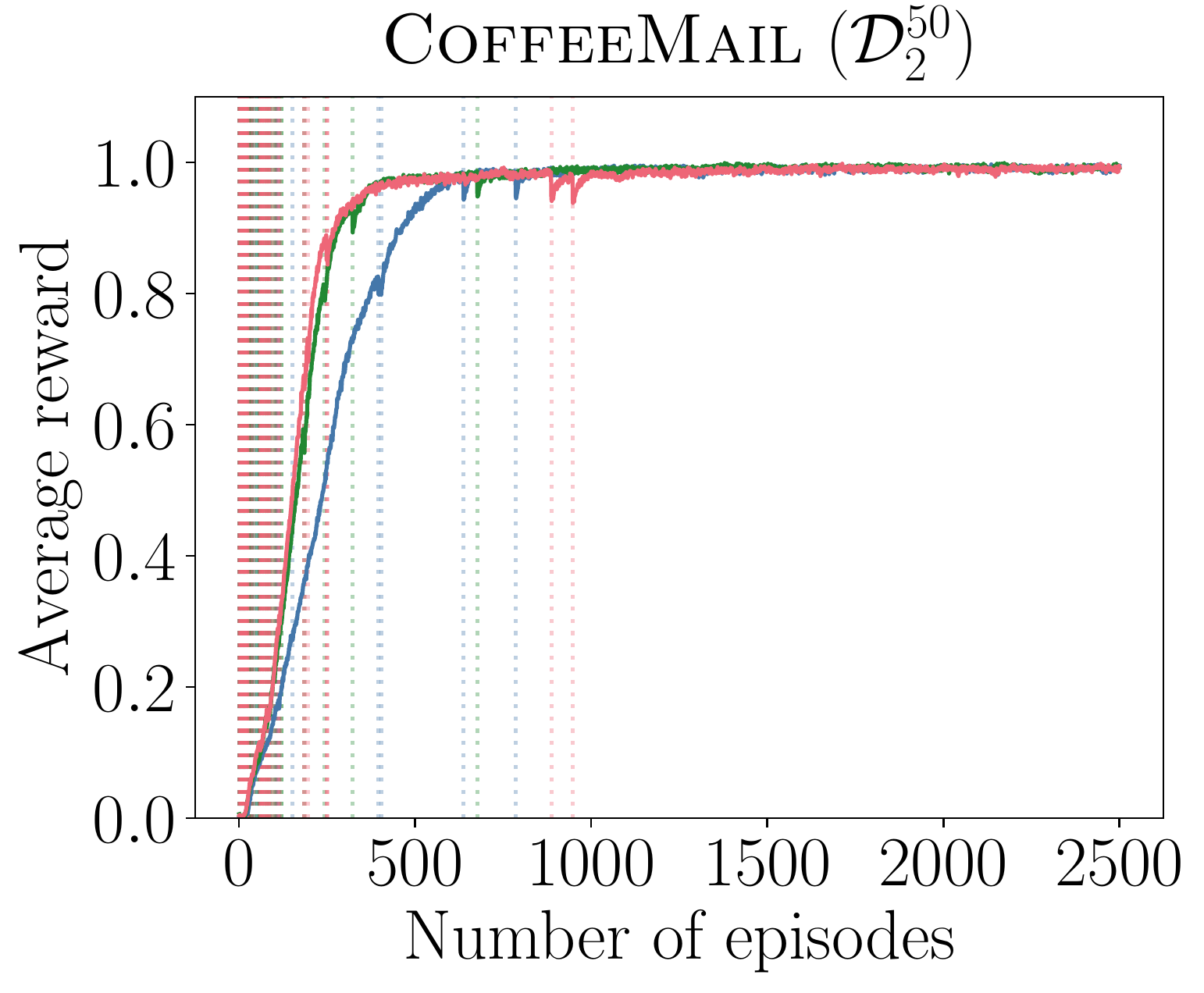}
		}
	}
	\subfloat{
		\resizebox{0.32\columnwidth}{!}{
			\includegraphics{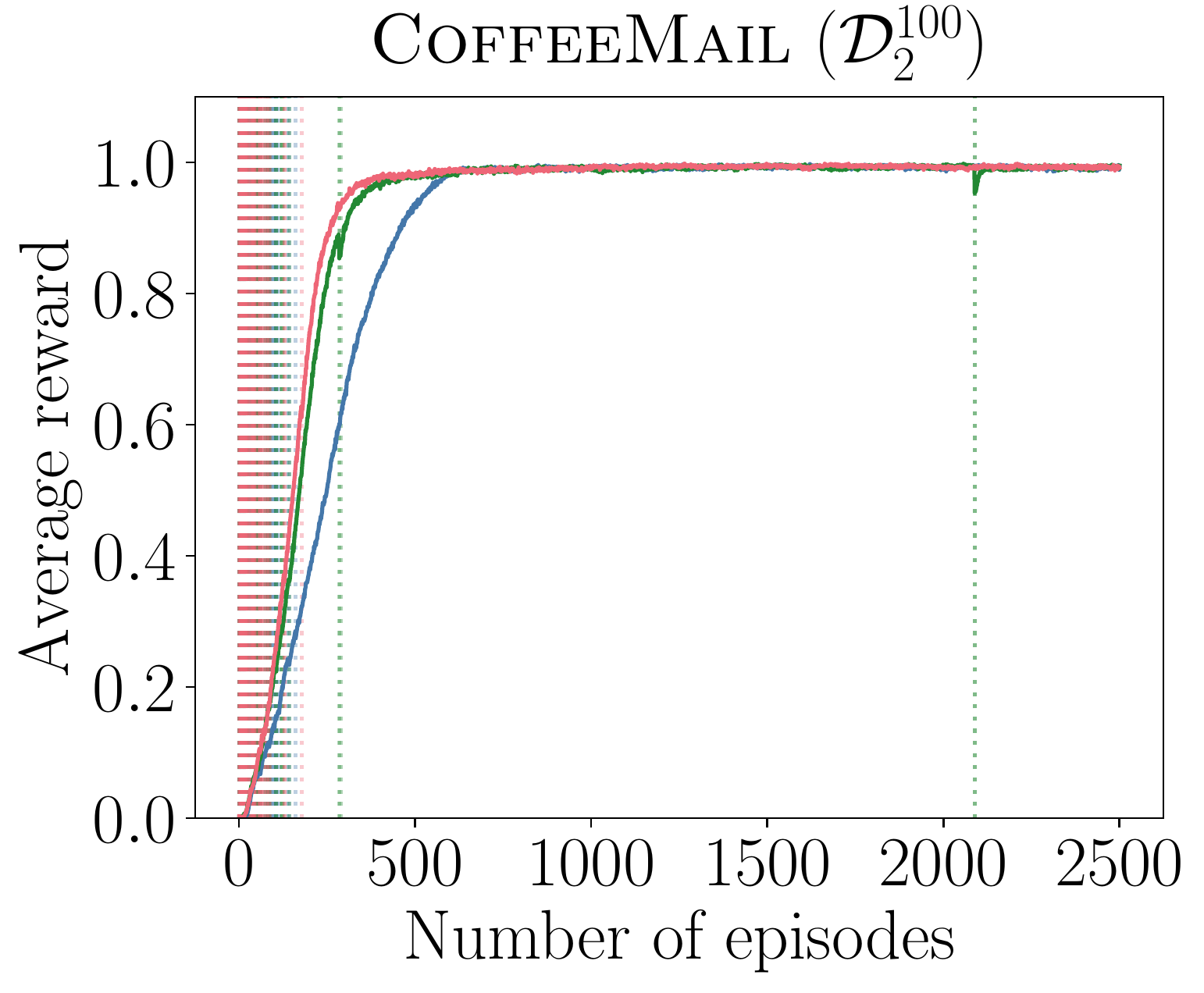}
		}
	}
	\newline
	\subfloat{
		\resizebox{0.32\columnwidth}{!}{
			\includegraphics{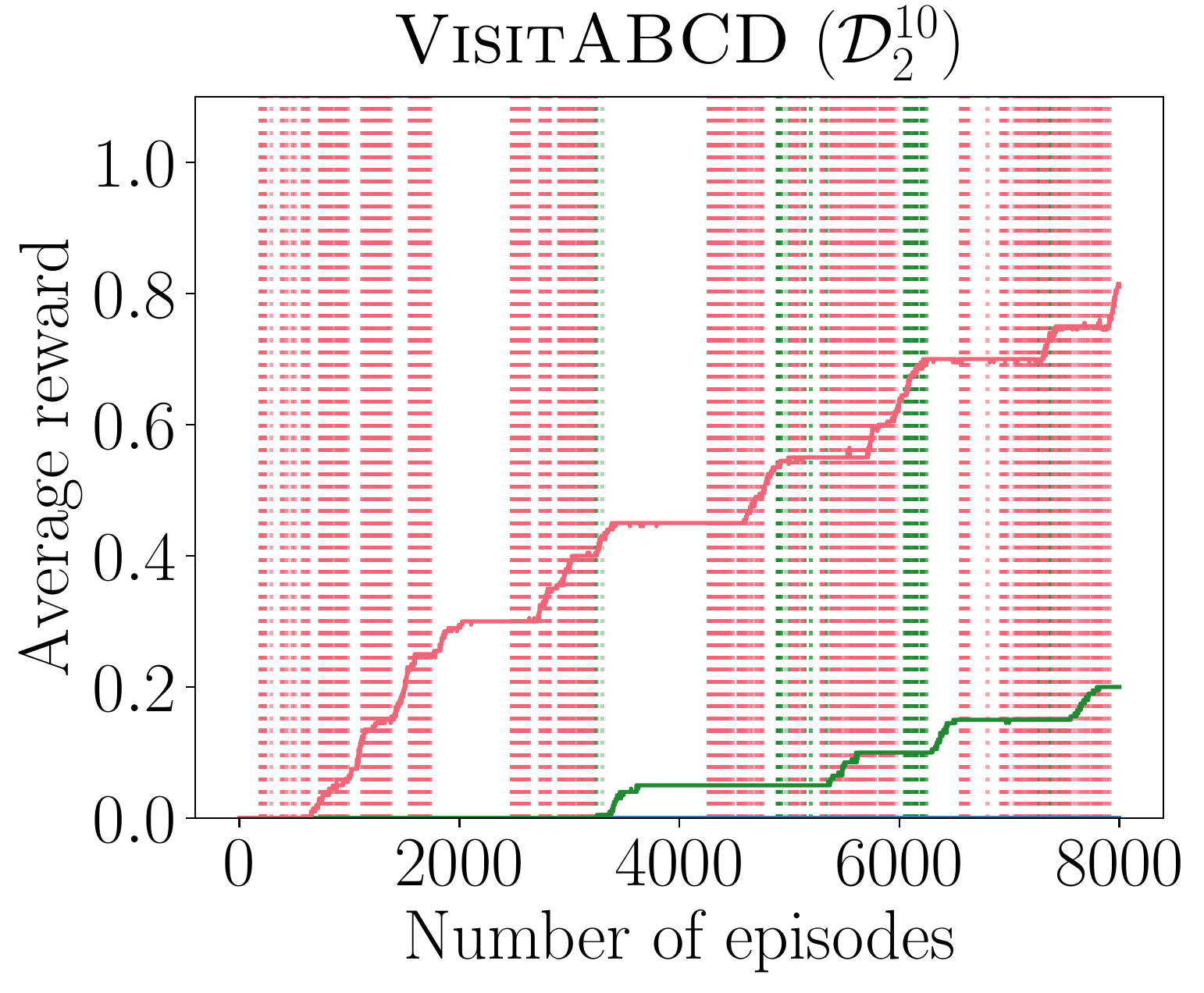}
		}
	}
	\subfloat{
		\resizebox{0.32\columnwidth}{!}{
			\includegraphics{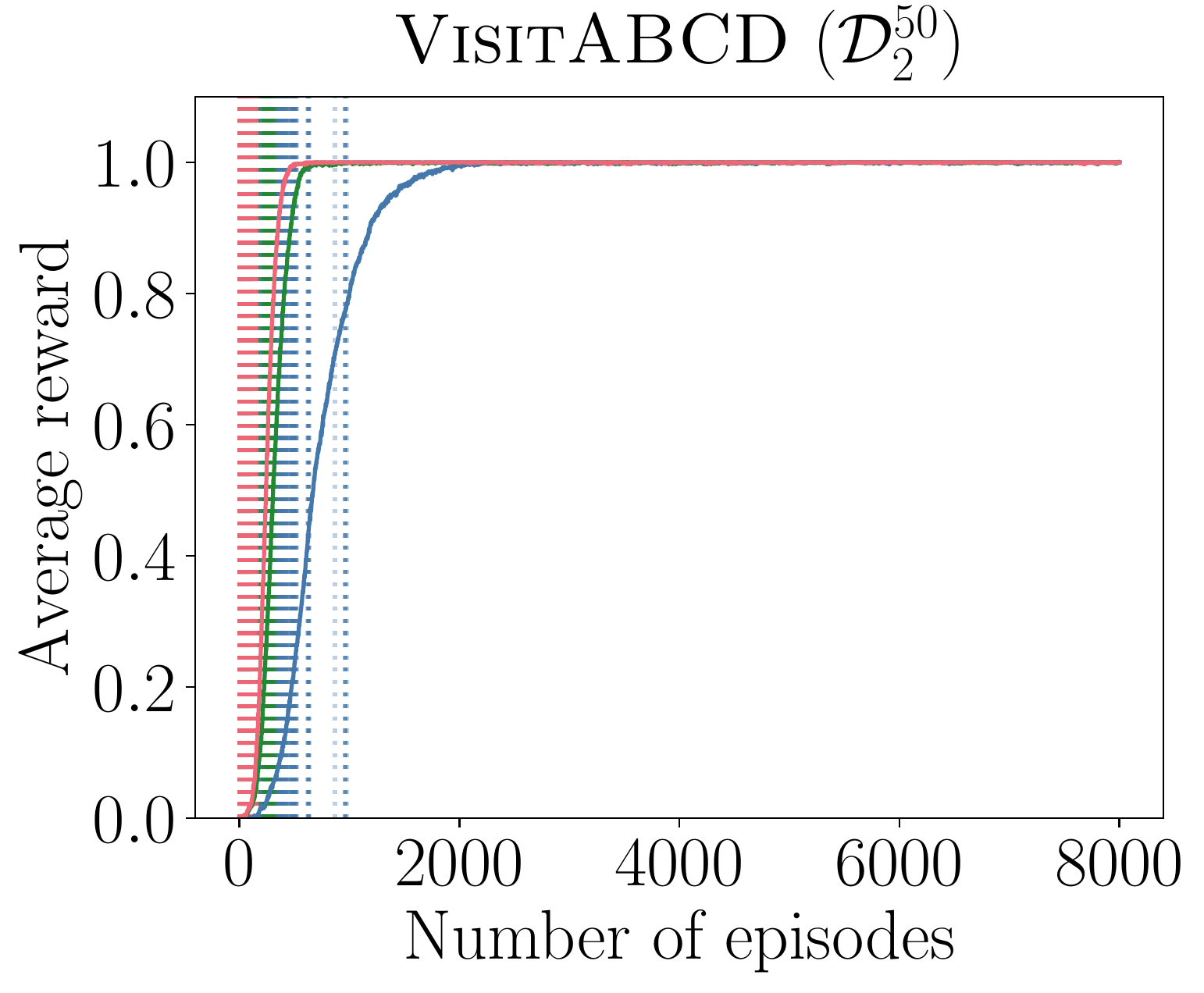}
		}
	}
	\subfloat{
		\resizebox{0.32\columnwidth}{!}{
			\includegraphics{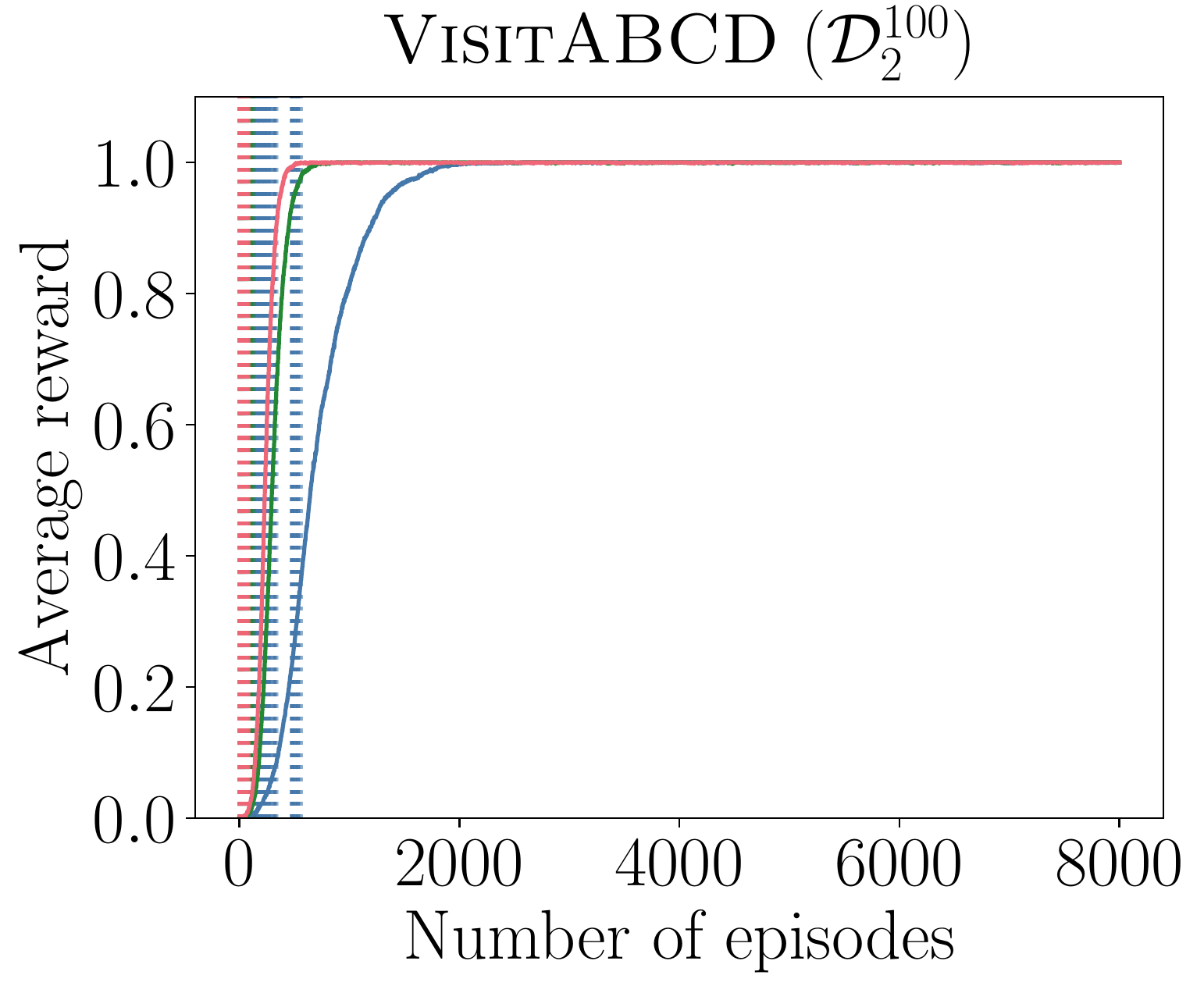}
		}
	}

	\begin{tikzpicture}
	\begin{customlegend}[legend columns=-1,legend style={column sep=1ex},legend cell align={left},legend entries={$N=100$,$N=250$,$N=500$}]
	\addlegendimage{pblue}
	\addlegendimage{pgreen}
	\addlegendimage{pred}
	\end{customlegend}
	\end{tikzpicture}
	\caption{Learning curves for different combinations of POMDP sets ($\mathcal{D}^{10}_2$, $\mathcal{D}^{50}_2$, $\mathcal{D}^{100}_2$) and maximum episode lengths (100, 250, 500).}
	\label{fig:officeworld_num_tasks_steps_dataset2}
\end{figure}

\subsubsection{Trace Compression}
Compressed traces are applicable to the \textsc{OfficeWorld} tasks since (1) empty observations are meaningless (i.e., the number of steps is unimportant), and (2) seeing an observation for more than one step in a row is equivalent to seeing it once. Table~\ref{tab:trace_compression} shows the impact of compressed observation traces in automaton learning. We omit the learning curves since the learned automata (if any) are similar for both types of traces and exhibit close performance. All runs using uncompressed traces finished on time for \textsc{Coffee} and \textsc{CoffeeMail}; in contrast, all such runs timed out for \textsc{VisitABCD}. Crucially, compressed traces allow to learn an automaton for \textsc{VisitABCD} on time. Besides, an automaton is also learned orders of magnitude faster in \textsc{CoffeeMail} using trace compression. Note that a compressed trace is considerably shorter than an uncompressed one even in the simplest task (\textsc{Coffee}).

\begin{table}
	\centering
	\resizebox{\textwidth}{!}{
		\begin{tabular}{lrrrrrrrr}
			\toprule
			& \multicolumn{2}{c}{Time (s.)}                 & & \multicolumn{2}{c}{\# Examples}               & &  \multicolumn{2}{c}{Example Length} \\
			\cmidrule{2-3} \cmidrule{5-6} \cmidrule{8-9}
			& \multicolumn{1}{c}{C} & \multicolumn{1}{c}{U} & & \multicolumn{1}{c}{C} & \multicolumn{1}{c}{U} & & \multicolumn{1}{c}{C} & \multicolumn{1}{c}{U} \\
			\midrule
			\textsc{Coffee}     & 0.4 (0.0)             & 1.5 (0.2)             & & 8.7 (0.4)             & 11.4 (0.4)            & & 2.8 (2.1)             & 58.1 (64.6) \\
			\textsc{CoffeeMail} & 18.9 (3.3)            & 9314.6 (1859.7)       & & 29.0 (1.5)            & 34.1 (1.4)            & & 4.0 (2.6)             & 78.1 (65.7) \\
			\textsc{VisitABCD}  & 163.2 (44.3)          & -                     & & 54.9 (3.8)            &  -                    & & 5.5 (3.1)             & - \\
			\bottomrule
		\end{tabular}
	}
	\caption{Automaton learning statistics when trace compression is on (C) and off (U).}
	\label{tab:trace_compression}
\end{table}

\subsubsection{Restricted Observable Set}
Table~\ref{tab:restricted_observable_set} shows how using the restricted observable set $\hat{\mathcal{O}}$ for each \textsc{OfficeWorld} task compares to using the full observable set $\mathcal{O}$. The learning curves are very similar for both cases, so we do not report them. The restricted observable set $\hat{\mathcal{O}}$ causes a sensible decrease in the automaton learning time, specially for the harder tasks. Similarly, fewer examples are needed to learn a helpful automaton and the average example length is also reduced. Intuitively, using only the observables that describe the subgoals eases the learning: the hypothesis space is smaller and no examples to discard irrelevant observables are needed.
\begin{table}
	\centering
	\resizebox{\textwidth}{!}{
		\begin{tabular}{lrrrrrrrr}
			\toprule
			& \multicolumn{2}{c}{Time (s.)}                                               & & \multicolumn{2}{c}{\# Examples}                                             & &  \multicolumn{2}{c}{Example Length} \\
			\cmidrule{2-3} \cmidrule{5-6} \cmidrule{8-9}
			& \multicolumn{1}{c}{$\mathcal{O}$} & \multicolumn{1}{c}{$\hat{\mathcal{O}}$} & & \multicolumn{1}{c}{$\mathcal{O}$} & \multicolumn{1}{c}{$\hat{\mathcal{O}}$} & & \multicolumn{1}{c}{$\mathcal{O}$} & \multicolumn{1}{c}{$\hat{\mathcal{O}}$} \\
			\midrule
			\textsc{Coffee}     & 0.4 (0.0)                         & 0.3 (0.0)                               & & 8.7 (0.4)                         & 6.4 (0.2)                               & & 2.8 (2.1)                         & 1.5 (0.6)                                \\
			\textsc{CoffeeMail} & 18.9 (3.3)                        & 1.5 (0.1)                               & & 29.0 (1.5)                        & 16.1 (0.6)                              & & 4.0 (2.6)                         & 2.7 (1.4)                                \\
			\textsc{VisitABCD}  & 163.2 (44.3)                      & 9.2 (1.6)                               & & 54.9 (3.8)                        & 30.9 (2.3)                              & & 5.5 (3.1)                         & 3.7 (1.8)                                \\
			\bottomrule
		\end{tabular}
	}
	\caption{Automaton learning statistics when the set of observables is unrestricted ($\mathcal{O}$) or restricted to a particular task ($\hat{\mathcal{O}}$).}
	\label{tab:restricted_observable_set}
\end{table}

\subsubsection{Cyclicity}
Table~\ref{tab:cyclic_vs_acylic_default} shows the effect that enforcing the automaton to be acyclic has on automaton learning. The tasks that we have considered so far do not require cycles. Therefore, we introduce two other tasks, \textsc{CoffeeDrop} and \textsc{CoffeeMailDrop}, whose minimal automata involve cycles to show that our approach can learn such structures. These tasks modify the \textsc{OfficeWorld} dynamics such that when the agent steps on a decoration ($\ast$), the coffee is dropped if the agent holds it. In such case, the agent must go back to the coffee location. Note that in these tasks there are not dead-end states and, thus, the learned automata do not have rejecting states.

The results show a considerable increase in the running time for \textsc{CoffeeMail} and \textsc{VisitABCD} when the automaton is allowed to have cycles. The fact that the hypothesis space is bigger does not only cause an increase in the running times, but also in the number of examples, which must rule out solutions with cycles. In contrast, the average example length remains approximately the same.

Regarding the performance in the additional tasks, the results in the cyclic setting for \textsc{CoffeeDrop} and \textsc{CoffeeMailDrop} are similar to the ones for \textsc{Coffee} and \textsc{CoffeeMail} respectively. Note that the average running time is lower in \textsc{CoffeeMailDrop} than in \textsc{CoffeeMail} since the former has less states than the latter (it does not have a rejecting state). Besides, the length of the examples is longer since there are no dead-end states to actively avoid. In contrast, in the acyclic setting an automaton is only found in 10/20 and 2/20 runs for \textsc{CoffeeDrop} and \textsc{CoffeeMailDrop} respectively. Note that the number of automaton states depends on the trace with the maximum number of times where the coffee has been picked and dropped. Clearly, these tasks are not well-suited to be expressed by an acyclic automaton.

\begin{table}
	\centering
	\resizebox{\textwidth}{!}{
		\begin{tabular}{lrrrrrrrr}
			\toprule
			& \multicolumn{2}{c}{Time (s.)}                                               & & \multicolumn{2}{c}{\# Examples}                                             & &  \multicolumn{2}{c}{Example Length} \\
			\cmidrule{2-3} \cmidrule{5-6} \cmidrule{8-9}
			& \multicolumn{1}{c}{Acyclic}       & \multicolumn{1}{c}{Cyclic}              & & \multicolumn{1}{c}{Acyclic}       & \multicolumn{1}{c}{Cyclic}              & & \multicolumn{1}{c}{Acyclic}       & \multicolumn{1}{c}{Cyclic} \\
			\midrule
			\textsc{Coffee}         & 0.4 (0.0)                         & 0.5 (0.0)                               & & 8.7 (0.4)                         & 9.6 (0.4)                               & & 2.8 (2.1)                         & 2.7 (2.2)                   \\
			\textsc{CoffeeMail}     & 18.9 (3.3)                        & 774.7 (434.4)                           & & 29.0 (1.5)                        & 33.8 (1.7)                              & & 4.0 (2.6)                         & 4.0 (2.6)                   \\
			\textsc{VisitABCD}      & 163.2 (44.3)                      & 1961.7 (1123.8)                         & & 54.9 (3.8)                        & 81.0 (6.6)                              & & 5.5 (3.1)                         & 5.5 (3.3)                   \\
			\midrule
			\textsc{CoffeeDrop}     & 13.9 (8.6)*                       & 0.6 (0.0)                               & & 15.0 (3.1)*                       & 9.9 (0.5)                               & & 5.4 (3.5)*                        & 5.3 (3.6)                            \\
			\textsc{CoffeeMailDrop} & -                                 & 312.2 (145.9)                           & & -                                 & 37.8 (1.7)                              & & -                                 & 7.0 (5.3)                            \\
			\bottomrule
		\end{tabular}
	}
	\caption{Comparison of different automaton learning statistics for the cases where automata must be acyclic and where automata can have cycles.}
	\label{tab:cyclic_vs_acylic_default}
\end{table}	

\subsubsection{Maximum Number of Edges between States}
Table~\ref{tab:number_of_disjunctions} shows the effect that increasing the maximum number of edges between two states from $\kappa=1$ to $\kappa=2$ has on automaton learning. Since the \textsc{OfficeWorld} tasks we consider have at most one edge from a state to another, we add a new task that can only be learned with $\kappa>1$ to show {\methodname} can learn such automata. The \textsc{CoffeeOrMail} task consists in going to the coffee or mail location (it does not matter which) and then go to the office while avoiding the decorations.

The results show that the automaton learner does not scale well to higher values of $\kappa$. The case of \textsc{VisitABCD} is the most notable one: only one run has not timed out with $\kappa=2$. The number of successful runs has also decreased in \textsc{CoffeeMail} from 20 to 15; besides, the running time is orders of magnitude higher. While the increase on the hypothesis space size has caused the running time to vastly increase, the average number of examples and the average example lengths remain similar with respect to $\kappa=1$.

Regarding the \textsc{CoffeeOrMail} task, note that it is not much harder than the \textsc{Coffee} task: a minimal automaton consists of 4 and 5 states for $\kappa=2$ and $\kappa=1$, respectively. Remember that a minimal automaton for the \textsc{Coffee} task consists of 4 states (regardless of $\kappa$). Despite of the increase on the number of states from $\kappa=1$ to $\kappa=2$, the running time and the other statistics remain almost the same. However, while $\kappa=1$ is effective in this case (although it cannot return a minimal automaton), it would not be enough to learn an automaton for any task requiring an edge with a disjunction to the accepting state (e.g., ``observe {\Coffeecup} or {\Letter}'').

\begin{table}
	\centering
	\resizebox{\textwidth}{!}{
		\begin{tabular}{lrrrrrrrr}
			\toprule
			& \multicolumn{2}{c}{Time (s.)}                                    & & \multicolumn{2}{c}{\# Examples}                                 & &  \multicolumn{2}{c}{Example Length} \\
			\cmidrule{2-3} \cmidrule{5-6} \cmidrule{8-9}
			& \multicolumn{1}{c}{$\kappa=1$} & \multicolumn{1}{c}{$\kappa=2$}  & & \multicolumn{1}{c}{$\kappa=1$} & \multicolumn{1}{c}{$\kappa=2$} & & \multicolumn{1}{c}{$\kappa=1$} & \multicolumn{1}{c}{$\kappa=2$} \\
			\midrule
			\textsc{Coffee}       & 0.4 (0.0)                      & 1.0 (0.1)                       & & 8.7 (0.4)                      & 11.7 (0.6)                     & & 2.8 (2.1)                      & 3.1 (2.4) \\
			\textsc{CoffeeMail}   & 18.9 (3.3)                     & 2252.7 (623.2)*                 & & 29.0 (1.5)                     & 32.5 (1.6)*                    & & 4.0 (2.6)                      & 4.0 (2.5)*\\
			\textsc{VisitABCD}    & 163.2 (44.3)                   & -                               & & 54.9 (3.8)                     & -                              & & 5.5 (3.1)                      & - \\
			\midrule
			\textsc{CoffeeOrMail} & 0.9 (0.1)                      & 1.0 (0.1)                       & & 11.7 (0.6)                     & 11.2 (0.4)                     & & 2.7 (1.8)                      &  2.4 (2.1) \\
			\bottomrule
		\end{tabular}
	}
	\caption{Automaton learning statistics for different maximum number of edges from one state to another ($\kappa$).}
	\label{tab:number_of_disjunctions}
\end{table}

\subsubsection{Symmetry Breaking}	
Table~\ref{tab:symmetry_breaking_comparison} shows the effect that symmetry breaking constraints have on the time required to learn an automaton. The average number of examples and example length barely change when symmetry breaking is used, so we do not report the results. The symmetry breaking constraints speed up automaton learning by an order of magnitude in \textsc{CoffeeMail} (acyclic, cyclic) and \textsc{VisitABCD} (acyclic).  Furthermore, while two runs have timed out for the case where automata can contain cycles and no symmetry breaking is used (one for \textsc{CoffeeMail} and one for \textsc{VisitABCD}), no run has timed out when the symmetry breaking is used.

\begin{table}
	\centering
	\begin{tabular}{lrrrrr}
		\toprule
		& \multicolumn{2}{c}{Acyclic}                        & & \multicolumn{2}{c}{Cyclic} \\
		\cmidrule{2-3} \cmidrule{5-6}
		& \multicolumn{1}{c}{No SB} & \multicolumn{1}{c}{SB} & & \multicolumn{1}{c}{No SB} & \multicolumn{1}{c}{SB} \\
		\midrule
		\textsc{Coffee}     & 0.5 (0.0)                 & 0.4 (0.0)              & & 0.5 (0.0)                 & 0.5 (0.0)              \\
		\textsc{CoffeeMail} & 277.4 (70.2)              & 18.9 (3.3)             & & 4204.3 (1334.4)*          & 774.7 (434.4)          \\
		\textsc{VisitABCD}  & 1070.0 (725.6)            & 163.2 (44.3)           & & 3293.5 (1199.2)*          & 1961.7 (1123.8)        \\
		\bottomrule
	\end{tabular}
	\caption{Total automaton learning time when symmetry breaking is disabled (No SB) and enabled (SB).}
	\label{tab:symmetry_breaking_comparison}
\end{table}	

\vskip 0.2in
\bibliography{journal_paper}

\begin{thebibliography}{}

\bibitem[\protect\BCAY{Andreas, Klein,\ \BBA\ Levine}{Andreas
  et~al.}{2017}]{AndreasKL17}
Andreas, J., Klein, D., \BBA\ Levine, S. \BBOP2017\BBCP.
\newblock \BBOQ {Modular Multitask Reinforcement Learning with Policy
  Sketches}\BBCQ\
\newblock In {\Bem Proceedings of the International Conference on Machine
  Learning {(ICML)}}, \BPGS\ 166--175.

\bibitem[\protect\BCAY{Angluin}{Angluin}{1980}]{Angluin80}
Angluin, D. \BBOP1980\BBCP.
\newblock \BBOQ {Inductive Inference of Formal Languages from Positive
  Data}\BBCQ\
\newblock {\Bem Inf. Control.}, {\Bem 45\/}(2), 117--135.

\bibitem[\protect\BCAY{Angluin}{Angluin}{1987}]{Angluin87}
Angluin, D. \BBOP1987\BBCP.
\newblock \BBOQ {Learning Regular Sets from Queries and Counterexamples}\BBCQ\
\newblock {\Bem Inf. Comput.}, {\Bem 75\/}(2), 87--106.

\bibitem[\protect\BCAY{Barto\ \BBA\ Mahadevan}{Barto\ \BBA\
  Mahadevan}{2003}]{BartoM03a}
Barto, A.~G.\BBACOMMA\  \BBA\ Mahadevan, S. \BBOP2003\BBCP.
\newblock \BBOQ {Recent Advances in Hierarchical Reinforcement Learning}\BBCQ\
\newblock {\Bem Discrete Event Dynamic Systems}, {\Bem 13\/}(4), 341--379.

\bibitem[\protect\BCAY{Bonet, Palacios,\ \BBA\ Geffner}{Bonet
  et~al.}{2009}]{BonetPG09}
Bonet, B., Palacios, H., \BBA\ Geffner, H. \BBOP2009\BBCP.
\newblock \BBOQ {Automatic Derivation of Memoryless Policies and Finite-State
  Controllers Using Classical Planners}\BBCQ\
\newblock In {\Bem Proceedings of the International Conference on Automated
  Planning and Scheduling ({ICAPS})}.

\bibitem[\protect\BCAY{Bradtke\ \BBA\ Duff}{Bradtke\ \BBA\
  Duff}{1994}]{BradtkeD94}
Bradtke, S.~J.\BBACOMMA\  \BBA\ Duff, M.~O. \BBOP1994\BBCP.
\newblock \BBOQ {Reinforcement Learning Methods for Continuous-Time Markov
  Decision Problems}\BBCQ\
\newblock In {\Bem Proceedings of the Advances in Neural Information Processing
  Systems ({NeurIPS}) Conference}, \BPGS\ 393--400.

\bibitem[\protect\BCAY{Brafman\ \BBA\ Tennenholtz}{Brafman\ \BBA\
  Tennenholtz}{2002}]{BrafmanT02}
Brafman, R.~I.\BBACOMMA\  \BBA\ Tennenholtz, M. \BBOP2002\BBCP.
\newblock \BBOQ {{R-MAX} - {A} General Polynomial Time Algorithm for
  Near-Optimal Reinforcement Learning}\BBCQ\
\newblock {\Bem J. Mach. Learn. Res.}, {\Bem 3}, 213--231.

\bibitem[\protect\BCAY{Brooks}{Brooks}{1989}]{Brooks89}
Brooks, R.~A. \BBOP1989\BBCP.
\newblock \BBOQ {A Robot that Walks; Emergent Behaviors from a Carefully
  Evolved Network}\BBCQ\
\newblock {\Bem Neural Computation}, {\Bem 1\/}(2), 253--262.

\bibitem[\protect\BCAY{Buckland}{Buckland}{2004}]{Buckland04}
Buckland, M. \BBOP2004\BBCP.
\newblock {\Bem {AI Game Programming by Example}}.
\newblock Wordware Publishing Inc.

\bibitem[\protect\BCAY{Calimeri, Faber, Gebser, Ianni, Kaminski, Krennwallner,
  Leone, Maratea, Ricca,\ \BBA\ Schaub}{Calimeri
  et~al.}{2020}]{CalimeriFGIKKLM20}
Calimeri, F., Faber, W., Gebser, M., Ianni, G., Kaminski, R., Krennwallner, T.,
  Leone, N., Maratea, M., Ricca, F., \BBA\ Schaub, T. \BBOP2020\BBCP.
\newblock \BBOQ {ASP-Core-2 Input Language Format}\BBCQ\
\newblock {\Bem Theory Pract. Log. Program.}, {\Bem 20\/}(2), 294--309.

\bibitem[\protect\BCAY{Camacho, Chen, Sanner,\ \BBA\ McIlraith}{Camacho
  et~al.}{2017}]{CamachoCSM17}
Camacho, A., Chen, O., Sanner, S., \BBA\ McIlraith, S.~A. \BBOP2017\BBCP.
\newblock \BBOQ {Non-Markovian Rewards Expressed in {LTL:} Guiding Search Via
  Reward Shaping}\BBCQ\
\newblock In {\Bem Proceedings of the International Symposium on Combinatorial
  Search ({SOCS})}, \BPGS\ 159--160.

\bibitem[\protect\BCAY{Camacho, Icarte, Klassen, Valenzano,\ \BBA\
  McIlraith}{Camacho et~al.}{2019}]{CamachoIKVM19}
Camacho, A., Icarte, R.~T., Klassen, T.~Q., Valenzano, R.~A., \BBA\ McIlraith,
  S.~A. \BBOP2019\BBCP.
\newblock \BBOQ {{LTL} and Beyond: Formal Languages for Reward Function
  Specification in Reinforcement Learning}\BBCQ\
\newblock In {\Bem Proceedings of the International Joint Conference on
  Artificial Intelligence ({IJCAI})}, \BPGS\ 6065--6073.

\bibitem[\protect\BCAY{Codish, Miller, Prosser,\ \BBA\ Stuckey}{Codish
  et~al.}{2019}]{CodishMPS19}
Codish, M., Miller, A., Prosser, P., \BBA\ Stuckey, P.~J. \BBOP2019\BBCP.
\newblock \BBOQ {Constraints for symmetry breaking in graph
  representation}\BBCQ\
\newblock {\Bem Constraints}, {\Bem 24\/}(1), 1--24.

\bibitem[\protect\BCAY{de~la Higuera}{de~la Higuera}{2010}]{DeLaHiguera10}
de~la Higuera, C. \BBOP2010\BBCP.
\newblock {\Bem {Grammatical Inference: Learning Automata and Grammars}}.
\newblock Cambridge University Press.

\bibitem[\protect\BCAY{Dietterich}{Dietterich}{2000}]{Dietterich00}
Dietterich, T.~G. \BBOP2000\BBCP.
\newblock \BBOQ {Hierarchical Reinforcement Learning with the {MAXQ} Value
  Function Decomposition}\BBCQ\
\newblock {\Bem J. Artif. Intell. Res.}, {\Bem 13}, 227--303.

\bibitem[\protect\BCAY{Drescher, Tifrea,\ \BBA\ Walsh}{Drescher
  et~al.}{2011}]{DrescherTW11}
Drescher, C., Tifrea, O., \BBA\ Walsh, T. \BBOP2011\BBCP.
\newblock \BBOQ {Symmetry-breaking Answer Set Solving}\BBCQ\
\newblock {\Bem {AI} Commun.}, {\Bem 24\/}(2), 177--194.

\bibitem[\protect\BCAY{Furelos{-}Blanco, Law, Russo, Broda,\ \BBA\
  Jonsson}{Furelos{-}Blanco et~al.}{2020}]{furelosblanco2020aaai}
Furelos{-}Blanco, D., Law, M., Russo, A., Broda, K., \BBA\ Jonsson, A.
  \BBOP2020\BBCP.
\newblock \BBOQ {Induction of Subgoal Automata for Reinforcement
  Learning}\BBCQ\
\newblock In {\Bem Proceedings of the {AAAI} Conference on Artificial
  Intelligence ({AAAI})}, \BPGS\ 3890--3897.

\bibitem[\protect\BCAY{Gaon\ \BBA\ Brafman}{Gaon\ \BBA\
  Brafman}{2020}]{GaonB20}
Gaon, M.\BBACOMMA\  \BBA\ Brafman, R.~I. \BBOP2020\BBCP.
\newblock \BBOQ {Reinforcement Learning with Non-Markovian Rewards}\BBCQ\
\newblock In {\Bem Proceedings of the {AAAI} Conference on Artificial
  Intelligence ({AAAI})}, \BPGS\ 3980--3987.

\bibitem[\protect\BCAY{Gelfond\ \BBA\ Kahl}{Gelfond\ \BBA\
  Kahl}{2014}]{GelfondK14}
Gelfond, M.\BBACOMMA\  \BBA\ Kahl, Y. \BBOP2014\BBCP.
\newblock {\Bem {Knowledge Representation, Reasoning, and the Design of
  Intelligent Agents: The Answer-Set Programming Approach}}.
\newblock Cambridge University Press.

\bibitem[\protect\BCAY{Gelfond\ \BBA\ Lifschitz}{Gelfond\ \BBA\
  Lifschitz}{1988}]{GelfondL88}
Gelfond, M.\BBACOMMA\  \BBA\ Lifschitz, V. \BBOP1988\BBCP.
\newblock \BBOQ {The Stable Model Semantics for Logic Programming}\BBCQ\
\newblock In {\Bem Proceedings of the International Conference and Symposium on
  Logic Programming}, \BPGS\ 1070--1080.

\bibitem[\protect\BCAY{Gold}{Gold}{1978}]{Gold78}
Gold, E.~M. \BBOP1978\BBCP.
\newblock \BBOQ {Complexity of Automaton Identification from Given Data}\BBCQ\
\newblock {\Bem Information and Control}, {\Bem 37\/}(3), 302--320.

\bibitem[\protect\BCAY{Heule\ \BBA\ Verwer}{Heule\ \BBA\
  Verwer}{2010}]{HeuleV10}
Heule, M.\BBACOMMA\  \BBA\ Verwer, S. \BBOP2010\BBCP.
\newblock \BBOQ {Exact {DFA} Identification Using {SAT} Solvers}\BBCQ\
\newblock In {\Bem Proceedings of the International Colloquium on Grammatical
  Inference: Algorithms and Applications ({ICGI})}, \BPGS\ 66--79.

\bibitem[\protect\BCAY{Ho, Abel, Griffiths,\ \BBA\ Littman}{Ho
  et~al.}{2019}]{HoAGL19}
Ho, M.~K., Abel, D., Griffiths, T.~L., \BBA\ Littman, M.~L. \BBOP2019\BBCP.
\newblock \BBOQ The value of abstraction\BBCQ\
\newblock {\Bem Current Opinion in Behavioral Sciences}, {\Bem 29}, 111 -- 116.

\bibitem[\protect\BCAY{Hu\ \BBA\ {De Giacomo}}{Hu\ \BBA\ {De
  Giacomo}}{2011}]{HuG11}
Hu, Y.\BBACOMMA\  \BBA\ {De Giacomo}, G. \BBOP2011\BBCP.
\newblock \BBOQ {Generalized Planning: Synthesizing Plans that Work for
  Multiple Environments}\BBCQ\
\newblock In {\Bem Proceedings of the International Joint Conference on
  Artificial Intelligence ({IJCAI})}, \BPGS\ 918--923.

\bibitem[\protect\BCAY{Kaelbling}{Kaelbling}{1993}]{Kaelbling93b}
Kaelbling, L.~P. \BBOP1993\BBCP.
\newblock \BBOQ {Hierarchical Learning in Stochastic Domains: Preliminary
  Results}\BBCQ\
\newblock In {\Bem Proceedings of the International Conference on Machine
  Learning (ICML)}, \BPGS\ 167--173.

\bibitem[\protect\BCAY{Kaelbling, Littman,\ \BBA\ Cassandra}{Kaelbling
  et~al.}{1998}]{KaelblingLC98}
Kaelbling, L.~P., Littman, M.~L., \BBA\ Cassandra, A.~R. \BBOP1998\BBCP.
\newblock \BBOQ {Planning and Acting in Partially Observable Stochastic
  Domains}\BBCQ\
\newblock {\Bem Artif. Intell.}, {\Bem 101\/}(1-2), 99--134.

\bibitem[\protect\BCAY{Kingma\ \BBA\ Ba}{Kingma\ \BBA\ Ba}{2015}]{KingmaB15}
Kingma, D.~P.\BBACOMMA\  \BBA\ Ba, J. \BBOP2015\BBCP.
\newblock \BBOQ {Adam: {A} Method for Stochastic Optimization}\BBCQ\
\newblock In {\Bem Proceedings of the International Conference on Learning
  Representations ({ICLR})}.

\bibitem[\protect\BCAY{Konidaris}{Konidaris}{2019}]{Konidaris19}
Konidaris, G. \BBOP2019\BBCP.
\newblock \BBOQ On the necessity of abstraction\BBCQ\
\newblock {\Bem Current Opinion in Behavioral Sciences}, {\Bem 29}, 1 -- 7.

\bibitem[\protect\BCAY{Koul, Fern,\ \BBA\ Greydanus}{Koul
  et~al.}{2019}]{KoulFG19}
Koul, A., Fern, A., \BBA\ Greydanus, S. \BBOP2019\BBCP.
\newblock \BBOQ {Learning Finite State Representations of Recurrent Policy
  Networks}\BBCQ\
\newblock In {\Bem Proceedings of the International Conference on Learning
  Representations ({ICLR})}.

\bibitem[\protect\BCAY{Kulkarni, Gupta, Ionescu, Borgeaud, Reynolds,
  Zisserman,\ \BBA\ Mnih}{Kulkarni et~al.}{2019}]{KulkarniGIBRZM19}
Kulkarni, T.~D., Gupta, A., Ionescu, C., Borgeaud, S., Reynolds, M., Zisserman,
  A., \BBA\ Mnih, V. \BBOP2019\BBCP.
\newblock \BBOQ {Unsupervised Learning of Object Keypoints for Perception and
  Control}\BBCQ\
\newblock In {\Bem Proceedings of the Advances in Neural Information Processing
  Systems ({NeurIPS}) Conference}, \BPGS\ 10723--10733.

\bibitem[\protect\BCAY{Lambeau, Damas,\ \BBA\ Dupont}{Lambeau
  et~al.}{2008}]{LambeauDD08}
Lambeau, B., Damas, C., \BBA\ Dupont, P. \BBOP2008\BBCP.
\newblock \BBOQ {State-Merging {DFA} Induction Algorithms with Mandatory Merge
  Constraints}\BBCQ\
\newblock In {\Bem Proceedings of the International Colloquium on Grammatical
  Inference: Algorithms and Applications ({ICGI})}, \BPGS\ 139--153.

\bibitem[\protect\BCAY{Lang, Pearlmutter,\ \BBA\ Price}{Lang
  et~al.}{1998}]{LangPP98}
Lang, K.~J., Pearlmutter, B.~A., \BBA\ Price, R.~A. \BBOP1998\BBCP.
\newblock \BBOQ {Results of the Abbadingo One {DFA} Learning Competition and a
  New Evidence-Driven State Merging Algorithm}\BBCQ\
\newblock In {\Bem Proceedings of the International Colloquium on Grammatical
  Inference: Algorithms and Applications ({ICGI})}, \BPGS\ 1--12.

\bibitem[\protect\BCAY{Lange\ \BBA\ Faisal}{Lange\ \BBA\
  Faisal}{2019}]{LangeF19}
Lange, R.~T.\BBACOMMA\  \BBA\ Faisal, A. \BBOP2019\BBCP.
\newblock \BBOQ {Semantic RL with Action Grammars: Data-Efficient Learning of
  Hierarchical Task Abstractions}\BBCQ\
\newblock {\Bem CoRR}, {\Bem abs/1907.12477}.

\bibitem[\protect\BCAY{Law, Russo,\ \BBA\ Broda}{Law et~al.}{2015a}]{LawRB15}
Law, M., Russo, A., \BBA\ Broda, K. \BBOP2015a\BBCP.
\newblock \BBOQ {Simplified Reduct for Choice Rules in ASP}\BBCQ\
\newblock \BTR, DTR2015-2, Imperial College of Science, Technology and
  Medicine, Department of Computing.

\bibitem[\protect\BCAY{Law, Russo,\ \BBA\ Broda}{Law
  et~al.}{2015b}]{ILASP_system}
Law, M., Russo, A., \BBA\ Broda, K. \BBOP2015b\BBCP.
\newblock \BBOQ {T}he {ILASP} {S}ystem for {L}earning {A}nswer {S}et
  {P}rograms\BBCQ.

\bibitem[\protect\BCAY{Law, Russo,\ \BBA\ Broda}{Law et~al.}{2016}]{LawRB16}
Law, M., Russo, A., \BBA\ Broda, K. \BBOP2016\BBCP.
\newblock \BBOQ {Iterative Learning of Answer Set Programs from Context
  Dependent Examples}\BBCQ\
\newblock {\Bem Theory Pract. Log. Program.}, {\Bem 16\/}(5-6), 834--848.

\bibitem[\protect\BCAY{Law, Russo,\ \BBA\ Broda}{Law et~al.}{2018}]{LawRB18}
Law, M., Russo, A., \BBA\ Broda, K. \BBOP2018\BBCP.
\newblock \BBOQ {The Meta-program Injection Feature in ILASP}\BBCQ\
\newblock \BTR.

\bibitem[\protect\BCAY{Leonetti, Iocchi,\ \BBA\ Patrizi}{Leonetti
  et~al.}{2012}]{LeonettiIP12}
Leonetti, M., Iocchi, L., \BBA\ Patrizi, F. \BBOP2012\BBCP.
\newblock \BBOQ {Automatic Generation and Learning of Finite-State
  Controllers}\BBCQ\
\newblock In {\Bem Proceedings of the International Conference on Artificial
  Intelligence: Methodology, Systems, Applications ({AIMSA})}, \BPGS\ 135--144.

\bibitem[\protect\BCAY{Machado, Bellemare,\ \BBA\ Bowling}{Machado
  et~al.}{2017}]{MachadoBB17}
Machado, M.~C., Bellemare, M.~G., \BBA\ Bowling, M.~H. \BBOP2017\BBCP.
\newblock \BBOQ {A Laplacian Framework for Option Discovery in Reinforcement
  Learning}\BBCQ\
\newblock In {\Bem Proceedings of the International Conference on Machine
  Learning ({ICML})}, \BPGS\ 2295--2304.

\bibitem[\protect\BCAY{McGovern\ \BBA\ Barto}{McGovern\ \BBA\
  Barto}{2001}]{McGovernB01}
McGovern, A.\BBACOMMA\  \BBA\ Barto, A.~G. \BBOP2001\BBCP.
\newblock \BBOQ {Automatic Discovery of Subgoals in Reinforcement Learning
  using Diverse Density}\BBCQ\
\newblock In {\Bem Proceedings of the International Conference on Machine
  Learning {(ICML)}}, \BPGS\ 361--368.

\bibitem[\protect\BCAY{Mehta, Ray, Tadepalli,\ \BBA\ Dietterich}{Mehta
  et~al.}{2008}]{MehtaRTD08}
Mehta, N., Ray, S., Tadepalli, P., \BBA\ Dietterich, T.~G. \BBOP2008\BBCP.
\newblock \BBOQ {Automatic Discovery and Transfer of {MAXQ} Hierarchies}\BBCQ\
\newblock In {\Bem Proceedings of the International Conference on Machine
  Learning {(ICML)}}, \BPGS\ 648--655.

\bibitem[\protect\BCAY{Menache, Mannor,\ \BBA\ Shimkin}{Menache
  et~al.}{2002}]{MenacheMS02}
Menache, I., Mannor, S., \BBA\ Shimkin, N. \BBOP2002\BBCP.
\newblock \BBOQ {Q-Cut - Dynamic Discovery of Sub-goals in Reinforcement
  Learning}\BBCQ\
\newblock In {\Bem Proceedings of the European Conference on Machine Learning
  ({ECML})}, \BPGS\ 295--306.

\bibitem[\protect\BCAY{Meuleau, Peshkin, Kim,\ \BBA\ Kaelbling}{Meuleau
  et~al.}{1999}]{MeuleauPKK99}
Meuleau, N., Peshkin, L., Kim, K., \BBA\ Kaelbling, L.~P. \BBOP1999\BBCP.
\newblock \BBOQ {Learning Finite-State Controllers for Partially Observable
  Environments}\BBCQ\
\newblock In {\Bem Proceedings of the Conference on Uncertainty in Artificial
  Intelligence ({UAI})}, \BPGS\ 427--436.

\bibitem[\protect\BCAY{Michalenko, Shah, Verma, Baraniuk, Chaudhuri,\ \BBA\
  Patel}{Michalenko et~al.}{2019}]{MichalenkoSVBCP19}
Michalenko, J.~J., Shah, A., Verma, A., Baraniuk, R.~G., Chaudhuri, S., \BBA\
  Patel, A.~B. \BBOP2019\BBCP.
\newblock \BBOQ {Representing Formal Languages: {A} Comparison Between Finite
  Automata and Recurrent Neural Networks}\BBCQ\
\newblock In {\Bem Proceedings of the International Conference on Learning
  Representations ({ICLR})}.

\bibitem[\protect\BCAY{Mnih, Kavukcuoglu, Silver, Rusu, Veness, Bellemare,
  Graves, Riedmiller, Fidjeland, Ostrovski, Petersen, Beattie, Sadik,
  Antonoglou, King, Kumaran, Wierstra, Legg,\ \BBA\ Hassabis}{Mnih
  et~al.}{2015}]{MnihKSRVBGRFOPB15}
Mnih, V., Kavukcuoglu, K., Silver, D., Rusu, A.~A., Veness, J., Bellemare,
  M.~G., Graves, A., Riedmiller, M.~A., Fidjeland, A., Ostrovski, G., Petersen,
  S., Beattie, C., Sadik, A., Antonoglou, I., King, H., Kumaran, D., Wierstra,
  D., Legg, S., \BBA\ Hassabis, D. \BBOP2015\BBCP.
\newblock \BBOQ {Human-level control through deep reinforcement learning}\BBCQ\
\newblock {\Bem Nature}, {\Bem 518\/}(7540), 529--533.

\bibitem[\protect\BCAY{Ng, Harada,\ \BBA\ Russell}{Ng et~al.}{1999}]{NgHR99}
Ng, A.~Y., Harada, D., \BBA\ Russell, S.~J. \BBOP1999\BBCP.
\newblock \BBOQ {Policy Invariance Under Reward Transformations: Theory and
  Application to Reward Shaping}\BBCQ\
\newblock In {\Bem Proceedings of the International Conference on Machine
  Learning {(ICML)}}, \BPGS\ 278--287.

\bibitem[\protect\BCAY{Parr\ \BBA\ Russell}{Parr\ \BBA\
  Russell}{1997}]{ParrR97}
Parr, R.\BBACOMMA\  \BBA\ Russell, S.~J. \BBOP1997\BBCP.
\newblock \BBOQ {Reinforcement Learning with Hierarchies of Machines}\BBCQ\
\newblock In {\Bem Proceedings of the Advances in Neural Information Processing
  Systems ({NeurIPS}) Conference}, \BPGS\ 1043--1049.

\bibitem[\protect\BCAY{{Segovia Aguas}, Jim{\'{e}}nez,\ \BBA\ Jonsson}{{Segovia
  Aguas} et~al.}{2018}]{Aguas0J18}
{Segovia Aguas}, J., Jim{\'{e}}nez, S., \BBA\ Jonsson, A. \BBOP2018\BBCP.
\newblock \BBOQ {Computing Hierarchical Finite State Controllers With Classical
  Planning}\BBCQ\
\newblock {\Bem J. Artif. Intell. Res.}, {\Bem 62}, 755--797.

\bibitem[\protect\BCAY{Silver, Hubert, Schrittwieser, Antonoglou, Lai, Guez,
  Lanctot, Sifre, Kumaran, Graepel, Lillicrap, Simonyan,\ \BBA\
  Hassabis}{Silver et~al.}{2018}]{Silver18}
Silver, D., Hubert, T., Schrittwieser, J., Antonoglou, I., Lai, M., Guez, A.,
  Lanctot, M., Sifre, L., Kumaran, D., Graepel, T., Lillicrap, T., Simonyan,
  K., \BBA\ Hassabis, D. \BBOP2018\BBCP.
\newblock \BBOQ {A general reinforcement learning algorithm that masters chess,
  shogi, and Go through self-play}\BBCQ\
\newblock {\Bem Science}, {\Bem 362\/}(6419), 1140--1144.

\bibitem[\protect\BCAY{Simsek\ \BBA\ Barto}{Simsek\ \BBA\
  Barto}{2004}]{SimsekB04}
Simsek, {\"{O}}.\BBACOMMA\  \BBA\ Barto, A.~G. \BBOP2004\BBCP.
\newblock \BBOQ {Using Relative Novelty to Identify Useful Temporal
  Abstractions in Reinforcement Learning}\BBCQ\
\newblock In {\Bem Proceedings of the International Conference on Machine
  Learning {(ICML)}}.

\bibitem[\protect\BCAY{Simsek, Wolfe,\ \BBA\ Barto}{Simsek
  et~al.}{2005}]{SimsekWB05}
Simsek, {\"{O}}., Wolfe, A.~P., \BBA\ Barto, A.~G. \BBOP2005\BBCP.
\newblock \BBOQ {Identifying Useful Subgoals in Reinforcement Learning by Local
  Graph Partitioning}\BBCQ\
\newblock In {\Bem Proceedings of the International Conference on Machine
  Learning {(ICML)}}, \BPGS\ 816--823.

\bibitem[\protect\BCAY{Srinivasan}{Srinivasan}{2001}]{Srinivasan01}
Srinivasan, A. \BBOP2001\BBCP.
\newblock \BBOQ {The Aleph Manual}\BBCQ.

\bibitem[\protect\BCAY{Stolle\ \BBA\ Precup}{Stolle\ \BBA\
  Precup}{2002}]{StolleP02}
Stolle, M.\BBACOMMA\  \BBA\ Precup, D. \BBOP2002\BBCP.
\newblock \BBOQ {Learning Options in Reinforcement Learning}\BBCQ\
\newblock In {\Bem Proceedings of the International Symposium on Abstraction,
  Reformulation and Approximation ({SARA})}, \BPGS\ 212--223.

\bibitem[\protect\BCAY{Sutton\ \BBA\ Barto}{Sutton\ \BBA\
  Barto}{1998}]{SuttonB98}
Sutton, R.~S.\BBACOMMA\  \BBA\ Barto, A.~G. \BBOP1998\BBCP.
\newblock {\Bem Reinforcement Learning: An Introduction}.
\newblock {MIT} Press.

\bibitem[\protect\BCAY{Sutton, Precup,\ \BBA\ Singh}{Sutton
  et~al.}{1998}]{SuttonPS98}
Sutton, R.~S., Precup, D., \BBA\ Singh, S.~P. \BBOP1998\BBCP.
\newblock \BBOQ {Intra-Option Learning about Temporally Abstract Actions}\BBCQ\
\newblock In {\Bem Proceedings of the International Conference on Machine
  Learning {(ICML)}}, \BPGS\ 556--564.

\bibitem[\protect\BCAY{Sutton, Precup,\ \BBA\ Singh}{Sutton
  et~al.}{1999}]{SuttonPS99}
Sutton, R.~S., Precup, D., \BBA\ Singh, S.~P. \BBOP1999\BBCP.
\newblock \BBOQ {Between MDPs and Semi-MDPs: {A} Framework for Temporal
  Abstraction in Reinforcement Learning}\BBCQ\
\newblock {\Bem Artif. Intell.}, {\Bem 112\/}(1-2), 181--211.

\bibitem[\protect\BCAY{{Toro Icarte}, Klassen, Valenzano,\ \BBA\
  McIlraith}{{Toro Icarte} et~al.}{2018}]{IcarteKVM18}
{Toro Icarte}, R., Klassen, T.~Q., Valenzano, R.~A., \BBA\ McIlraith, S.~A.
  \BBOP2018\BBCP.
\newblock \BBOQ {Using Reward Machines for High-Level Task Specification and
  Decomposition in Reinforcement Learning}\BBCQ\
\newblock In {\Bem Proceedings of the International Conference on Machine
  Learning ({ICML})}, \BPGS\ 2112--2121.

\bibitem[\protect\BCAY{{Toro Icarte}, Waldie, Klassen, Valenzano, Castro,\
  \BBA\ McIlraith}{{Toro Icarte} et~al.}{2019}]{IcarteWKVCM19}
{Toro Icarte}, R., Waldie, E., Klassen, T.~Q., Valenzano, R.~A., Castro, M.~P.,
  \BBA\ McIlraith, S.~A. \BBOP2019\BBCP.
\newblock \BBOQ {Learning Reward Machines for Partially Observable
  Reinforcement Learning}\BBCQ\
\newblock In {\Bem Proceedings of the Advances in Neural Information Processing
  Systems ({NeurIPS}) Conference}, \BPGS\ 15497--15508.

\bibitem[\protect\BCAY{Torrey, Shavlik, Walker,\ \BBA\ Maclin}{Torrey
  et~al.}{2007}]{TorreySWM07}
Torrey, L., Shavlik, J.~W., Walker, T., \BBA\ Maclin, R. \BBOP2007\BBCP.
\newblock \BBOQ {Relational Macros for Transfer in Reinforcement
  Learning}\BBCQ\
\newblock In {\Bem Proceedings of the International Conference on Inductive
  Logic Programming ({ILP})}, \BPGS\ 254--268.

\bibitem[\protect\BCAY{Ulyantsev, Zakirzyanov,\ \BBA\ Shalyto}{Ulyantsev
  et~al.}{2015}]{UlyantsevZS15}
Ulyantsev, V., Zakirzyanov, I., \BBA\ Shalyto, A. \BBOP2015\BBCP.
\newblock \BBOQ {BFS-Based Symmetry Breaking Predicates for {DFA}
  Identification}\BBCQ\
\newblock In {\Bem Proceedings of the International Conference on Language and
  Automata Theory and Applications ({LATA})}, \BPGS\ 611--622.

\bibitem[\protect\BCAY{Ulyantsev, Zakirzyanov,\ \BBA\ Shalyto}{Ulyantsev
  et~al.}{2016}]{UlyantsevZS16}
Ulyantsev, V., Zakirzyanov, I., \BBA\ Shalyto, A. \BBOP2016\BBCP.
\newblock \BBOQ {Symmetry Breaking Predicates for SAT-based {DFA}
  Identification}\BBCQ\
\newblock {\Bem CoRR}, {\Bem abs/1602.05028}.

\bibitem[\protect\BCAY{van Hasselt, Guez,\ \BBA\ Silver}{van Hasselt
  et~al.}{2016}]{HasseltGS16}
van Hasselt, H., Guez, A., \BBA\ Silver, D. \BBOP2016\BBCP.
\newblock \BBOQ {Deep Reinforcement Learning with Double Q-Learning}\BBCQ\
\newblock In {\Bem Proceedings of the {AAAI} Conference on Artificial
  Intelligence ({AAAI})}, \BPGS\ 2094--2100.

\bibitem[\protect\BCAY{Watkins}{Watkins}{1989}]{Watkins89}
Watkins, C. \BBOP1989\BBCP.
\newblock {\Bem Learning from Delayed Rewards}.
\newblock Ph.D.\ thesis, King's College, Cambridge, UK.

\bibitem[\protect\BCAY{Weiss, Goldberg,\ \BBA\ Yahav}{Weiss
  et~al.}{2018}]{WeissGY18}
Weiss, G., Goldberg, Y., \BBA\ Yahav, E. \BBOP2018\BBCP.
\newblock \BBOQ {Extracting Automata from Recurrent Neural Networks Using
  Queries and Counterexamples}\BBCQ\
\newblock In {\Bem Proceedings of the International Conference on Machine
  Learning ({ICML})}, \BPGS\ 5244--5253.

\bibitem[\protect\BCAY{Xu, Gavran, Ahmad, Majumdar, Neider, Topcu,\ \BBA\
  Wu}{Xu et~al.}{2020}]{XuGAMNTW20}
Xu, Z., Gavran, I., Ahmad, Y., Majumdar, R., Neider, D., Topcu, U., \BBA\ Wu,
  B. \BBOP2020\BBCP.
\newblock \BBOQ {Joint Inference of Reward Machines and Policies for
  Reinforcement Learning}\BBCQ\
\newblock In {\Bem Proceedings of the International Conference on Automated
  Planning and Scheduling ({ICAPS})}, \BPGS\ 590--598.

\bibitem[\protect\BCAY{Zakirzyanov, Morgado, Ignatiev, Ulyantsev,\ \BBA\
  Marques{-}Silva}{Zakirzyanov et~al.}{2019}]{ZakirzyanovMIUM19}
Zakirzyanov, I., Morgado, A., Ignatiev, A., Ulyantsev, V., \BBA\
  Marques{-}Silva, J. \BBOP2019\BBCP.
\newblock \BBOQ {Efficient Symmetry Breaking for SAT-Based Minimum {DFA}
  Inference}\BBCQ\
\newblock In {\Bem Proceedings of the International Conference on Language and
  Automata Theory and Applications ({LATA})}, \BPGS\ 159--173.

\end{thebibliography}
\bibliographystyle{theapa}

\end{document}